\documentclass[english]{article}
\usepackage[T1]{fontenc}
\usepackage[latin9]{inputenc}
\usepackage{geometry}
\usepackage{babel}
\usepackage{verbatim}
\usepackage{float}
\usepackage{bm}
\usepackage{amsmath}
\usepackage{amssymb}
\usepackage{graphicx}
\usepackage{hyperref}
\hypersetup{
 colorlinks,linkcolor=red,anchorcolor=blue,citecolor=blue}

\makeatletter

\providecommand{\tabularnewline}{\\}
\floatstyle{ruled}
\newfloat{algorithm}{tbp}{loa}
\providecommand{\algorithmname}{Algorithm}
\floatname{algorithm}{\protect\algorithmname}

\usepackage{babel}

\usepackage{cite}\usepackage{amsthm}\usepackage{dsfont}\usepackage{array}\usepackage{mathrsfs}\usepackage{comment}\onecolumn

\usepackage{color}\usepackage{babel}

\allowdisplaybreaks

\usepackage{enumitem}
\setlist[itemize]{leftmargin=1em}
\setlist[enumerate]{leftmargin=1em}

\usepackage{babel}
\usepackage{algorithm}
\usepackage{algorithmic}
\usepackage{arydshln}

\newcommand{\bx}{\bm{x}}

\newcommand{\cS}{{\mathcal{S}}}

\newcommand{\EE}{\mathbb{E}}

\newcommand{\PP}{\mathbb{P}}

\newcommand{\RR}{\mathbb{R}}

\DeclareMathOperator{\ind}{\mathds{1}}  

\definecolor{yxc}{RGB}{255,0,0}
\definecolor{yjc}{RGB}{125,0,0}
\definecolor{cm}{RGB}{0,0,200}
\definecolor{kzw}{RGB}{0,150,0}

\usepackage{babel}

\makeatother

\begin{document}
\theoremstyle{plain} \newtheorem{lemma}{\textbf{Lemma}} \newtheorem{prop}{\textbf{Proposition}}\newtheorem{theorem}{\textbf{Theorem}}\setcounter{theorem}{0}
\newtheorem{corollary}{\textbf{Corollary}} \newtheorem{assumption}{\textbf{Assumption}}
\newtheorem{example}{\textbf{Example}} \newtheorem{definition}{\textbf{Definition}}
\newtheorem{fact}{\textbf{Fact}} \theoremstyle{definition}

\theoremstyle{remark}\newtheorem{remark}{\textbf{Remark}}\newtheorem{condition}{Condition}\newtheorem{claim}{Claim}

\title{Gradient Descent with Random Initialization: \\
 Fast Global Convergence for Nonconvex Phase Retrieval\footnotetext{Author names are sorted alphabetically.}}
\author
{
	Yuxin Chen\thanks{Department of Electrical Engineering, Princeton University, Princeton, NJ 08544, USA; Email:
		\texttt{yuxin.chen@princeton.edu}.}
	\and Yuejie Chi\thanks{Department of Electrical and Computer Engineering, Carnegie Mellon University, Pittsburgh, PA 15213, USA; Email:
		\texttt{yuejiechi@cmu.edu}. }
	\and Jianqing Fan\thanks{Department of Operations Research and Financial Engineering, Princeton University, Princeton, NJ 08544, USA; Email:
		\texttt{\{jqfan, congm\}@princeton.edu}.}
	\and Cong Ma\footnotemark[3]
}

\date{March 2018; \quad Revised April 2019}

\maketitle

\begin{abstract}
This paper considers the problem of solving systems of quadratic equations,
namely, recovering an object of interest $\bm{x}^{\natural}\in\mathbb{R}^{n}$
from $m$ quadratic equations\,/\,samples $y_{i}=(\bm{a}_{i}^{\top}\bm{x}^{\natural})^{2}$,
$1\leq i\leq m$. This problem, also dubbed as phase retrieval, spans
multiple domains including physical sciences and machine learning.

We investigate the efficacy of gradient descent (or Wirtinger flow)
designed for the nonconvex least squares problem. We prove that
under Gaussian designs, gradient descent --- when randomly initialized
--- yields an $\epsilon$-accurate solution in $O\big(\log n+\log(1/\epsilon)\big)$
iterations given nearly minimal samples, thus achieving near-optimal
computational and sample complexities at once. This provides the first
global convergence guarantee concerning vanilla gradient descent for
phase retrieval, without the need of (i) carefully-designed initialization,
(ii) sample splitting, or (iii) sophisticated saddle-point escaping
schemes. All of these are achieved by exploiting the statistical models
in analyzing optimization algorithms, via a leave-one-out approach
that enables the decoupling of certain statistical dependency between
the gradient descent iterates and the data.

\end{abstract}

\tableofcontents

\section{Introduction\label{sec:Introduction}}

Suppose we are interested in learning an unknown object $\bm{x}^{\natural}\in\mathbb{R}^{n}$,
but only have access to a few quadratic equations of the form
\begin{equation}
y_{i}=\left(\bm{a}_{i}^{\top}\bm{x}^{\natural}\right)^{2},\qquad1\leq i\leq m,\label{eq:quadratic-systems}
\end{equation}
where $y_{i}$ is the sample we collect and $\bm{a}_{i}$ is the design
vector known \emph{a priori}. Is it feasible to reconstruct $\bm{x}^{\natural}$
in an accurate and efficient manner?

The problem of solving systems of quadratic equations (\ref{eq:quadratic-systems})
is of fundamental importance and finds applications in numerous contexts.
Perhaps one of the best-known applications is the so-called \emph{phase
retrieval} problem arising in physical sciences \cite{candes2013phase,shechtman2015phase}.
In X-ray crystallography, due to the ultra-high frequency of the X-rays,
the optical sensors and detectors are incapable of recording the phases
of the diffractive waves; rather, only intensity measurements are
collected. The phase retrieval problem comes down to reconstructing
the specimen of interest given intensity-only measurements. If one
thinks of $\bm{x}^{\natural}$ as the specimen under study and uses
$\{y_{i}\}$ to represent the intensity measurements, then phase retrieval
is precisely about inverting the quadratic system (\ref{eq:quadratic-systems}).

Moving beyond physical sciences, the above problem also spans various
machine learning applications. One example is \emph{mixed linear regression},
where one wishes to estimate two unknown vectors $\bm{\beta}_{1}$
and $\bm{\beta}_{2}$ from unlabeled linear measurements \cite{chen2014convex}.
The acquired data $\{\bm{a}_{i},b_{i}\}_{1\leq i\leq m}$ take the
form of either $b_{i}\approx\bm{a}_{i}^{\top}\bm{\beta}_{1}$ or $b_{i}\approx\bm{a}_{i}^{\top}\bm{\beta}_{2}$,
without knowing which of the two vectors generates the data. In a
simple symmetric case with $\bm{\beta}_{1}=-\bm{\beta}_{2}=\bm{x}^{\natural}$
(so that $b_{i}\approx\pm\bm{a}_{i}^{\top}\bm{x}^{\natural}$), the
squared measurements $y_{i}=b_{i}^{2}\approx(\bm{a}_{i}^{\top}\bm{x}^{\natural})^{2}$
become the sufficient statistics, and hence mixed linear regression
can be converted to learning $\bm{x}^{\natural}$ from $\{\bm{a}_{i},y_{i}\}$.
Furthermore, the quadratic measurement model in (\ref{eq:quadratic-systems}) allows to represent a single neuron associated with a quadratic activation function, where $\{\bm{a}_{i},y_{i}\}$
are the data and $\bm{x}^{\natural}$ encodes the parameters to be learned. As described
in \cite{soltanolkotabi2017theoretical,li2017algorithmic}, \emph{learning
neural nets with quadratic activations} involves solving systems
of quadratic equations.

\subsection{Nonconvex optimization via gradient descent}

A natural strategy for inverting the system of quadratic equations
(\ref{eq:quadratic-systems}) is to solve the following nonconvex least squares estimation problem
\begin{equation}
\text{minimize}_{\bm{x}\in\mathbb{R}^{n}}\quad f(\bm{x}):=\frac{1}{4m}\sum_{i=1}^{m}\left[\left(\bm{a}_{i}^{\top}\bm{x}\right)^{2}-y_{i}\right]^{2}.\label{eq:PR-loss}
\end{equation}
Under Gaussian designs where $\bm{a}_{i}\overset{\text{i.i.d.}}{\sim}\mathcal{N}(\bm{0},\bm{I}_{n})$,
the solution to (\ref{eq:PR-loss}) is known to be exact \textemdash{}
up to some global sign \textemdash{} with high probability, as soon as the number $m$ of
equations (samples) exceeds the order of the number $n$ of unknowns \cite{bandeira2014saving}.
However, the loss function in (\ref{eq:PR-loss}) is highly nonconvex,
thus resulting in severe computational challenges. With this issue
in mind, can we still hope to find the global minimizer of (\ref{eq:PR-loss})
via low-complexity algorithms which, ideally, run in time proportional
to that taken to read the data?

Fortunately, in spite of nonconvexity, a variety of optimization-based
methods are shown to be effective in the presence of proper statistical
models. Arguably, one of the simplest algorithms for solving (\ref{eq:PR-loss})
is vanilla gradient descent (GD), which attempts recovery via the
update rule
\begin{equation}
\bm{x}^{t+1}=\bm{x}^{t}-\eta_{t}\nabla f\left(\bm{x}^{t}\right),\qquad t=0,1,\cdots\label{eq:gradient_update-WF}
\end{equation}
with $\eta_{t}$ being the stepsize$\,$/$\,$learning rate. The above
iterative procedure is also dubbed \emph{Wirtinger flow} for phase
retrieval, which can accommodate the complex-valued case as well \cite{candes2014wirtinger}.
This simple algorithm is remarkably efficient under Gaussian designs:
in conjunction with carefully-designed initialization and stepsize
rules, GD provably converges to the truth $\bm{x}^{\natural}$ at
a linear rate\footnote{An iterative algorithm is said to enjoy linear convergence if the
iterates $\{\bm{x}^{t}\}$ converge geometrically fast to the minimizer~$\bm{x}^{\natural}$. }, provided that the ratio $m/n$ of the number of equations to the
number of unknowns exceeds some logarithmic factor \cite{candes2014wirtinger,soltanolkotabi2014algorithms,ma2017implicit}.

One crucial element in prior convergence analysis is initialization.
In order to guarantee linear convergence, prior works typically recommend
spectral initialization or its variants \cite{candes2014wirtinger,ChenCandes15solving,wang2017solving,zhang2017reshaped,ma2017implicit,lu2017phase,mondelli2017fundamental}.
Specifically, the spectral method forms an initial estimate $\bm{x}^{0}$
using the (properly scaled) leading eigenvector of a certain data
matrix. Two important features are worth emphasizing:
\begin{itemize}
\item $\bm{x}^{0}$ falls within a local $\ell_{2}$-ball surrounding $\bm{x}^{\natural}$ with a reasonably small radius,
where $f(\cdot)$ enjoys strong convexity;
\item $\bm{x}^{0}$ is incoherent with all the design vectors $\{\bm{a}_{i}\}$ --- in the sense that $|\bm{a}_{i}^{\top}\bm{x}^{0}|$ is reasonably small
for all $1\leq i\leq m$ --- and hence $\bm{x}^{0}$ falls within a region where $f(\cdot)$
enjoys desired smoothness conditions.
\end{itemize}
These two properties taken collectively allow gradient descent to
converge rapidly from the very beginning.

\subsection{Random initialization?\label{subsec:Random-initialization}}

The enormous success of spectral initialization gives rise to a curious
question: is carefully-designed initialization necessary for achieving
fast convergence? Obviously, vanilla GD cannot start from arbitrary
points, since it may get trapped in undesirable stationary points
(e.g.~saddle points). However, is there any \emph{simpler} initialization
approach that avoids such stationary points and works
equally well as spectral initialization?

A strategy that practitioners often like to employ is to initialize
GD randomly. The advantage is clear: compared with spectral methods,
random initialization is model-agnostic and is usually more robust
vis-a-vis model mismatch. Despite its wide use in practice, however,
GD with random initialization is poorly understood in theory. One
way to study this method is through a geometric lens \cite{sun2016geometric}:
under Gaussian designs, the loss function $f(\cdot)$ (cf.~(\ref{eq:PR-loss}))
does not have any spurious local minima as long as the sample size
$m$ is on the order of $n\log^{3}n$. Moreover, all saddle points
are strict \cite{ge2015escaping}, meaning that the associated Hessian
matrices have at least one negative eigenvalue if they are not local
minima. Armed with these two conditions, the theory of Lee et al.~\cite{lee2016gradient}
implies that vanilla GD converges {\em almost surely} to the truth. However,
the convergence rate remains unsettled.  In fact, we are not aware
of any theory that guarantees polynomial-time convergence of vanilla
GD for phase retrieval in the absence of carefully-designed initialization.

Motivated by this, we aim to pursue a formal understanding about the
convergence properties of GD with random initialization. Before embarking
on theoretical analyses, we first assess its practical efficiency
through numerical experiments. Generate the true object $\bm{x}^{\natural}$
and the initial guess $\bm{x}^{0}$ randomly as
\[
\bm{x}^{\natural}\sim\mathcal{N}(\bm{0},n^{-1}\bm{I}_{n})\qquad\text{and}\qquad\bm{x}^{0}\sim\mathcal{N}(\bm{0},n^{-1}\bm{I}_{n}).
\]
We vary the number $n$ of unknowns (i.e.~$n=100,200,500,800,1000$),
set $m=10n$, and take a constant stepsize $\eta_{t}\equiv0.1$. Here the measurement vectors are generated from Gaussian distributions, i.e.~$\bm{a}_{i}\overset{\mathrm{i.i.d.}}{\sim} \mathcal{N}(\bm{0},\bm{I}_{n})$ for $1\leq i \leq m$. The
relative $\ell_{2}$ errors $\text{dist}(\bm{x}^{t},\bm{x}^{\natural})/\|\bm{x}^{\natural}\|_{2}$
of the GD iterates in a random trial are plotted in Figure~\ref{fig:Relative-error-GD-Gaussian},
where
\begin{equation}
\mathrm{dist}(\bm{x}^{t},\bm{x}^{\natural}):=\min\big\{\|\bm{x}^{t}-\bm{x}^{\natural}\|_{2},\|\bm{x}^{t}+\bm{x}^{\natural}\|_{2}\big\}\label{eq:defn-dist}
\end{equation}
represents the $\ell_{2}$ distance between $\bm{x}^{t}$ and $\bm{x}^{\natural}$
modulo the unrecoverable global sign.

In all experiments carried out in Figure~\ref{fig:Relative-error-GD-Gaussian},
we observe two stages for GD: (1) Stage 1: the relative error
of $\bm{x}^{t}$ stays nearly flat; (2) Stage 2: the relative error
of $\bm{x}^{t}$ experiences geometric decay. Interestingly, Stage
1 lasts only for a few tens of iterations. These numerical findings taken
together reveal appealing computational efficiency of GD in the presence
of random initialization \textemdash{} it attains 5-digit accuracy
within about 200 iterations!

\begin{figure}
\centering

\includegraphics[width=0.4\textwidth]{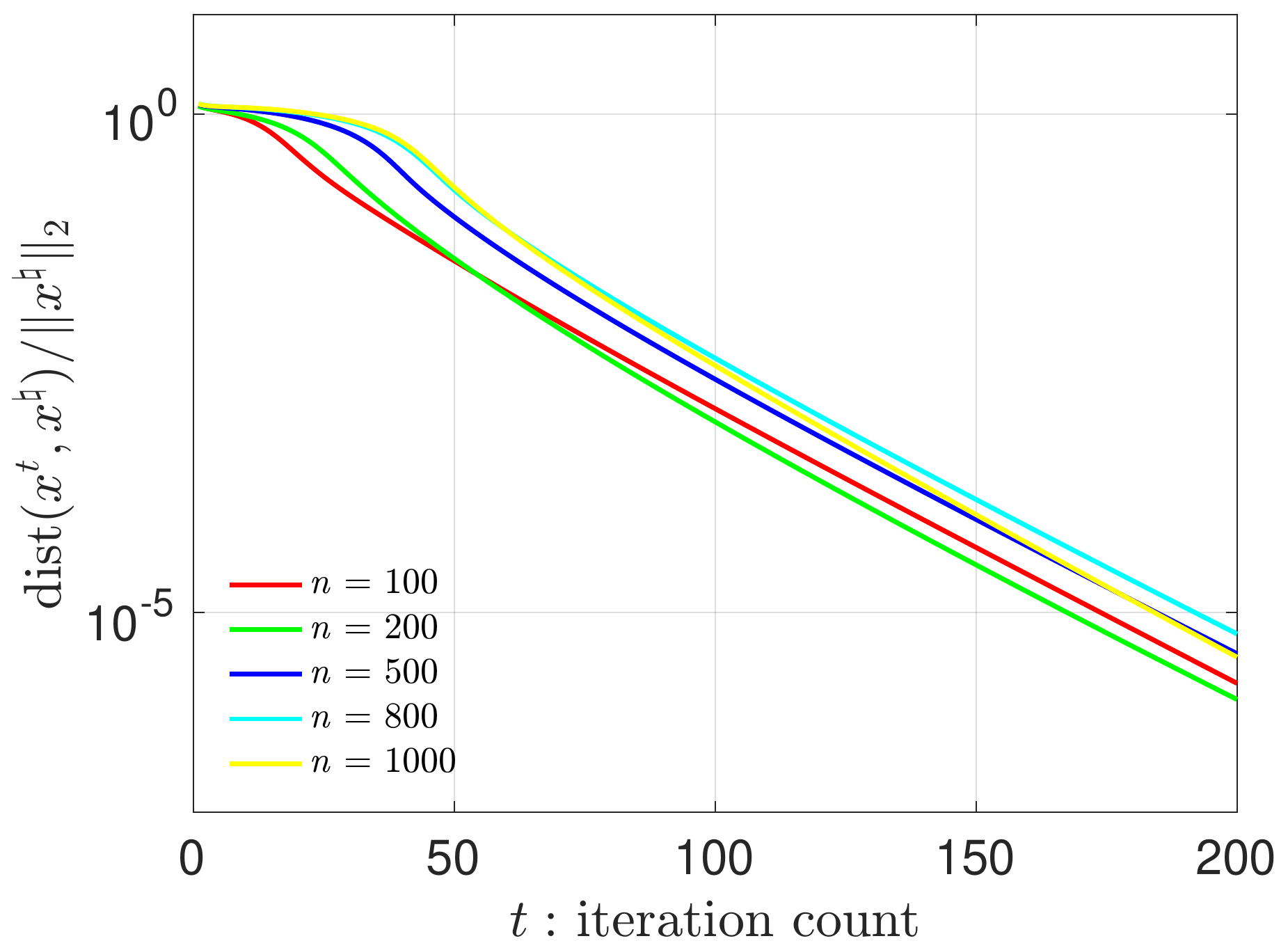}\caption{The relative $\ell_{2}$ error vs.~iteration count for GD with random
initialization, plotted semi-logarithmically. The results are shown for $n=100,200,500,800,1000$
with $m=10n$ and $\eta_{t}\equiv0.1$. \label{fig:Relative-error-GD-Gaussian}}
\end{figure}

To further illustrate this point, we take a closer inspection of the
signal component $\langle\bm{x}^{t},\bm{x}^{\natural}\rangle\bm{x}^{\natural}$
and the orthogonal component $\bm{x}^{t}-\langle\bm{x}^{t},\bm{x}^{\natural}\rangle\bm{x}^{\natural}$,
where we normalize $\|\bm{x}^{\natural}\|_{2}=1$ for simplicity.
Denote by $\|\bm{x}_{\perp}^{t}\|_{2}$ the $\ell_{2}$ norm of the
orthogonal component. We highlight two important and somewhat surprising
observations that allude to why random initialization works.
\begin{itemize}
\item \emph{The strength ratio of the signal to the
orthogonal components grows exponentially.} The ratio, $|\langle\bm{x}^{t},\bm{x}^{\natural}\rangle|\,/\,\|\bm{x}_{\perp}^{t}\|_{2}$,
grows exponentially fast throughout the execution of the algorithm,
as demonstrated in Figure~\ref{fig:SNR-GD-Gaussian}(a). This metric $|\langle\bm{x}^{t},\bm{x}^{\natural}\rangle|\,/\,\|\bm{x}_{\perp}^{t}\|_{2}$ in some sense captures the signal-to-noise ratio of the
running iterates.
\item \emph{Exponential growth of the signal strength in Stage 1}. While
the $\ell_{2}$ estimation error of $\bm{x}^{t}$ may not drop significantly
during Stage 1, the size $|\langle\bm{x}^{t},\bm{x}^{\natural}\rangle|$
of the signal component increases exponentially fast and becomes the
dominant component within several tens of iterations, as demonstrated in Figure~\ref{fig:SNR-GD-Gaussian}(b). This helps explain why Stage 1 lasts only for a short duration.
\end{itemize}
The central question then amounts to whether one can develop a mathematical
theory to interpret such intriguing numerical performance. In particular,
how many iterations does Stage 1 encompass, and how fast can the algorithm
converge in Stage 2?

\begin{figure}
\centering

\begin{tabular}{cc}
\includegraphics[width=0.4\textwidth]{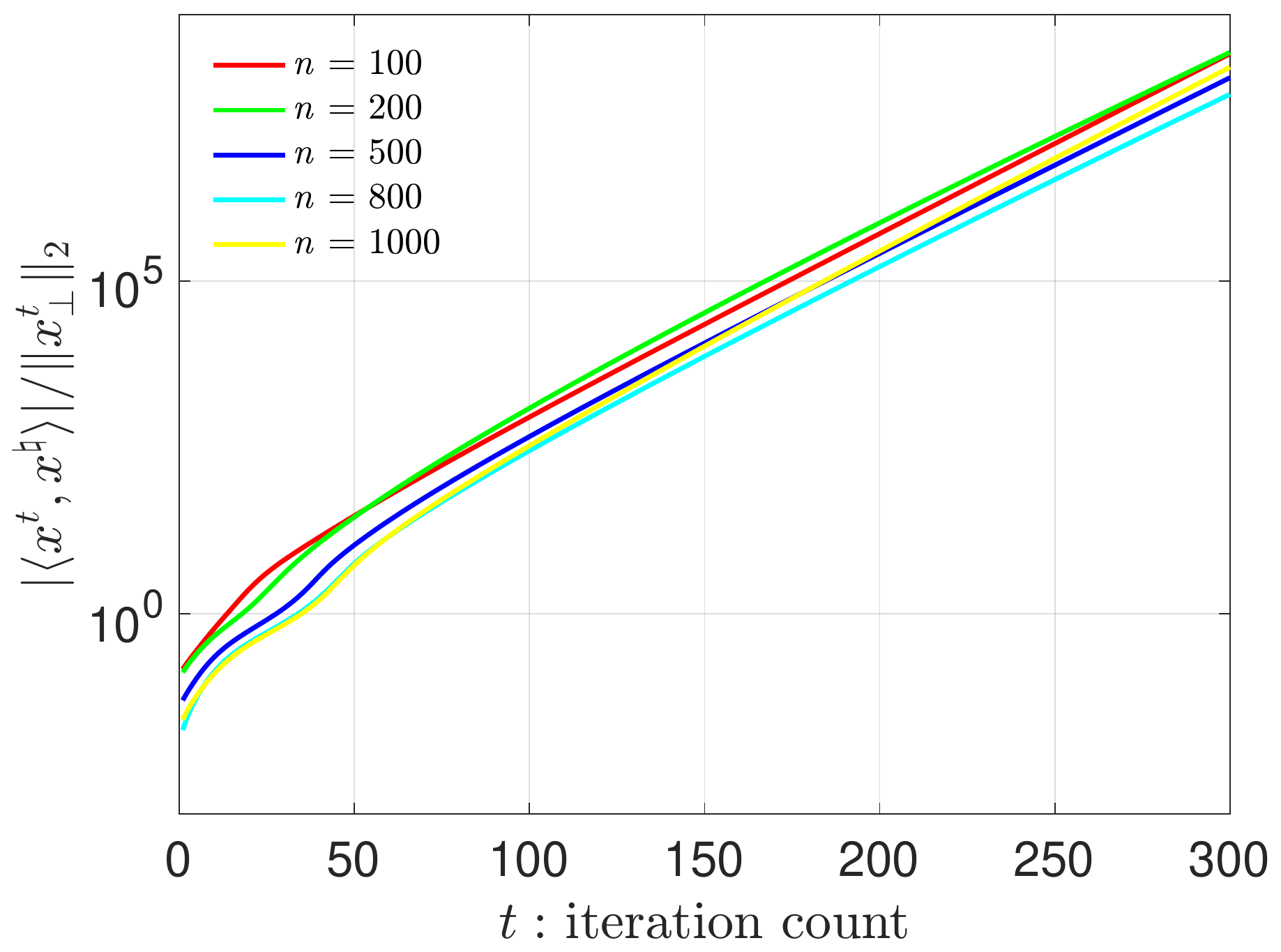} $\quad$ & $\quad$\includegraphics[width=0.4\textwidth]{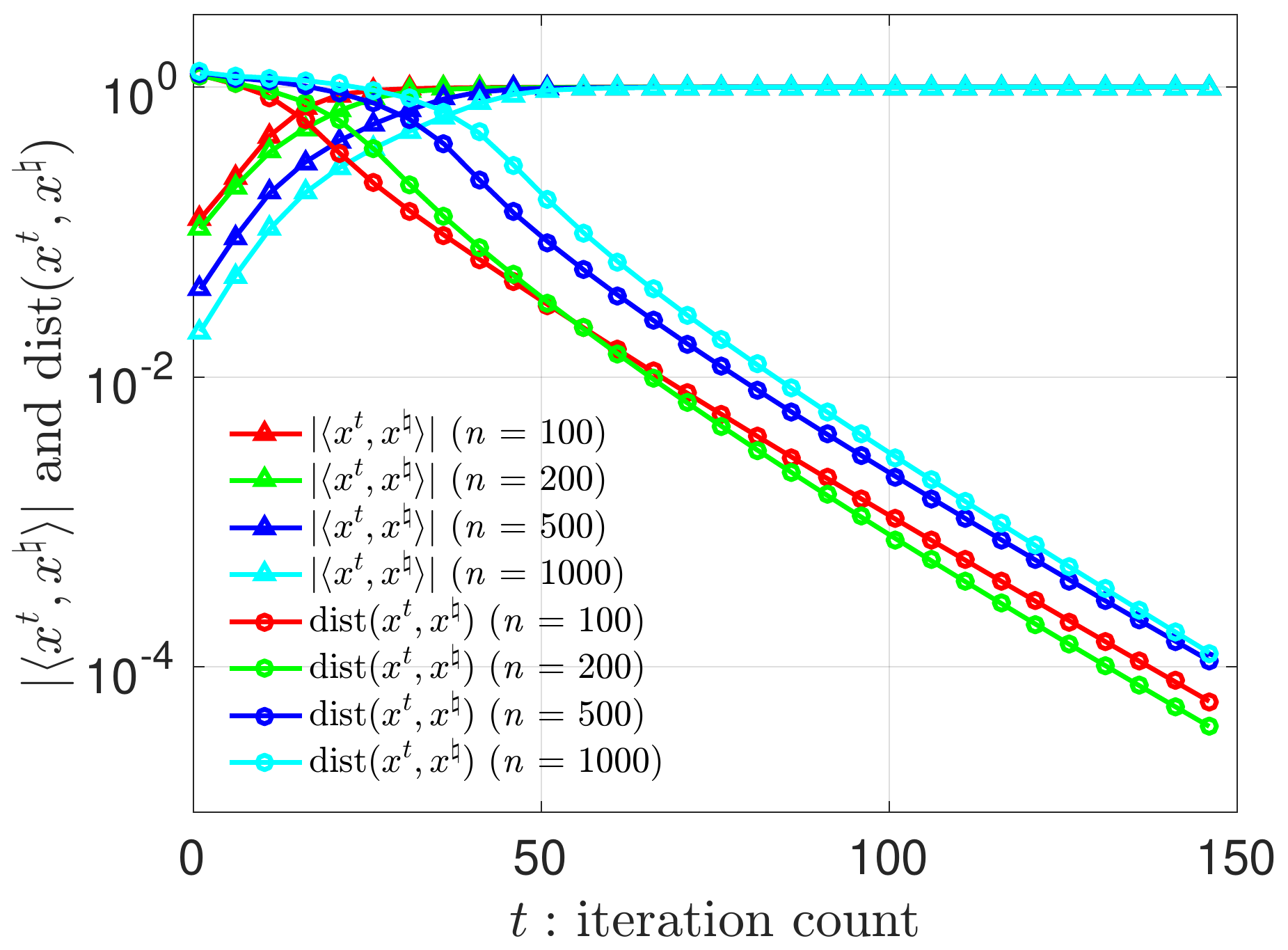}\tabularnewline
(a)  & (b)\tabularnewline
\end{tabular}\caption{(a) The ratio $|\langle\bm{x}^{t},\bm{x}^{\natural}\rangle|\,/\,\|\bm{x}_{\perp}^{t}\|_{2}$, and 
(b) the size $|\langle\bm{x}^{t},\bm{x}^{\natural}\rangle|$ of
the signal component and the $\ell_{2}$ error vs.~iteration count, both plotted on semilogarithmic scales.
The results are shown for $n=100,200,500,800,1000$ with $m=10n$,
$\eta_{t}\equiv0.1$, and $\|\bm{x}^{\natural}\|_{2}=1$. \label{fig:SNR-GD-Gaussian}}
\end{figure}

\subsection{Main findings}

The objective of the current paper is to demystify the computational
efficiency of GD with random initialization, thus bridging the gap
between theory and practice. Assuming a tractable random design model
in which $\bm{a}_{i}$'s follow Gaussian distributions, our main findings
are summarized in the following theorem. Here and throughout, the
notation $f(n)\lesssim g(n)$ or $f(n)=O(g(n))$ (resp.~$f(n)\gtrsim g(n)$,
$f(n)\asymp g(n)$) means that there exist constants $c_{1},c_{2}>0$
such that $f(n)\leq c_{1}g(n)$ (resp.~$f(n)\geq c_{2}g(n)$, $c_{1}g(n)\leq f(n)\leq c_{2}g(n)$).

\begin{theorem}\label{thm:intro}Fix $\bm{x}^{\natural}\in\mathbb{R}^{n}$
with $\|\bm{x}^{\natural}\|_{2}=1$. Suppose that $\bm{a}_{i}\overset{\mathrm{i.i.d.}}{\sim}\mathcal{N}(\bm{0},\bm{I}_{n})$ for $1\leq i \leq m$,
$\bm{x}^{0}\sim\mathcal{N}(\bm{0},n^{-1}\bm{I}_{n})$, and $\eta_{t}\equiv\eta=c/\|\bm{x}^{\natural}\|_{2}^{2}$
for some sufficiently small constant $c>0$. Then with probability
approaching one, there exist some sufficiently small constant $0<\gamma<1$
and $T_{\gamma}\lesssim\log n$ such that the GD iterates (\ref{eq:gradient_update-WF})
obey
\[
\mathrm{dist}\big(\bm{x}^{t},\bm{x}^{\natural}\big)\leq\gamma(1-\rho)^{t-T_{\gamma}},\qquad\forall\,t\geq T_{\gamma}
\]
for some absolute constant $0<\rho<1$, provided that the sample size $m\gtrsim n\,\mathrm{poly}\log(m)$.
\end{theorem}\begin{remark}The readers are referred to Theorem \ref{thm:main}
for a more general statement. \end{remark}
\begin{figure}
\centering

\begin{tabular}{cc}
\includegraphics[width=0.4\textwidth]{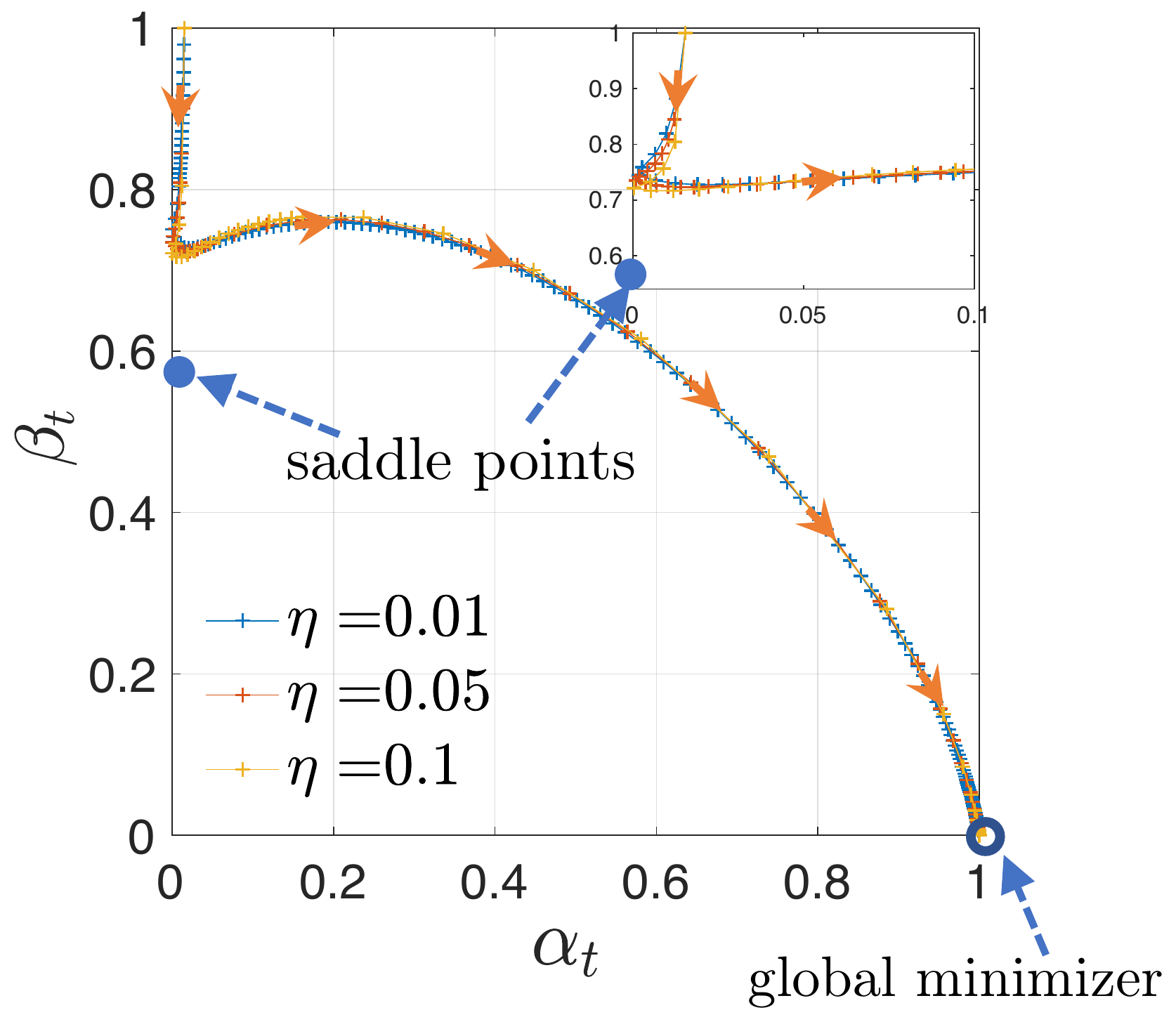}
$\quad$ & $\quad$\includegraphics[width=0.4\textwidth]{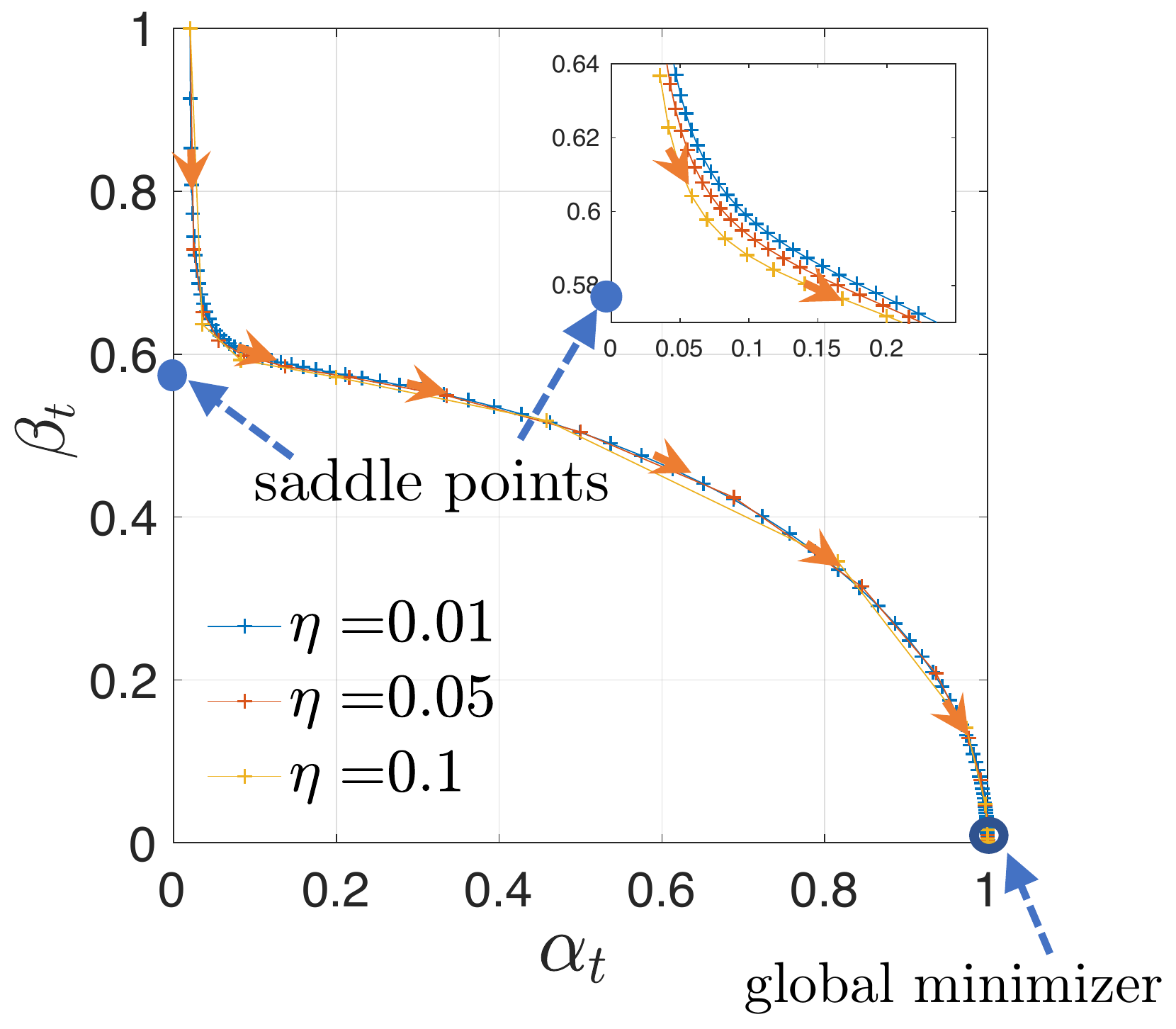}\tabularnewline
(a)  & (b)\tabularnewline
\end{tabular}\caption{The trajectory of $(\alpha_{t},\beta_{t})$, where $\alpha_{t}=|\langle\bm{x}^{t},\bm{x}^{\natural}\rangle|$
and $\beta_{t}=\|\bm{x}^{t}-\langle\bm{x}^{t},\bm{x}^{\natural}\rangle\bm{x}^{\natural}\|_2$
represent respectively the size of the signal component and that of the orthogonal
component of the GD iterates (assume $\|\bm{x}^{\natural}\|_{2}=1$).
(a) The results are shown for $n=1000$ with $m=10n$, and $\eta_{t}=0.01,0.05,0.1$.
(b) The results are shown for $n=1000$ with $m$ approaching infinity,
and $\eta_{t}=0.01,0.05,0.1$. The blue filled circles represent the
population-level saddle points, and the orange arrows indicate the directions
of increasing $t$. \label{fig:scatter_alpha_beta}}
\end{figure}
Here, the stepsize is taken to be a fixed constant throughout all
iterations, and we reuse the same data across all iterations (i.e.~no
sample splitting is needed to establish this theorem). The GD trajectory
is divided into 2 stages: (1) Stage 1 consists of the first $T_{\gamma}$
iterations, corresponding to the first tens of iterations discussed
in Section~\ref{subsec:Random-initialization}; (2) Stage 2 consists
of all remaining iterations, where the estimation error contracts
linearly. Several important implications$\,$/$\,$remarks follow
immediately.
\begin{itemize}
\item \emph{Stage 1 takes $O(\log n)$ iterations.} When seeded with a random
initial guess, GD is capable of entering a local region surrounding
$\bm{x}^{\natural}$ within $T_{\gamma}\lesssim\log n$ iterations,
namely,
\[
\mathrm{dist}\big(\bm{x}^{T_{\gamma}},\bm{x}^{\natural}\big)\leq\gamma
\]
 for some sufficiently small constant $\gamma>0$. Even though Stage
1 may not enjoy linear convergence in terms of the estimation error,
it is of fairly short duration.
\item \emph{Stage 2 takes $O(\log(1/\epsilon))$ iterations.} After entering
the local region, GD converges linearly to the ground truth $\bm{x}^{\natural}$
with a contraction rate $1-\rho$. This tells
us that GD reaches $\epsilon$-accuracy (in a relative sense) within
\emph{$O(\log(1/\epsilon))$} iterations.
\item \emph{Near linear-time computational complexity.} Taken collectively, these imply
that the iteration complexity of GD with random initialization is
\[
O\left(\log n+\log\frac{1}{\epsilon}\right).
\]
Given that the cost of each iteration mainly lies in calculating the
gradient $\nabla f(\bm{x}^{t})$, the whole algorithm takes nearly
linear time, namely, it enjoys a computational complexity proportional
to the time taken to read the data (modulo some logarithmic factor).

\item \emph{Near-minimal sample complexity. }The preceding computational
guarantees occur as soon as the sample size exceeds $m\gtrsim n\,\mathrm{poly}\log(m)$.
Given that one needs at least $n$ samples to recover $n$ unknowns,
the sample complexity of randomly initialized GD is optimal up to
some logarithmic factor.

\item \emph{Saddle points?} The GD iterates never hit the saddle points
(see Figure \ref{fig:scatter_alpha_beta} for an illustration). In
fact, after a constant number of iterations at the very beginning,
GD will follow a path that increasingly distances itself from the set
of saddle points as the algorithm progresses. There is no need to
adopt sophisticated saddle-point escaping schemes developed in generic optimization
theory (e.g.~cubic regularization \cite{nesterov2006cubic}, perturbed GD \cite{jin2017escape}).

\item \emph{Weak dependency w.r.t.~the design vectors. }As we will elaborate in Section \ref{sec:Analysis}, the statistical dependency
between the GD iterates $\{\bm{x}^{t}\}$ and certain components of the design vectors $\{\bm{a}_{i}\}$
stays at an exceedingly weak level. Consequently, the GD iterates
$\{\bm{x}^{t}\}$ proceed \emph{as if} fresh samples were employed
in each iteration. This statistical observation plays a crucial role
in characterizing the dynamics of the algorithm without the need of
sample splitting.
\end{itemize}
It is worth emphasizing that the entire trajectory of GD is automatically
confined within a certain region enjoying favorable geometry. For
example, the GD iterates are always incoherent with the design vectors,
stay sufficiently away from any saddle point, and exhibit desired
smoothness conditions, which we will formalize in Section \ref{sec:Analysis}.
Such delicate geometric properties underlying the GD trajectory are
not explained by prior papers. In light of this, convergence analysis
based on global geometry \cite{sun2016geometric} \textemdash{} which
provides valuable insights into algorithm designs with \emph{arbitrary}
initialization \textemdash{} results in suboptimal (or even pessimistic)
computational guarantees when analyzing a specific algorithm like
GD. In contrast, the current paper establishes near-optimal performance
guarantees by paying particular attention to finer dynamics of the
algorithm. As will be seen later, this is accomplished by heavily
exploiting the statistical properties in each iterative update.

\section{Why random initialization works?\label{sec:why}}

Before diving into the proof of the main theorem, we pause to develop
intuitions regarding why gradient descent with random initialization
is expected to work. We will build our understanding step by step:
(i) we first investigate the dynamics of the population gradient sequence
(the case where we have infinite samples); (ii) we then turn to the
finite-sample case and present a heuristic argument assuming independence
between the iterates and the design vectors; (iii) finally, we argue
that the true trajectory is remarkably close to the one heuristically
analyzed in the previous step, which arises from a key property concerning
the ``near-independence'' between $\{\bm{x}^{t}\}$ and the design
vectors $\{\bm{a}_{i}\}$.

Without loss of generality, we assume $\bm{x}^{\natural}=\bm{e}_{1}$
throughout this section, where $\bm{e}_{1}$ denotes the first standard
basis vector. For notational simplicity, we denote by
\begin{equation}
x_{\parallel}^{t}:=x_{1}^{t}\qquad\text{and}\qquad\bm{x}_{\perp}^{t}:=[x_{i}^{t}]_{2\leq i\leq n}\label{eq:defn-xperp-xpara}
\end{equation}
the first entry and the 2nd through the $n$th entries of $\bm{x}^{t}$,
respectively. Since $\bm{x}^{\natural}=\bm{e}_{1}$, it is easily
seen that
\begin{equation}
\underset{\text{signal component}}{\underbrace{x_{\parallel}^{t}\bm{e}_{1}=\langle\bm{x}^{t},\bm{x}^{\natural}\rangle\bm{x}^{\natural}}}\qquad\text{and}\qquad\underset{\text{orthogonal component}}{\underbrace{\left[\begin{array}{c}
0\\
\bm{x}_{\perp}^{t}
\end{array}\right]=\bm{x}^{t}-\langle\bm{x}^{t},\bm{x}^{\natural}\rangle\bm{x}^{\natural}}}\label{eq:defn-xperp-xpara-full}
\end{equation}
represent respectively the components of $\bm{x}^{t}$ along and orthogonal
to the signal direction. In what follows, we focus our attention on
the following two quantities that reflect the sizes of the preceding
two components\footnote{Here, we do not take the absolute value of $x_{\parallel}^{t}$. As
we shall see later, the $x_{\parallel}^{t}$'s are of the same sign
throughout the execution of the algorithm.}
\begin{equation}
\alpha_{t}:=x_{\parallel}^{t}\qquad\text{and}\qquad\beta_{t}:=\left\Vert \bm{x}_{\perp}^{t}\right\Vert _{2}.\label{eq:alpha-beta-defn}
\end{equation}
Without loss of generality, assume that $\alpha_{0}>0$.

\subsection{Population dynamics}

To start with, we consider the unrealistic case where the iterates
$\{\bm{x}^{t}\}$ are constructed using the population gradient (or
equivalently, the gradient when the sample size $m$ approaches infinity), i.e.
\[
\bm{x}^{t+1}=\bm{x}^{t}-\eta\nabla F(\bm{x}^{t}).
\]
Here, $\nabla F(\bm{x})$ represents the population gradient given
by
\[
\nabla F(\bm{x}):=(3\|\bm{{x}}\|_{2}^{2}-1)\bm{x}-2(\bm{x}^{\natural\top}\bm{x})\bm{x}^{\natural},
\]
which can be computed by $\nabla F(\bm{x})=\mathbb{E}[\nabla f(\bm{x})]=\mathbb{E}\big[\{(\bm{a}_{i}^{\top}\bm{x})^{2}-(\bm{a}_{i}^{\top}\bm{x}^{\natural})^{2}\}\bm{a}_{i}\bm{a}_{i}^{\top}\bm{x}\big]$
assuming that $\bm{x}$ and the $\bm{a}_{i}$'s are independent. Simple
algebraic manipulation reveals the dynamics for both the signal and
the orthogonal components: \begin{subequations}\label{subeq:population-dynamics}
\begin{align}
x_{\parallel}^{t+1} & =\left\{ 1+3\eta\left(1-\|\bm{x}^{t}\|_{2}^{2}\right)\right\} x_{\parallel}^{t};\label{eq:population-alpha-t}\\
\bm{x}_{\perp}^{t+1} & =\left\{ 1+\eta\left(1-3\|\bm{x}^{t}\|_{2}^{2}\right)\right\} \bm{x}_{\perp}^{t}.\label{eq:population-beta-t}
\end{align}
\end{subequations}Assuming that $\eta$ is sufficiently small and
recognizing that $\|\bm{x}^{t}\|_{2}^{2}=\alpha_{t}^{2}+\beta_{t}^{2}$,
we arrive at the following population-level state evolution for both
$\alpha_{t}$ and $\beta_{t}$ (cf.~(\ref{eq:alpha-beta-defn})):
\begin{subequations}\label{subeq:population-iterative}
\begin{align}
\alpha_{t+1} & =\left\{ 1+3\eta\left[1-\left(\alpha_{t}^{2}+\beta_{t}^{2}\right)\right]\right\} \alpha_{t};\label{eq:state-evolution-population-alpha}\\
\beta_{t+1} & =\left\{ 1+\eta\left[1-3\left(\alpha_{t}^{2}+\beta_{t}^{2}\right)\right]\right\} \beta_{t}.\label{eq:state-evolution-population-beta}
\end{align}
\end{subequations}This recursive system has three \emph{fixed points}:
\[
(\alpha,\beta)=(1,0),\qquad(\alpha,\beta)=(0,0),\qquad\text{and}\qquad(\alpha,\beta)=(0,1/\sqrt{3}),
\]
which correspond to the global minimizer, the local maximizer, and
the saddle points, respectively, of the population objective function.

We make note of the following key observations in the presence of
a randomly initialized $\bm{x}^{0}$, which will be formalized later
in Lemma \ref{lemma:iterative}:
\begin{itemize}
\item the ratio $\alpha_{t}/\beta_{t}$ of the size of the signal component to that of the
orthogonal component increases exponentially fast;
\item the size $\alpha_{t}$ of the signal component keeps growing until
it plateaus around $1$;
\item the size $\beta_{t}$ of the orthogonal component eventually drops towards
zero.
\end{itemize}
In other words, when randomly initialized, $(\alpha^{t},\beta^{t})$
converges to $(1, 0)$ rapidly, thus indicating rapid convergence of $\bm{x}^{t}$
to the truth $\bm{x}^{\natural}$, without getting stuck at any undesirable
saddle points. We also illustrate these phenomena numerically. Set
$n=1000$, $\eta_{t}\equiv0.1$ and $\bm{x}^{0}\sim\mathcal{N}(\bm{0},n^{-1}\bm{I}_{n})$.
Figure~\ref{fig:intuition-population} displays the dynamics of $\alpha_{t}/\beta_{t}$,
$\alpha_{t}$, and $\beta_{t}$, which are precisely as discussed
above.

\begin{figure}
\centering

\begin{tabular}{cc}
\includegraphics[width=0.4\textwidth]{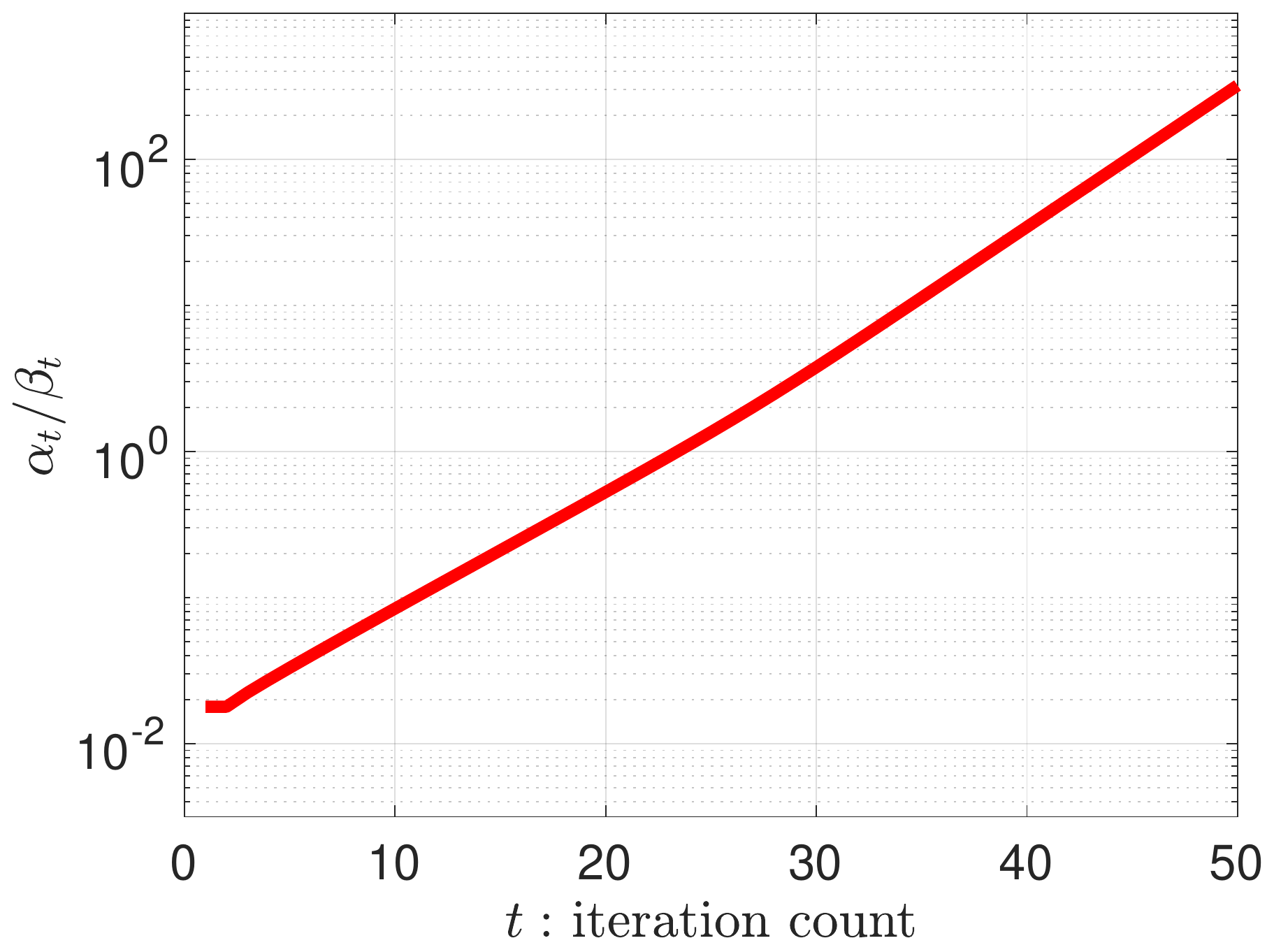} $\quad$ & $\quad$\includegraphics[width=0.4\textwidth]{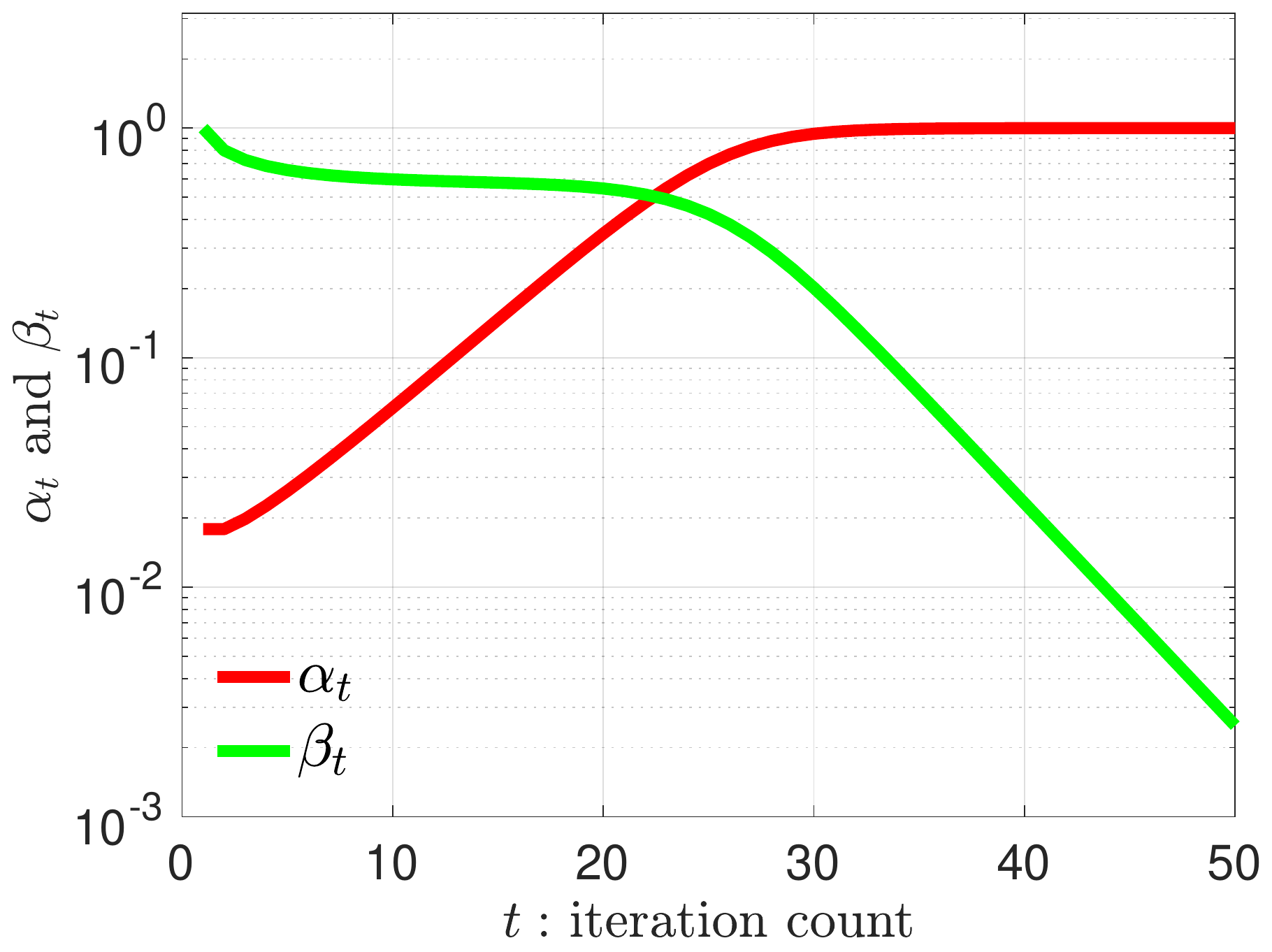}\tabularnewline
(a) $\alpha_{t}/\beta_{t}$  & (b) $\alpha_{t}$ and $\beta_{t}$\tabularnewline
\end{tabular}\caption{Population-level state evolution, plotted semilogarithmically: (a) the ratio $\alpha_{t}/\beta_{t}$
vs.~iteration count, and (b) $\alpha_{t}$ and $\beta_{t}$ vs.~iteration count. The results
are shown for $n=1000$, $\eta_{t}\equiv0.1$, and $\bm{x}^{0}\sim\mathcal{N}(\bm{0},n^{-1}\bm{I}_{n})$
(assuming $\alpha_{0}>0$ though). \label{fig:intuition-population}
}
\end{figure}

\subsection{Finite-sample analysis: a heuristic treatment\label{subsec:Finite-sample-analysis:-heuristi}}

We now move on to the finite-sample regime, and examine how many samples
are needed in order for the population dynamics to be reasonably accurate.
Notably, the arguments in this subsection are heuristic in nature,
but they are useful in developing insights into the true dynamics
of the GD iterates.

Rewrite the gradient update rule (\ref{eq:gradient_update-WF}) as
\begin{align}
\bm{x}^{t+1} & =\bm{x}^{t}-\eta\nabla f(\bm{x}^{t})=\bm{x}^{t}-\eta\nabla F(\bm{x}^{t})-\eta\big(\underset{:=\bm{r}(\bm{x}^{t})}{\underbrace{\nabla f(\bm{x}^{t})-\nabla F(\bm{x}^{t})}}\big),\label{eq:grad-difference-r}
\end{align}
where $\nabla f(\bm{x})= m^{-1}\sum_{i=1}^{m}[(\bm{a}_{i}^{\top}\bm{x})^{2}-(\bm{a}_{i}^{\top}\bm{x}^{\natural})^{2}]\bm{a}_{i}\bm{a}_{i}^{\top}\bm{x}$.
Assuming (unreasonably) that the iterate $\bm{x}^{t}$ is \emph{independent
of} $\{\bm{a}_{i}\}$, the central limit theorem (CLT) allows us to
control the size of the fluctuation term $\bm{r}(\bm{x}^{t})$. Take
the signal component as an example: simple calculations give
\begin{align*}
x_{\|}^{t+1} & =x_{\|}^{t}-\eta\left(\nabla F(\bm{x}^{t})\right)_{1}-\eta r_{1}(\bm{x}^{t}),
\end{align*}
where
\begin{equation}
r_{1}(\bm{x}):=\frac{1}{m}\sum_{i=1}^{m}\left[\big(\bm{a}_{i}^{\top}\bm{x}\big)^{3}-a_{i,1}^{2}\big(\bm{a}_{i}^{\top}\bm{x}\big)\right]a_{i,1}-\mathbb{E}\left[\left\{ \big(\bm{a}_{i}^{\top}\bm{x}\big)^{3}-a_{i,1}^{2}\big(\bm{a}_{i}^{\top}\bm{x}\big)\right\} a_{i,1}\right]\label{eq:defn-r1}
\end{equation}
with $a_{i,1}$ the first entry of $\bm{a}_{i}$. Owing to the preceding
independence assumption, $r_{1}$ is the sum of $m$ i.i.d.~zero-mean
random variables. Assuming that $\bm{x}^{t}$ never blows up so that
$\|\bm{x}^{t}\|_{2}=O(1)$, one can apply the CLT to demonstrate that
\begin{equation}
|r_{1}(\bm{x}^{t})|\lesssim\sqrt{\text{Var}(r_{1}(\bm{x}^{t}))\,\mathrm{poly}\log(m)}\lesssim\sqrt{\frac{\mathrm{poly}\log(m)}{m}}\label{eq:why-defn-residual}
\end{equation}
with high probability, which is often negligible compared to the other
terms. For instance, for the random initial guess $\bm{x}^{0}\sim\mathcal{N}(\bm{0},n^{-1}\bm{I}_{n})$
one has $\big|x_{||}^{0}\big|\gtrsim1/\sqrt{n\log n}$ with probability
approaching one,  telling us that
\[
|r_{1}(\bm{x}^{0})|\lesssim\sqrt{\frac{\mathrm{poly}\log(m)}{m}}\ll|x_{||}^{0}|
\]
as long as $m\gtrsim n\,\mathrm{poly}\log(m)$. This combined with the fact that $|x_{||}^{0}-\eta(\nabla F(\bm{x}^{0}))_{1}|\asymp|x_{||}^{0}|$ reveals $|r_{1}(\bm{x}^{0})|\lesssim |x_{||}^{0}-\eta(\nabla F(\bm{x}^{0}))_{1}|$. Similar observations
hold true for the orthogonal component $\bm{x}_{\perp}^{t}$.

In summary, by assuming independence between $\bm{x}^{t}$ and $\{\bm{a}_{i}\}$,
we arrive at an approximate state evolution for the finite-sample
regime: \begin{subequations}\label{subeq:approx-state-evolution}
\begin{align}
\alpha_{t+1} & \approx\left\{ 1+3\eta\left[1-\left(\alpha_{t}^{2}+\beta_{t}^{2}\right)\right]\right\} \alpha_{t};\label{eq:approx-state-evolution-alpha}\\
\beta_{t+1} & \approx\left\{ 1+\eta\left[1-3\left(\alpha_{t}^{2}+\beta_{t}^{2}\right)\right]\right\} \beta_{t},\label{eq:approx-state-evolution-beta}
\end{align}
\end{subequations}
with the proviso that $m\gtrsim n\,\mathrm{poly}\log(m)$.

\subsection{Key analysis ingredients: near-independence and leave-one-out tricks\label{subsec:why-loo}}

\begin{figure}
\centering

\includegraphics[width=0.35\textwidth]{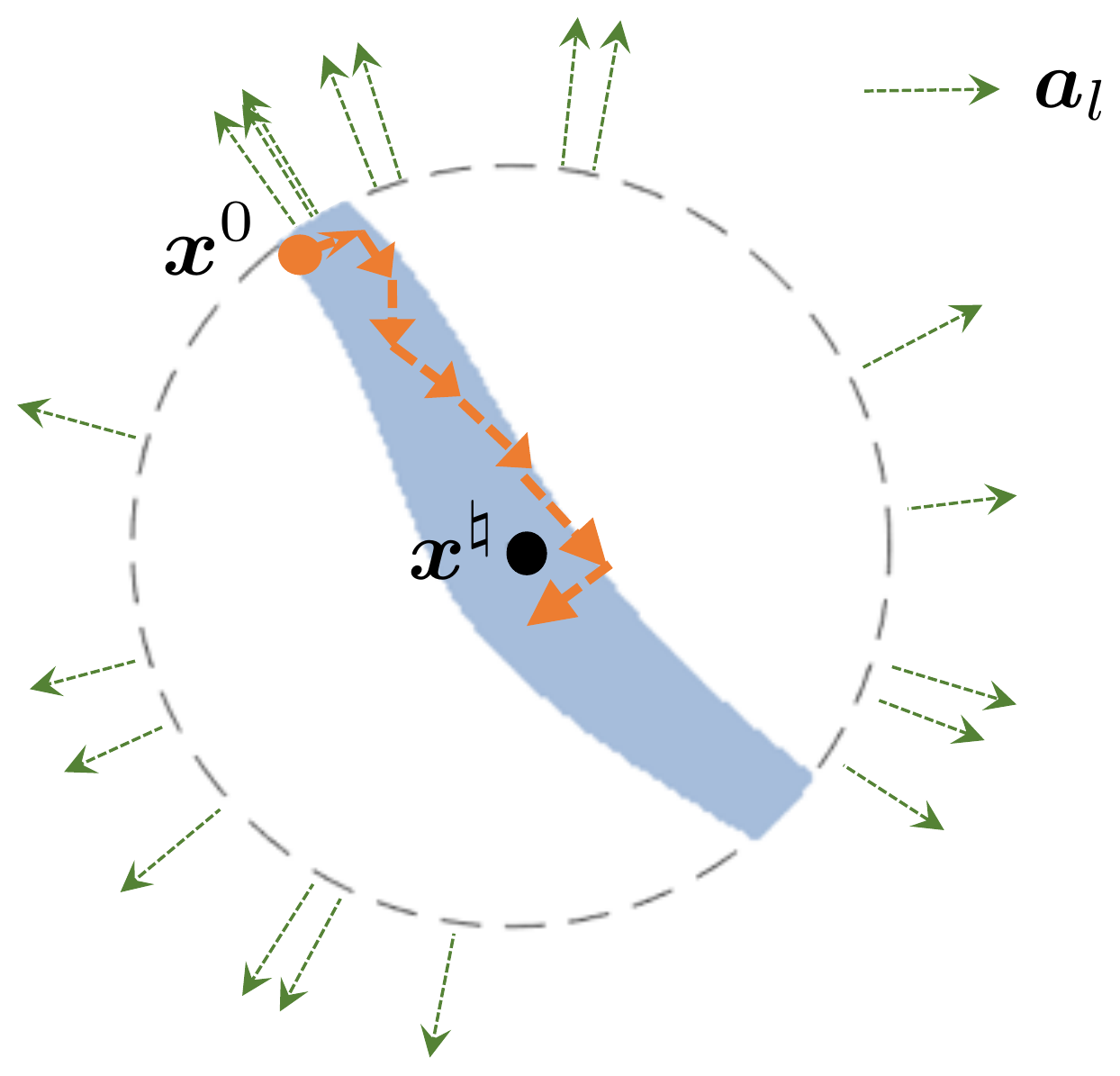}

\caption{Illustration of the region satisfying the ``near-independence''
property. Here, the green arrows represent the directions of $\{\bm{a}_{i}\}_{1\leq i\leq20}$,
and the blue region consists of all points such that the first entry
$r_{1}(\bm{x})$ of the fluctuation $\bm{r}(\bm{x})=\nabla f(\bm{x})-\nabla F(\bm{x})$
is bounded above in magnitude by $|x^{t}_{\parallel}| / 5$ (or $|\langle\bm{x},\bm{x}^{\natural}\rangle|/5$).
\label{fig:Illustration-of-near-independence}}
\end{figure}

The preceding heuristic argument justifies the approximate validity
of the population dynamics, under an independence assumption that
never holds unless we use fresh samples in each iteration. On closer
inspection, what we essentially need is the fluctuation term $\bm{r}(\bm{x}^{t})$
(cf.~(\ref{eq:grad-difference-r})) being well-controlled. For instance,
when focusing on the signal component, one needs $|r_{1}(\bm{x}^{t})|\ll\big|x_{\|}^{t}\big|$
for all $t\geq0$. In particular, in the beginning iterations, $|x_{\|}^{t}|$
is as small as $O(1/\sqrt{n})$. Without the independence assumption,
the CLT types of results fail to hold due to the complicated dependency
between $\bm{x}^{t}$ and $\{\bm{a}_{i}\}$. In fact, one can easily
find many points that result in much larger remainder terms (as large
as $O(1)$) and that violate the approximate state evolution \eqref{subeq:approx-state-evolution}.
See Figure \ref{fig:Illustration-of-near-independence} for a caricature
of the region where the fluctuation term $\bm{r}(\bm{x}^{t})$ is
well-controlled. As can be seen, it only occupies a tiny fraction
of the neighborhood of $\bm{x}^{\natural}$ .

Fortunately, despite the complicated dependency across iterations,
one can provably guarantee that $\bm{x}^{t}$ always stays within
the preceding desirable region in which $\bm{r}(\bm{x}^{t})$ is well-controlled.
The key idea is to exploit a certain ``near-independence'' property
between $\{\bm{x}^{t}\}$ and $\{\bm{a}_{i}\}$. Towards this, we
make use of a leave-one-out trick proposed in \cite{ma2017implicit} for analyzing nonconvex iterative
methods. In particular, we construct auxiliary sequences that are
\begin{enumerate}
\item independent of \emph{certain components} of the design vectors $\{\bm{a}_{i}\}$;
and
\item extremely close to the original gradient sequence $\{\bm{x}^{t}\}_{t\geq0}$.
\end{enumerate}
\begin{figure}
\centering

\begin{tabular}{cccc}
\includegraphics[height=0.16\textheight]{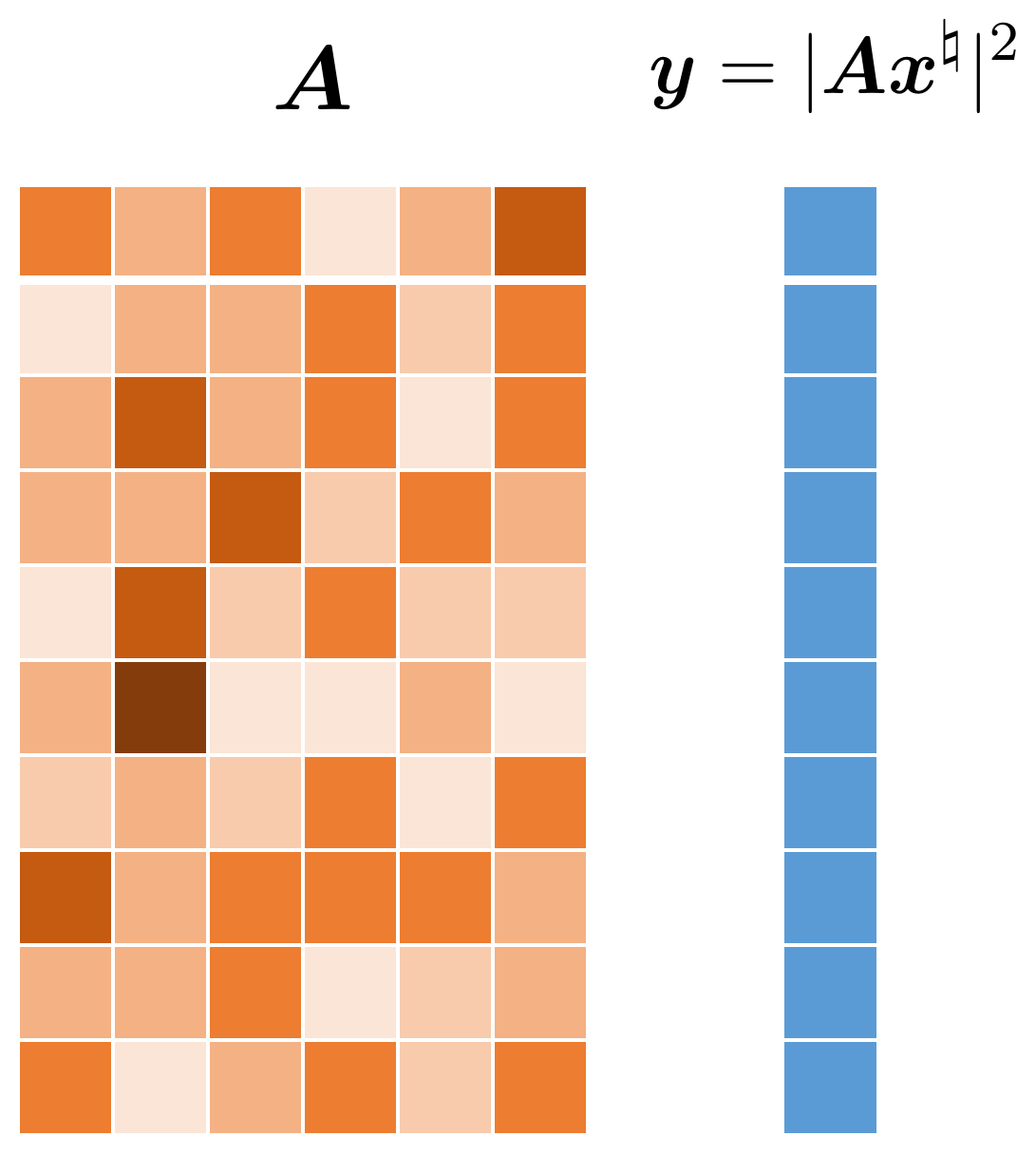} & \includegraphics[height=0.16\textheight]{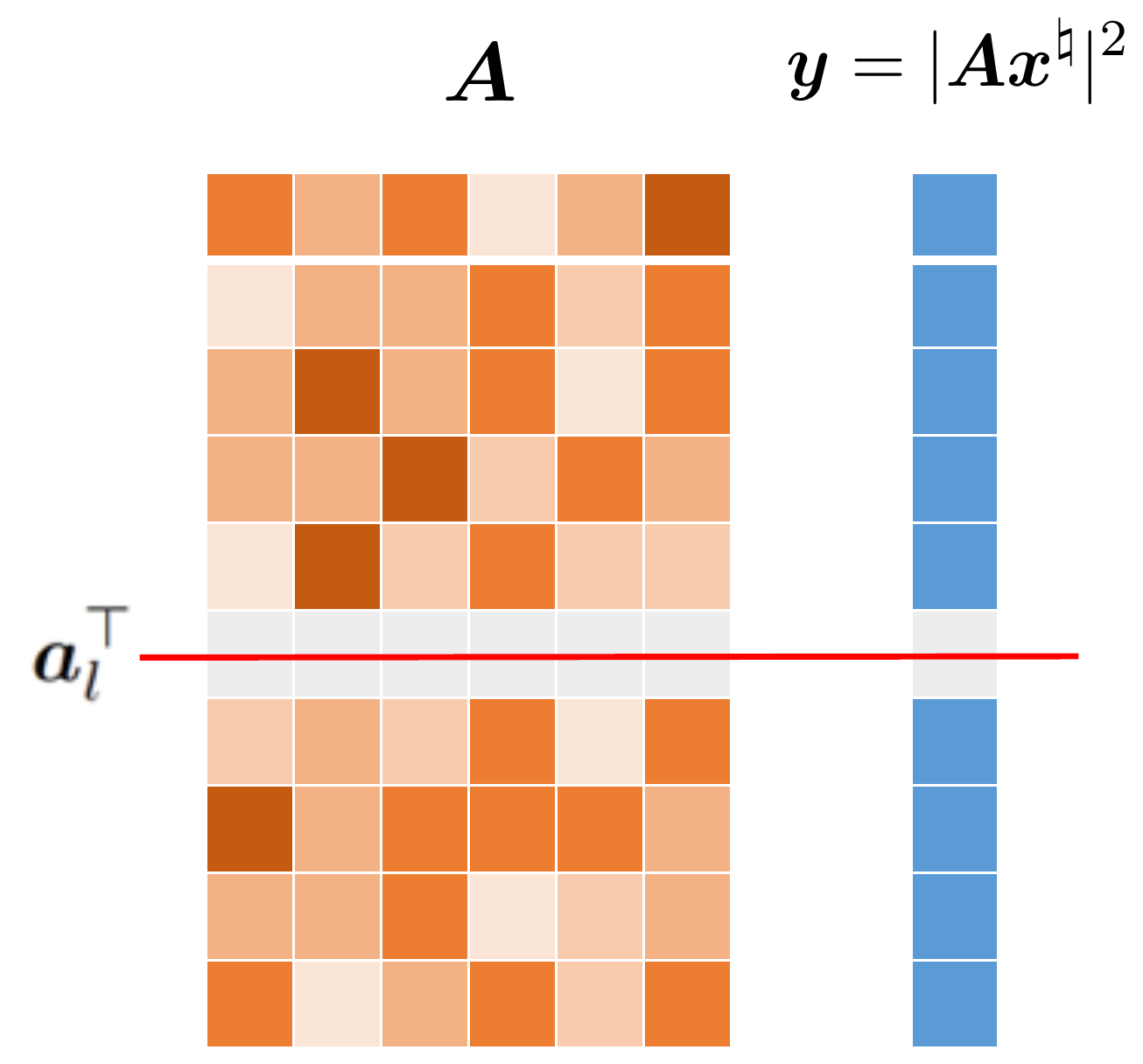} & $\,\,$$\,\,$$\,\,$\includegraphics[height=0.16\textheight]{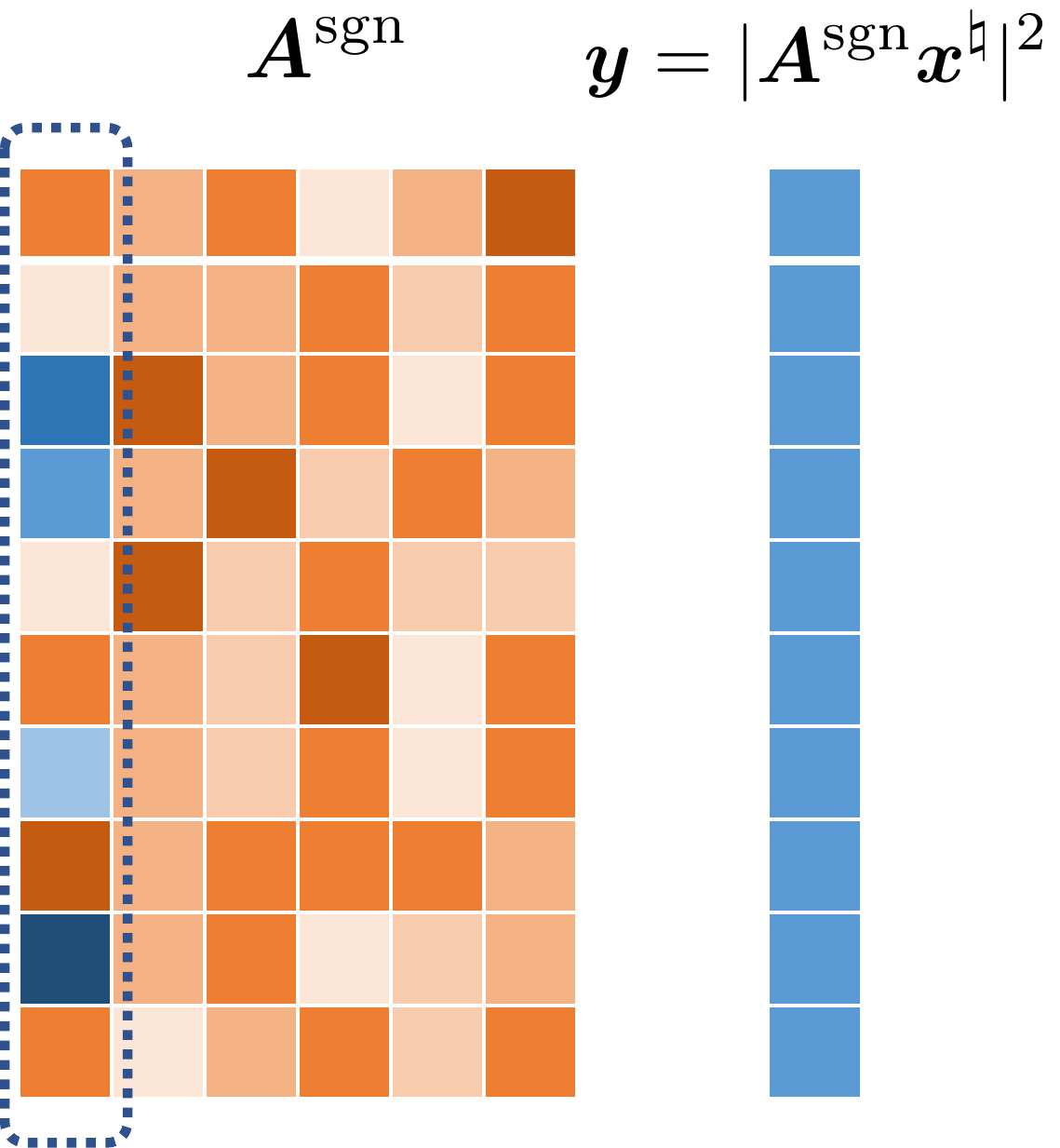} & \includegraphics[height=0.16\textheight]{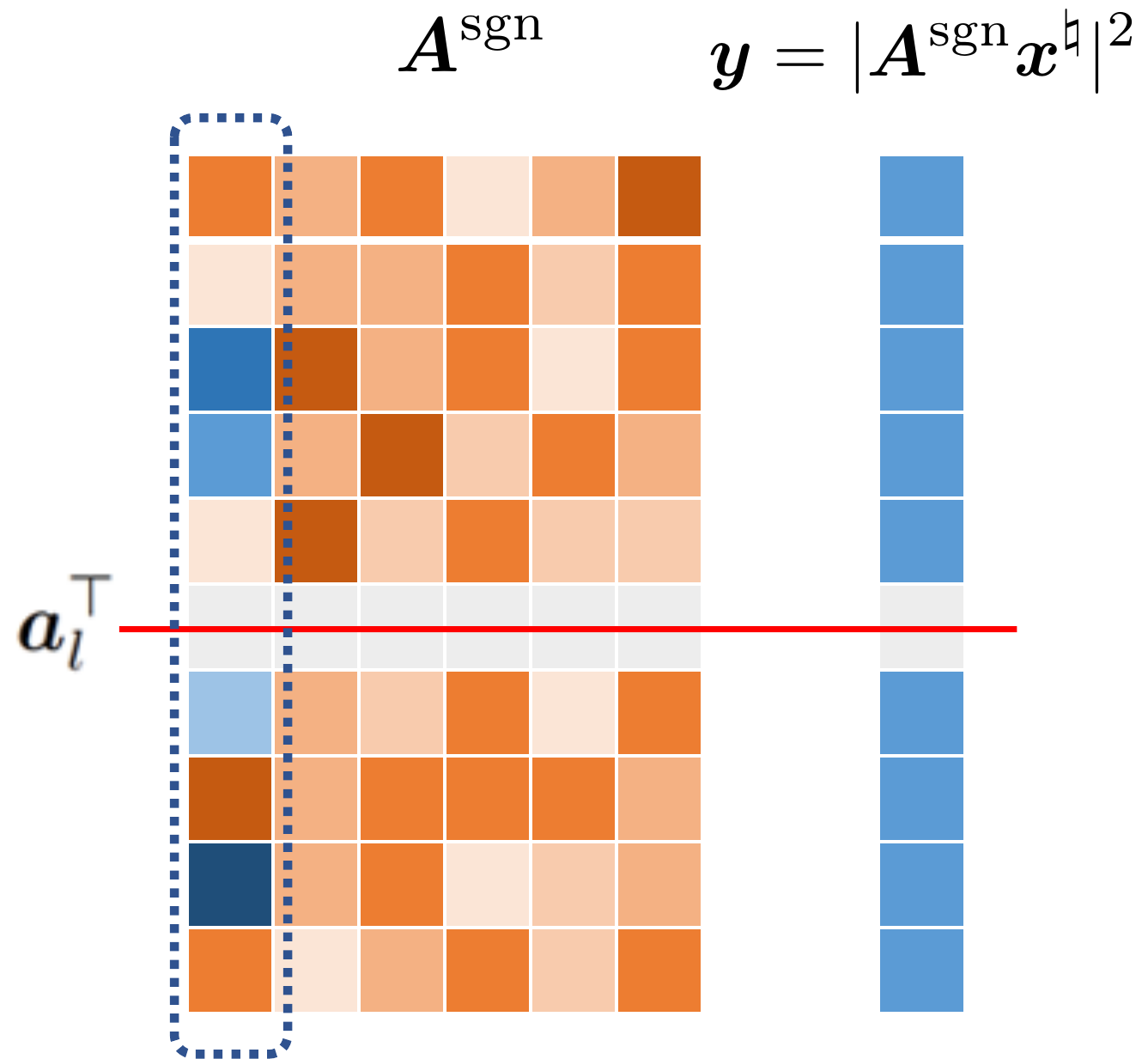}\tabularnewline
(a) $\bm{x}^{t}$  & (b) $\bm{x}^{t,(l)}$ & (c) $\bm{x}^{t,\text{sgn}}$  & (d) $\bm{x}^{t,\text{sgn},(l)}$\tabularnewline
\end{tabular}\caption{Illustration of the leave-one-out and random-sign sequences. (a)
$\{\bm{x}^{t}\}$ is constructed using all data $\{\bm{a}_{i},y_{i}\}$;
(b) $\{\bm{x}^{t,(l)}\}$ is constructed by discarding the $l$th
sample $\{\bm{a}_{l},y_{l}\}$; (c) $\{\bm{x}^{t,\text{sgn}}\}$ is
constructed by using auxiliary design vectors $\{\bm{a}_{i}^{\text{sgn}}\}$,
where $\bm{a}_{i}^{\text{sgn}}$ is obtained by randomly flipping
the sign of the first entry of $\bm{a}_{i}$; (d) $\{ \bm{x}^{t,\text{sgn},(l)}\} $
is constructed by discarding the $l$th sample $\{\bm{a}_{l}^{\text{sgn}},y_{l}\}$.
\label{fig:Illustration-of-leave-one-out}}
\end{figure}

As it turns out, we need to construct several auxiliary sequences
$\{\bm{x}^{t,(l)}\}_{t\geq0}$, $\{\bm{x}^{t,\text{sgn}}\}_{t\geq0}$
and $\{\bm{x}^{t,\text{sgn},(l)}\}_{t\geq0}$, where $\{\bm{x}^{t,(l)}\}_{t\geq0}$
is independent of the $l$th sampling vector $\bm{a}_{l}$, $\{\bm{x}^{t,\text{sgn}}\}_{t\geq0}$
is independent of the sign information of the first entries of all
$\bm{a}_{i}$'s, and $\{\bm{x}^{t,\text{sgn},(l)}\}$ is independent
of both. In addition, these auxiliary sequences are constructed by
slightly perturbing the original data (see Figure \ref{fig:Illustration-of-leave-one-out}
for an illustration), and hence one can expect all of them to stay
close to the original sequence throughout the execution of the algorithm.
Taking these two properties together, one can propagate the above
statistical independence underlying each auxiliary sequence to the
true iterates $\{\bm{x}^{t}\}$, which in turn allows us to obtain
near-optimal control of the fluctuation term $\bm{r}(\bm{x}^{t})$.
The details are postponed to Section \ref{sec:Analysis}.

\section{Related work \label{sec:Related-work}}

Solving systems of quadratic equations, or phase retrieval, has been
studied extensively in the recent literature; see \cite{shechtman2015phase}
for an overview. One popular method is convex relaxation (e.g.~\emph{PhaseLift
}\cite{candes2012phaselift}), which is guaranteed to work as long
as $m/n$ exceeds some large enough constant \cite{candes2012solving,demanet2012stable,chen2013exact,cai2015rop,kueng2017low}.
However, the resulting semidefinite program is computationally prohibitive
for solving large-scale problems. To address this issue, \cite{candes2014wirtinger}
proposed the Wirtinger flow algorithm with spectral initialization,
which provides the first convergence guarantee for nonconvex methods
without sample splitting. Both the sample and computation complexities
were further improved by \cite{ChenCandes15solving} with an adaptive
truncation strategy. Other nonconvex phase retrieval methods include
\cite{netrapalli2013phase,cai2016optimal,soltanolkotabi2017structured,wang2017solving,zhang2017reshaped,wang2017solving2,chi2016kaczmarz,duchi2017solving,gao2016phase,chen2015phase,wei2015solving,bendory2017non,tan2017phase,cai2017fast,zhang2017compressive,qu2017convolutional,zhang2016provable,yang2017misspecified,chen2017robust,zhang2017phase,ma2018optimization,chi2018nonconvex}.
Almost all of these nonconvex methods require carefully-designed initialization
to guarantee a sufficiently accurate initial point. One exception
is the approximate message passing algorithm proposed in \cite{ma2018optimization},
which works as long as the correlation between the truth and the initial
signal is bounded away from zero. This, however, does not accommodate
the case when the initial signal strength is vanishingly small (like
random initialization). Other works \cite{zhang2017phase,li2015phase} explored
the global convergence of alternating minimization\,/\,projection with random initialization
which, however, require fresh samples at least in each of the first $O(\log n)$
iterations in order to enter the local basin. In addition, \cite{li2017algorithmic}
explored low-rank recovery from quadratic measurements with near-zero
initialization. Using a truncated least-squares objective, \cite{li2017algorithmic}
established approximate (but non-exact) recovery of over-parametrized
GD. Notably, if we do not over-parametrize the phase retrieval problem,
then GD with near-zero initialization is (nearly) equivalent to running
the power method for spectral initialization\footnote{More specifically, the GD update
$\bm{x}^{t+1}=\bm{x}^{t}-m^{-1}{\eta_{t}}\sum_{i=1}^{m}\big[(\bm{a}_{i}^{\top}\bm{x}^{t})^{2}-y_{i}\big]\bm{a}_{i}\bm{a}_{i}^{\top}\bm{x}_{t}\approx(\bm{I}+m^{-1}{\eta_{t}}\sum_{i=1}^{m}y_{i}\bm{a}_{i}\bm{a}_{i}^{\top})\bm{x}_{t}$ when $\bm{x}_t\approx \bm{0}$, which is equivalent to a power iteration (without normalization) w.r.t.~the data matrix $\bm{I}+m^{-1}{\eta_{t}}\sum_{i=1}^{m}y_{i}\bm{a}_{i}\bm{a}_{i}^{\top}$.
}, which can be understood
using prior theory.

Another related line of research is the design of generic saddle-point
escaping algorithms, where the goal is to locate a second-order stationary
point (i.e.~the point with a vanishing gradient and a positive-semidefinite
Hessian). As mentioned earlier, it has been shown by \cite{sun2016geometric}
that as soon as $m\gg n\log^{3}n$, all local minima are global and
all the saddle points are strict. With these two geometric properties
in mind, saddle-point escaping algorithms are guaranteed to converge
globally for phase retrieval. Existing saddle-point escaping algorithms include but are not limited to Hessian-based methods \cite{nesterov2006cubic,sun2016geometric}
(see also \cite{agarwal2016finding,allen2017natasha,jin2017escape}
for some reviews), noisy stochastic gradient descent \cite{ge2015escaping},
perturbed gradient descent \cite{jin2017escape}, and normalized gradient
descent \cite{murray2017revisiting}. On the one hand, the results
developed in these works are fairly general: they establish polynomial-time
convergence guarantees under a few generic geometric conditions. On
the other hand, the iteration complexity derived therein may be pessimistic
when specialized to a particular problem.

Take phase retrieval and the perturbed gradient descent algorithm
\cite{jin2017escape} as an example. It has been shown in \cite[Theorem 5]{jin2017escape}
that for an objective function that is $L$-gradient Lipschitz,
$\rho$-Hessian Lipschitz, $(\theta,\gamma,\zeta)$-strict saddle,
and also locally $\alpha$-strongly convex and $\beta$-smooth (see
definitions in \cite{jin2017escape}), it takes\footnote{When applied to phase retrieval with $m\asymp n\,\mathrm{poly}\log n$,
one has $L\asymp n$, $\rho\asymp n$, $\theta\asymp\gamma\asymp1$
(see \cite[Theorem 2.2]{sun2016geometric}), $\alpha\asymp1$, and
$\beta\gtrsim n$ (ignoring logarithmic factors).}
\[
O\Bigg(\frac{L}{\left[\min\left(\theta,\gamma^{2}/\rho\right)\right]^{2}}+\frac{\beta}{\alpha}\log\frac{1}{\epsilon}\Bigg)=O\left(n^{3}+n\log\frac{1}{\epsilon}\right)
\]
iterations (ignoring logarithmic factors) for perturbed gradient descent
to converge to $\epsilon$-accuracy. In fact, even with Nesterov's
accelerated scheme \cite{jin2017accelerated}, the iteration complexity
for entering the local region is at least
\[
O\Bigg(\frac{L^{1/2}\rho^{1/4}}{\left[\min\left(\theta,\gamma^{2}/\rho\right)\right]^{7/4}}\Bigg)=O\left(n^{2.5}\right).
\]
Both of them are much larger than the $O\big(\log n+\log(1/\epsilon)\big)$
complexity established herein. This is primarily due to the following
facts: (i) the Lipschitz constants of both the gradients and the Hessians
are quite large, i.e.~$L\asymp n$ and $\rho\asymp n$ (ignoring
log factors), which are, however, treated as dimension-independent
constants in the aforementioned papers; (ii) the local condition number
is also large, i.e.~$\beta/\alpha\asymp n$. In comparison, as suggested
by our theory, the GD iterates with random initialization are always
confined within a restricted region enjoying much more benign geometry
than the worst-case$\,$/$\,$global characterization.

Furthermore, the above saddle-escaping first-order methods are often
more complicated than vanilla GD. Despite its algorithmic simplicity
and wide use in practice, the convergence rate of GD with random initialization
remains largely unknown. In fact, Du et al.~\cite{du2017gradient}
demonstrated that there exist non-pathological functions such that
GD can take exponential time to escape the saddle points when initialized
randomly. In contrast, as we have demonstrated, saddle points are
not an issue for phase retrieval; the GD iterates with random initialization
never get trapped in the saddle points.

Finally, the leave-one-out arguments have been invoked to analyze
other high-dimensional statistical inference problems including robust
M-estimators \cite{el2013robust,el2015impact}, and
maximum likelihood theory for
logistic regression \cite{sur2017likelihood}, etc. In addition, \cite{zhong2017near,chen2017spectral,abbe2017entrywise}
made use of the leave-one-out trick to derive entrywise perturbation
bounds for eigenvectors resulting from certain spectral methods. The
techniques have also been applied by \cite{ma2017implicit,li2018nonconvex}
to establish local linear convergence of vanilla GD for nonconvex
statistical estimation problems in the presence of proper spectral
initialization.

\section{Analysis\label{sec:Analysis}}

In this section, we first provide a more general version of Theorem
\ref{thm:intro} as follows.  It spells out exactly the conditions on $\bx^0$ in order for vanilla GD with random initialization to succeed.

\begin{theorem}\label{thm:main}Fix $\bm{x}^{\natural}\in\RR^{n}$.
Suppose $\bm{a}_{i}\overset{\mathrm{i.i.d.}}{\sim}\mathcal{N}\left(\bm{0},\bm{I}_{n}\right)$
$(1\leq i\leq m)$ and $m\geq Cn\log^{13}m$  for some sufficiently
large constant $C>0$. Assume that the initialization $\bm{x}^{0}$
is independent of $\{\bm{a}_{i}\}$ and obeys
\begin{equation}
\frac{\big|\langle\bm{x}^{0},\bm{x}^{\natural}\rangle\big|}{\|\bm{x}^{\natural}\|_{2}^{2}}\geq\frac{1}{\sqrt{n\log n}}\qquad\text{and}\qquad\left(1-\frac{1}{\log n}\right)\|\bm{x}^{\natural}\|_{2}\leq\|\bm{x}^{0}\|_{2}\leq\left(1+\frac{1}{\log n}\right)\|\bm{x}^{\natural}\|_{2},\label{eq:initilization-condition}
\end{equation}
and that the stepsize satisfies $\eta_{t}\equiv\eta=c/\|\bm{x}^{\natural}\|_{2}^{2}$
for some sufficiently small constant $c>0$. Then there exist a sufficiently
small absolute constant $0<\gamma<1$ and $T_{\gamma}\lesssim\log n$
such that with probability at least $1-O(m^{2}e^{-1.5n})-O(m^{-9})$,
\begin{enumerate}
\item the GD iterates (\ref{eq:gradient_update-WF}) converge linearly to
$\bm{x}^{\natural}$ after $t\geq T_{\gamma}$, namely,
\begin{align*}
\mathrm{dist}\left(\bm{x}^{t},\bm{x}^{\natural}\right) & \leq\left(1-\frac{\eta}{2}\left\Vert \bm{x}^{\natural}\right\Vert _{2}^{2}\right)^{t-T_{\gamma}}\cdot\gamma\left\Vert \bm{x}^{\natural}\right\Vert _{2},\qquad\forall\,t\geq T_{\gamma};
\end{align*}
\item the strength ratio of the signal component $\frac{\langle\bm{x}^{t},\bm{x}^{\natural}\rangle}{\|\bm{x}^{\natural}\|_{2}^{2}}\bm{x}^{\natural}$
to the orthogonal component $\bm{x}^{t}-\frac{\langle\bm{x}^{t},\bm{x}^{\natural}\rangle}{\|\bm{x}^{\natural}\|_{2}^{2}}\bm{x}^{\natural}$
obeys
\begin{equation}
\frac{\Big\|\frac{\langle\bm{x}^{t},\bm{x}^{\natural}\rangle}{\|\bm{x}^{\natural}\|_{2}^{2}}\bm{x}^{\natural}\Big\|_{2}}{\Big\|\bm{x}^{t}-\frac{\langle\bm{x}^{t},\bm{x}^{\natural}\rangle}{\|\bm{x}^{\natural}\|_{2}^{2}}\bm{x}^{\natural}\Big\|_{2}}\gtrsim\frac{1}{\sqrt{n\log n}}(1+c_{1}\eta^2)^{t},\qquad t=0,1,\cdots\label{eq:SNR-lower-bound}
\end{equation}
for some constant $c_{1}>0$.
\end{enumerate}
\end{theorem}
Several remarks regarding Theorem \ref{thm:main} are in order.
\begin{itemize}
	\item Our current sample complexity reads $m\gtrsim n \log^{13}m$, which is optimal up to logarithmic factors. It is possible to further reduce the logarithmic factors using more refined probabilistic tools, which we leave for future work.
	\item We can also prove similar performance guarantees for noisy phase retrieval. For brevity, we do not provide the exact theorem and the detailed proofs. The readers will find them in the last author's Ph.D. thesis. 
	\item The random initialization $\bm{x}^{0}\sim\mathcal{N}(\bm{0},n^{-1}\|\bm{x}^{\natural}\|_{2}^{2}\bm{I}_{n})$
obeys the condition (\ref{eq:initilization-condition}) with probability
exceeding $1-O(1/\sqrt{\log n})$, which in turn establishes Theorem
\ref{thm:intro}.
	\item Theorem \ref{thm:main} requires an initialization $\bm{x}^{0}$ which is independent of the data and the knowledge of $\|\bm{x}^{\natural}\|$, which is not practical.   One possible method is to estimate it from the data, which results in an initial value that depends on the data.  The following theorem demonstrate both independent initial value and known $\|\bm{x}^{\natural}\|$ are not necessary, resulting a practical algorithm.

	\begin{theorem}\label{thm:dependent-init}Let
			\[
	\bm{x}^{0} = \sqrt{\frac{1}{m}\sum_{i=1}^{m}y_{i}} \cdot \bm{u},
\]
where $\bm{u}$ is uniformly distributed over the unit sphere. With probability
at least $1-O(1/\sqrt{\log n})$ all the claims in Theorem \ref{thm:main}
			continue to hold.  \end{theorem}

\begin{proof}The proof is very similar to that of Theorem \ref{thm:main},
with only a few changes. See Appendix \ref{sec:Proof-of-Theorem-dependent} for detailed explanations.
\end{proof}
	
\end{itemize}

The remainder of this section is then devoted to proving Theorem \ref{thm:main}.
Without loss of generality\footnote{This is because of the rotational invariance of Gaussian distributions.},
we will assume throughout that
\begin{equation}
\bm{x}^{\natural}=\bm{e}_{1}\qquad\text{and}\qquad x_{1}^{0}>0.\label{eq:assumption-x0-xnatural}
\end{equation}
Given this, one can decompose
\begin{equation}
\bm{x}^{t}=x_{\|}^{t}\bm{e}_{1}+\left[\begin{array}{c}
0\\
\bm{x}_{\perp}^{t}
\end{array}\right]\label{eq:xt-decompose-e1}
\end{equation}
where $x_{\|}^{t}=x_{1}^{t}$ and $\bm{x}_{\perp}^{t}=[x_{i}^{t}]_{2\leq i\leq n}$
as introduced in Section \ref{sec:why}. For notational
simplicity, we define
\begin{equation}
\alpha_{t}:=x_{\parallel}^{t}\qquad\text{and}\qquad\beta_{t}:=\|\bm{x}_{\perp}^{t}\|_{2}.\label{eq:defn-alphat-betat}
\end{equation}
Intuitively, $\alpha_{t}$ represents the size of the signal component,
whereas $\beta_{t}$ measures the size of the component orthogonal
to the signal direction. In view of (\ref{eq:assumption-x0-xnatural}), we have $\alpha_{0}>0$.

\subsection{Outline of the proof}

To begin with, it is easily seen that if $\alpha_{t}$ and $\beta_{t}$
(cf.~\eqref{eq:defn-alphat-betat}) obey $|\alpha_{t}-1|\leq\gamma/2$
and $\beta_{t}\leq\gamma/2$, then
\[
\mathrm{dist}\left(\bm{x}^{t},\bm{x}^{\natural}\right)\leq\|\bm{x}^{t}-\bm{x}^{\natural}\|_{2}\leq\big|\alpha_{t}-1\big|+\big|\beta_{t}\big|\leq\gamma.
\]
Therefore, our first step --- which is concerned with proving $\mathrm{dist}(\bm{x}^{t},\bm{x}^{\natural})\leq\gamma$
--- comes down to  the following two steps.
\begin{enumerate}
\item Show that if $\alpha_{t}$ and $\beta_{t}$ satisfy the approximate
state evolution (see (\ref{subeq:approx-state-evolution})), then
there exists some $T_{\gamma}=O\left(\log n\right)$ such that
\begin{equation}
\left|\alpha_{T_{\gamma}}-1\right|\leq\gamma/2\qquad\text{and}\qquad\beta_{T_{\gamma}}\leq\gamma/2,\label{eq:defn-local}
\end{equation}
which would immediately imply that
\[
\mathrm{dist}\left(\bm{x}^{T_{\gamma}},\bm{x}^{\natural}\right)
\leq\gamma.
\]
Along the way, we will also show that the ratio $\alpha_{t}/\beta_{t}$ grows exponentially fast.
\item Justify that $\alpha_{t}$ and $\beta_{t}$ satisfy the approximate
state evolution with high probability, using (some variants of) leave-one-out
arguments.
\end{enumerate}
After $t\geq T_{\gamma}$, we can invoke prior theory \cite{ma2017implicit} concerning local
convergence to show that with high probability,
\[
\mathrm{dist}\left(\bm{x}^{t},\bm{x}^{\natural}\right)\leq(1-\rho)^{t-T_{\gamma}}\|\bm{x}^{T_{\gamma}}-\bm{x}^{\natural}\|_{2},\qquad\forall \,t>T_{\gamma}
\]
for some constant $0<\rho<1$ independent of $n$ and $m$.

\subsection{Dynamics of approximate state evolution }

This subsection formalizes our intuition in Section \ref{sec:why}:
as long as the approximate state evolution holds, then one can find
$T_{\gamma}\lesssim\log n$ obeying condition (\ref{eq:defn-local}).
In particular, the approximate state evolution is given by\begin{subequations}\label{subeq:state-evolution}
\begin{align}
\alpha_{t+1} & =\left\{ 1+3\eta\left[1-\left(\alpha_{t}^{2}+\beta_{t}^{2}\right)\right]+\eta\zeta_{t}\right\} \alpha_{t},\label{eq:alphat}\\
\beta_{t+1} & =\left\{ 1+\eta\left[1-3\left(\alpha_{t}^{2}+\beta_{t}^{2}\right)\right]+\eta\rho_{t}\right\} \beta_{t},\label{eq:betat}
\end{align}
\end{subequations}where $\left\{ \zeta_{t}\right\} $ and $\left\{ \rho_{t}\right\} $
represent the perturbation terms. Our result is this:

\begin{lemma}\label{lemma:iterative}Let $\gamma>0$ be some sufficiently small constant, and consider
the approximate state evolution \eqref{subeq:state-evolution}. Suppose
the initial point obeys
\begin{equation}
\alpha_{0}\geq\frac{1}{\sqrt{n\log n}}\qquad\text{and}\qquad1-\frac{1}{\log n}\leq\sqrt{\alpha_{0}^{2}+\beta_{0}^{2}}\leq1+\frac{1}{\log n}.\label{eq:initial-point-alpha-beta}
\end{equation}
and the perturbation terms satisfy
\[
\max\left\{ \left|\zeta_{t}\right|,\left|\rho_{t}\right|\right\} \leq\frac{c_{3}}{\log n},\qquad t=0,1,\cdots
\]
for some sufficiently small constant $c_{3}>0$.

(a) Let
\begin{align}
T_{\gamma} & :=\min\big\{ t: \left|\alpha_{t}-1\right|\leq\gamma/2\text{ and }\beta_{t}\leq\gamma/2\big\}.
\end{align}
Then for any sufficiently large $n$ and $m$ and any sufficiently
small constant $\eta>0$, one has
\begin{equation}
T_{\gamma}\lesssim\log n,\label{eq:Tgamma-UB}
\end{equation}
and there exist some constants $c_{5},c_{10}>0$ independent of $n$
and $m$ such that
\begin{equation}
\frac{1}{2\sqrt{n\log n}}\leq\alpha_{t}\leq2,\quad c_{5}\leq\beta_{t}\leq1.5\quad\text{and}\quad\frac{\alpha_{t+1}/\alpha_{t}}{\beta_{t+1}/\beta_{t}}\geq1+c_{10}\eta^2,\qquad0\leq t\leq T_{\gamma}.\label{eq:alpha-beta-range-SNR}
\end{equation}

(b) If we define
\begin{align}
T_{0} & :=\min\left\{ t: \alpha_{t+1}\geq c_{6}/\log^{5}m\right\} ,\label{eq:defn-T0-clean-1}\\
T_{1} & :=\min\left\{ t: \alpha_{t+1}>c_{4}\right\} ,\label{eq:defn-T1}
\end{align}
for some arbitrarily small constants $c_{4},c_{6}>0$, then
	\begin{enumerate}
\item[1)] $T_{0}\leq T_{1}\leq T_{\gamma}\lesssim\log n$; $T_{1}-T_{0}\lesssim\log\log m$;
$T_{\gamma}-T_{1}\lesssim1$;
\item[2)] For $T_{0}<t\leq T_{\gamma}$, one has $\alpha_{t}\geq c_{6}/\log^{5}m$.
\end{enumerate}
\end{lemma}\begin{proof}See Appendix \ref{sec:Proof-of-Lemma-iterative}.\end{proof}

\begin{remark}
	Recall that $\gamma$ is sufficiently small and $(\alpha,\beta)=(1,0)$ represents the global minimizer. Since $|\alpha_0- 1|\approx 1$, one has $T_{\gamma}>0$, which denotes the first time when the iterates enter the local region surrounding the global minimizer. In addition, the fact that $\alpha_0 \lesssim 1/\sqrt{n} $ gives $T_0>0$ and $T_1>0$, both of which indicate the first time when the signal strength is sufficiently large.
\end{remark}

Lemma \ref{lemma:iterative} makes precise that under the approximate
state evolution, the first stage enjoys a fairly short duration $T_{\gamma}\lesssim\log n$.
Moreover, the size of the signal component grows faster than that
of the orthogonal component for any iteration $t<T_{\gamma}$,
thus confirming the exponential growth of $\alpha_{t}/\beta_{t}$.

In addition, Lemma \ref{lemma:iterative} identifies two midpoints
$T_{0}$ and $T_{1}$ when the sizes of the signal component $\alpha_{t}$
become sufficiently large. These are helpful in our subsequent analysis.
In what follows, we will divide Stage~1 (which consists of all iterations
up to $T_{\gamma}$) into two phases:
\begin{itemize}
\item \emph{Phase I}: consider the duration $0\leq t\leq T_{0}$;
\item \emph{Phase II}: consider all iterations with $T_{0}<t\leq T_{\gamma}$.
\end{itemize}
We will justify the approximate state evolution (\ref{subeq:state-evolution})
for these two phases separately.

\subsection{Motivation of the leave-one-out approach\label{subsec:Motivation}}

As we have alluded in Section \ref{subsec:why-loo}, the main difficulty
in establishing the approximate state evolution (\ref{subeq:state-evolution})
lies in controlling the perturbation terms to the desired orders (i.e.~$\left|\zeta_{t}\right|,\left|\rho_{t}\right|\ll1/\log n$
in Lemma~\ref{lemma:iterative}). To achieve this, we advocate the
use of (some variants of) leave-one-out sequences to help establish
certain ``near-independence'' between $\bm{x}^{t}$ and certain components of $\{\bm{a}_{i}\}$.

We begin by taking a closer look at the perturbation terms. Regarding
the signal component, it is easily seen from (\ref{eq:defn-r1}) that
\[
	x_{\parallel}^{t+1}=\left\{ 1+3\eta\left(1-\|\bm{x}^{t}\|_{2}^{2}\right)\right\} x_{\parallel}^{t}-\eta r_{1}(\bm{x}^{t}),
\]
where the perturbation term $r_{1}(\bm{x}^{t})$ obeys
\begin{align}
r_{1}(\bm{x}^{t}) & =\underset{:=I_{1}}{\underbrace{\left[1-\big(x_{\|}^{t}\big)^{2}\right]x_{\|}^{t} \left(\frac{1}{m}\sum_{i=1}^{m}a_{i,1}^{4}-3\right)}}+\underset{:=I_{2}}{\underbrace{\left[1-3\big(x_{\|}^{t}\big)^{2}\right] \frac{1}{m}\sum_{i=1}^{m}a_{i,1}^{3}\bm{a}_{i,\perp}^{\top}\bm{x}_{\perp}^{t}}}\nonumber \\
 & \quad-\underset{:=I_{3}}{\underbrace{3x_{\|}^{t} \left(\frac{1}{m}\sum_{i=1}^{m}\big(\bm{a}_{i,\perp}^{\top}\bm{x}_{\perp}^{t}\big)^{2}a_{i,1}^{2}-\left\Vert \bm{x}_{\perp}^{t}\right\Vert _{2}^{2}\right)}}-\underset{:=I_{4}}{\underbrace{\frac{1}{m}\sum_{i=1}^{m}\big(\bm{a}_{i,\perp}^{\top}\bm{x}_{\perp}^{t}\big)^{3}a_{i,1}}}.\label{eq:r1-I1-4-defn}
\end{align}
Here and throughout the paper, for any vector $\bm{v}\in\mathbb{R}^{n}$, $\bm{v}_{\perp}\in\RR^{n-1}$
denotes the 2nd through the $n$th entries of $\bm{v}$. Due to the
dependency between $\bm{x}^{t}$ and $\{\bm{a}_{i}\}$, it is challenging
to obtain sharp control of some of these terms.

In what follows, we use the term $I_{4}$ to explain and motivate
our leave-one-out approach. As discussed in Section \ref{subsec:why-loo},
$I_{4}$ needs to be controlled to the level $O(1/(\sqrt{n}\,\mathrm{poly}\log(n)))$.
This precludes us from seeking a uniform bound on the function $h(\bm{x}):=m^{-1}\!\sum_{i=1}^{m}(\bm{a}_{i,\perp}^{\top}\bm{x}_{\perp})^{3}a_{i,1}$
over all $\bm{x}$ (or even all $\bm{x}$ within the set $\mathcal{C}$
incoherent with $\{\bm{a}_{i}\}$), since the uniform bound $\sup_{\bm{x}\in\mathcal{C}}|h(\bm{x})|$
can be $O(\sqrt{n}/\text{poly}\log(n))$ times larger than the desired
order.

In order to control $I_{4}$ to the desirable order, one strategy
is to approximate it by a sum of independent variables and then invoke
the CLT. Specifically, we first rewrite $I_{4}$ as
\[
I_{4}=\frac{1}{m}\sum_{i=1}^{m}\left(\bm{a}_{i,\perp}^{\top}\bm{x}_{\perp}^{t}\right)^{3}\left|a_{i,1}\right|\xi_{i}
\]
with $\xi_{i}:=\mathrm{sgn}(a_{i,1})$. Here $\text{sgn}(\cdot)$
denotes the usual sign function. To exploit the statistical independence
between $\xi_{i}$ and $\{|a_{i,1}|,\bm{a}_{i,\perp}\}$, we would
like to identify some vector independent of $\xi_{i}$ that well approximates
$\bm{x}^{t}$. If this can be done, then one may treat $I_{4}$ as
a weighted independent sum of $\{\xi_{i}\}$. Viewed in this light,
our plan is the following:
\begin{enumerate}
\item Construct a sequence $\{\bm{x}^{t,\mathrm{sgn}}\}$ independent of
$\{\xi_{i}\}$ obeying $\bm{x}^{t,\mathrm{sgn}}\approx\bm{x}^{t}$,
so that
\[
I_{4}\approx\frac{1}{m}\sum_{i=1}^{m}\underset{:=w_{i}}{\underbrace{\left(\bm{a}_{i,\perp}^{\top}\bm{x}_{\perp}^{t,\mathrm{sgn}}\right)^{3}\left|a_{i,1}\right|}}\,\xi_{i}.
\]
One can then apply standard concentration results (e.g.~the Bernstein
inequality) to control $I_{4}$, as long as none of the weight $w_{i}$
is exceedingly large.
\item Demonstrate that the weight $w_{i}$ is well-controlled, or equivalently,
$\big|\bm{a}_{i,\perp}^{\top}\bm{x}_{\perp}^{t,\mathrm{sgn}}\big|$
($1\leq i\leq m$) is not much larger than its typical size. This
can be accomplished by identifying another sequence $\{\bm{x}^{t,(i)}\}$
independent of $\bm{a}_{i}$ such that $\bm{x}^{t,(i)}\approx\bm{x}^{t}\approx\bm{x}^{t,\mathrm{sgn}}$,
followed by the argument:
\begin{equation} \label{eq:incoherence-intuition}
\big|\bm{a}_{i,\perp}^{\top}\bm{x}_{\perp}^{t,\mathrm{sgn}}\big|\approx\big|\bm{a}_{i,\perp}^{\top}\bm{x}_{\perp}^{t} \big| \approx\big|\bm{a}_{i,\perp}^{\top}\bm{x}_{\perp}^{t,(i)}\big|\lesssim\sqrt{\log m}\big\|\bm{x}_{\perp}^{t,(i)}\big\|_{2}\approx\sqrt{\log m}\big\|\bm{x}_{\perp}^{t}\big\|_{2}.
\end{equation}
Here, the inequality follows from standard Gaussian tail bounds and
		the independence between $\bm{a}_{i}$ and $\bm{x}^{t,(i)}$. This explains why we would like to construct $\{\bm{x}^{t,(i)}\}$ for each $1\leq i\leq m$.
\end{enumerate}
As we will detail in the next subsection, such auxiliary sequences
are constructed by leaving out a small amount of relevant information
from the collected data before running the GD algorithm, which is
a variant of the ``leave-one-out'' approach rooted in probability
theory and random matrix theory.

\subsection{Leave-one-out and random-sign sequences\label{subsec:Leave-one-out-sequences}}

We now describe how to design auxiliary sequences to help establish
certain independence properties between the gradient iterates $\left\{ \bm{x}^{t}\right\} $
and the design vectors $\left\{ \bm{a}_{i}\right\} $. In the sequel,
we formally define the three sets of auxiliary sequences $\{\bm{x}^{t,\left(l\right)}\},\{\bm{x}^{t,\text{sgn}}\},\{\bm{x}^{t,\text{sgn},\left(l\right)}\}$
as introduced in Section \ref{subsec:why-loo} and Section \ref{subsec:Motivation}.

\begin{algorithm}[t]
\caption{The $l$th leave-one-out sequence}

\label{alg:wf-LOO}\begin{algorithmic}

\STATE \textbf{{Input}}: $\{\bm{a}_{i}\}_{1\leq i\leq m,i\neq l}$,
$\{y_{i}\}_{1\leq i\leq m,i\neq l}$, and $\bm{x}^{0}$.


\STATE \textbf{{Gradient updates}}: \textbf{for} $t=0,1,2,\ldots,T-1$
\textbf{do}

\STATE
\begin{equation}
\bm{x}^{t+1,(l)}=\bm{x}^{t,(l)}-\eta_{t}\nabla f^{(l)}(\bm{x}^{t,(l)}),\label{eq:gradient-update-leave-WF}
\end{equation}
where $\bm{x}^{0,(l)}=\bm{x}^{0}$ and $f^{(l)}(\bm{x})=(1/4m)\!\cdot\!\sum_{i:i\neq l}[(\bm{a}_{i}^{\top}\bm{x})^{2}-(\bm{a}_{i}^{\top}\bm{x}^{\natural})^{2}]^{2}$.

\end{algorithmic}
\end{algorithm}

\begin{algorithm}[t]
\caption{The random-sign sequence }

\label{alg:wf-LOO-sign}\begin{algorithmic}

\STATE \textbf{{Input}}: $\{|a_{i,1}|\}_{1\leq i\leq m}$, $\{\bm{a}_{i,\perp}\}_{1\leq i\leq m}$,
$\{\xi_{i}^{\mathrm{sgn}}\}_{1\le i\leq m}$, $\{y_{i}\}_{1\leq i\leq m}$,
$\bm{x}^{0}$.


\STATE \textbf{{Gradient updates}}: \textbf{for} $t=0,1,2,\ldots,T-1$
\textbf{do}

\STATE
\begin{equation}
\bm{x}^{t+1,\mathrm{sgn}}=\bm{x}^{t,\mathrm{sgn}}-\eta_{t}\nabla f^{\mathrm{sgn}}(\bm{x}^{t,\mathrm{sgn}}),\label{eq:gradient-update-leave-WF-1}
\end{equation}
where $\bm{x}^{0,\mathrm{sgn}}=\bm{x}^{0}$, $f^{\mathrm{sgn}}(\bm{x})=\frac{{1}}{4m}\sum_{i=1}^{m}[(\bm{a}_{i}^{\mathrm{sgn}\top}\bm{x})^{2}-(\bm{a}_{i}^{\mathrm{sgn}\top}\bm{x}^{\natural})^{2}]^{2}$
with $\bm{a}_{i}^{\mathrm{sgn}}:=\left[\begin{array}{c}
\xi_{i}^{\mathrm{sgn}}|a_{i,1}|\\
\bm{a}_{i,\perp}
\end{array}\right]$.

\end{algorithmic}
\end{algorithm}

\begin{algorithm}[t]
\caption{The $l$th leave-one-out and random-sign sequence}

\label{alg:wf-LOO-sign-l}\begin{algorithmic}

\STATE \textbf{{Input}}:$\{|a_{i,1}|\}_{1\leq i\leq m,i\neq l}$,
$\{\bm{a}_{i,\perp}\}_{1\leq i\leq m,i\neq l}$, $\{\xi_{i}^{\mathrm{sgn}}\}_{1\le i\leq m,,i\neq l}$,
$\{y_{i}\}_{1\leq i\leq m,i\neq l}$, $\bm{x}^{0}$.


\STATE \textbf{{Gradient updates}}: \textbf{for} $t=0,1,2,\ldots,T-1$
\textbf{do}

\STATE
\begin{equation}
\bm{x}^{t+1,\mathrm{sgn},(l)}=\bm{x}^{t,\mathrm{sgn},(l)}-\eta_{t}\nabla f^{\mathrm{sgn},(l)}(\bm{x}^{t,\mathrm{sgn},(l)}),\label{eq:gradient-update-leave-WF-1-1}
\end{equation}
where $\bm{x}^{0,\mathrm{sgn},(l)}=\bm{x}^{0}$, $f^{\mathrm{sgn},(l)}\big(\bm{x}\big)=\frac{1}{4m}\sum_{i:i\neq l}\left[\big(\bm{a}_{i}^{\mathrm{sgn}\top}\bm{x}\big)^{2}-\big(\bm{a}_{i}^{\mathrm{sgn}\top}\bm{x}^{\natural}\big)^{2}\right]^{2}$
with $\bm{a}_{i}^{\mathrm{sgn}}:=\left[\begin{array}{c}
\xi_{i}^{\mathrm{sgn}}|a_{i,1}|\\
\bm{a}_{i,\perp}
\end{array}\right]$.

\end{algorithmic}
\end{algorithm}
\begin{itemize}
\item \emph{Leave-one-out sequences} $\{\bm{x}^{t,(l)}\}_{t\geq0}$. For
each $1\leq l\leq m$, we introduce a sequence $\{\bm{x}^{t,(l)}\}$,
which  drops the $l$th sample and runs GD w.r.t.~the auxiliary objective function
\begin{equation}
f^{(l)}\left(\bm{x}\right)=\frac{1}{4m}\sum_{i:i\neq l}\left[\left(\bm{a}_{i}^{\top}\bm{x}\right)^{2}-\left(\bm{a}_{i}^{\top}\bm{x}^{\natural}\right)^{2}\right]^{2}.\label{eq:auxiliary-loss-LOO}
\end{equation}
See Algorithm~\ref{alg:wf-LOO} for details and also
Figure~\ref{fig:Illustration-of-leave-one-out}(a) for an illustration.
One of the most important features of $\{\bm{x}^{t,(l)}\}$ is that
all of its iterates are statistically independent of $(\bm{a}_{l},y_{l})$,
and hence are incoherent with $\bm{a}_{l}$ with high probability,
in the sense that $\big|\bm{a}_{l}^{\top}\bm{x}^{t,(l)}\big|\lesssim\sqrt{\log m} \|\bm{x}^{t,(l)} \|_{2}$.
		Such incoherence properties further allow us to control  both $\big|\bm{a}_{l}^{\top}\bm{x}^{t}\big|$ and $\big|\bm{a}_{l}^{\top}\bm{x}^{t,\mathrm{sgn}}\big|$ (see \eqref{eq:incoherence-intuition}), which is crucial for controlling
the size of the residual terms (e.g.~$r_{1}(\bm{x}^{t})$ as defined
		in (\ref{eq:defn-r1})). Notably, the sequence $\{\bm{x}^{t,(l)}\}$ has also been applied by \cite{ma2017implicit} to justify the success of GD with spectral initialization for several nonconvex statistical estimation problems.
\item \emph{Random-sign sequence} $\left\{ \bm{x}^{t,\mathrm{sgn}}\right\} _{t\geq0}$.
Introduce a collection of auxiliary design vectors $\left\{ \bm{a}_{i}^{\mathrm{sgn}}\right\} _{1\leq i\leq m}$
defined as
\begin{equation}
\bm{a}_{i}^{\mathrm{\mathrm{sgn}}}:=\left[\begin{array}{c}
\xi_{i}^{\mathrm{sgn}}\left|a_{i,1}\right|\\
\bm{a}_{i,\perp}
\end{array}\right],\label{eq:auxiliary-random-sign-vector}
\end{equation}
where $\left\{ \xi_{i}^{\mathrm{sgn}}\right\} _{1\leq i\leq m}$ is
a set of Rademacher random variables independent of $\left\{ \bm{a}_{i}\right\} $, i.e.
\begin{equation}
\xi_{i}^{\mathrm{sgn}}\overset{\text{i.i.d.}}{=}\begin{cases}
1, & \text{with probability 1/2},\\
-1, & \text{else},
\end{cases}\qquad1\leq i\leq m.\label{eq:random-sign-xi}
\end{equation}
In words, $\bm{a}_{i}^{\mathrm{sgn}}$ is generated by randomly flipping
the sign of the first entry of $\bm{a}_{i}$. To simplify the notations
hereafter, we also denote
\begin{equation}
\xi_{i}=\text{sgn}(a_{i,1}).\label{eq:defn-xi}
\end{equation}
As a result, $\bm{a}_{i}$ and $\bm{a}_{i}^{\mathrm{sgn}}$ differ
only by a single bit of information. With these auxiliary design vectors
in place, we generate a sequence $\{\bm{x}^{t,\mathrm{sgn}}\}$ by
running GD w.r.t.~the auxiliary loss function
\begin{equation}
f^{\mathrm{sgn}}\big(\bm{x}\big)=\frac{1}{4m}\sum_{i=1}^{m}\left[\big(\bm{a}_{i}^{\mathrm{sgn}\top}\bm{x}\big)^{2}-\big(\bm{a}_{i}^{\mathrm{sgn}\top}\bm{x}^{\natural}\big)^{2}\right]^{2}.\label{eq:f-sgn-LOO}
\end{equation}
One simple yet important feature associated with these new design
vectors is that it produces the same measurements as $\left\{ \bm{a}_{i}\right\} $:
\begin{equation}
\big(\bm{a}_{i}^{\top}\bm{x}^{\natural}\big)^{2}=\big(\bm{a}_{i}^{\text{sgn}\top}\bm{x}^{\natural}\big)^{2}=\left|a_{i,1}\right|^{2},\qquad1\leq i\leq m.\label{eq:by-construction}
\end{equation}
See Figure~\ref{fig:Illustration-of-leave-one-out}(b) for an illustration
and Algorithm~\ref{alg:wf-LOO-sign} for the detailed procedure.
This sequence is introduced in order to ``randomize'' certain Gaussian
polynomials (e.g.~$I_{4}$ in (\ref{eq:r1-I1-4-defn})), which in
turn enables optimal control of these quantities. This is particularly
crucial at the initial stage of the algorithm.
\item \emph{Leave-one-out and random-sign sequences} $\left\{ \bm{x}^{t,\mathrm{sgn},(l)}\right\} _{t\geq0}$.
Furthermore, we also need to introduce another collection of sequences
$\{\bm{x}^{t,\mathrm{sgn},(l)}\}$ by simultaneously employing the
new design vectors $\{\bm{a}_{i}^{\text{sgn}}\}$ and discarding a
single sample $(\bm{a}^{\text{sgn}}_{l},y^{\text{sgn}}_{l})$. This enables us to propagate the
kinds of independence properties across the above two sets of sequences,
which is useful in demonstrating that $\bm{x}^{t}$ is jointly ``nearly-independent''
of both $\bm{a}_{l}$ and $\{\mathrm{sgn}(a_{i,1})\}$. See Algorithm~\ref{alg:wf-LOO-sign-l}
and Figure~\ref{fig:Illustration-of-leave-one-out}(c).
\end{itemize}
As a remark, all of these leave-one-out and random-sign procedures
are assumed to start from the same initial point as the original sequence,
namely,
\begin{equation}
\bm{x}^{0}=\bm{x}^{0,\left(l\right)}=\bm{x}^{0,\text{sgn}}=\bm{x}^{0,\text{sgn},\left(l\right)},\qquad1\leq l\leq m.\label{eq:same-initial}
\end{equation}

\subsection{Justification of approximate state evolution for Phase I of Stage 1}

Recall that Phase I consists of the iterations $0\leq t\leq T_{0}$,
where
\begin{equation}
T_{0}=\min\left\{ t: \alpha_{t+1}\geq\frac{c_{6}}{\log^{5}m}\right\} .\label{eq:defn-T0-analysis}
\end{equation}
Our goal here is to show that the approximate state evolution (\ref{subeq:state-evolution})
for both the size $\alpha_{t}$ of the signal component and the size
$\beta_{t}$ of the orthogonal component holds true throughout
Phase I. Our proof will be inductive in nature. Specifically,
we will first identify a set of induction hypotheses that are helpful
in proving the validity of the approximate state evolution (\ref{subeq:state-evolution}),
and then proceed by establishing these hypotheses via induction.

\subsubsection{Induction hypotheses}

For the sake of clarity, we first list all the induction hypotheses.
\begin{subequations}\label{subeq:induction}
\begin{align}
\max_{1\leq l\leq m}\big\|\bm{x}^{t}-\bm{x}^{t,\left(l\right)}\big\|_{2} & \leq\beta_{t}\left(1+\frac{1}{\log m}\right)^{t}C_{1}\frac{\sqrt{n\log^{5}m}}{m},\label{eq:induction-xt-l}\\
\max_{1\leq l\leq m}\left|x_{\parallel}^{t}-x_{\parallel}^{t,\left(l\right)}\right| & \leq\alpha_{t}\left(1+\frac{1}{\log m}\right)^{t}C_{2}\frac{\sqrt{n\log^{12}m}}{m},\label{eq:induction-xt-l-signal}\\
\left\Vert \bm{x}^{t}-\bm{x}^{t,\text{sgn}}\right\Vert _{2} & \leq\alpha_{t}\left(1+\frac{1}{\log m}\right)^{t}C_{3}\sqrt{\frac{n\log^{5}m}{m}},\label{eq:induction-xt-sgn}\\
\max_{1\leq l\leq m}\left\Vert \bm{x}^{t}-\bm{x}^{t,\text{sgn}}-\bm{x}^{t,\left(l\right)}+\bm{x}^{t,\text{sgn},\left(l\right)}\right\Vert _{2} & \leq\alpha_{t}\left(1+\frac{1}{\log m}\right)^{t}C_{4}\frac{\sqrt{n\log^{9}m}}{m},\label{eq:induction-double}\\
c_{5}\leq\left\Vert \bm{x}_{\perp}^{t}\right\Vert _{2} & \leq\left\Vert \bm{x}^{t}\right\Vert _{2}\leq C_{5},\label{eq:induction-norm-size}\\
\left\Vert \bm{x}^{t}\right\Vert _{2} & \leq4\alpha_{t}\sqrt{n\log m},\label{eq:induction-norm-relative}
\end{align}
where $C_{1},\cdots,C_{5}$ and $c_{5}$ are some absolute positive
constants.\end{subequations}

Now we are ready to prove an immediate consequence of the induction
hypotheses (\ref{subeq:induction}): if (\ref{subeq:induction}) hold
for the $t^{\text{th}}$ iteration, then $\alpha_{t+1}$ and $\beta_{t+1}$
follow the approximate state evolution (see (\ref{subeq:state-evolution})).
This is justified in the following lemma.

\begin{lemma}\label{lemma:xt-signal}Suppose $m\geq Cn\log^{11}m$
for some sufficiently large constant $C>0$. For any $0\leq t\leq T_{0}$
(cf.~\eqref{eq:defn-T0-analysis}), if the $t^{\mathrm{th}}$ iterates satisfy
the induction hypotheses \eqref{subeq:induction}, then with probability
at least $1-O(me^{-1.5n})-O(m^{-10})$,\begin{subequations}
\begin{align}
\alpha_{t+1} & =\left\{ 1+3\eta\left[1-\left(\alpha_{t}^{2}+\beta_{t}^{2}\right)\right]+\eta\zeta_{t}\right\} \alpha_{t};\label{eq:approximate_state_evolution_phase_1_alpha}\\
\beta_{t+1} & =\left\{ 1+\eta\left[1-3\left(\alpha_{t}^{2}+\beta_{t}^{2}\right)\right]+\eta\rho_{t}\right\} \beta_{t}\label{eq:approximate_state_evolution_phase_1_beta}
\end{align}
\end{subequations}hold for some $\left|\zeta_{t}\right|\ll1/\log m$
and $\left|\rho_{t}\right|\ll1/\log m$. \end{lemma}\begin{proof}See
Appendix \ref{sec:Proof-of-Lemma-xt-signal}.\end{proof}

It remains to inductively show that the hypotheses hold for all $0\leq t\leq T_{0}$.
Before proceeding to this induction step, it is helpful to first develop
more understanding about the preceding hypotheses.
\begin{enumerate}
\item In words, (\ref{eq:induction-xt-l}), (\ref{eq:induction-xt-l-signal}),
(\ref{eq:induction-xt-sgn}) specify that the leave-one-out sequences
$\left\{ \bm{x}^{t,\left(l\right)}\right\} $ and $\left\{ \bm{x}^{t,\text{sgn}}\right\} $
are exceedingly close to the original sequence $\left\{ \bm{x}^{t}\right\} $.
Similarly, the difference between $\bm{x}^{t}-\bm{x}^{t,\text{sgn}}$
and $\bm{x}^{t,\left(l\right)}-\bm{x}^{t,\text{sgn},\left(l\right)}$
is extremely small, as asserted in (\ref{eq:induction-double}). The
hypothesis (\ref{eq:induction-norm-size}) says that the norm of the
iterates $\left\{ \bm{x}^{t}\right\} $ is always bounded from above
and from below in Phase I. The last one (\ref{eq:induction-norm-relative})
indicates that the size $\alpha_{t}$ of the signal component is never
too small compared with $\|\bm{x}^{t}\|_{2}$.
\item Another property that is worth mentioning is the growth rate (with
respect to $t$) of the quantities appeared in the induction hypotheses
(\ref{subeq:induction}). For instance, $\big|x_{\parallel}^{t}-x_{\parallel}^{t,\left(l\right)}\big|,\|\bm{x}^{t}-\bm{x}^{t,\text{sgn}}\|_{2}$
and $\|\bm{x}^{t}-\bm{x}^{t,\text{sgn}}-\bm{x}^{t,\left(l\right)}+\bm{x}^{t,\text{sgn},\left(l\right)}\|_{2}$
grow more or less at the same rate as $\alpha_{t}$ (modulo some $(1+1/\log m)^{T_{0}}$ factor). In
contrast, $\|\bm{x}^{t}-\bm{x}^{t,\left(l\right)}\|_{2}$ shares the
same growth rate with $\beta_{t}$ (modulo the $(1+1/\log m)^{T_{0}}$ factor). See Figure~\ref{fig:loo-trend}
for an illustration. The difference in the growth rates turns out
to be crucial in establishing the advertised result.
\item Last but not least, we emphasize the sizes of the quantities
of interest in (\ref{subeq:induction}) for $t=1$ under the Gaussian
initialization. Ignoring all of the $\log m$ terms and recognizing that
$\alpha_{1}\asymp1/\sqrt{n}$ and $\beta_{1}\asymp1$, one sees that
$\|\bm{x}^{1}-\bm{x}^{1,\left(l\right)}\|_{2}\lesssim1/\sqrt{m}$,
$|x_{\parallel}^{1}-x_{\parallel}^{1,\left(l\right)}|\lesssim1/m$,
$\|\bm{x}^{1}-\bm{x}^{1,\text{sgn}}\|_{2}\lesssim1/\sqrt{m}$ and
$\|\bm{x}^{1}-\bm{x}^{1,\text{sgn}}-\bm{x}^{1,(l)}+\bm{x}^{1,\text{sgn},(l)}\|_{2}\lesssim1/m$.
See Figure~\ref{fig:loo-trend} for an illustration of the trends of
the above four quantities.
\end{enumerate}
\begin{figure}
\centering

\begin{tabular}{cc}
\includegraphics[width=0.4\textwidth]{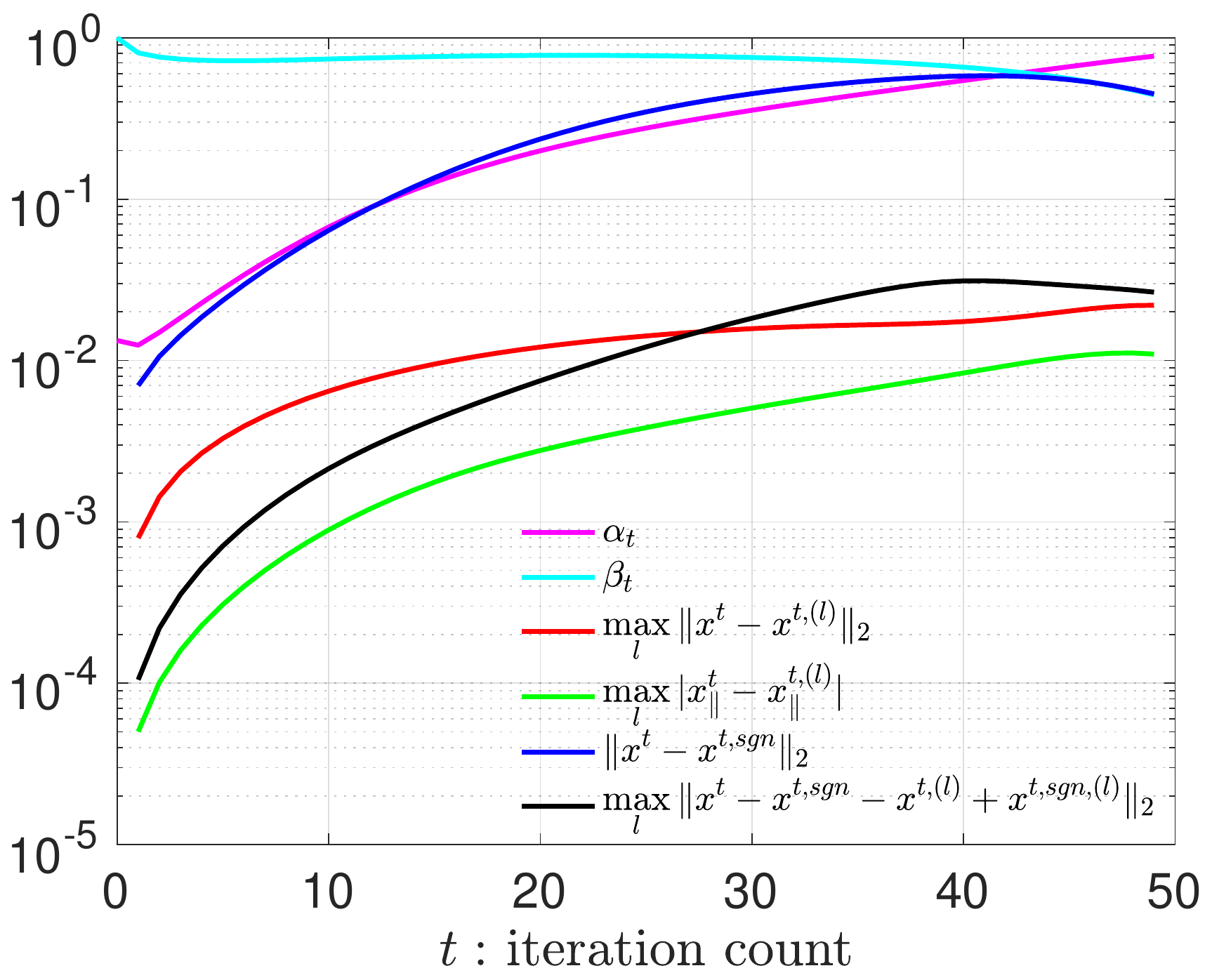} $\quad$ & $\quad$\includegraphics[width=0.4\textwidth]{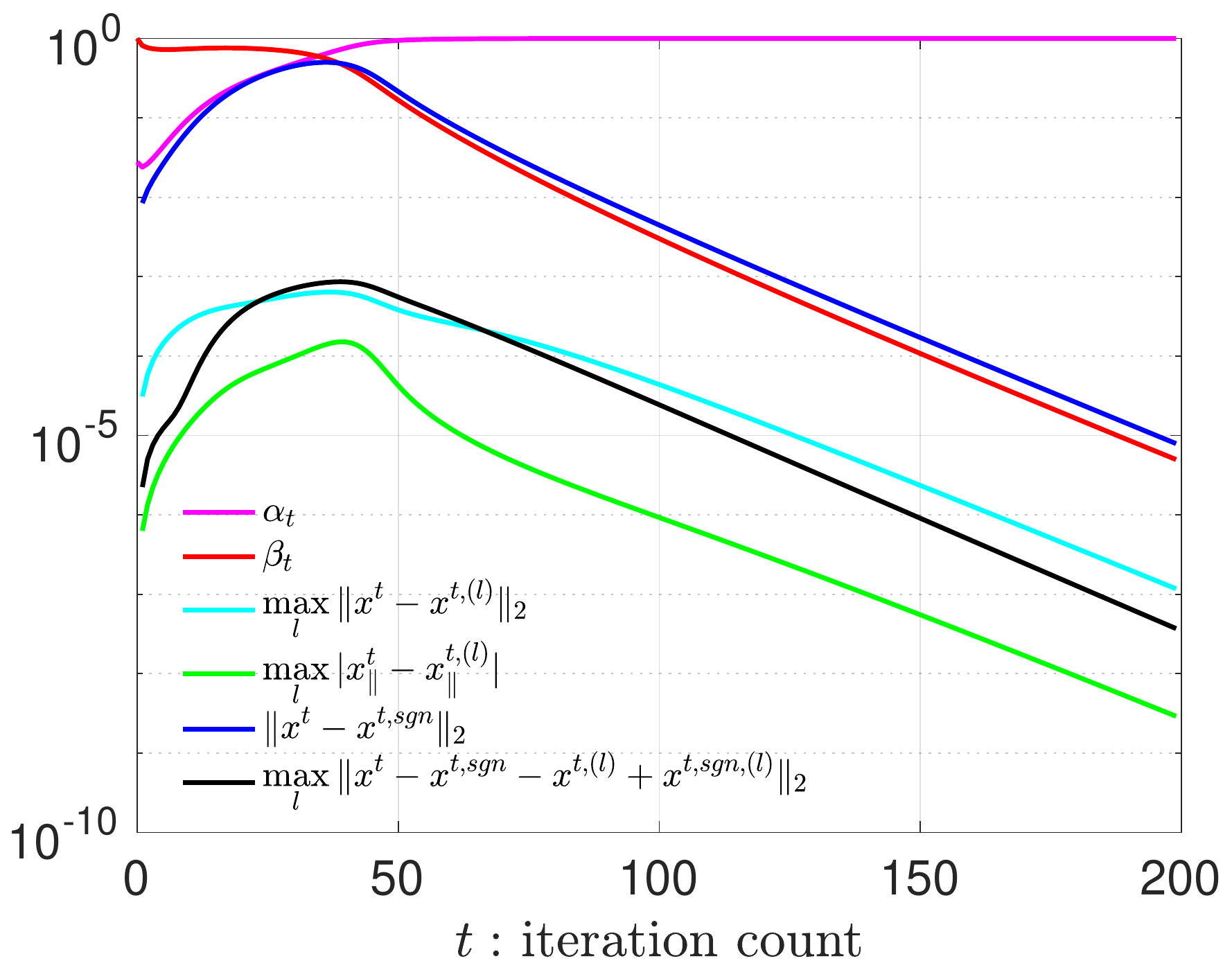}\tabularnewline
(a) Stage 1  & (b) Stage 1 and Stage 2\tabularnewline
\end{tabular}\caption{Illustration of the differences among leave-one-out and original sequences
	vs.~iteration count, plotted semilogarithmically. The results are shown for $n=1000$ with $m=10n$,
$\eta_{t}\equiv0.1$, and $\|\bm{x}^{\natural}\|_{2}=1$. (a) The four differences increases in Stage 1. From the induction hypotheses (\ref{subeq:induction}), our upper bounds on $|x^t_{\parallel}- x^{t,(l)}_{\parallel}|$, $\Vert \bm{x}^{t}-\bm{x}^{t,\text{sgn}}\Vert _{2}$ and $\Vert \bm{x}^{t}-\bm{x}^{t,\text{sgn}}-\bm{x}^{t,\left(l\right)}+\bm{x}^{t,\text{sgn},\left(l\right)}\Vert _{2}$ scale linearly with $\alpha_{t}$, whereas the upper bound on $\|\bm{x}^t -\bm{x}^{t,(l)}\|_2$ is proportional to $\beta_{t}$. In addition, $\|\bm{x}^{1}-\bm{x}^{1,\left(l\right)}\|_{2}\lesssim1/\sqrt{m}$,
$|x_{\parallel}^{1}-x_{\parallel}^{1,\left(l\right)}|\lesssim1/m$,
$\|\bm{x}^{1}-\bm{x}^{1,\text{sgn}}\|_{2}\lesssim1/\sqrt{m}$ and
$\|\bm{x}^{1}-\bm{x}^{1,\text{sgn}}-\bm{x}^{1,(l)}+\bm{x}^{1,\text{sgn},(l)}\|_{2}\lesssim1/m$. (b) The four differences converge to zero geometrically fast in Stage 2, as all the (variants of) leave-one-out sequences and the original sequence converge to the truth $\bm{x}^{\natural}$. \label{fig:loo-trend}
}
\end{figure}
%
Several consequences of (\ref{subeq:induction}) regarding the incoherence
between $\{\bm{x}^{t}\}$, $\{\bm{x}^{t,\text{sgn}}\}$ and $\{\bm{a}_{i}\}$,
$\{\bm{a}_{i}^{\text{sgn}}\}$ are immediate, as summarized in the
following lemma.

\begin{lemma}\label{lemma:consequence-main-text}
Suppose that $m\geq Cn\log^{6}m$
for some sufficiently large constant $C>0$ and the $t^{\mathrm{th}}$ iterates
satisfy the induction hypotheses (\ref{subeq:induction}) for $t\leq T_{0}$,
then with probability at least $1-O(me^{-1.5n})-O(m^{-10})$,\begin{subequations}
\begin{align*}
\max_{1\leq l\leq m}\left|\bm{a}_{l}^{\top}\bm{x}^{t}\right| & \lesssim\sqrt{\log m}\big\|\bm{x}^{t}\big\|_{2};\\
\max_{1\leq l\leq m}\big|\bm{a}_{l,\perp}^{\top}\bm{x}_{\perp}^{t}\big| & \lesssim\sqrt{\log m}\big\|\bm{x}_{\perp}^{t}\big\|_{2};\\
\max_{1\leq l\leq m}\big|\bm{a}_{l}^{\top}\bm{x}^{t,\mathrm{sgn}}\big| & \lesssim\sqrt{\log m}\big\|\bm{x}^{t,\mathrm{sgn}}\big\|_{2};\\
\max_{1\leq l\leq m}\big|\bm{a}_{l,\perp}^{\top}\bm{x}_{\perp}^{t,\mathrm{sgn}}\big| & \lesssim\sqrt{\log m}\big\|\bm{x}_{\perp}^{t,\mathrm{sgn}}\big\|_{2};\\
\max_{1\leq l\leq m}\big|\bm{a}_{l}^{\mathrm{sgn}\top}\bm{x}^{t,\mathrm{sgn}}\big| & \lesssim\sqrt{\log m}\big\|\bm{x}^{t,\mathrm{sgn}}\big\|_{2}.
\end{align*}
\end{subequations}
\end{lemma}
\begin{proof} These incoherence conditions typically arise from the independence between $\{\bm{x}^{t,(l)}\}$ and $\bm{a}_l$. For instance, the first line follows since
	\[
		\big| \bm{a}_l^{\top} \bm{x}^t \big| \approx  \big| \bm{a}_l^{\top} \bm{x}^{t,(l)} \big| \lesssim \sqrt{\log m} \| \bm{x}^{t,(l)}\|_2 \asymp \sqrt{\log m} \| \bm{x}^{t}\|_2.
	\]
See Appendix \ref{sec:Proof-of-Lemma-consequence} for detailed proofs.
\end{proof}

\subsubsection{Induction step}

We then turn to showing that the induction hypotheses (\ref{subeq:induction})
hold throughout Phase I, i.e.~for $0\leq t\leq T_{0}$. The base
case can be easily verified because of the identical initial points
(\ref{eq:same-initial}). Now we move on to the inductive step, i.e.~we
aim to show that if the hypotheses (\ref{subeq:induction}) are valid
up to the $t^{\mathrm{th}}$ iteration for some $t\leq T_{0}$, then they continue
to hold for the $\left(t+1\right)^{\mathrm{th}}$ iteration.

The first lemma concerns the difference between the leave-one-out
sequence $\bm{x}^{t+1,\left(l\right)}$ and the true sequence $\bm{x}^{t+1}$
(see (\ref{eq:induction-xt-l})).

\begin{lemma}\label{lemma:xt-xt-l}Suppose $m\geq Cn\log^{5}m$ for
some sufficiently large constant $C>0$. If the induction hypotheses
(\ref{subeq:induction}) hold true up to the $t^{\mathrm{th}}$ iteration for
some $t\leq T_{0}$, then with probability at least $1-O(me^{-1.5n})-O(m^{-10})$,
\begin{equation}
\max_{1\leq l\leq m}\big\|\bm{x}^{t+1}-\bm{x}^{t+1,\left(l\right)}\big\|_{2}\leq\beta_{t+1}\left(1+\frac{1}{\log m}\right)^{t+1}C_{1}\frac{\sqrt{n\log^{5}m}}{m}
\end{equation}
holds as long as $\eta>0$ is a sufficiently small constant and $C_{1}>0$
is sufficiently large. \end{lemma}\begin{proof}See Appendix \ref{sec:Proof-of-Lemma-xt-xt-l}.\end{proof}

The next lemma characterizes a finer relation between $\bm{x}^{t+1}$
and $\bm{x}^{t+1,\left(l\right)}$ when projected onto the signal
direction (cf.~(\ref{eq:induction-xt-l-signal})).

\begin{lemma}\label{lemma:xt-xt-l-signal}Suppose $m\geq Cn\log^{6}m$
for some sufficiently large constant $C>0$. If the induction hypotheses
(\ref{subeq:induction}) hold true up to the $t^{\mathrm{th}}$ iteration for
some $t\leq T_{0}$, then with probability at least $1-O(me^{-1.5n})-O(m^{-10})$,
\begin{equation}
\max_{1\leq l\leq m}\big|x_{\parallel}^{t+1}-x_{\parallel}^{t+1,\left(l\right)}\big|\leq\alpha_{t+1}\left(1+\frac{1}{\log m}\right)^{t+1}C_{2}\frac{\sqrt{n\log^{12}m}}{m}
\end{equation}
holds as long as $\eta>0$ is a sufficiently small constant and $C_{2}\gg C_{4}$.
\end{lemma}\begin{proof}See Appendix \ref{sec:Proof-of-Lemma-xt-xt-l-signal}.\end{proof}

Regarding the difference between $\bm{x}^{t}$ and $\bm{x}^{t,\text{sgn}}$
(see (\ref{eq:induction-xt-sgn})), we have the following result.

\begin{lemma}\label{lemma:xt-xt-sgn}Suppose $m\geq Cn\log^{5}m$
for some sufficiently large constant $C>0$. If the induction hypotheses
(\ref{subeq:induction}) hold true up to the $t^{\mathrm{th}}$ iteration for
some $t\leq T_{0}$, then with probability at least $1-O(me^{-1.5n})-O\left(m^{-10}\right)$,
\begin{equation}
\left\Vert \bm{x}^{t+1}-\bm{x}^{t+1,\mathrm{sgn}}\right\Vert _{2}\leq\alpha_{t+1}\left(1+\frac{1}{\log m}\right)^{t+1}C_{3}\sqrt{\frac{n\log^{5}m}{m}}
\end{equation}
holds as long as $\eta>0$ is a sufficiently small constant and $C_{3}$
is a sufficiently large positive constant. \end{lemma}\begin{proof}See
Appendix \ref{sec:Proof-of-Lemma-xt-xt-sgn}.\end{proof}

We are left with the double difference $\bm{x}^{t+1}-\bm{x}^{t+1,\mathrm{sgn}}-\bm{x}^{t+1,\left(l\right)}+\bm{x}^{t+1,\mathrm{sgn},\left(l\right)}$
(cf.~(\ref{eq:induction-double})), for which one has the following
lemma.

\begin{lemma}\label{lemma:double-diff}Suppose $m\geq Cn\log^{8}m$
for some sufficiently large constant $C>0$. If the induction hypotheses
(\ref{subeq:induction}) hold true up to the $t^{\mathrm{th}}$ iteration for
some $t\leq T_{0}$, then with probability at least $1-O(me^{-1.5n})-O(m^{-10})$,
\begin{equation}
\max_{1\leq l\leq m}\left\Vert \bm{x}^{t+1}-\bm{x}^{t+1,\mathrm{sgn}}-\bm{x}^{t+1,\left(l\right)}+\bm{x}^{t+1,\mathrm{sgn},\left(l\right)}\right\Vert _{2}\leq\alpha_{t+1}\left(1+\frac{1}{\log m}\right)^{t+1}C_{4}\frac{\sqrt{n\log^{9}m}}{m}
\end{equation}
holds as long as $\eta>0$ is a sufficiently small constant and $C_{4}>0$
is sufficiently large. \end{lemma}\begin{proof}See Appendix \ref{sec:Proof-of-Lemma-double-diff}.\end{proof}

Assuming the induction hypotheses (\ref{subeq:induction}) hold up
to the $t^{\mathrm{th}}$ iteration for some $t\leq T_{0}$, we know from Lemma
\ref{lemma:xt-signal} that the approximate state evolution for both
$\alpha_{t}$ and $\beta_{t}$ (see (\ref{subeq:state-evolution}))
holds up to $t+1$. As a result, the last two hypotheses (\ref{eq:induction-norm-size})
and (\ref{eq:induction-norm-relative}) for the $\left(t+1\right)^{\mathrm{th}}$
iteration can be easily verified.

\subsection{Justification of approximate state evolution for Phase II of Stage 1}

Recall from Lemma \ref{lemma:iterative} that Phase II refers to the
iterations $T_{0}<t\leq T_{\gamma}$ (see the definition of $T_{0}$
in Lemma~\ref{lemma:iterative}), for which one has
\begin{equation}
\alpha_{t}\geq\frac{c_{6}}{\log^{5}m}
\end{equation}
as long as the approximate state evolution (\ref{subeq:state-evolution})
holds. Here $c_{6}>0$ is the same constant as in Lemma \ref{lemma:iterative}.
Similar to Phase I, we invoke an inductive argument to prove that
the approximate state evolution (\ref{subeq:state-evolution}) continues
to hold for $T_{0}<t\leq T_{\gamma}$.

\subsubsection{Induction hypotheses}

In Phase I, we rely on the leave-one-out sequences and the random-sign
sequences $\{\bm{x}^{t,(l)}\},\{\bm{x}^{t,\text{sgn}}\}$ and $\{\bm{x}^{t,\text{sgn},(l)}\}$
to establish certain ``near-independence'' between $\{\bm{x}^{t}\}$
	and $\{\bm{a}_{l}\}$, which in turn allows us to obtain
sharp control of the residual terms $\bm{r}\left(\bm{x}^{t}\right)$
(cf.~(\ref{eq:grad-difference-r})) and $r_{1}\left(\bm{x}^{t}\right)$
(cf.~(\ref{eq:defn-r1})). As it turns out, once the size $\alpha_{t}$
of the signal component obeys $\alpha_{t}\gtrsim1/\text{poly}\log(m)$,
then $\{\bm{x}^{t,\left(l\right)}\}$ alone is sufficient for our
purpose to establish the ``near-independence'' property. More precisely,
in Phase II we only need to impose the following induction hypotheses.\begin{subequations}\label{subeq:induction-phase-2}
\begin{align}
\max_{1\leq l\leq m}\big\|\bm{x}^{t}-\bm{x}^{t,\left(l\right)}\big\|_{2} & \leq\alpha_{t}\left(1+\frac{1}{\log m}\right)^{t}C_{6}\frac{\sqrt{n\log^{15}m}}{m};\label{eq:induction-xt-xt-l-phase-2}\\
c_{5}\leq\left\Vert \bm{x}_{\perp}^{t}\right\Vert _{2} & \leq\left\Vert \bm{x}^{t}\right\Vert _{2}\leq C_{5}.\label{eq:induction-norm-phase-2}
\end{align}
\end{subequations}

A direct consequence of (\ref{subeq:induction-phase-2}) is the incoherence
between $\bm{x}^{t}$ and $\{\bm{a}_{l}\}$, namely, \begin{subequations}
\begin{align}
\max_{1\leq l\leq m}\left|\bm{a}_{l,\perp}^{\top}\bm{x}_{\perp}^{t}\right| & \lesssim\sqrt{\log m}\left\Vert \bm{x}_{\perp}^{t}\right\Vert _{2};\label{eq:incoherence-phase-2}\\
\max_{1\leq l\leq m}\left|\bm{a}_{l}^{\top}\bm{x}^{t}\right| & \lesssim\sqrt{\log m}\left\Vert \bm{x}^{t}\right\Vert _{2}.\label{eq:incoherence-full-phase-2}
\end{align}
\end{subequations}To see this, one can use the triangle inequality
to show that
\begin{align*}
\left|\bm{a}_{l,\perp}^{\top}\bm{x}_{\perp}^{t}\right| & \leq\left|\bm{a}_{l,\perp}^{\top}\bm{x}_{\perp}^{t,\left(l\right)}\right|+\left|\bm{a}_{l,\perp}^{\top}\big(\bm{x}_{\perp}^{t}-\bm{x}_{\perp}^{t,\left(l\right)}\big)\right|\\
 & \overset{\left(\text{i}\right)}{\lesssim}\sqrt{\log m}\big\|\bm{x}_{\perp}^{t,\left(l\right)}\big\|_{2}+\sqrt{n}\big\|\bm{x}^{t}-\bm{x}^{t,\left(l\right)}\big\|_{2}\\
 & \lesssim\sqrt{\log m}\left(\big\|\bm{x}_{\perp}^{t}\big\|_{2}+\big\|\bm{x}^{t}-\bm{x}^{t,\left(l\right)}\big\|_{2}\right)+\sqrt{n}\big\|\bm{x}^{t}-\bm{x}^{t,\left(l\right)}\big\|_{2}\\
 & \overset{\left(\text{ii}\right)}{\lesssim}\sqrt{\log m}+\frac{\sqrt{n\log^{15}m}}{m}\sqrt{n}\lesssim\sqrt{\log m},
\end{align*}
where (i) follows from the independence between $\bm{a}_{l}$ and
$\bm{x}^{t,(l)}$ and the Cauchy-Schwarz inequality, and the last
line (ii) arises from $\left(1+1/\log m\right)^{t}\lesssim1$ for
$t\leq T_{\gamma}\lesssim\log n$ and $m\gg n\log^{15/2}m$. This
combined with the fact that $\left\Vert \bm{x}_{\perp}^{t}\right\Vert _{2}\geq c_{5}/2$
results in
\begin{equation}
\max_{1\leq l\leq m}\left|\bm{a}_{l,\perp}^{\top}\bm{x}_{\perp}^{t}\right|\lesssim\sqrt{\log m}\left\Vert \bm{x}_{\perp}^{t}\right\Vert _{2}.\label{eq:incoherence-Phase2-perp}
\end{equation}
The condition (\ref{eq:incoherence-full-phase-2}) follows using nearly
identical arguments, which are omitted here.

As in Phase I, we need to justify the approximate state evolution
(\ref{subeq:state-evolution}) for both $\alpha_{t}$ and $\beta_{t}$,
given that the $t^{\mathrm{th}}$ iterates satisfy the induction hypotheses (\ref{subeq:induction-phase-2}).
This is stated in the following lemma.

\begin{lemma}\label{lemma:xt-signal-phase-2}Suppose $m\geq Cn\log^{13}m$
for some sufficiently large constant $C>0$. If the $t^{\mathrm{th}}$ iterates
satisfy the induction hypotheses (\ref{subeq:induction-phase-2})
for $T_{0}<t<T_{\gamma}$ , then with probability at least $1-O(me^{-1.5n})-O(m^{-10})$,
\begin{subequations}
\begin{align}
\alpha_{t+1} & =\left\{ 1+3\eta\left[1-\left(\alpha_{t}^{2}+\beta_{t}^{2}\right)\right]+\eta\zeta_{t}\right\} \alpha_{t};\label{eq:approximate_state_evolution_phase_2_alpha}\\
\beta_{t+1} & =\left\{ 1+\eta\left[1-3\left(\alpha_{t}^{2}+\beta_{t}^{2}\right)\right]+\eta\rho_{t}\right\} \beta_{t},\label{eq:approximate_state_evolution_phase_2_beta}
\end{align}
\end{subequations}for some $\left|\zeta_{t}\right|\ll1/\log m$ and
$\rho_{t}\ll1/\log m$. \end{lemma}\begin{proof}See Appendix \ref{sec:Proof-of-Lemma-xt-signal-phase-2}
for the proof of (\ref{eq:approximate_state_evolution_phase_2_alpha}).
The proof of (\ref{eq:approximate_state_evolution_phase_2_beta})
follows exactly the same argument as in proving (\ref{eq:approximate_state_evolution_phase_1_beta}),
and is hence omitted. \end{proof}

\subsubsection{Induction step}

We proceed to complete the induction argument. Towards this end, one
has the following lemma in regard to the induction on $\max_{1\leq l\leq m}\|\bm{x}^{t+1}-\bm{x}^{t+1,\left(l\right)}\|_{2}$
(see (\ref{eq:induction-xt-xt-l-phase-2})).

\begin{lemma}\label{lemma:xt-xt-l-phase-2}Suppose $m\geq Cn\log^{5}m$
for some sufficiently large constant $C>0$, and consider any $T_{0}<t<T_{\gamma}$.
If the induction hypotheses (\ref{subeq:induction}) are valid throughout Phase I and (\ref{subeq:induction-phase-2}) are valid
from the $T_{0}$th to the $t^{\mathrm{th}}$ iterations, then with probability
at least $1-O(me^{-1.5n})-O(m^{-10})$,
\[
\max_{1\leq l\leq m}\big\|\bm{x}^{t+1}-\bm{x}^{t+1,\left(l\right)}\big\|_{2}\leq\alpha_{t+1}\left(1+\frac{1}{\log m}\right)^{t+1}C_{6}\frac{\sqrt{n\log^{13}m}}{m}
\]
holds as long as $\eta>0$ is sufficiently small and $C_{6}>0$ is
sufficiently large. \end{lemma}\begin{proof}See Appendix \ref{sec:Proof-of-Lemma-xt-xt-l-phase-2}.\end{proof}

As in Phase I, since we assume the induction hypotheses (\ref{subeq:induction})
(resp.~(\ref{subeq:induction-phase-2})) hold for all iterations
up to the $T_{0}$th iteration (resp.~between the $T_{0}$th and
the $t^{\mathrm{th}}$ iteration), we know from Lemma \ref{lemma:xt-signal-phase-2}
that the approximate state evolution for both $\alpha_{t}$ and $\beta_{t}$
(see (\ref{subeq:state-evolution})) holds up to $t+1$. The last
induction hypothesis (\ref{eq:induction-norm-phase-2}) for the $\left(t+1\right)^{\mathrm{th}}$
iteration can be easily verified from Lemma \ref{lemma:iterative}.

It remains to check the case when $t=T_{0}+1$. It can be seen from
the analysis in Phase I that
\begin{align*}
\max_{1\leq l\leq m}\big\|\bm{x}^{T_{0}+1}-\bm{x}^{T_{0}+1,\left(l\right)}\big\|_{2} & \leq\beta_{T_{0}+1}\left(1+\frac{1}{\log m}\right)^{T_{0}+1}C_{1}\frac{\sqrt{n\log^{5}m}}{m}\\
 & \leq\alpha_{T_{0}+1}\left(1+\frac{1}{\log m}\right)^{T_{0}+1}C_{6}\frac{\sqrt{n\log^{15}m}}{m},
\end{align*}
for some constant condition $C_{6}\gg1$, where the second line holds
since $\beta_{T_{0}+1}\leq C_{5}$, $\alpha_{T_{0}+1}\geq c_{6}/\log^{5}m$.

\subsection{Analysis for Stage 2}

Combining the analyses in \emph{Phase I} and \emph{Phase II}, we finish
the proof of Theorem \ref{thm:main} for Stage 1, i.e.~$t\leq T_{\gamma}$.
In addition to $\text{dist}\left(\bm{x}^{T_{\gamma}},\bm{x}^{\natural}\right)\leq\gamma$,
we can also see from (\ref{eq:incoherence-full-phase-2}) that
\[
\max_{1\leq i\leq m}\left|\bm{a}_{i}^{\top}\bm{x}^{T_{\gamma}}\right|\lesssim\sqrt{\log m},
\]
which in turn implies that
\[
\max_{1\leq i\leq m}\left|\bm{a}_{i}^{\top}\left(\bm{x}^{T_{\gamma}}-\bm{x}^{\natural}\right)\right|\lesssim\sqrt{\log m}.
\]
Armed with these properties, one can apply the arguments in \cite[Section 6]{ma2017implicit}
to prove that for $t\geq T_{\gamma}+1$,
\begin{equation}
\text{dist}\left(\bm{x}^{t},\bm{x}^{\natural}\right)\leq\left(1-\frac{\eta}{2}\right)^{t-T_{\gamma}}\text{dist}\left(\bm{x}^{T_{\gamma}},\bm{x}^{\natural}\right)\leq\left(1-\frac{\eta}{2}\right)^{t-T_{\gamma}}\cdot\gamma.\label{eq:state-2-dist}
\end{equation}
 Notably, the theorem therein \cite[Theorem 1]{ma2017implicit} works under the stepsize $\eta_{t}\equiv\eta\asymp c/\log n$
when $m\gg n\log n$. Nevertheless, as remarked by the authors, when
the sample complexity exceeds $m\gg n\log^{3}m$, a constant stepsize
is allowed.

We are left with proving (\ref{eq:SNR-lower-bound}) for Stage 2.
Note that we have already shown that the ratio $\alpha_{t}/\beta_{t}$
increases exponentially fast in Stage 1. Therefore,
\[
\frac{\alpha_{T_{1}}}{\beta_{T_{1}}}\geq\frac{1}{\sqrt{2n\log n}}(1+c_{10}\eta^2)^{T_{1}}
\]
and, by the definition of $T_{1}$ (see \eqref{eq:defn-T1}) and Lemma
\ref{lemma:iterative}, one has $\alpha_{T_{1}}\asymp\beta_{T_{1}}\asymp1$
and hence
\begin{equation}
\frac{\alpha_{T_{1}}}{\beta_{T_{1}}}\asymp1.\label{eq:SNR-T1}
\end{equation}
When it comes to $t>T_{\gamma}$, in view of (\ref{eq:state-2-dist}),
one has
\begin{align*}
\frac{\alpha_{t}}{\beta_{t}} & \geq \frac{1- \text{dist}\left(\bm{x}^{t},\bm{x}^{\natural}\right)}{\text{dist}\left(\bm{x}^{t},\bm{x}^{\natural}\right)} \geq \frac{1-\gamma}{\left(1-\frac{\eta}{2}\right)^{t-T_{\gamma}}\cdot\gamma} \\
 & \geq\frac{1-\gamma}{\gamma}\left(1+\frac{\eta}{2}\right)^{t-T_{\gamma}}\overset{(\text{i})}{\asymp}\frac{\alpha_{T_{1}}}{\beta_{T_{1}}}\left(1+\frac{\eta}{2}\right)^{t-T_{\gamma}}\\
 & \gtrsim\frac{1}{\sqrt{n\log n}}\left(1+c_{10}\eta^2\right)^{T_{1}}\left(1+\frac{\eta}{2}\right)^{t-T_{\gamma}}\\
 & \overset{(\text{ii})}{\asymp}\frac{1}{\sqrt{n\log n}}\left(1+c_{10}\eta^2\right)^{T_{\gamma}}\left(1+\frac{\eta}{2}\right)^{t-T_{\gamma}}\\
 & \gtrsim\frac{1}{\sqrt{n\log n}}\left(1+c_{10}\eta^2\right)^{t},
\end{align*}
where (i) arises from (\ref{eq:SNR-T1}) and the fact that $\gamma$
is a constant, (ii) follows since $T_{\gamma}-T_{1}\asymp1$ according
to Lemma \ref{lemma:iterative}, and the last line holds as long as
$c_{10}>0$ and $\eta$ are sufficiently small. This concludes the proof regarding
the lower bound on $\alpha_{t}/\beta_{t}$.

\section{Discussions\label{sec:Discussion}}

The current paper justifies the fast global convergence of gradient
descent with random initialization for phase retrieval. Specifically,
we demonstrate that GD with random initialization takes only $O\big(\log n+\log(1/\epsilon)\big)$
iterations to achieve a relative $\epsilon$-accuracy in terms of
the estimation error. It is likely that such fast global convergence
properties also arise in other nonconvex statistical estimation problems.
The technical tools developed herein may also prove useful for other
settings. We conclude our paper with a few directions
worthy of future investigation.

\begin{itemize}
\item \emph{Sample complexity and phase transition. } We have proved in Theorem
\ref{thm:main} that GD with random initialization enjoys fast convergence, 
with the proviso that $m\gg n\log^{13}m$. It is possible to improve
the sample complexity via more sophisticated arguments. In addition,
it would be interesting to examine the phase transition phenomenon
of GD with random initialization.

\item \emph{Other nonconvex statistical estimation problems. }We use the
phase retrieval problem to showcase the efficiency of GD with random
initialization. It is certainly interesting to investigate whether
this fast global convergence carries over to other nonconvex statistical
estimation problems including \emph{low-rank matrix and tensor recovery
}\cite{KesMonSew2010,sun2016guaranteed,chen2015fast,tu2016low,zheng2016convergence,zhao2015nonconvex-estimation,ma2017implicit,chen2017memory,chen2016projected, chen2018asymmetry, hao2018sparse}\emph{,
blind deconvolution }\cite{DBLP:journals/corr/LiLSW16,ma2017implicit,huang2017blind}\emph{
}and \emph{neural networks} \cite{soltanolkotabi2017theoretical,li2017algorithmic,fu2018local}.
The leave-one-out sequences and the ``near-independence'' property
introduced$\,$/$\,$identified in this paper might be useful in proving
efficiency of randomly initialized GD for the aforementioned problems.
%
\item \emph{Other iterative optimization methods. }Apart from gradient descent,
other iterative procedures have been applied to solve the phase retrieval
problem. Partial examples include \emph{alternating minimization},\emph{ Kaczmarz
algorithm,} and \emph{truncated gradient
descent (Truncated Wirtinger flow)}. In conjunction with random initialization,
whether the iterative algorithms mentioned above enjoy fast global
convergence is an interesting open problem. For example, it has been
shown that truncated WF together with truncated spectral initialization
achieves optimal sample complexity (i.e.~$m\asymp n$) and computational
complexity simultaneously \cite{ChenCandes15solving}. Does truncated
Wirtinger flow still enjoy optimal sample complexity when initialized randomly?

\item \emph{Beyond Gaussian sampling vectors. }In this work, we consider the Gaussian phase retrieval problem where the sampling vectors are i.i.d.~Gaussian vectors. We expect our results to generalize to other sampling vectors. Experimentally, we can verify that random initialization also converges fast under a Rademacher sampling model; see Figure \ref{fig:bernoulli}. 

\begin{figure}
\centering

\includegraphics[width=0.4\textwidth]{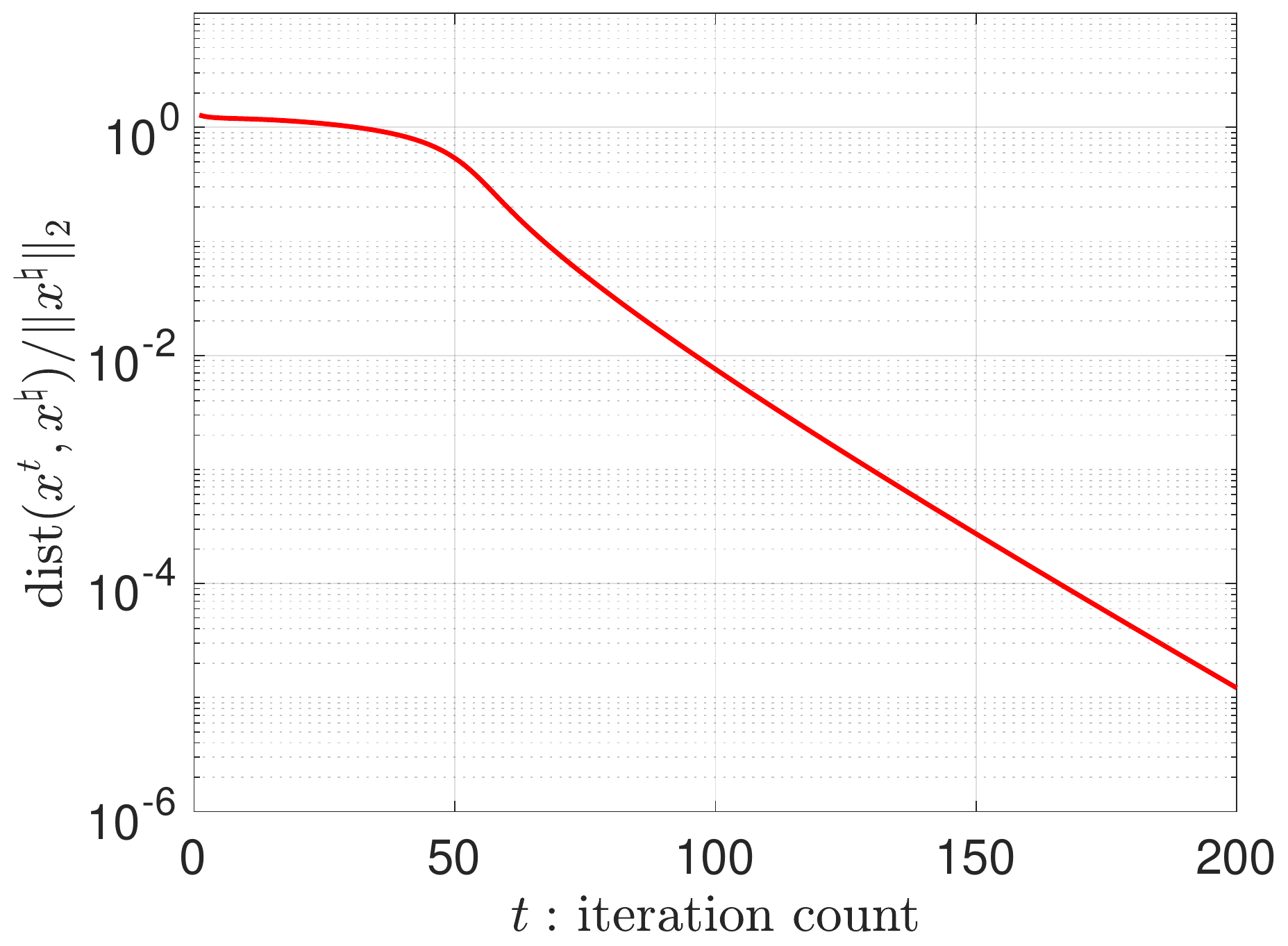}\caption{The relative $\ell_{2}$ error vs.~iteration count for GD with random
initialization, plotted semilogarithmically. The results are shown for $n=1000$
with $m=10n$ and $\eta_{t}\equiv0.1$. Here the entries of the sampling vectors $\bm{a}_i$ are drawn {\em i.i.d.} from a Rademacher distribution.\label{fig:bernoulli}}
\end{figure}

\item \emph{Applications of leave-one-out tricks. }In this paper, we heavily
deploy the \emph{leave-one-out }trick to demonstrate the ``near-independence''
between the iterates $\bm{x}^{t}$ and the sampling vectors $\left\{ \bm{a}_{i}\right\} $.
The basic idea is to construct an auxiliary sequence that is (i) independent
w.r.t.~certain components of the design vectors, and (ii) extremely
close to the original sequence. These two properties allow us to propagate
 the desired independence properties to $\bm{x}^{t}$. As mentioned
in Section \ref{sec:Related-work}, the leave-one-out trick has served
as a very powerful hammer for decoupling the dependency between random
vectors in several high-dimensional estimation problems. We expect
this powerful trick to be useful in broader settings.
\end{itemize}

\bibliographystyle{alphaabbr}

\section*{Acknowledgements}

Y.~Chen is supported in part by the AFOSR YIP award FA9550-19-1-0030, by the ARO  grant W911NF-18-1-0303, by the ONR grant N00014-19-1-2120,  and by the Princeton SEAS innovation award.  Y.~Chi is supported in part by AFOSR under the grant FA9550-15-1-0205, by ONR under the grant N00014-18-1-2142, by ARO under the grant W911NF-18-1-0303, and by NSF under the grants CAREER ECCS-1818571 and CCF-1806154.  J.~Fan is supported in part by the NSF grants DMS-1662139 and DMS-1712591, the ONR grant N00014-19-1-2120, and the NIH grant 2R01-GM072611-13.

\bibliography{../revised_manuscript/bibfileNonconvex}

\clearpage
\appendix

\section{Preliminaries }

We first gather two standard concentration inequalities
used throughout the appendix. The first lemma is the multiplicative
form of the Chernoff bound, while the second lemma is a user-friendly
version of the Bernstein inequality.

\begin{lemma}\label{lemma:chernoff}Suppose $X_{1},\cdots,X_{m}$
are independent random variables taking values in $\left\{ 0,1\right\} $.
Denote $X=\sum_{i=1}^{m}X_{i}$ and $\mu=\EE\left[X\right]$. Then
for any $\delta\geq1$, one has
\[
\PP\left(X\geq\left(1+\delta\right)\mu\right)\leq e^{-\delta\mu/3}.
\]
\end{lemma}

\begin{lemma}\label{lemma:bernstein}Consider $m$ independent random
variables $z_{l}\;\left(1\leq l\leq m\right)$, each satisfying $\left|z_{l}\right|\leq B$.
For any $a\geq2$, one has
\[
\left|\sum_{l=1}^{m}z_{l}-\sum_{l=1}^{m}\EE\left[z_{l}\right]\right|\leq\sqrt{2a\log m\sum_{l=1}^{m}\EE\left[z_{l}^{2}\right]}+\frac{2a}{3}B\log m
\]
with probability at least $1-2m^{-a}$. \end{lemma}

Next, we list a few simple facts. The gradient and the Hessian of
the nonconvex loss function (\ref{eq:PR-loss}) are given respectively
by
\begin{align}
\nabla f\left(\bm{x}\right) & =\frac{1}{m}\sum_{i=1}^{m}\left[\left(\bm{a}_{i}^{\top}\bm{x}\right)^{2}-\left(\bm{a}_{i}^{\top}\bm{x}^{\natural}\right)^{2}\right]\bm{a}_{i}\bm{a}_{i}^{\top}\bm{x};\label{eq:gradient}\\
\nabla^{2}f\left(\bm{x}\right) & =\frac{1}{m}\sum_{i=1}^{m}\left[3\left(\bm{a}_{i}^{\top}\bm{x}\right)^{2}-\left(\bm{a}_{i}^{\top}\bm{x}^{\natural}\right)^{2}\right]\bm{a}_{i}\bm{a}_{i}^{\top}.\label{eq:hessian}
\end{align}
In addition, recall that $\bm{x}^{\natural}$ is assumed to be $\bm{x}^{\natural}=\bm{e}_{1}$
throughout the proof. For each $1\leq i\leq m$, we have the decomposition
$\bm{a}_{i}=\left[\begin{array}{c}
a_{i,1}\\
\bm{a}_{i,\perp}
\end{array}\right]$, where $\bm{a}_{i,\perp}$ contains the $2$nd through the $n$th entries
of $\bm{a}_{i}$. The standard concentration inequality reveals that
\begin{equation}
\max_{1\leq i\leq m}\left|\bm{a}_{i}^{\top}\bm{x}^{\natural}\right| = \max_{1\leq i\leq m}\left|a_{i,1}\right| \leq5\sqrt{\log m}\label{eq:max-a-i-1}
\end{equation}
with probability $1-O\left(m^{-10}\right)$. Additionally, apply the
standard concentration inequality to see that
\begin{equation}
\max_{1\leq i\leq m}\left\Vert \bm{a}_{i}\right\Vert _{2}\leq\sqrt{6n}\label{eq:max-a-i-norm}
\end{equation}
with probability $1-O\left(me^{-1.5n}\right)$.

The next lemma provides concentration bounds regarding polynomial
functions of $\left\{ \bm{a}_{i}\right\} $.

\begin{lemma}\label{lemma:ai-uniform-concentration}Consider any
$\epsilon>3/n$. Suppose that $\bm{a}_{i}\overset{\mathrm{i.i.d.}}{\sim}\mathcal{N}(\bm{0},\bm{I}_{n})$
for $1\leq i\leq m$. Let
\[
\mathcal{S}:=\left\{ \bm{z}\in\mathbb{R}^{n-1}\Big| \max_{1\leq i\leq m}\left|\bm{a}_{i,\perp}^{\top}\bm{z}\right|\leq\beta\left\Vert \bm{z}\right\Vert _{2}\right\} ,
\]
where $\beta$ is any value obeying $\beta\geq c_{1}\sqrt{\log m}$
for some sufficiently large constant $c_{1}>0$. Then with probability
exceeding $1-O\left(m^{-10}\right)$, one has
\begin{enumerate}
\item $\left|\frac{1}{m}\sum_{i=1}^{m}a_{i,1}^{3}\bm{a}_{i,\perp}^{\top}\bm{z}\right|\leq\epsilon\left\Vert \bm{z}\right\Vert _{2}$
for all $\bm{z}\in\mathcal{S}$, provided that $m\geq c_{0}\max\left\{ \frac{1}{\epsilon^{2}}n\log n,\text{ }\frac{1}{\epsilon}\beta n\log^{\frac{5}{2}}m\right\} $;
\item $\left|\frac{1}{m}\sum_{i=1}^{m}a_{i,1}\big(\bm{a}_{i,\perp}^{\top}\bm{z}\big)^{3}\right|\leq\epsilon\left\Vert \bm{z}\right\Vert _{2}^{3}$
for all $\bm{z}\in\mathcal{S}$, provided that $m\geq c_{0}\max\left\{ \frac{1}{\epsilon^{2}}n\log n,\text{ }\frac{1}{\epsilon}\beta^{3}n\log^{\frac{3}{2}}m\right\} $;
\item $\left|\frac{1}{m}\sum_{i=1}^{m}a_{i,1}^{2}\big(\bm{a}_{i,\perp}^{\top}\bm{z}\big)^{2}-\left\Vert \bm{z}\right\Vert _{2}^{2}\right|\leq\epsilon\left\Vert \bm{z}\right\Vert _{2}^{2}$
for all $\bm{z}\in\mathcal{S}$, provided that $m\geq c_{0}\max\left\{ \frac{1}{\epsilon^{2}}n\log n,\text{ }\frac{1}{\epsilon}\beta^{2}n\log^{2}m\right\} $;
\item $\left|\frac{1}{m}\sum_{i=1}^{m}a_{i,1}^{6}\big(\bm{a}_{i,\perp}^{\top}\bm{z}\big)^{2}-15\left\Vert \bm{z}\right\Vert _{2}^{2}\right|\leq\epsilon\left\Vert \bm{z}\right\Vert _{2}^{2}$
for all $\bm{z}\in\mathcal{S}$, provided that $m\geq c_{0}\max\left\{ \frac{1}{\epsilon^{2}}n\log n,\text{ }\frac{1}{\epsilon}\beta^{2}n\log^{4}m\right\} $;
\item $\left|\frac{1}{m}\sum_{i=1}^{m}a_{i,1}^{2}\big(\bm{a}_{i,\perp}^{\top}\bm{z}\big)^{6}-15\left\Vert \bm{z}\right\Vert _{2}^{6}\right|\leq\epsilon\left\Vert \bm{z}\right\Vert _{2}^{6}$
for all $\bm{z}\in\mathcal{S}$, provided that $m\geq c_{0}\max\left\{ \frac{1}{\epsilon^{2}}n\log n,\text{ }\frac{1}{\epsilon}\beta^{6}n\log^{2}m\right\} $;
\item $\left|\frac{1}{m}\sum_{i=1}^{m}a_{i,1}^{2}\big(\bm{a}_{i,\perp}^{\top}\bm{z}\big)^{4}-3\left\Vert \bm{z}\right\Vert _{2}^{4}\right|\leq\epsilon\left\Vert \bm{z}\right\Vert _{2}^{4}$
for all $\bm{z}\in\mathcal{S}$, provided that $m\geq c_{0}\max\left\{ \frac{1}{\epsilon^{2}}n\log n,\text{ }\frac{1}{\epsilon}\beta^{4}n\log^{2}m\right\} $.
\end{enumerate}
Here, $c_{0}>0$ is some sufficiently large constant. \end{lemma}\begin{proof}See
Appendix \ref{sec:Proof-of-Lemma-ai-uniform-concentration}.\end{proof}

The next lemmas provide the (uniform) matrix concentration inequalities
about $\left\{ \bm{a}_{i}\bm{a}_{i}^{\top}\right\} $.

\begin{lemma}[{\cite[Corollary 5.35]{Vershynin2012}}]\label{lemma:ai-ai-spectral-upper-bound}Suppose
that $\bm{a}_{i}\overset{\mathrm{i.i.d.}}{\sim}\mathcal{N}\left(\bm{0},\bm{I}_{n}\right)$
for $1\leq i\leq m$. With probability at least $1-ce^{-\tilde{c}m}$,
one has
\[
\left\Vert \frac{1}{m}\sum_{i=1}^{m}\bm{a}_{i}\bm{a}_{i}^{\top}\right\Vert \leq2,
\]
as long as $m\geq c_{0}n$ for some sufficiently large constant $c_{0}>0$.
Here, $c,\tilde{c}>0$ are some absolute constants. \end{lemma}

\begin{lemma}\label{lemma:hessian-concentration}Fix some $\bm{x}^{\natural}\in\mathbb{R}^{n}$.
Suppose that $\bm{a}_{i}\overset{\mathrm{i.i.d.}}{\sim}\mathcal{N}\left(\bm{0},\bm{I}_{n}\right)$,
$1\leq i\leq m$. With probability at least $1-O\left(m^{-10}\right)$,
one has
\begin{equation}
\left\Vert \frac{1}{m}\sum_{i=1}^{m}\left(\bm{a}_{i}^{\top}\bm{x}^{\natural}\right)^{2}\bm{a}_{i}\bm{a}_{i}^{\top}-\left\Vert \bm{x}^{\natural}\right\Vert _{2}^{2}\bm{I}_{n}-2\bm{x}^{\natural}\bm{x}^{\natural\top}\right\Vert \leq c_{0}\sqrt{\frac{n\log^{3}m}{m}}\left\Vert \bm{x}^{\natural}\right\Vert _{2}^{2},\label{eq:hessian-first-claim}
\end{equation}
provided that $m>c_{1}n\log^{3}m$. Here, $c_{0},c_{1}$ are some
universal positive constants. Furthermore, fix any $c_{2}>1$ and
suppose that $m>c_{1}n\log^{3}m$ for some sufficiently large constant
$c_{1}>0$. Then with probability exceeding $1-O\left(m^{-10}\right)$,
\begin{equation}
\left\Vert \frac{1}{m}\sum_{i=1}^{m}\left(\bm{a}_{i}^{\top}\bm{z}\right)^{2}\bm{a}_{i}\bm{a}_{i}^{\top}-\left\Vert \bm{z}\right\Vert _{2}^{2}\bm{I}_{n}-2\bm{z}\bm{z}^{\top}\right\Vert \leq c_{0}\sqrt{\frac{n\log^{3}m}{m}}\left\Vert \bm{z}\right\Vert _{2}^{2}\label{eq:hessian-second-claim}
\end{equation}
holds simultaneously for all $\bm{z}\in\mathbb{R}^{n}$ obeying $\max_{1\leq i\leq m}\left|\bm{a}_{i}^{\top}\bm{z}\right|\leq c_{2}\sqrt{\log m}\left\Vert \bm{z}\right\Vert _{2}$.
On this event, we have

\begin{equation}
\left\Vert \frac{1}{m}\sum_{i=1}^{m}|a_{i,1}|^{2}\bm{a}_{i,\perp}\bm{a}_{i,\perp}^{\top}\right\Vert \leq\left\Vert \frac{1}{m}\sum_{i=1}^{m}|a_{i,1}|^{2}\bm{a}_{i}\bm{a}_{i}^{\top}\right\Vert \leq4.\label{eq:hessian-third-claim}
\end{equation}

\end{lemma}\begin{proof}See Appendix \ref{sec:Proof-of-Lemma-hessian-concentration}.\end{proof}

The following lemma provides the concentration results regarding the
Hessian matrix $\nabla^{2}f\left(\bm{x}\right)$.

\begin{lemma}\label{lemma:Hessian-UB-Stage1}
Fix any constant $c_{0}>1$.
Suppose that $m>c_{1}n\log^{3}m$ for some sufficiently large constant
$c_{1}>0$. Then with probability exceeding $1-O\left(m^{-10}\right)$,
\[
\left\Vert \left(\bm{I}_{n}-\eta\nabla^{2}f\left(\bm{z}\right)\right)-\left\{ \big(1-3\eta\left\Vert \bm{z}\right\Vert _{2}^{2}+\eta\big)\bm{I}_{n}+2\eta\bm{x}^{\natural}\bm{x}^{\natural\top}-6\eta\bm{z}\bm{z}^{\top}\right\} \right\Vert \lesssim\sqrt{\frac{n\log^{3}m}{m}}\max\left\{ \left\Vert \bm{z}\right\Vert _{2}^{2},1\right\}
\]
\begin{align*}
	\text{and} \qquad \left\Vert \nabla^{2}f\left(\bm{z}\right)\right\Vert  & \leq10\|\bm{z}\|_{2}^{2}+4
\end{align*}
hold simultaneously for all $\bm{z}$ obeying $\max_{1\leq i\leq m}\left|\bm{a}_{i}^{\top}\bm{z}\right|\leq c_{0}\sqrt{\log m}\left\Vert \bm{z}\right\Vert _{2}$,
provided that $0<\eta<\frac{c_{2}}{\max\{\left\Vert \bm{z}\right\Vert _{2}^{2},1\}}$
for some sufficiently small constant $c_{2}>0$.\end{lemma}\begin{proof}See
Appendix \ref{sec:Proof-of-Lemma-Hessian-UB-Stage1}.\end{proof}

Finally, we note that there are a few immediate consequences of the
induction hypotheses (\ref{subeq:induction}), which we summarize
below. These conditions are useful in the subsequent analysis. Note that Lemma \ref{lemma:consequence-main-text}
is incorporated here.

\begin{lemma}\label{lemma:consequence}Suppose that $m\geq Cn\log^{6}m$
for some sufficiently large constant $C>0$. Then under the hypotheses
(\ref{subeq:induction}) for $t\lesssim\log n$, with probability
at least $1-O(me^{-1.5n})-O(m^{-10})$ one has \begin{subequations}\label{subeq:consequence-norm}
\begin{align}
c_{5}/2\leq\big\|\bm{x}_{\perp}^{t,\left(l\right)}\big\|_{2} & \leq\big\|\bm{x}^{t,\left(l\right)}\big\|_{2}\leq2C_{5};\label{eq:consequence-norm-xt-l}\\
c_{5}/2\leq\big\|\bm{x}_{\perp}^{t,\mathrm{sgn}}\big\|_{2} & \leq\big\|\bm{x}^{t,\mathrm{sgn}}\big\|_{2}\leq2C_{5};\label{eq:consequence-norm-xt-sgn}\\
c_{5}/2\leq\big\|\bm{x}_{\perp}^{t,\mathrm{sgn},\left(l\right)}\big\|_{2} & \leq\big\|\bm{x}^{t,\mathrm{sgn},\left(l\right)}\big\|_{2}\leq2C_{5};\label{eq:consequence-norm-xt-sgn-l}
\end{align}
\end{subequations}\vspace{-2em}\begin{subequations}\label{subeq:consequence-incoherence}
\begin{align}
\max_{1\leq l\leq m}\left|\bm{a}_{l}^{\top}\bm{x}^{t}\right| & \lesssim\sqrt{\log m}\big\|\bm{x}^{t}\big\|_{2}; \label{eq:consequence-incoherence-original}\\
\max_{1\leq l\leq m}\big|\bm{a}_{l,\perp}^{\top}\bm{x}_{\perp}^{t}\big| & \lesssim\sqrt{\log m}\big\|\bm{x}_{\perp}^{t}\big\|_{2};\label{eq:consequence-incoherence-ai-perp-x-t}\\
\max_{1\leq l\leq m}\big|\bm{a}_{l}^{\top}\bm{x}^{t,\mathrm{sgn}}\big| & \lesssim\sqrt{\log m}\big\|\bm{x}^{t,\mathrm{sgn}}\big\|_{2};\\
\max_{1\leq l\leq m}\big|\bm{a}_{l,\perp}^{\top}\bm{x}_{\perp}^{t,\mathrm{sgn}}\big| & \lesssim\sqrt{\log m}\big\|\bm{x}_{\perp}^{t,\mathrm{sgn}}\big\|_{2};\label{eq:consequence-incoherence-ai-perp-x-sgn}\\
\max_{1\leq l\leq m}\big|\bm{a}_{l}^{\mathrm{sgn}\top}\bm{x}^{t,\mathrm{sgn}}\big| & \lesssim\sqrt{\log m}\big\|\bm{x}^{t,\mathrm{sgn}}\big\|_{2};\label{eq:consequence-incoherence-ai-sgn-x-sgn}
\end{align}
\end{subequations}\vspace{-2em}\begin{subequations}\label{eq:consequence-upper-bounds}
\begin{align}
\max_{1\leq l\leq m}\big\|\bm{x}^{t}-\bm{x}^{t,\left(l\right)}\big\|_{2} & \ll\frac{1}{\log m};\label{eq:1-logm-xt-xt-l}\\
\big\|\bm{x}^{t}-\bm{x}^{t,\mathrm{sgn}}\big\|_{2} & \ll\frac{1}{\log m};\label{eq:1-logm-xt-xt-sgn}\\
\max_{1\leq l\leq m}\big|x_{\parallel}^{t,\left(l\right)}\big| & \leq2\alpha_{t}.\label{eq:xt-l-signal-upper-bound}
\end{align}
\end{subequations}\end{lemma}\begin{proof}See Appendix \ref{sec:Proof-of-Lemma-consequence}.
\end{proof}

\section{Proof of Lemma \ref{lemma:iterative}\label{sec:Proof-of-Lemma-iterative}}

We focus on the case when 
\[
\frac{1}{\sqrt{n\log n}}\leq\alpha_{0}\leq\frac{\log n}{\sqrt{n}}\qquad\text{and}\qquad1-\frac{1}{\log n}\leq\beta_{0}\leq1+\frac{1}{\log n}
\]
The other cases can be proved using very similar arguments as below,
and hence omitted.

Let $\eta>0$ and $c_{4}>0$ be some sufficiently small constants
independent of $n$. In the sequel, we divide Stage 1 (iterations
up to $T_{\gamma}$) into several substages. See Figure \ref{fig:substages} for an illustration. 

\begin{figure}
	\centering\includegraphics[width=0.5\textwidth]{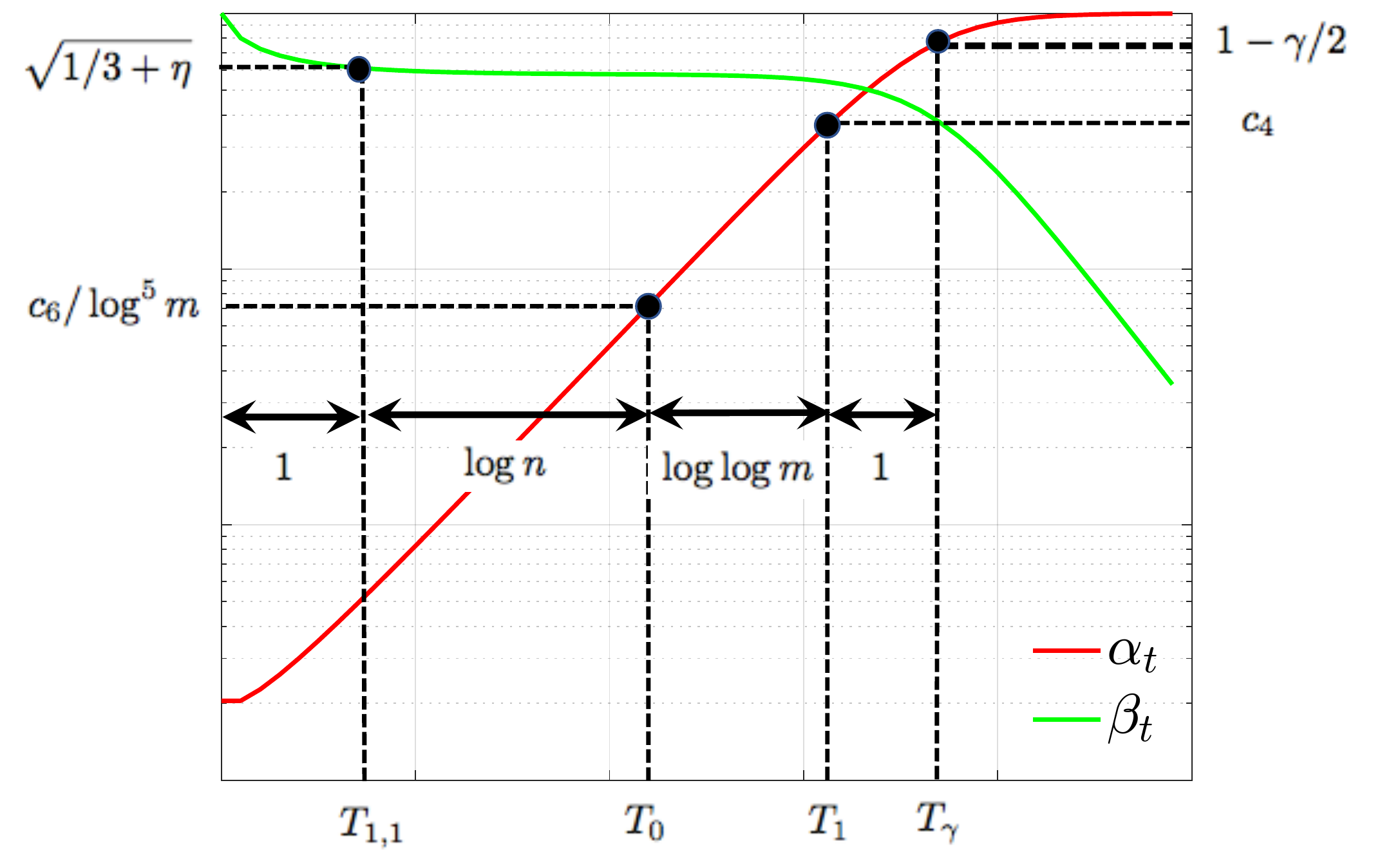}\caption{Illustration of the substages for the proof of Lemma \ref{lemma:iterative}.\label{fig:substages}}
\end{figure}

\begin{itemize}
\item \textbf{Stage 1.1: }consider the period when $\alpha_{t}$ is sufficiently
small, which consists of all iterations $0\leq t\leq T_{1}$ with
$T_{1}$ given in (\ref{eq:defn-T1}). We claim that, throughout this
substage, 
\begin{subequations}\label{eq:claim-range}
\begin{align}
\alpha_{t} & >\frac{1}{2\sqrt{n\log n}},\label{eq:claim-alpha-range} \\
\sqrt{0.5}<\beta_{t} & <\sqrt{1.5}.\label{eq:claim-beta-range}
\end{align}
\end{subequations}
If this claim holds, then we would have $\alpha_{t}^{2}+\beta_{t}^{2}<c_{4}^{2}+1.5<2$
as long as $c_{4}$ is small enough. This immediately reveals that
$1+\eta\left(1-3\alpha_{t}^{2}-3\beta_{t}^{2}\right)\geq1-6\eta$,
which further gives 
\begin{align}
\beta_{t+1} & \geq\left\{ 1+\eta\left(1-3\alpha_{t}^{2}-3\beta_{t}^{2}\right)+\eta\rho_{t}\right\} \beta_{t}\nonumber \\
 & \geq\left(1-6\eta-\frac{c_{3}\eta}{\log n}\right)\beta_{t}\nonumber \\
 & \geq(1-7\eta)\beta_{t}.\label{eq:betat-LB1}
\end{align}
In what follows, we further divide this stage into multiple sub-phases. 
\begin{itemize}
\item \textbf{Stage 1.1.1}: consider the iterations $0\leq t\leq T_{1,1}$
with 
\begin{equation}
T_{1,1}=\min\left\{ t\mid\beta_{t+1}\leq\sqrt{1/3+\eta}\right\} .\label{eq:defn-T11-1}
\end{equation}
\begin{fact}\label{fact:Stage-1.1.1}For any sufficiently small $\eta>0$,
one has
\begin{align}
\beta_{t+1} & \leq(1-2\eta^{2})\beta_{t},\qquad0\leq t\leq T_{1,1};\label{eq:beta-growth-T11}\\
\alpha_{t+1} & \leq(1+4\eta)\alpha_{t},\qquad0\leq t\leq T_{1,1};\nonumber \\
\alpha_{t+1} & \geq(1+2\eta^{3})\alpha_{t},\qquad1\leq t\leq T_{1,1};\label{eq:alpha-growth-T11}\\
\alpha_{1} & \geq\alpha_{0}/2;\nonumber \\
\beta_{T_{1,1}+1} & \geq\frac{1-7\eta}{\sqrt{3}};\nonumber \\
T_{1,1} & \lesssim\frac{1}{\eta^{2}}.
\end{align}
Moreover, $\alpha_{T_{1,1}}\ll c_{4}$ and hence $T_{1,1}<T_{1}$.
\end{fact}From Fact \ref{fact:Stage-1.1.1}, we see that in this substage, $\alpha_{t}$ keeps increasing (at least for $t\geq 1$) with
\[
c_{4}>\alpha_{t}\geq\frac{\alpha_{0}}{2}\geq\frac{1}{2\sqrt{n\log n}},\qquad0\leq t\leq T_{1,1},
\]
and $\beta_t$ is strictly decreasing with
\[
1.5>\beta_{0}\geq\beta_{t}\geq\beta_{T_{1,1}+1}\geq\frac{1-7\eta}{\sqrt{3}},\qquad0\leq t\leq T_{1,1},
\]
which justifies \eqref{eq:claim-range}. In addition, combining (\ref{eq:beta-growth-T11}) with (\ref{eq:alpha-growth-T11}),
we arrive at the growth rate of $\alpha_t/\beta_t$ as
\[
\frac{\alpha_{t+1}/\alpha_{t}}{\beta_{t+1}/\beta_{t}}\geq\frac{1+2\eta^{3}}{1-2\eta^{2}}=1+O(\eta^{2}).
\]
These demonstrate (\ref{eq:alpha-beta-range-SNR}) for this substage. 
\item \textbf{Stage 1.1.2:} this substage contains all iterations obeying
$T_{1,1}<t\leq T_{1}$. We claim the following result. \begin{fact}\label{fact:Stage-1-1-2}Suppose
that $\eta>0$ is sufficiently small. Then for any $T_{1,1}<t\leq T_{1}$,
\begin{align}
\beta_{t} & \in\left[\frac{(1-7\eta)^{2}}{\sqrt{3}},\frac{1+30\eta}{\sqrt{3}}\right];\label{eq:beta-bound-stage12}\\
\beta_{t+1} & \leq(1+30\eta^{2})\beta_{t}.\label{eq:betat-growth-stage12}
\end{align}
\end{fact}

Furthermore, since 
\[
\alpha_{t}^{2}+\beta_{t}^{2}\leq c_{4}^{2}+\frac{(1+30\eta)^{2}}{3}<\frac{1}{2},
\]
we have, for sufficiently small $c_{3}$, that 
\begin{align}
\alpha_{t+1} & \geq\left\{ 1+3\eta\left(1-\alpha_{t}^{2}-\beta_{t}^{2}\right)-\eta|\zeta_{t}|\right\} \alpha_{t}\nonumber \\
 & \geq\left(1+1.5\eta-\frac{c_{3}\eta}{\log n}\right)\alpha_{t}\nonumber \\
 & \geq(1+1.4\eta)\alpha_{t},\label{eq:alpha-t-growth-right}
\end{align}
and hence $\alpha_{t}$ keeps increasing. This means $\alpha_{t}\geq\alpha_{1}\geq\frac{1}{2\sqrt{n\log n}}$,
which justifies the claim (\ref{eq:claim-range}) together with \eqref{eq:beta-bound-stage12} for this substage.
As a consequence, 
\begin{align*}
T_{1}-T_{1,1} & \lesssim\frac{\log\frac{c_{4}}{\alpha_{0}}}{\log(1+1.4\eta)}\lesssim\frac{\log n}{\eta};\\
T_{1}-T_{0} & \lesssim\frac{\log\frac{c_{4}}{\frac{c_{6}}{\log^{5}m}}}{\log\left(1+1.4\eta\right)}\lesssim\frac{\log\log m}{\eta}.
\end{align*}
Moreover, combining (\ref{eq:alpha-t-growth-right}) with (\ref{eq:betat-growth-stage12})
yields the growth rate of $\alpha_t/\beta_t$ as
\begin{align*}
\frac{\alpha_{t+1}/\alpha_{t}}{\beta_{t+1}/\beta_{t}} & \geq\frac{1+1.4\eta}{1+30\eta^{2}}\geq1+\eta
\end{align*}
for $\eta>0$ sufficiently small. 
\item Taken collectively, the preceding bounds imply that 
\[
T_{1}=T_{1,1}+(T_{1}-T_{1,1})\lesssim\frac{1}{\eta^{2}}+\frac{\log n}{\eta}\lesssim\frac{\log n}{\eta^{2}}.
\]
\end{itemize}
\item \textbf{Stage 1.2: }in this stage, we consider all iterations $T_{1}<t\leq T_{2}$,
where 
\[
T_{2}:=\min\left\{ t\mid\frac{\alpha_{t+1}}{\beta_{t+1}}>\frac{2}{\gamma}\right\} .
\]
From the preceding analysis, it is seen that, for $\eta$ sufficiently
small, 
\[
\frac{\alpha_{T_{1,1}}}{\beta_{T_{1,1}}}\leq\frac{c_{4}}{\frac{\left(1-7\eta\right)^{2}}{\sqrt{3}}}\leq\frac{\sqrt{3}c_{4}}{1-15\eta}.
\]
In addition, we have: \begin{fact}\label{fact-Stage1-2}Suppose $\eta>0$
is sufficiently small. Then for any $T_{1}<t\leq T_{2}$, one has
\begin{align}
\alpha_{t}^{2}+\beta_{t}^{2} & \leq2;\label{eq:alpha2-beta2-UB-1}\\
\frac{\alpha_{t+1}/\beta_{t+1}}{\alpha_{t}/\beta_{t}} & \geq1+\eta;\\
\alpha_{t+1} & \geq\left\{ 1-3.1\eta\right\} \alpha_{t};\\
\beta_{t+1} & \geq\left\{ 1-5.1\eta\right\} \beta_{t}.
\end{align}
In addition, 
\[
T_{2}-T_{1}\lesssim\frac{1}{\eta}.
\]
\end{fact}With this fact in place, one has
\[
\alpha_{t}\geq(1-3.1\eta)^{t-T_{1}}\alpha_{T_{1}}\gtrsim1,\qquad T_{1}<t\leq T_{2}.
\]
and hence 
\[
\beta_{t}\geq(1-5.1\eta)^{t-T_{1}}\beta_{T_{1}}\gtrsim1,\qquad T_{1}<t\leq T_{2}.
\]
These taken collectively demonstrate (\ref{eq:alpha-beta-range-SNR})
for any $T_{1}<t\leq T_{2}$. Finally, if $T_{2}\geq T_{\gamma}$,
then we complete the proof as
\[
T_{\gamma}\leq T_{2}=T_{1}+(T_{2}-T_{1})\lesssim\frac{\log n}{\eta^{2}}.
\]
Otherwise we move to the next stage. 
\item \textbf{Stage 1.3:} this stage is composed of all iterations
$T_{2}<t\leq T_{\gamma}$. We break the discussion into two cases.
\begin{itemize}
\item If $\alpha_{T_{2}+1}>1+\gamma$, then $\alpha_{T_{2}+1}^{2}+\beta_{T_{2}+1}^{2}\geq\alpha_{T_{2}+1}^{2}>1+2\gamma$.
This means that 
\begin{align*}
\alpha_{T_{2}+2} & \leq\left\{ 1+3\eta\left(1-\alpha_{T_{2}+1}^{2}-\beta_{T_{2}+1}^{2}\right)+\eta|\zeta_{T_{2}+1}|\right\} \alpha_{T_{2}+1}\\
 & \leq\left\{ 1-6\eta\gamma-\frac{\eta c_{3}}{\log n}\right\} \alpha_{T_{2}+1}\\
 & \leq\left\{ 1-5\eta\gamma\right\} \alpha_{T_{2}+1}
\end{align*}
when $c_{3}>0$ is sufficiently small. Similarly, one also gets $\beta_{T_{2}+2}\leq(1-5\eta\gamma)\beta_{T_{2}+1}$.
As a result, both $\alpha_{t}$ and $\beta_{t}$ will decrease. Repeating
this argument reveals that 
\begin{align*}
\alpha_{t+1} & \leq(1-5\eta\gamma)\alpha_{t},\\
\beta_{t+1} & \leq(1-5\eta\gamma)\beta_{t}
\end{align*}
until $\alpha_{t}\leq1+\gamma$. In addition, applying the same argument
as for Stage 1.2 yields 
\[
\frac{\alpha_{t+1}/\alpha_{t}}{\beta_{t+1}/\beta_{t}}\geq1+c_{10}\eta
\]
for some constant $c_{10}>0$. Therefore, when $\alpha_{t}$ drops
below $1+\gamma$, one has 
\[
\alpha_{t}\geq(1-3\eta)(1+\gamma)\geq1-\gamma
\]
and 
\[
\beta_{t}\leq\frac{\gamma}{2}\alpha_{t}\leq\gamma.
\]
This justifies that 
\[
T_{\gamma}-T_{2}\lesssim\frac{\log\frac{2}{1-\gamma}}{-\log(1-5\eta\gamma)}\lesssim\frac{1}{\eta}.
\]
\item If $c_{4}\leq\alpha_{T_{2}+1}<1-\gamma$, take very similar arguments
as in Stage 1.2 to reach that 
\[
\frac{\alpha_{t+1}/\alpha_{t}}{\beta_{t+1}/\beta_{t}}\geq1+c_{10}\eta,\qquad T_{\gamma}-T_{2}\lesssim\frac{1}{\eta}
\]
\[
\text{and}\qquad\alpha_{t}\gtrsim1,\quad\beta_{t}\gtrsim1\qquad T_{2}\leq t\leq T_{\gamma}
\]
for some constant $c_{10}>0$. We omit the details for brevity. 

\end{itemize}
In either case, we see that $\alpha_{t}$ is always bounded away from
0. We can also repeat the argument for Stage 1.2 to show that $\beta_{t}\gtrsim1$. 
\end{itemize}
In conclusion, we have established that 
\[
T_{\gamma}=T_{1}+(T_{2}-T_{1})+(T_{\gamma}-T_{2})\lesssim\frac{\log n}{\eta^{2}},\qquad0\leq t<T_{\gamma}
\]
\[
\text{and}\qquad\frac{\alpha_{t+1}/\alpha_{t}}{\beta_{t+1}/\beta_{t}}\geq1+c_{10}\eta^{2},\qquad c_{5}\leq\beta_{t}\leq1.5,\qquad\frac{1}{2\sqrt{n\log n}}\leq\alpha_{t}\leq2,\qquad0\leq t<T_{\gamma}
\]
for some constants $c_{5},c_{10}>0$. 

\begin{proof}[Proof of Fact \ref{fact:Stage-1.1.1}]The proof proceeds
as follows.
\begin{itemize}
\item First of all, for any $0\leq t\leq T_{1,1}$, one has $\beta_{t}\geq\sqrt{1/3+\eta}$
and $\alpha_{t}^{2}+\beta_{t}^{2}\geq1/3+\eta$ and, as a result,
\begin{align}
\beta_{t+1} & \leq\left\{ 1+\eta\left(1-3\alpha_{t}^{2}-3\beta_{t}^{2}\right)+\eta|\rho_{t}|\right\} \beta_{t}\nonumber \\
 & \leq\left(1-3\eta^{2}+\frac{c_{3}\eta}{\log n}\right)\beta_{t}\nonumber \\
 & \leq(1-2\eta^{2})\beta_{t}\label{eq:betat-T11}
\end{align}
 as long as $c_{3}$ and $\eta$ are both constants. In other words,
$\beta_{t}$ is strictly decreasing before $T_{1,1}$, which also
justifies the claim (\ref{eq:claim-beta-range}) for this substage. 
\item Moreover, given that the contraction factor of $\beta_{t}$ is at
least $1-2\eta^{2}$, we have 
\[
T_{1,1}\lesssim\frac{\log\frac{\beta_{0}}{\sqrt{1/3+\eta}}}{-\log\left(1-2\eta^{2}\right)}\asymp\frac{1}{\eta^{2}}.
\]
This upper bound also allows us to conclude that $\beta_{t}$ will
cross the threshold $\sqrt{1/3+\eta}$ before $\alpha_{t}$ exceeds
$c_{4}$, namely, $T_{1,1}<T_{1}$. To see this, we note that the
growth rate of $\{\alpha_{t}\}$ within this substage is upper bounded
by 
\begin{align}
\alpha_{t+1} & \leq\left\{ 1+3\eta\left(1-\alpha_{t}^{2}-\beta_{t}^{2}\right)+\eta|\zeta_{t}|\right\} \alpha_{t}\nonumber \\
 & \leq\left(1+3\eta+\frac{c_{3}\eta}{\log n}\right)\alpha_{t}\nonumber \\
 & \leq(1+4\eta)\alpha_{t}.\label{eq:alpha_UB_state11_step}
\end{align}
This leads to an upper bound 
\begin{equation}
|\alpha_{T_{1,1}}|\leq(1+4\eta)^{T_{1,1}}|\alpha_{0}|\leq(1+4\eta)^{O(\eta^{-2})}\frac{\log n}{\sqrt{n}}\ll c_{4}.\label{eq:alpha_UB_state11}
\end{equation}
\item Furthermore, we can also lower bound $\alpha_{t}$. First of all,
\begin{align*}
\alpha_{1} & \geq\left\{ 1+3\eta\left(1-\alpha_{0}^{2}-\beta_{0}^{2}\right)-\eta|\zeta_{t}|\right\} \alpha_{0}\\
 & \geq\left(1-3\eta-\frac{c_{3}\eta}{\log n}\right)\alpha_{0}\\
 & \geq(1-4\eta)\alpha_{0}\geq\frac{1}{2}\alpha_{0}
\end{align*}
for $\eta$ sufficiently small. For all $1\leq t\leq T_{1,1}$, using
(\ref{eq:alpha_UB_state11_step}) we have 
\[
\alpha_{t}^{2}+\beta_{t}^{2}\leq(1+4\eta)^{T_{1,1}}\alpha_{0}^{2}+\beta_{1}^{2}\leq o(1)+(1-2\eta^{2})\beta_{0}\leq1-\eta^{2},
\]
allowing one to deduce that 
\begin{align*}
\alpha_{t+1} & \geq\left\{ 1+3\eta\left(1-\alpha_{t}^{2}-\beta_{t}^{2}\right)-\eta|\zeta_{t}|\right\} \alpha_{t}\\
 & \geq\left(1+3\eta^{3}-\frac{c_{3}\eta}{\log n}\right)\alpha_{t}\\
 & \geq(1+2\eta^{3})\alpha_{t}.
\end{align*}
In other words, $\alpha_{t}$ keeps increasing throughout all $1\leq t\leq T_{1,1}$.
This verifies the condition (\ref{eq:claim-alpha-range}) for this
substage. 
\item Finally, we make note of one useful lower bound 
\begin{align}
\beta_{T_{1,1}+1} & \geq(1-7\eta)\beta_{T_{1,1}}\geq\frac{1-7\eta}{\sqrt{3}},\label{eq:beta_T11_LB}
\end{align}
which follows by combining (\ref{eq:betat-LB1}) and the condition
$\beta_{T_{1,1}}\geq\sqrt{1/3+\eta}$ . 
\end{itemize}
\end{proof}

\begin{proof}[Proof of Fact \ref{fact:Stage-1-1-2}]Clearly, $\beta_{T_{1,1}+1}$
falls within this range according to (\ref{eq:defn-T11-1}) and (\ref{eq:beta_T11_LB}).
We now divide into several cases. 
\begin{itemize}
\item If $\frac{1+\eta}{\sqrt{3}}\leq\beta_{t}<\frac{1+30\eta}{\sqrt{3}}$,
then $\alpha_{t}^{2}+\beta_{t}^{2}\geq\beta_{t}^{2}\geq(1+\eta)^{2}/3$,
and hence the next iteration obeys 
\begin{align}
\beta_{t+1} & \leq\left\{ 1+\eta\left(1-3\beta_{t}^{2}\right)+\eta|\rho_{t}|\right\} \beta_{t}\nonumber \\
 & \leq\left(1+\eta\left(1-(1+\eta)^{2}\right)+\frac{c_{3}\eta}{\log n}\right)\beta_{t}\nonumber \\
 & \leq(1-\eta^{2})\beta_{t}\label{eq:eta-t-growth-right}
\end{align}
and, in view of (\ref{eq:betat-LB1}), $\beta_{t+1}\geq(1-7\eta)\beta_{t}\geq\frac{1-7\eta}{\sqrt{3}}.$
In summary, in this case one has $\beta_{t+1}\in\left[\frac{1-7\eta}{\sqrt{3}},\frac{1+30\eta}{\sqrt{3}}\right]$,
which still resides within the range (\ref{eq:beta-bound-stage12}). 
\item If $\frac{(1-7\eta)^{2}}{\sqrt{3}}\leq\beta_{t}\leq\frac{1-7\eta}{\sqrt{3}}$,
then $\alpha_{t}^{2}+\beta_{t}^{2}<c_{4}^{2}+(1-7\eta)^{2}/3<(1-7\eta)/3$
for $c_{4}$ sufficiently small. Consequently, for a small enough
$c_{3}$ one has 
\begin{align*}
\beta_{t+1} & \geq\left\{ 1+\eta\left(1-3\alpha_{t}^{2}-3\beta_{t}^{2}\right)-\eta|\rho_{t}|\right\} \beta_{t}\\
 & \geq\big(1+7\eta^{2}-\frac{c_{3}\eta}{\log n}\big)\beta_{t}\\
 & \geq(1+6\eta^{2})\beta_{t}.
\end{align*}
In other words, $\beta_{t+1}$ is strictly larger than $\beta_{t}$.
Moreover, recognizing that $\alpha_{t}^{2}+\beta_{t}^{2}>(1-7\eta)^{4}/3>(1-29\eta)/3$,
one has 
\begin{align}
\beta_{t+1} & \leq\left\{ 1+\eta\left(1-3\alpha_{t}^{2}-3\beta_{t}^{2}\right)+\eta|\rho_{t}|\right\} \beta_{t}\nonumber \\
 & \leq\big(1+29\eta^{2}+\frac{c_{3}\eta}{\log n}\big)\beta_{t}\leq(1+30\eta^{2})\beta_{t}\label{eq:eta-t-growth-left}\\
 & <\frac{1+30\eta^{2}}{\sqrt{3}}.\nonumber 
\end{align}
Therefore, we have shown that $\beta_{t+1}\in\left[\frac{(1-7\eta)^{2}}{\sqrt{3}},\frac{1+30\eta}{\sqrt{3}}\right]$,
which continues to lie within the range (\ref{eq:beta-bound-stage12}). 
\item Finally, if $\frac{1-7\eta}{\sqrt{3}}<\beta_{t}<\frac{1+\eta}{\sqrt{3}}$,
we have $\alpha_{t}^{2}+\beta_{t}^{2}\geq\frac{(1-7\eta)^{2}}{3}\geq\frac{1-15\eta}{3}$
for $\eta$ sufficiently small, which implies 
\begin{align}
\beta_{t+1} & \leq\left\{ 1+15\eta^{2}+\eta|\rho_{t}|\right\} \beta_{t}\leq(1+16\eta^{2})\beta_{t}\label{eq:eta-t-growth-mid}\\
 & \leq\frac{(1+16\eta^{2})(1+\eta)}{\sqrt{3}}\leq\frac{1+2\eta}{\sqrt{3}}\nonumber 
\end{align}
for small $\eta>0$. In addition, it comes from (\ref{eq:beta_T11_LB})
that $\beta_{t+1}\geq(1-7\eta)\beta_{t}\geq\frac{(1-7\eta)^{2}}{\sqrt{3}}.$
This justifies that $\beta_{t+1}$ falls within the range (\ref{eq:beta-bound-stage12}). 
\end{itemize}
Combining all of the preceding cases establishes the claim (\ref{eq:beta-bound-stage12})
for all $T_{1,1}<t\leq T_{1}$. \end{proof}

\begin{proof}[Proof of Fact \ref{fact-Stage1-2}]We first demonstrate
that 
\begin{equation}
\alpha_{t}^{2}+\beta_{t}^{2}\leq2\label{eq:alpha2-beta2-UB}
\end{equation}
throughout this substage. In fact, if $\alpha_{t}^{2}+\beta_{t}^{2}\leq1.5$,
then 
\[
\alpha_{t+1}\leq\left\{ 1+3\eta\left(1-\alpha_{t}^{2}-\beta_{t}^{2}\right)+\eta|\zeta_{t}|\right\} \alpha_{t}\leq\left(1+4\eta\right)\alpha_{t}
\]
and, similarly, $\beta_{t+1}\leq(1+4\eta)\beta_{t}$. These taken
together imply that 
\[
\alpha_{t+1}^{2}+\beta_{t+1}^{2}\leq\left(1+4\eta\right)^{2}\left(\alpha_{t}^{2}+\beta_{t}^{2}\right)\leq1.5(1+9\eta)<2.
\]
Additionally, if $1.5<\alpha_{t}^{2}+\beta_{t}^{2}\leq2$, then 
\begin{align*}
\alpha_{t+1} & \leq\left\{ 1+3\eta\left(1-\alpha_{t}^{2}-\beta_{t}^{2}\right)+\eta|\zeta_{t}|\right\} \alpha_{t}\\
 & \leq\left(1-1.5\eta+\frac{c_{3}\eta}{\log n}\right)\alpha_{t}\\
 & \leq(1-\eta)\alpha_{t}
\end{align*}
and, similarly, $\beta_{t+1}\leq(1-\eta)\beta_{t}.$ These reveal
that 
\begin{align*}
\alpha_{t+1}^{2}+\beta_{t+1}^{2} & \leq\alpha_{t}^{2}+\beta_{t}^{2}.
\end{align*}
Put together the above argument to establish the claim (\ref{eq:alpha2-beta2-UB}). 

With the claim (\ref{eq:alpha2-beta2-UB}) in place, we can deduce
that 
\begin{align}
\alpha_{t+1} & \geq\left\{ 1+3\eta\left(1-\alpha_{t}^{2}-\beta_{t}^{2}\right)-\eta|\zeta_{t}|\right\} \alpha_{t}\nonumber \\
 & \geq\left\{ 1+3\eta\left(1-\alpha_{t}^{2}-\beta_{t}^{2}\right)-0.1\eta\right\} \alpha_{t}\label{eq:alphat-LB-stage2}
\end{align}
and 
\begin{align*}
\beta_{t+1} & \leq\left\{ 1+\eta\left(1-3\alpha_{t}^{2}-3\beta_{t}^{2}\right)+\eta|\rho_{t}|\right\} \beta_{t}\\
 & \leq\left\{ 1+\eta\left(1-3\alpha_{t}^{2}-3\beta_{t}^{2}\right)+0.1\eta\right\} \beta_{t}.
\end{align*}
Consequently, 
\begin{align*}
\frac{\alpha_{t+1}/\beta_{t+1}}{\alpha_{t}/\beta_{t}} & =\frac{\alpha_{t+1}/\alpha_{t}}{\beta_{t+1}/\beta_{t}}\geq\frac{1+3\eta\left(1-\alpha_{t}^{2}-\beta_{t}^{2}\right)-0.1\eta}{1+\eta\left(1-3\alpha_{t}^{2}-3\beta_{t}^{2}\right)+0.1\eta}\\
 & =1+\frac{1.8\eta}{1+\eta\left(1-3\alpha_{t}^{2}-3\beta_{t}^{2}\right)+0.1\eta}\\
 & \geq1+\frac{1.8\eta}{1+2\eta}\geq1+\eta
\end{align*}
for $\eta>0$ sufficiently small. This immediately implies that 
\[
T_{2}-T_{1}\lesssim\frac{\log\left(\frac{2/\gamma}{\alpha_{T_{1}}/\beta_{T_{1}}}\right)}{\log\left(1+\eta\right)}\asymp\frac{1}{\eta}.
\]

Moreover, combine (\ref{eq:alpha2-beta2-UB}) and (\ref{eq:alphat-LB-stage2})
to arrive at 
\begin{align}
\alpha_{t+1} & \geq\left\{ 1-3.1\eta\right\} \alpha_{t},\label{eq:alphat-LB-stage2-1}
\end{align}
Similarly, one can show that $\beta_{t+1}\geq\left\{ 1-5.1\eta\right\} \beta_{t}.$
\end{proof}

\section{Proof of Lemma \ref{lemma:xt-signal} \label{sec:Proof-of-Lemma-xt-signal}}

\subsection{Proof of (\ref{eq:approximate_state_evolution_phase_1_alpha})}

In view of the gradient update rule (\ref{eq:gradient_update-WF}),
we can express the signal component $x_{||}^{t+1}$ as follows 
\begin{align*}
x_{\|}^{t+1} & =x_{\|}^{t}-\frac{\eta}{m}\sum_{i=1}^{m}\left[\left(\bm{a}_{i}^{\top}\bm{x}^{t}\right)^{3}-a_{i,1}^{2}\left(\bm{a}_{i}^{\top}\bm{x}^{t}\right)\right]a_{i,1}.
\end{align*}
Expanding this expression using $\bm{a}_{i}^{\top}\bm{x}^{t}=x_{\|}^{t}a_{i,1}+\bm{a}_{i,\perp}^{\top}\bm{x}_{\perp}^{t}$
and rearranging terms, we are left with 
\begin{align*}
x_{\|}^{t+1} & =x_{\|}^{t}+\underset{:=J_{1}}{\underbrace{\eta\left[1-\big(x_{\|}^{t}\big)^{2}\right]x_{\|}^{t}\cdot\frac{1}{m}\sum_{i=1}^{m}a_{i,1}^{4}}}+\underset{:=J_{2}}{\underbrace{\eta\left[1-3\big(x_{\|}^{t}\big)^{2}\right]\cdot\frac{1}{m}\sum_{i=1}^{m}a_{i,1}^{3}\bm{a}_{i,\perp}^{\top}\bm{x}_{\perp}^{t}}}\\
 & \quad-\underset{:=J_{3}}{\underbrace{3\eta x_{\|}^{t}\cdot\frac{1}{m}\sum_{i=1}^{m}\left(\bm{a}_{i,\perp}^{\top}\bm{x}_{\perp}^{t}\right)^{2}a_{i,1}^{2}}}-\underset{:=J_{4}}{\underbrace{\eta\cdot\frac{1}{m}\sum_{i=1}^{m}\left(\bm{a}_{i,\perp}^{\top}\bm{x}_{\perp}^{t}\right)^{3}a_{i,1}}}.
\end{align*}
In the sequel, we control the above four terms $J_{1}$, $J_{2}$,
$J_{3}$ and $J_{4}$ separately. 
\begin{itemize}
\item With regard to the first term $J_{1}$, it follows from the standard
concentration inequality for Gaussian polynomials \cite[Theorem 1.9]{schudy2012concentration}
that 
\[
\mathbb{P}\left(\left|\frac{1}{m}\sum_{i=1}^{m}a_{i,1}^{4}-3\right|\geq\tau\right)\leq e^{2}e^{-c_{1}m^{1/4}\tau^{1/2}}
\]
for some absolute constant $c_{1}>0$. Taking $\tau\asymp\frac{\log^{3}m}{\sqrt{m}}$
reveals that with probability exceeding $1-O\left(m^{-10}\right)$,
\begin{align}
J_{1} & =3\eta\left[1-\big(x_{\|}^{t}\big)^{2}\right]x_{\|}^{t}+\left(\frac{1}{m}\sum_{i=1}^{m}a_{i,1}^{4}-3\right)\eta\left[1-\big(x_{\|}^{t}\big)^{2}\right]x_{\|}^{t}\nonumber \\
 & =3\eta\left[1-\big(x_{\|}^{t}\big)^{2}\right]x_{\|}^{t}+r_{1},\label{eq:J-1-result}
\end{align}
where the remainder term $r_{1}$ obeys 
\[
\left|r_{1}\right|=O\left(\frac{\eta\log^{3}m}{\sqrt{m}}\big|x_{\|}^{t}\big|\right).
\]
Here, the last line also uses the fact that
\begin{equation}
\left|1-\big(x_{\|}^{t}\big)^{2}\right|\leq1+\left\Vert \bm{x}^{t}\right\Vert _{2}^{2}\lesssim1,\label{eq:1-alpha-t-square}
\end{equation}
with the last relation coming from the induction hypothesis (\ref{eq:induction-norm-size}). 
\item For the third term $J_{3}$, it is easy to see that 
\begin{equation}
\frac{1}{m}\sum_{i=1}^{m}\left(\bm{a}_{i,\perp}^{\top}\bm{x}_{\perp}^{t}\right)^{2}a_{i,1}^{2}-\left\Vert \bm{x}_{\perp}^{t}\right\Vert _{2}^{2}=\bm{x}_{\perp}^{t\top}\Bigg[\underbrace{\frac{1}{m}\sum_{i=1}^{m}\left(\bm{a}_{i}^{\top}\bm{x}^{\natural}\right)^{2}\bm{a}_{i,\perp}\bm{a}_{i,\perp}^{\top}}_{:=\bm{U}}-\bm{I}_{n-1}\Bigg]\bm{x}_{\perp}^{t},\label{eq:U-defn}
\end{equation}
where $\bm{U}-\bm{I}_{n-1}$ is a submatrix of the following matrix
(obtained by removing its first row and column)
\begin{equation}
\frac{1}{m}\sum_{i=1}^{m}\left(\bm{a}_{i}^{\top}\bm{x}^{\natural}\right)^{2}\bm{a}_{i}\bm{a}_{i}^{\top}-\left(\bm{I}_{n}+2\bm{x}^\natural\bm{x}^{\natural\top}\right).\label{eq:hessian-diff-exp}
\end{equation}
This fact combined with Lemma \ref{lemma:hessian-concentration} reveals
that 
\[
\left\Vert \bm{U}-\bm{I}_{n-1}\right\Vert \leq\left\Vert \frac{1}{m}\sum_{i=1}^{m}\left(\bm{a}_{i}^{\top}\bm{x}^{\natural}\right)^{2}\bm{a}_{i}\bm{a}_{i}^{\top}-\left(\bm{I}_{n}+2\bm{x}^\natural\bm{x}^{\natural\top}\right)\right\Vert \lesssim\sqrt{\frac{n\log^{3}m}{m}}
\]
with probability at least $1-O\left(m^{-10}\right)$, provided that
$m\gg n\log^{3}m$. This further implies 
\begin{equation}
J_{3}=3\eta\left\Vert \bm{x}_{\perp}^{t}\right\Vert _{2}^{2}x_{\|}^{t}+r_{2},\label{eq:J-3-result}
\end{equation}
where the size of the remaining term $r_{2}$ satisfies
\[
\left|r_{2}\right|\lesssim\eta\sqrt{\frac{n\log^{3}m}{m}}\big|x_{\|}^{t}\big|\left\Vert \bm{x}_{\perp}^{t}\right\Vert _{2}^{2}\lesssim\eta\sqrt{\frac{n\log^{3}m}{m}}\big|x_{\|}^{t}\big|.
\]
Here, the last inequality holds under the hypothesis (\ref{eq:induction-norm-size})
that $\left\Vert \bm{x}_{\perp}^{t}\right\Vert _{2}^{2}\leq\left\Vert \bm{x}^{t}\right\Vert _{2}^{2}\lesssim1$. 
\item When it comes to $J_{2}$, our analysis relies on the random-sign 
sequence $\left\{ \bm{x}^{t,\mathrm{sgn}}\right\} $. Specifically,
one can decompose 
\begin{equation}
\frac{1}{m}\sum_{i=1}^{m}a_{i,1}^{3}\bm{a}_{i,\perp}^{\top}\bm{x}_{\perp}^{t}=\frac{1}{m}\sum_{i=1}^{m}a_{i,1}^{3}\bm{a}_{i,\perp}^{\top}\bm{x}_{\perp}^{t,\mathrm{sgn}}+\frac{1}{m}\sum_{i=1}^{m}a_{i,1}^{3}\bm{a}_{i,\perp}^{\top}\left(\bm{x}_{\perp}^{t}-\bm{x}_{\perp}^{t,\mathrm{sgn}}\right).\label{eq:J2}
\end{equation}
For the first term on the right-hand side of (\ref{eq:J2}), note
that $|a_{i,1}|^{3}\bm{a}_{i,\perp}^{\top}\bm{x}_{\perp}^{t,\mathrm{sgn}}$
is statistically independent of $\xi_{i}=\text{sgn}\left(a_{i,1}\right)$.
Therefore we can treat $\frac{1}{m}\sum_{i=1}^{m}a_{i,1}^{3}\bm{a}_{i,\perp}^{\top}\bm{x}_{\perp}^{t,\mathrm{sgn}}$
as a weighted sum of the $\xi_{i}$'s and apply the Bernstein inequality
(see Lemma \ref{lemma:bernstein}) to arrive at 
\begin{equation}
\left|\frac{1}{m}\sum_{i=1}^{m}a_{i,1}^{3}\bm{a}_{i,\perp}^{\top}\bm{x}_{\perp}^{t,\mathrm{sgn}}\right|=\left|\frac{1}{m}\sum_{i=1}^{m}\xi_{i}\left(\left|a_{i,1}\right|^{3}\bm{a}_{i,\perp}^{\top}\bm{x}_{\perp}^{t,\mathrm{sgn}}\right)\right|\lesssim\frac{1}{m}\left(\sqrt{V_{1}\log m}+B_{1}\log m\right)\label{eq:bernstein-1}
\end{equation}
with probability exceeding $1-O\left(m^{-10}\right)$, where 
\[
V_{1}:=\sum_{i=1}^{m}\left|a_{i,1}\right|^{6}\left(\bm{a}_{i,\perp}^{\top}\bm{x}_{\perp}^{t,\mathrm{sgn}}\right)^{2}\qquad\text{and}\qquad B_{1}:=\max_{1\leq i\leq m}\left|a_{i,1}\right|^{3}\left|\bm{a}_{i,\perp}^{\top}\bm{x}_{\perp}^{t,\mathrm{sgn}}\right|.
\]
Make use of Lemma \ref{lemma:ai-uniform-concentration} and the incoherence
condition (\ref{eq:consequence-incoherence-ai-perp-x-sgn}) to deduce
that with probability at least $1-O\left(m^{-10}\right)$, 
\begin{align*}
\frac{1}{m}V_{1} & =\frac{1}{m}\sum_{i=1}^{m}\left|a_{i,1}\right|^{6}\left(\bm{a}_{i,\perp}^{\top}\bm{x}_{\perp}^{t,\mathrm{sgn}}\right)^{2}\lesssim\left\Vert \bm{x}_{\perp}^{t,\mathrm{sgn}}\right\Vert _{2}^{2}
\end{align*}
with the proviso that $m\gg n\log^{5}m$. Furthermore, the incoherence
condition (\ref{eq:consequence-incoherence-ai-perp-x-sgn}) together
with the fact (\ref{eq:max-a-i-1}) implies that 
\[
B_{1}\lesssim\log^{2}m\left\Vert \bm{x}_{\perp}^{t,\mathrm{sgn}}\right\Vert _{2}.
\]
Substitute the bounds on $V_{1}$ and \textbf{$B_{1}$ }back to (\ref{eq:bernstein-1})
to obtain 
\begin{equation}
\left|\frac{1}{m}\sum_{i=1}^{m}a_{i,1}^{3}\bm{a}_{i,\perp}^{\top}\bm{x}_{\perp}^{t,\mathrm{sgn}}\right|\lesssim\sqrt{\frac{\log m}{m}}\left\Vert \bm{x}_{\perp}^{t,\mathrm{sgn}}\right\Vert _{2}+\frac{\log^{3}m}{m}\left\Vert \bm{x}_{\perp}^{t,\mathrm{sgn}}\right\Vert _{2}\asymp\sqrt{\frac{\log m}{m}}\left\Vert \bm{x}_{\perp}^{t,\mathrm{sgn}}\right\Vert _{2}\label{eq:J2-1}
\end{equation}
as long as $m\gtrsim\log^{5}m$. Additionally, regarding the second
term on the right-hand side of (\ref{eq:J2}), one sees that 
\begin{equation}
\frac{1}{m}\sum_{i=1}^{m}a_{i,1}^{3}\bm{a}_{i,\perp}^{\top}\left(\bm{x}_{\perp}^{t}-\bm{x}_{\perp}^{t,\mathrm{sgn}}\right)=\underbrace{\frac{1}{m}\sum_{i=1}^{m}\left(\bm{a}_{i}^{\top}\bm{x}^{\natural}\right)^{2}a_{i,1}\bm{a}_{i,\perp}^{\top}}_{:=\bm{u}^{\top}}\left(\bm{x}_{\perp}^{t}-\bm{x}_{\perp}^{t,\mathrm{sgn}}\right),\label{eq:u-defn}
\end{equation}
where $\bm{u}$ is the first column of (\ref{eq:hessian-diff-exp}) without
the first entry. Hence we have 
\begin{align}
\left|\frac{1}{m}\sum_{i=1}^{m}a_{i,1}^{3}\bm{a}_{i,\perp}^{\top}\left(\bm{x}_{\perp}^{t}-\bm{x}_{\perp}^{t,\mathrm{sgn}}\right)\right| & \leq\left\Vert \bm{u}\right\Vert _{2}\left\Vert \bm{x}_{\perp}^{t}-\bm{x}_{\perp}^{t,\mathrm{sgn}}\right\Vert _{2}\lesssim\sqrt{\frac{n\log^{3}m}{m}}\left\Vert \bm{x}_{\perp}^{t}-\bm{x}_{\perp}^{t,\mathrm{sgn}}\right\Vert _{2},\label{eq:J2-2}
\end{align}
with probability exceeding $1-O\left(m^{-10}\right)$, with the proviso
that $m\gg n\log^{3}m$. Substituting the above two bounds (\ref{eq:J2-1})
and (\ref{eq:J2-2}) back into (\ref{eq:J2}) gives 
\begin{align*}
\left|\frac{1}{m}\sum_{i=1}^{m}a_{i,1}^{3}\bm{a}_{i,\perp}^{\top}\bm{x}_{\perp}^{t}\right| & \leq\left|\frac{1}{m}\sum_{i=1}^{m}a_{i,1}^{3}\bm{a}_{i,\perp}^{\top}\bm{x}_{\perp}^{t,\mathrm{sgn}}\right|+\left|\frac{1}{m}\sum_{i=1}^{m}a_{i,1}^{3}\bm{a}_{i,\perp}^{\top}\left(\bm{x}_{\perp}^{t}-\bm{x}_{\perp}^{t,\mathrm{sgn}}\right)\right|\\
 & \lesssim\sqrt{\frac{\log m}{m}}\left\Vert \bm{x}_{\perp}^{t,\mathrm{sgn}}\right\Vert _{2}+\sqrt{\frac{n\log^{3}m}{m}}\left\Vert \bm{x}_{\perp}^{t}-\bm{x}_{\perp}^{t,\mathrm{sgn}}\right\Vert _{2}.
\end{align*}
As a result, we arrive at the following bound on $J_{2}$: 
\begin{align*}
\left|J_{2}\right| & \lesssim\eta\left|1-3\big(x_{\|}^{t}\big)^{2}\right|\left(\sqrt{\frac{\log m}{m}}\left\Vert \bm{x}_{\perp}^{t,\mathrm{sgn}}\right\Vert _{2}+\sqrt{\frac{n\log^{3}m}{m}}\left\Vert \bm{x}_{\perp}^{t}-\bm{x}_{\perp}^{t,\mathrm{sgn}}\right\Vert _{2}\right)\\
 & \overset{\left(\text{i}\right)}{\lesssim}\eta\sqrt{\frac{\log m}{m}}\left\Vert \bm{x}_{\perp}^{t,\mathrm{sgn}}\right\Vert _{2}+\eta\sqrt{\frac{n\log^{3}m}{m}}\left\Vert \bm{x}_{\perp}^{t}-\bm{x}_{\perp}^{t,\mathrm{sgn}}\right\Vert _{2}\\
 & \overset{\left(\text{ii}\right)}{\lesssim}\eta\sqrt{\frac{\log m}{m}}\left\Vert \bm{x}_{\perp}^{t}\right\Vert _{2}+\eta\sqrt{\frac{n\log^{3}m}{m}}\left\Vert \bm{x}_{\perp}^{t}-\bm{x}_{\perp}^{t,\mathrm{sgn}}\right\Vert _{2},
\end{align*}
where (i) uses (\ref{eq:1-alpha-t-square}) again and (ii) comes from
the triangle inequality $\left\Vert \bm{x}_{\perp}^{t,\mathrm{sgn}}\right\Vert _{2}\leq\left\Vert \bm{x}_{\perp}^{t}\right\Vert _{2}+\left\Vert \bm{x}_{\perp}^{t}-\bm{x}_{\perp}^{t,\mathrm{sgn}}\right\Vert _{2}$
and the fact that $\sqrt{\frac{\log m}{m}}\leq\sqrt{\frac{n\log^{3}m}{m}}$. 
\item It remains to control $J_{4}$, towards which we resort to the random-sign
sequence $\left\{ \bm{x}^{t,\text{sgn}}\right\} $ once again. Write
\begin{equation}
\frac{1}{m}\sum_{i=1}^{m}\left(\bm{a}_{i,\perp}^{\top}\bm{x}_{\perp}^{t}\right)^{3}a_{i,1}=\frac{1}{m}\sum_{i=1}^{m}\left(\bm{a}_{i,\perp}^{\top}\bm{x}_{\perp}^{t,\mathrm{sgn}}\right)^{3}a_{i,1}+\frac{1}{m}\sum_{i=1}^{m}\left[\left(\bm{a}_{i,\perp}^{\top}\bm{x}_{\perp}^{t}\right)^{3}-\left(\bm{a}_{i,\perp}^{\top}\bm{x}_{\perp}^{t,\mathrm{sgn}}\right)^{3}\right]a_{i,1}.\label{eq:J4}
\end{equation}
For the first term in (\ref{eq:J4}), since $\xi_{i}=\mathrm{sgn}\left(a_{i,1}\right)$
is statistically independent of $\big(\bm{a}_{i,\perp}^{\top}\bm{x}_{\perp}^{t,\mathrm{sgn}}\big)^{3}\left|a_{i,1}\right|$,
we can upper bound the first term using the Bernstein inequality (see
Lemma \ref{lemma:bernstein}) as 
\[
\left|\frac{1}{m}\sum_{i=1}^{m}\left(\bm{a}_{i,\perp}^{\top}\bm{x}_{\perp}^{t,\mathrm{sgn}}\right)^{3}\left|a_{i,1}\right|\xi_{i}\right|\lesssim\frac{1}{m}\left(\sqrt{V_{2}\log m}+B_{2}\log m\right),
\]
where the quantities $V_{2}$ and $B_{2}$ obey 
\begin{align*}
V_{2} & :=\sum_{i=1}^{m}\left(\bm{a}_{i,\perp}^{\top}\bm{x}_{\perp}^{t,\mathrm{sgn}}\right)^{6}\left|a_{i,1}\right|^{2}\qquad\text{and}\qquad B_{2}:=\max_{1\leq i\leq m}\left|\bm{a}_{i,\perp}^{\top}\bm{x}_{\perp}^{t,\mathrm{sgn}}\right|^{3}\left|a_{i,1}\right|.
\end{align*}
Using similar arguments as in bounding (\ref{eq:bernstein-1}) yields
\[
V_{2}\lesssim m\left\Vert \bm{x}_{\perp}^{t,\mathrm{sgn}}\right\Vert _{2}^{6}\qquad\text{and}\qquad B_{2}\lesssim\log^{2}m\left\Vert \bm{x}_{\perp}^{t,\mathrm{sgn}}\right\Vert _{2}^{3}
\]
with the proviso that $m\gg n\log^{5}m$ and 
\begin{equation}
\left|\frac{1}{m}\sum_{i=1}^{m}\left(\bm{a}_{i,\perp}^{\top}\bm{x}_{\perp}^{t,\mathrm{sgn}}\right)^{3}\left|a_{i,1}\right|\xi_{i}\right|\lesssim\sqrt{\frac{\log m}{m}}\left\Vert \bm{x}_{\perp}^{t,\mathrm{sgn}}\right\Vert _{2}^{3}+\frac{\log^{3}m}{m}\left\Vert \bm{x}_{\perp}^{t,\mathrm{sgn}}\right\Vert _{2}^{3}\asymp\sqrt{\frac{\log m}{m}}\left\Vert \bm{x}_{\perp}^{t,\mathrm{sgn}}\right\Vert _{2}^{3},\label{eq:J4-1}
\end{equation}
with probability exceeding $1-O(m^{-10})$ as soon as $m\gtrsim\log^{5}m$.
Regarding the second term in (\ref{eq:J4}), 
\begin{align}
 & \left|\frac{1}{m}\sum_{i=1}^{m}\left[\left(\bm{a}_{i,\perp}^{\top}\bm{x}_{\perp}^{t}\right)^{3}-\left(\bm{a}_{i,\perp}^{\top}\bm{x}_{\perp}^{t,\mathrm{sgn}}\right)^{3}\right]a_{i,1}\right|\nonumber \\
 & \quad\overset{(\text{i})}{=}\frac{1}{m}\sum_{i=1}^{m}\left|\left\{ \bm{a}_{i,\perp}^{\top}\left(\bm{x}_{\perp}^{t}-\bm{x}_{\perp}^{t,\mathrm{sgn}}\right)\left[\left(\bm{a}_{i,\perp}^{\top}\bm{x}_{\perp}^{t}\right)^{2}+\left(\bm{a}_{i,\perp}^{\top}\bm{x}_{\perp}^{t,\mathrm{sgn}}\right)^{2}+\left(\bm{a}_{i,\perp}^{\top}\bm{x}_{\perp}^{t}\right)\left(\bm{a}_{i,\perp}^{\top}\bm{x}_{\perp}^{t,\mathrm{sgn}}\right)\right]\right\} a_{i,1}\right|\nonumber \\
 & \quad\overset{(\text{ii})}{\leq}\sqrt{\frac{1}{m}\sum_{i=1}^{m}\left[\bm{a}_{i,\perp}^{\top}\left(\bm{x}_{\perp}^{t}-\bm{x}_{\perp}^{t,\mathrm{sgn}}\right)\right]^{2}}\sqrt{\frac{1}{m}\sum_{i=1}^{m}\left[5\big(\bm{a}_{i,\perp}^{\top}\bm{x}_{\perp}^{t}\big)^{4}+5\big(\bm{a}_{i,\perp}^{\top}\bm{x}_{\perp}^{t,\mathrm{sgn}}\big)^{4}\right]a_{i,1}^{2}}.\label{eq:J4-2-prior}
\end{align}
Here, the first equality (i) utilizes the elementary identity $a^{3}-b^{3}=\left(a-b\right)\left(a^{2}+b^{2}+ab\right)$,
and (ii) follows from the Cauchy-Schwarz inequality as well as the
inequality 
\[
(a^{2}+b^{2}+ab)^{2}\leq(1.5a^{2}+1.5b^{2})^{2}\leq5a^{4}+5b^{4}.
\]
Use Lemma \ref{lemma:ai-ai-spectral-upper-bound} to reach 
\[
\sqrt{\frac{1}{m}\sum_{i=1}^{m}\left[\bm{a}_{i,\perp}^{\top}\left(\bm{x}_{\perp}^{t}-\bm{x}_{\perp}^{t,\mathrm{sgn}}\right)\right]^{2}}=\sqrt{\left(\bm{x}_{\perp}^{t}-\bm{x}_{\perp}^{t,\mathrm{sgn}}\right)^{\top}\left(\frac{1}{m}\sum_{i=1}^{m}\bm{a}_{i,\perp}\bm{a}_{i,\perp}^{\top}\right)\left(\bm{x}_{\perp}^{t}-\bm{x}_{\perp}^{t,\mathrm{sgn}}\right)}\lesssim\left\Vert \bm{x}_{\perp}^{t}-\bm{x}_{\perp}^{t,\mathrm{sgn}}\right\Vert _{2}.
\]
Additionally, combining Lemma \ref{lemma:ai-uniform-concentration}
and the incoherence conditions (\ref{eq:consequence-incoherence-ai-perp-x-t})
and (\ref{eq:consequence-incoherence-ai-perp-x-sgn}), we can obtain
\[
\sqrt{\frac{1}{m}\sum_{i=1}^{m}\left[5\big(\bm{a}_{i,\perp}^{\top}\bm{x}_{\perp}^{t}\big)^{4}+5\big(\bm{a}_{i,\perp}^{\top}\bm{x}_{\perp}^{t,\mathrm{sgn}}\big)^{4}\right]a_{i,1}^{2}}\lesssim\left\Vert \bm{x}_{\perp}^{t}\right\Vert _{2}^{2}+\left\Vert \bm{x}_{\perp}^{t,\text{sgn}}\right\Vert _{2}^{2}\lesssim1,
\]
as long as $m\gg n\log^{6}m$. Here, the last relation comes from
the norm conditions (\ref{eq:induction-norm-size}) and (\ref{eq:consequence-norm-xt-sgn}).
These in turn imply 
\begin{equation}
\left|\frac{1}{m}\sum_{i=1}^{m}\left[\left(\bm{a}_{i,\perp}^{\top}\bm{x}_{\perp}^{t}\right)^{3}-\left(\bm{a}_{i,\perp}^{\top}\bm{x}_{\perp}^{t,\mathrm{sgn}}\right)^{3}\right]a_{i,1}\right|\lesssim\left\Vert \bm{x}_{\perp}^{t}-\bm{x}_{\perp}^{t,\mathrm{sgn}}\right\Vert _{2}.\label{eq:J4-2}
\end{equation}
Combining the above bounds (\ref{eq:J4-1}) and (\ref{eq:J4-2}),
we get 
\begin{align*}
\left|J_{4}\right| & \leq\eta\left|\frac{1}{m}\sum_{i=1}^{m}\left(\bm{a}_{i,\perp}^{\top}\bm{x}_{\perp}^{t,\mathrm{sgn}}\right)^{3}a_{i,1}\right|+\eta\left|\frac{1}{m}\sum_{i=1}^{m}\left[\left(\bm{a}_{i,\perp}^{\top}\bm{x}_{\perp}^{t}\right)^{3}-\left(\bm{a}_{i,\perp}^{\top}\bm{x}_{\perp}^{t,\mathrm{sgn}}\right)^{3}\right]a_{i,1}\right|\\
 & \lesssim\eta\sqrt{\frac{\log m}{m}}\left\Vert \bm{x}_{\perp}^{t,\mathrm{sgn}}\right\Vert _{2}^{3}+\eta\left\Vert \bm{x}_{\perp}^{t}-\bm{x}_{\perp}^{t,\mathrm{sgn}}\right\Vert _{2}\\
 & \lesssim\eta\sqrt{\frac{\log m}{m}}\left\Vert \bm{x}_{\perp}^{t,\mathrm{sgn}}\right\Vert _{2}+\eta\left\Vert \bm{x}_{\perp}^{t}-\bm{x}_{\perp}^{t,\mathrm{sgn}}\right\Vert _{2}\\
 & \lesssim\eta\sqrt{\frac{\log m}{m}}\left\Vert \bm{x}_{\perp}^{t}\right\Vert _{2}+\eta\left\Vert \bm{x}_{\perp}^{t}-\bm{x}_{\perp}^{t,\mathrm{sgn}}\right\Vert _{2},
\end{align*}
where the penultimate inequality arises from the norm condition (\ref{eq:consequence-norm-xt-sgn})
and the last one comes from the triangle inequality $\left\Vert \bm{x}_{\perp}^{t,\mathrm{sgn}}\right\Vert _{2}\leq\left\Vert \bm{x}_{\perp}^{t}\right\Vert _{2}+\left\Vert \bm{x}_{\perp}^{t}-\bm{x}_{\perp}^{t,\mathrm{sgn}}\right\Vert _{2}$
and the fact that $\sqrt{\frac{\log m}{m}}\leq1.$ 
\item Putting together the above estimates for $J_{1},J_{2},J_{3}$ and
$J_{4}$, we reach 
\begin{align}
x_{\|}^{t+1} & =x_{\|}^{t}+J_{1}-J_{3}+J_{2}-J_{4}\nonumber \\
 & =x_{\|}^{t}+3\eta\left[1-\big(x_{\parallel}^{t}\big)^{2}\right]x_{\|}^{t}-3\eta\left\Vert \bm{x}_{\perp}^{t}\right\Vert _{2}^{2}x_{||}^{t}+R_{1}\nonumber \\
 & =\left\{ 1+3\eta\left(1-\left\Vert \bm{x}^{t}\right\Vert _{2}^{2}\right)\right\} x_{\|}^{t}+R_{1},\label{eq:x-t+1-signal}
\end{align}
where $R_{1}$ is the residual term obeying 
\[
\left|R_{1}\right|\lesssim\eta\sqrt{\frac{n\log^{3}m}{m}}\big|x_{\|}^{t}\big|+\eta\sqrt{\frac{\log m}{m}}\left\Vert \bm{x}_{\perp}^{t}\right\Vert _{2}+\eta\left\Vert \bm{x}^{t}-\bm{x}^{t,\mathrm{sgn}}\right\Vert _{2}.
\]
Substituting the hypotheses (\ref{subeq:induction}) into (\ref{eq:x-t+1-signal})
and recalling that $\alpha_{t}=\langle\bm{x}^{t},\bm{x}^{\natural}\rangle$
lead us to conclude that 
\begin{align}
\alpha_{t+1} & =\left\{ 1+3\eta\left(1-\left\Vert \bm{x}^{t}\right\Vert _{2}^{2}\right)\right\} \alpha_{t}+O\left(\eta\sqrt{\frac{n\log^{3}m}{m}}\alpha_{t}\right)+O\left(\eta\sqrt{\frac{\log m}{m}}\beta_{t}\right)\nonumber \\
 & \quad+O\left(\eta\alpha_{t}\left(1+\frac{1}{\log m}\right)^{t}C_{3}\sqrt{\frac{n\log^{5}m}{m}}\right)\nonumber \\
 & =\left\{ 1+3\eta\left(1-\left\Vert \bm{x}^{t}\right\Vert _{2}^{2}\right)+\eta\zeta_{t}\right\} \alpha_{t},\label{eq:alpha-t-iterative}
\end{align}
for some $|\zeta_{t}|\ll\frac{1}{\log m}$, provided that \begin{subequations}
\begin{align}
\sqrt{\frac{n\log^{3}m}{m}} & \ll\frac{1}{\log m}\label{eq:alpha-t-as-long-as-1}\\
\sqrt{\frac{\log m}{m}}\beta_{t} & \ll\frac{1}{\log m}\alpha_{t}\label{eq:alpha-t-as-long-as-2}\\
\left(1+\frac{1}{\log m}\right)^{t}C_{3}\sqrt{\frac{n\log^{5}m}{m}} & \ll\frac{1}{\log m}.\label{eq:alpha-t-as-long-as-3}
\end{align}
\end{subequations}Here, the first condition (\ref{eq:alpha-t-as-long-as-1})
naturally holds under the sample complexity $m\gg n\log^{5}m$, whereas
the second condition (\ref{eq:alpha-t-as-long-as-2}) is true since
$\beta_{t}\leq\|\bm{x}^{t}\|_{2}\lesssim\alpha_{t}\sqrt{n\log m}$
(cf.~the induction hypothesis (\ref{eq:induction-norm-relative}))
and $m\gg n\log^{4}m$. For the last condition (\ref{eq:alpha-t-as-long-as-3}),
observe that for $t\leq T_{0}=O\left(\log n\right)$, 
\[
\left(1+\frac{1}{\log m}\right)^{t}=O\left(1\right),
\]
which further implies 
\[
\left(1+\frac{1}{\log m}\right)^{t}C_{3}\sqrt{\frac{n\log^{5}m}{m}}\lesssim C_{3}\sqrt{\frac{n\log^{5}m}{m}}\ll\frac{1}{\log m}
\]
as long as the number of samples obeys $m\gg n\log^{7}m$. This concludes
the proof. 
\end{itemize}

\subsection{Proof of (\ref{eq:approximate_state_evolution_phase_1_beta})\label{sec:Proof-of-Lemma-xt-perp}}

Given the gradient update rule (\ref{eq:gradient_update-WF}), the
orthogonal component $\bm{x}_{\perp}^{t+1}$ can be decomposed
as 
\begin{align}
\bm{x}_{\perp}^{t+1} & =\bm{x}_{\perp}^{t}-\frac{\eta}{m}\sum_{i=1}^{m}\left[\left(\bm{a}_{i}^{\top}\bm{x}^{t}\right)^{2}-\left(\bm{a}_{i}^{\top}\bm{x}^{\natural}\right)^{2}\right]\bm{a}_{i,\perp}\bm{a}_{i}^{\top}\bm{x}^{t}\nonumber \\
 & =\bm{x}_{\perp}^{t}+\underset{:=\bm{v}_{1}}{\underbrace{\frac{\eta}{m}\sum_{i=1}^{m}\left(\bm{a}_{i}^{\top}\bm{x}^{\natural}\right)^{2}\bm{a}_{i,\perp}\bm{a}_{i}^{\top}\bm{x}^{t}}}-\underset{:=\bm{v}_{2}}{\underbrace{\frac{\eta}{m}\sum_{i=1}^{m}\left(\bm{a}_{i}^{\top}\bm{x}^{t}\right)^{3}\bm{a}_{i,\perp}}}.\label{eq:defn-v1-v2}
\end{align}
In what follows, we bound $\bm{v}_{1}$ and $\bm{v}_{2}$ in turn. 
\begin{itemize}
\item We begin with $\bm{v}_{1}$. Using the identity $\bm{a}_{i}^{\top}\bm{x}^{t}=a_{i,1}x_{\parallel}^{t}+\bm{a}_{i,\perp}^{\top}\bm{x}_{\perp}^{t}$,
one can further decompose $\bm{v}_{1}$ into the following two terms:
\begin{align*}
\frac{1}{\eta}\bm{v}_{1} & =x_{\|}^{t}\cdot\frac{1}{m}\sum_{i=1}^{m}\left(\bm{a}_{i}^{\top}\bm{x}^{\natural}\right)^{2}a_{i,1}\bm{a}_{i,\perp}+\frac{1}{m}\sum_{i=1}^{m}\left(\bm{a}_{i}^{\top}\bm{x}^{\natural}\right)^{2}\bm{a}_{i,\perp}\bm{a}_{i,\perp}^{\top}\bm{x}_{\perp}^{t}\\
 & =x_{\|}^{t}\bm{u}+\bm{U}\bm{x}_{\perp}^{t},
\end{align*}
where $\bm{U}$, $\bm{u}$ are as defined, respectively, in (\ref{eq:U-defn})
and (\ref{eq:u-defn}). Recall that we have shown that
\begin{align*}
\left\Vert \bm{u}\right\Vert _{2} & \lesssim\sqrt{\frac{n\log^{3}m}{m}}\qquad\text{and}\qquad\left\Vert \bm{U}-\bm{I}_{n-1}\right\Vert \lesssim\sqrt{\frac{n\log^{3}m}{m}}
\end{align*}
hold with probability exceeding $1-O\left(m^{-10}\right)$. Consequently,
one has 
\begin{align}
\bm{v}_{1} & =\eta\bm{x}_{\perp}^{t}+\bm{r}_{1},\label{eq:orthogonal-v1}
\end{align}
where the residual term $\bm{r}_{1}$ obeys 
\begin{align}
\left\Vert \bm{r}_{1}\right\Vert _{2} & \lesssim\eta\sqrt{\frac{n\log^{3}m}{m}}\left\Vert \bm{x}_{\perp}^{t}\right\Vert _{2}+\eta\sqrt{\frac{n\log^{3}m}{m}}\big|x_{\|}^{t}\big|.\label{eq:orthogonal-v1-residual}
\end{align}
\item It remains to bound $\bm{v}_{2}$ in (\ref{eq:defn-v1-v2}). To this
end, we make note of the following fact 
\begin{align}
\frac{1}{m}\sum_{i=1}^{m}\left(\bm{a}_{i}^{\top}\bm{x}^{t}\right)^{3}\bm{a}_{i,\perp} & =\frac{1}{m}\sum_{i=1}^{m}\left(\bm{a}_{i,\perp}^{\top}\bm{x}_{\perp}^{t}\right)^{3}\bm{a}_{i,\perp}+\big(x_{\|}^{t}\big)^{3}\frac{1}{m}\sum_{i=1}^{m}a_{i,1}^{3}\bm{a}_{i,\perp}\nonumber \\
 & \quad+\frac{3x_{\|}^{t}}{m}\sum_{i=1}^{m}a_{i,1}\left(\bm{a}_{i,\perp}^{\top}\bm{x}_{\perp}^{t}\right)^{2}\bm{a}_{i,\perp}+3\big(x_{\|}^{t}\big)^{2}\frac{1}{m}\sum_{i=1}^{m}a_{i,1}^{2}\bm{a}_{i,\perp}\bm{a}_{i,\perp}^{\top}\bm{x}_{\perp}^{t}\nonumber \\
= & \frac{1}{m}\sum_{i=1}^{m}\left(\bm{a}_{i,\perp}^{\top}\bm{x}_{\perp}^{t}\right)^{3}\bm{a}_{i,\perp}+\frac{3x_{\|}^{t}}{m}\sum_{i=1}^{m}a_{i,1}\left(\bm{a}_{i,\perp}^{\top}\bm{x}_{\perp}^{t}\right)^{2}\bm{a}_{i,\perp}+\big(x_{\|}^{t}\big)^{3}\bm{u}+3\big(x_{\|}^{t}\big)^{2}\bm{U}\bm{x}_{\perp}^{t}.\label{eq:ax3a}
\end{align}
Applying Lemma \ref{lemma:hessian-concentration} and using the incoherence
condition (\ref{eq:consequence-incoherence-ai-perp-x-t}), we get
\[
\left\Vert \frac{1}{m}\sum_{i=1}^{m}\left(\bm{a}_{i,\perp}^{\top}\bm{x}_{\perp}^{t}\right)^{2}\bm{a}_{i,\perp}\bm{a}_{i,\perp}^{\top}-\left\Vert \bm{x}_{\perp}^{t}\right\Vert _{2}^{2}\bm{I}_{n-1}-2\bm{x}_{\perp}^{t}\bm{x}_{\perp}^{t\top}\right\Vert \lesssim\sqrt{\frac{n\log^{3}m}{m}}\left\Vert \bm{x}_{\perp}^{t}\right\Vert _{2}^{2},
\]
\[
\left\Vert \frac{1}{m}\sum_{i=1}^{m}\left(\bm{a}_{i}^{\top}\left[\begin{array}{c}
0\\
\bm{x}_{\perp}^{t}
\end{array}\right]\right)^{2}\bm{a}_{i}\bm{a}_{i}^{\top}-\left\Vert \bm{x}_{\perp}^{t}\right\Vert _{2}^{2}\bm{I}_{n}-2\left[\begin{array}{c}
0\\
\bm{x}_{\perp}^{t}
\end{array}\right]\left[\begin{array}{c}
0\\
\bm{x}_{\perp}^{t}
\end{array}\right]^{\top}\right\Vert \lesssim\sqrt{\frac{n\log^{3}m}{m}}\left\Vert \bm{x}_{\perp}^{t}\right\Vert _{2}^{2},
\]
as long as $m\gg n\log^{3}m$. These two together allow us to derive
\begin{align*}
\left\Vert \frac{1}{m}\sum_{i=1}^{m}\left(\bm{a}_{i,\perp}^{\top}\bm{x}_{\perp}^{t}\right)^{3}\bm{a}_{i,\perp}-3\left\Vert \bm{x}_{\perp}^{t}\right\Vert _{2}^{2}\bm{x}_{\perp}^{t}\right\Vert _{2} & =\left\Vert \left\{ \frac{1}{m}\sum_{i=1}^{m}\left(\bm{a}_{i,\perp}^{\top}\bm{x}_{\perp}^{t}\right)^{2}\bm{a}_{i,\perp}\bm{a}_{i,\perp}^{\top}-\left\Vert \bm{x}_{\perp}^{t}\right\Vert _{2}^{2}\bm{I}_{n-1}-2\bm{x}_{\perp}^{t}\bm{x}_{\perp}^{t\top}\right\} \bm{x}_{\perp}^{t}\right\Vert _{2}\\
 & \leq\left\Vert \frac{1}{m}\sum_{i=1}^{m}\left(\bm{a}_{i,\perp}^{\top}\bm{x}_{\perp}^{t}\right)^{2}\bm{a}_{i,\perp}\bm{a}_{i,\perp}^{\top}-\left\Vert \bm{x}_{\perp}^{t}\right\Vert _{2}^{2}\bm{I}_{n-1}-2\bm{x}_{\perp}^{t}\bm{x}_{\perp}^{t\top}\right\Vert \left\Vert \bm{x}_{\perp}^{t}\right\Vert _{2}\\
 & \lesssim\sqrt{\frac{n\log^{3}m}{m}}\left\Vert \bm{x}_{\perp}^{t}\right\Vert _{2}^{3};
\end{align*}
and
\begin{align*}
\left\Vert \frac{1}{m}\sum_{i=1}^{m}a_{i,1}\left(\bm{a}_{i,\perp}^{\top}\bm{x}_{\perp}^{t}\right)^{2}\bm{a}_{i,\perp}\right\Vert _{2} & \leq\underset{:=\bm{A}}{\Bigg\|\underbrace{\frac{1}{m}\sum_{i=1}^{m}\left(\bm{a}_{i}^{\top}\left[\begin{array}{c}
0\\
\bm{x}_{\perp}^{t}
\end{array}\right]\right)^{2}\bm{a}_{i}\bm{a}_{i}^{\top}-\left\Vert \bm{x}_{\perp}^{t}\right\Vert _{2}^{2}\bm{I}_{n}-2\left[\begin{array}{c}
0\\
\bm{x}_{\perp}^{t}
\end{array}\right]\left[\begin{array}{c}
0\\
\bm{x}_{\perp}^{t}
\end{array}\right]^{\top}}\Bigg\|}\\
 & \lesssim\sqrt{\frac{n\log^{3}m}{m}}\left\Vert \bm{x}_{\perp}^{t}\right\Vert _{2}^{2},
\end{align*}
where the second one follows since $\frac{1}{m}\sum_{i=1}^{m}a_{i,1}\big(\bm{a}_{i,\perp}^{\top}\bm{x}_{\perp}^{t}\big)^{2}\bm{a}_{i,\perp}$
is the first column of $\bm{A}$ except
for the first entry. Substitute the preceding bounds into (\ref{eq:ax3a})
to arrive at 
\begin{align*}
 & \left\Vert \frac{1}{m}\sum_{i=1}^{m}\big(\bm{a}_{i}^{\top}\bm{x}^{t}\big)^{3}\bm{a}_{i,\perp}-3\left\Vert \bm{x}_{\perp}^{t}\right\Vert _{2}^{2}\bm{x}_{\perp}^{t}-3\big(x_{\|}^{t}\big)^{2}\bm{x}_{\perp}^{t}\right\Vert _{2}\\
 & \quad\leq\left\Vert \frac{1}{m}\sum_{i=1}^{m}\left(\bm{a}_{i,\perp}^{\top}\bm{x}_{\perp}^{t}\right)^{3}\bm{a}_{i,\perp}-3\left\Vert \bm{x}_{\perp}^{t}\right\Vert _{2}^{2}\bm{x}_{\perp}^{t}\right\Vert _{2}+3\left|x_{\|}^{t}\right|\left\Vert \frac{1}{m}\sum_{i=1}^{m}a_{i,1}\left(\bm{a}_{i,\perp}^{\top}\bm{x}_{\perp}^{t}\right)^{2}\bm{a}_{i,\perp}\right\Vert _{2}\\
 & \quad\qquad+\left\Vert \big(x_{\|}^{t}\big)^{3}\bm{u}\right\Vert _{2}+3\big(x_{\|}^{t}\big)^{2}\left\Vert \left(\bm{U}-\bm{I}_{n-1}\right)\bm{x}_{\perp}^{t}\right\Vert _{2}\\
 & \quad\lesssim\sqrt{\frac{n\log^{3}m}{m}}\left(\left\Vert \bm{x}_{\perp}^{t}\right\Vert _{2}^{3}+\big|x_{\|}^{t}\big|\left\Vert \bm{x}_{\perp}^{t}\right\Vert _{2}^{2}+\big|x_{\|}^{t}\big|^{3}+\big|x_{\|}^{t}\big|^{2}\left\Vert \bm{x}_{\perp}^{t}\right\Vert _{2}\right)\lesssim\sqrt{\frac{n\log^{3}m}{m}}\left\Vert \bm{x}^{t}\right\Vert _{2}
\end{align*}
with probability at least $1-O(m^{-10})$. Here, the last relation
holds owing to the norm condition (\ref{eq:induction-norm-size})
and the fact that 
\[
\left\Vert \bm{x}_{\perp}^{t}\right\Vert _{2}^{3}+\big|x_{\|}^{t}\big|\left\Vert \bm{x}_{\perp}^{t}\right\Vert _{2}^{2}+\big|x_{\|}^{t}\big|^{3}+\big|x_{\|}^{t}\big|^{2}\left\Vert \bm{x}_{\perp}^{t}\right\Vert _{2}\asymp\left\Vert \bm{x}^{t}\right\Vert _{2}^{3}\lesssim\left\Vert \bm{x}^{t}\right\Vert _{2}.
\]
This in turn tells us that 
\[
\bm{v}_{2}=\frac{\eta}{m}\sum_{i=1}^{m}\big(\bm{a}_{i}^{\top}\bm{x}^{t}\big)^{3}\bm{a}_{i,\perp}=3\eta\left\Vert \bm{x}_{\perp}^{t}\right\Vert _{2}^{2}\bm{x}_{\perp}^{t}+3\eta\big(x_{\|}^{t}\big)^{2}\bm{x}_{\perp}^{t}+\bm{r}_{2}=3\eta\left\Vert \bm{x}^{t}\right\Vert _{2}^{2}\bm{x}_{\perp}^{t}+\bm{r}_{2},
\]
where the residual term $\bm{r}_{2}$ is bounded by 
\[
\left\Vert \bm{r}_{2}\right\Vert _{2}\lesssim\eta\sqrt{\frac{n\log^{3}m}{m}}\left\Vert \bm{x}^{t}\right\Vert _{2}.
\]
\item Putting the above estimates on $\bm{v}_{1}$ and $\bm{v}_{2}$ together,
we conclude that 
\begin{align*}
\bm{x}_{\perp}^{t+1} & =\bm{x}_{\perp}^{t}+\bm{v}_{1}-\bm{v}_{2}=\left\{ 1+\eta\left(1-3\left\Vert \bm{x}^{t}\right\Vert _{2}^{2}\right)\right\} \bm{x}_{\perp}^{t}+\bm{r}_{3},
\end{align*}
where $\bm{r}_{3}=\bm{r}_{1}-\bm{r}_{2}$ satisfies 
\[
\left\Vert \bm{r}_{3}\right\Vert _{2}\lesssim\eta\sqrt{\frac{n\log^{3}m}{m}}\left\Vert \bm{x}^{t}\right\Vert _{2}.
\]
Plug in the definitions of $\alpha_{t}$ and $\beta_{t}$ to realize that 
\begin{align*}
\beta_{t+1} & =\left\{ 1+\eta\left(1-3\left\Vert \bm{x}^{t}\right\Vert _{2}^{2}\right)\right\} \beta_{t}+O\left(\eta\sqrt{\frac{n\log^{3}m}{m}}\left(\alpha_{t}+\beta_{t}\right)\right)\\
 & =\left\{ 1+\eta\left(1-3\left\Vert \bm{x}^{t}\right\Vert _{2}^{2}\right)+\eta\rho_{t}\right\} \beta_{t},
\end{align*}
for some $|\rho_{t}|\ll\frac{1}{\log m}$, with the proviso that $m\gg n\log^{5}m$
and
\begin{equation}
\sqrt{\frac{n\log^{3}m}{m}}\alpha_{t}\ll\frac{1}{\log m}\beta_{t}.\label{eq:beta-t-as-long-as-1}
\end{equation}
The last condition holds true since 
\[
\sqrt{\frac{n\log^{3}m}{m}}\alpha_{t}\lesssim\sqrt{\frac{n\log^{3}m}{m}}\frac{1}{\log^{5}m}\ll\frac{1}{\log m}\ll\frac{1}{\log m}\beta_{t},
\]
where we have used the assumption $\alpha_{t}\lesssim\frac{1}{\log^{5}m}$
(see definition of $T_{0}$), the sample size condition $m\gg n\log^{11}m$
and the induction hypothesis $\beta_{t}\geq c_{5}$ (see (\ref{eq:induction-norm-size})).
This finishes the proof. 
\end{itemize}

\section{Proof of Lemma \ref{lemma:xt-xt-l}\label{sec:Proof-of-Lemma-xt-xt-l}}

It follows from the gradient update rules (\ref{eq:gradient_update-WF})
and (\ref{eq:gradient-update-leave-WF}) that 
\begin{align}
\bm{x}^{t+1}-\bm{x}^{t+1,(l)} & =\bm{x}^{t}-\eta\nabla f\left(\bm{x}^{t}\right)-\left(\bm{x}^{t,(l)}-\eta\nabla f^{(l)}\big(\bm{x}^{t,(l)}\big)\right)\nonumber \\
 & =\bm{x}^{t}-\eta\nabla f\left(\bm{x}^{t}\right)-\left(\bm{x}^{t,(l)}-\eta\nabla f\big(\bm{x}^{t,(l)}\big)\right)+\eta\nabla f^{(l)}\big(\bm{x}^{t,(l)}\big)-\eta\nabla f\big(\bm{x}^{t,(l)}\big)\nonumber \\
 & =\left[ \bm{I}_{n}-\eta\int_{0}^{1}\nabla^{2}f\left(\bm{x}\left(\tau\right)\right)\mathrm{d}\tau\right]
  \big(\bm{x}^{t}-\bm{x}^{t,(l)}\big)-\frac{\eta}{m}\left[\big(\bm{a}_{l}^{\top}\bm{x}^{t,(l)}\big)^{2}-\big(\bm{a}_{l}^{\top}\bm{x}^{\natural}\big)^{2}\right]\bm{a}_{l}\bm{a}_{l}^{\top}\bm{x}^{t,(l)},\label{eq:xt-xt-l}
\end{align}
where we denote $\bm{x}\left(\tau\right):=\bm{x}^{t}+\tau\left(\bm{x}^{t,(l)}-\bm{x}^{t}\right)$.
Here, the last identity is due to the fundamental theorem of calculus
\cite[Chapter XIII, Theorem 4.2]{lang1993real}. 
\begin{itemize}
\item Controlling the first term in (\ref{eq:xt-xt-l}) requires exploring
the properties of the Hessian $\nabla^{2}f\left(\bm{x}\right)$. Since
$\bm{x}\left(\tau\right)$ lies between $\bm{x}^{t}$ and $\bm{x}^{t,\left(l\right)}$
for any $0\leq\tau\leq1$, we have the following two consequences
\begin{equation}
\left\Vert \bm{x}_{\perp}\left(\tau\right)\right\Vert _{2}\leq\left\Vert \bm{x}\left(\tau\right)\right\Vert _{2}\leq2C_{5}\qquad\text{and}\qquad\max_{1\leq i\leq m}\left|\bm{a}_{i}^{\top}\bm{x}\left(\tau\right)\right|\lesssim\sqrt{\log m}\lesssim\sqrt{\log m}\left\Vert \bm{x}\left(\tau\right)\right\Vert _{2}.\label{eq:x-tau-implication}
\end{equation}
To see the left statement in (\ref{eq:x-tau-implication}), one has 
\[
\left\Vert \bm{x}\left(\tau\right)\right\Vert _{2}\leq\max\{ \Vert \bm{x}^{t}\Vert _{2},\Vert \bm{x}^{t,(l)}\Vert _{2}\} \leq2C_{5},
\]
where the last inequality follows from (\ref{eq:induction-norm-size}) and (\ref{eq:consequence-norm-xt-l}). Moreover, for
the right statement in (\ref{eq:x-tau-implication}), one can see 
\begin{align*}
\max_{1\leq i\leq m}\left|\bm{a}_{i}^{\top}\bm{x}\left(\tau\right)\right| & =\max_{1\leq i\leq m}\left|\left(1-\tau\right)\bm{a}_{i}^{\top}\bm{x}^{t}+\tau\bm{a}_{i}^{\top}\bm{x}^{t,\left(l\right)}\right|\\
 & \leq\max_{1\leq i\leq m}\left|\left(1-\tau\right)\bm{a}_{i}^{\top}\bm{x}^{t}+\tau\bm{a}_{i}^{\top}\left(\bm{x}^{t,\left(l\right)}-\bm{x}^{t,\left(i\right)}\right)+\tau\bm{a}_{i}^{\top}\bm{x}^{t,\left(i\right)}\right|\\
 & \leq\left(1-\tau\right)\max_{1\leq i\leq m}\left|\bm{a}_{i}^{\top}\bm{x}^{t}\right|+\tau\max_{1\leq i\leq m}\left|\bm{a}_{i}^{\top}\left(\bm{x}^{t,\left(l\right)}-\bm{x}^{t,\left(i\right)}\right)\right|+\tau\max_{1\leq i\leq m}\left|\bm{a}_{i}^{\top}\bm{x}^{t,\left(i\right)}\right|.
\end{align*}
In view of (\ref{eq:consequence-incoherence-original}), we have
\[
\max_{1\leq i\leq m}\left|\bm{a}_{i}^{\top}\bm{x}^{t}\right|\lesssim\log m.
\]
Furthermore, due to the independence between $\bm{a}_{i}$ and $\bm{x}^{t,\left(i\right)}$,
one can apply standard Gaussian concentration inequalities to show
that with high probability 
\[
\max_{1\leq i\leq m}\left|\bm{a}_{i}^{\top}\bm{x}^{t,\left(i\right)}\right|\lesssim\sqrt{\log m}.
\]
We are left with the middle term, which can be controlled using Cauchy-Schwarz
as follows: 
\begin{align*}
\max_{1\leq i\leq m}\left|\bm{a}_{i}^{\top}\left(\bm{x}^{t,\left(l\right)}-\bm{x}^{t,\left(i\right)}\right)\right| & \leq\max_{i}\left\Vert \bm{a}_{i}\right\Vert _{2}\max_{1\leq i\leq m}\left\Vert \bm{x}^{t,\left(l\right)}-\bm{x}^{t,\left(i\right)}\right\Vert _{2}\\
 & \overset{\text{(i)}}{\lesssim}\sqrt{n}\cdot\max_{1\leq i\leq m}\left(\left\Vert \bm{x}^{t,\left(l\right)}-\bm{x}^{t}\right\Vert _{2}+\left\Vert \bm{x}^{t}-\bm{x}^{t,\left(i\right)}\right\Vert _{2}\right)\\
 & \lesssim\sqrt{n}\cdot\max_{1\leq i\leq m}\left\Vert \bm{x}^{t}-\bm{x}^{t,\left(i\right)}\right\Vert _{2}\\
 & \overset{\text{(ii)}}{\lesssim}\sqrt{n}\cdot\beta_{t}\left(1+\frac{1}{\log m}\right)^{t}C_{1}\eta\frac{\sqrt{n\log^{5}m}}{m}\\
 & \overset{\text{(iii)}}{\lesssim}\sqrt{\log m}.
\end{align*}
Here, the inequality (i) arises from the concentration of norm of
Gaussian vectors and the triangle inequality; the relation (ii) holds
because of the induction hypothesis (\ref{eq:induction-xt-l}) and the last inequality
(iii) holds true under the sample size condition $m\gg n\log^{2}m$. 

In addition, combining (\ref{eq:induction-norm-size}) and (\ref{eq:consequence-upper-bounds})
leads to
\begin{equation}
\left\Vert \bm{x}_{\perp}\left(\tau\right)\right\Vert _{2}\geq\left\Vert \bm{x}_{\perp}^{t}\right\Vert _{2}-\big\|\bm{x}^{t}-\bm{x}^{t,(l)}\big\|_{2}\geq c_{5}-\log^{-1}m\geq c_{5}/4.\label{eq:x-tau-implication-LB}
\end{equation}
Armed with these bounds, we can readily apply Lemma \ref{lemma:Hessian-UB-Stage1}
to obtain 
\begin{align*}
 & \left\Vert \bm{I}_{n}-\eta\nabla^{2}f\left(\bm{x}\left(\tau\right)\right)-\left\{ \left(1-3\eta\left\Vert \bm{x}\left(\tau\right)\right\Vert _{2}^{2}+\eta\right)\bm{I}_{n}+2\eta\bm{x}^{\natural}\bm{x}^{\natural\top}-6\eta\bm{x}\left(\tau\right)\bm{x}\left(\tau\right)^{\top}\right\} \right\Vert \\
 & \quad\lesssim\eta\sqrt{\frac{n\log^{3}m}{m}}\max\left\{ \left\Vert \bm{x}(\tau)\right\Vert _{2}^{2},1\right\} \lesssim\eta\sqrt{\frac{n\log^{3}m}{m}}.
\end{align*}
This further allows one to derive 
\begin{align*}
 & \left\Vert \left\{ \bm{I}_{n}-\eta\nabla^{2}f\left(\bm{x}\left(\tau\right)\right)\right\} \big(\bm{x}^{t}-\bm{x}^{t,(l)}\big)\right\Vert _{2}\\
 & \leq\left\Vert \left\{ \left(1-3\eta\left\Vert \bm{x}\left(\tau\right)\right\Vert _{2}^{2}+\eta\right)\bm{I}_{n}+2\eta\bm{x}^{\natural}\bm{x}^{\natural\top}-6\eta\bm{x}\left(\tau\right)\bm{x}\left(\tau\right)^{\top}\right\} \big(\bm{x}^{t}-\bm{x}^{t,(l)}\big)\right\Vert _{2}+O\left(\eta\sqrt{\frac{n\log^{3}m}{m}}\big\|\bm{x}^{t}-\bm{x}^{t,(l)}\big\|_{2}\right).
\end{align*}
Moreover, we can apply the triangle inequality to get 
\begin{align*}
 & \left\Vert \left\{ \left(1-3\eta\left\Vert \bm{x}\left(\tau\right)\right\Vert _{2}^{2}+\eta\right)\bm{I}_{n}+2\eta\bm{x}^{\natural}\bm{x}^{\natural\top}-6\eta\bm{x}\left(\tau\right)\bm{x}\left(\tau\right)^{\top}\right\} \big(\bm{x}^{t}-\bm{x}^{t,(l)}\big)\right\Vert _{2}\\
 & \quad\leq\left\Vert \left\{ \left(1-3\eta\left\Vert \bm{x}\left(\tau\right)\right\Vert _{2}^{2}+\eta\right)\bm{I}_{n}-6\eta\bm{x}\left(\tau\right)\bm{x}\left(\tau\right)^{\top}\right\} \big(\bm{x}^{t}-\bm{x}^{t,(l)}\big)\right\Vert _{2}+\left\Vert 2\eta\bm{x}^{\natural}\bm{x}^{\natural\top}\big(\bm{x}^{t}-\bm{x}^{t,(l)}\big)\right\Vert _{2}\\
 & \quad\overset{\left(\text{i}\right)}{=}\left\Vert \left\{ \left(1-3\eta\left\Vert \bm{x}\left(\tau\right)\right\Vert _{2}^{2}+\eta\right)\bm{I}_{n}-6\eta\bm{x}\left(\tau\right)\bm{x}\left(\tau\right)^{\top}\right\} \big(\bm{x}^{t}-\bm{x}^{t,(l)}\big)\right\Vert _{2}+2\eta\big|x_{\parallel}^{t}-x_{\parallel}^{t,\left(l\right)}\big|\\
 & \quad\overset{\left(\text{ii}\right)}{\leq}\left(1-3\eta\left\Vert \bm{x}\left(\tau\right)\right\Vert _{2}^{2}+\eta\right)\big\|\bm{x}^{t}-\bm{x}^{t,(l)}\big\|_{2}+2\eta\big|x_{\parallel}^{t}-x_{\parallel}^{t,\left(l\right)}\big|,
\end{align*}
where (i) holds since $\bm{x}^{\natural\top}\left(\bm{x}^{t}-\bm{x}^{t,(l)}\right)=x_{\parallel}^{t}-x_{\parallel}^{t,\left(l\right)}$
(recall that $\bm{x}^{\natural}=\bm{e}_{1}$) and (ii) follows from
the fact that 
\[
\left(1-3\eta\left\Vert \bm{x}\left(\tau\right)\right\Vert _{2}^{2}+\eta\right)\bm{I}_{n}-6\eta\bm{x}\left(\tau\right)\bm{x}\left(\tau\right)^{\top}\succeq\bm{0},
\]
as long as $\eta\leq1/\left(18C_{5}\right)$. This further reveals
\begin{align}
 & \left\Vert \left\{ \bm{I}_{n}-\eta\nabla^{2}f\left(\bm{x}\left(\tau\right)\right)\right\} \big(\bm{x}^{t}-\bm{x}^{t,(l)}\big)\right\Vert _{2}\nonumber \\
 & \quad\leq\left\{ 1+\eta\left(1-3\left\Vert \bm{x}\left(\tau\right)\right\Vert _{2}^{2}\right)+O\left(\eta\sqrt{\frac{n\log^{3}m}{m}}\right)\right\} \big\|\bm{x}^{t}-\bm{x}^{t,(l)}\big\|_{2}+2\eta\big|x_{\parallel}^{t}-x_{\parallel}^{t,\left(l\right)}\big|\nonumber \\
 & \quad\overset{\left(\text{i}\right)}{\leq}\left\{ 1+\eta\left(1-3\left\Vert \bm{x}^{t}\right\Vert _{2}^{2}\right)+O\left(\eta\big\|\bm{x}^{t}-\bm{x}^{t,(l)}\big\|_{2}\right)+O\left(\eta\sqrt{\frac{n\log^{3}m}{m}}\right)\right\} \big\|\bm{x}^{t}-\bm{x}^{t,(l)}\big\|_{2}+2\eta\big|x_{\parallel}^{t}-x_{\parallel}^{t,\left(l\right)}\big|\nonumber \\
 & \quad\overset{\left(\text{ii}\right)}{\leq}\left\{ 1+\eta\left(1-3\left\Vert \bm{x}^{t}\right\Vert _{2}^{2}\right)+\eta\phi_{1}\right\} \big\|\bm{x}^{t}-\bm{x}^{t,(l)}\big\|_{2}+2\eta\big|x_{\parallel}^{t}-x_{\parallel}^{t,\left(l\right)}\big|,\label{eq:xt-xt-l-signal-main}
\end{align}
for some $|\phi_{1}|\ll\frac{1}{\log m}$, where (i) holds since for
every $0\leq\tau\leq1$ 
\begin{align}
\left\Vert \bm{x}\left(\tau\right)\right\Vert _{2}^{2} & \geq\left\Vert \bm{x}^{t}\right\Vert _{2}^{2}-\left|\left\Vert \bm{x}\left(\tau\right)\right\Vert _{2}^{2}-\left\Vert \bm{x}^{t}\right\Vert _{2}^{2}\right|\nonumber \\
 & \geq\left\Vert \bm{x}^{t}\right\Vert _{2}^{2}-\left\Vert \bm{x}\left(\tau\right)-\bm{x}^{t}\right\Vert _{2}\left(\left\Vert \bm{x}\left(\tau\right)\right\Vert _{2}+\left\Vert \bm{x}^{t}\right\Vert _{2}\right)\nonumber \\
 & \geq\left\Vert \bm{x}^{t}\right\Vert _{2}^{2}-O\left(\big\|\bm{x}^{t}-\bm{x}^{t,(l)}\big\|_{2}\right),\label{eq:lower-bound-xt-xt-l}
\end{align}
and (ii) comes from the fact (\ref{eq:1-logm-xt-xt-l}) and the sample
complexity assumption $m\gg n\log^{5}m$. 
\item We then move on to the second term of (\ref{eq:xt-xt-l}). Observing
that $\bm{x}^{t,(l)}$ is statistically independent of $\bm{a}_{l}$,
we have 
\begin{align}
\left\Vert \frac{1}{m}\left[\big(\bm{a}_{l}^{\top}\bm{x}^{t,(l)}\big)^{2}-\big(\bm{a}_{l}^{\top}\bm{x}^{\natural}\big)^{2}\right]\bm{a}_{l}\bm{a}_{l}^{\top}\bm{x}^{t,(l)}\right\Vert _{2} & \leq\frac{1}{m}\left[\big(\bm{a}_{l}^{\top}\bm{x}^{t,(l)}\big)^{2}+\big(\bm{a}_{l}^{\top}\bm{x}^{\natural}\big)^{2}\right]\left|\bm{a}_{l}^{\top}\bm{x}^{t,(l)}\right|\left\Vert \bm{a}_{l}\right\Vert _{2}\nonumber \\
 & \lesssim\frac{1}{m}\cdot\log m\cdot\sqrt{\log m}\left\Vert \bm{x}^{t,(l)}\right\Vert _{2}\cdot\sqrt{n}\nonumber \\
 & \asymp\frac{\sqrt{n\log^{3}m}}{m}\left\Vert \bm{x}^{t,(l)}\right\Vert _{2},\label{eq:xt-xt-l-signal-remainder}
\end{align}
where the second inequality makes use of the facts (\ref{eq:max-a-i-1}),
(\ref{eq:max-a-i-norm}) and the standard concentration results 
\[
\left|\bm{a}_{l}^{\top}\bm{x}^{t,(l)}\right|\lesssim\sqrt{\log m}\big\|\bm{x}^{t,(l)}\big\|_{2}\lesssim\sqrt{\log m}.
\]
\item Combine the previous two bounds (\ref{eq:xt-xt-l-signal-main}) and
(\ref{eq:xt-xt-l-signal-remainder}) to reach
\begin{align*}
 & \big\|\bm{x}^{t+1}-\bm{x}^{t+1,(l)}\big\|_{2}\\
 & \quad\leq\left\Vert \left\{ \bm{I}-\eta\int_{0}^{1}\nabla^{2}f\left(\bm{x}(\tau)\right)\mathrm{d}\tau\right\} \big(\bm{x}^{t}-\bm{x}^{t,(l)}\big)\right\Vert _{2}+\eta\left\Vert \frac{1}{m}\left[\left(\bm{a}_{l}^{\top}\bm{x}^{t,(l)}\right)^{2}-\left(\bm{a}_{l}^{\top}\bm{x}^{\natural}\right)^{2}\right]\bm{a}_{l}\bm{a}_{l}^{\top}\bm{x}^{t,(l)}\right\Vert _{2}\\
 & \quad\leq\left\{ 1+\eta\left(1-3\left\Vert \bm{x}^{t}\right\Vert _{2}^{2}\right)+\eta\phi_{1}\right\} \big\|\bm{x}^{t}-\bm{x}^{t,(l)}\big\|_{2}+2\eta\left|x_{\parallel}^{t}-x_{\parallel}^{t,\left(l\right)}\right|+O\left(\frac{\eta\sqrt{n\log^{3}m}}{m}\left\Vert \bm{x}^{t,(l)}\right\Vert _{2}\right)\\
 & \quad\leq\left\{ 1+\eta\left(1-3\left\Vert \bm{x}^{t}\right\Vert _{2}^{2}\right)+\eta\phi_{1}\right\} \big\|\bm{x}^{t}-\bm{x}^{t,(l)}\big\|_{2}+O\left(\frac{\eta\sqrt{n\log^{3}m}}{m}\right)\left\Vert \bm{x}^{t}\right\Vert _{2}+2\eta\big|x_{\parallel}^{t}-x_{\parallel}^{t,\left(l\right)}\big|.
\end{align*}
Here the last relation holds because of the triangle inequality 
\[
\big\|\bm{x}^{t,(l)}\big\|_{2}\leq\left\Vert \bm{x}^{t}\right\Vert _{2}+\big\|\bm{x}^{t}-\bm{x}^{t,(l)}\big\|_{2}
\]
and the fact that $\frac{\sqrt{n\log^{3}m}}{m}\ll\frac{1}{\log m}$.

In view of the inductive hypotheses (\ref{subeq:induction}), one
has 
\begin{align*}
\big\|\bm{x}^{t+1}-\bm{x}^{t+1,(l)}\big\|_{2} & \overset{\left(\text{i}\right)}{\leq}\left\{ 1+\eta\left(1-3\left\Vert \bm{x}^{t}\right\Vert _{2}^{2}\right)+\eta\phi_{1}\right\} \beta_{t}\left(1+\frac{1}{\log m}\right)^{t}C_{1}\frac{\sqrt{n\log^{5}m}}{m}\\
 & \quad\quad+O\left(\frac{\eta\sqrt{n\log^{3}m}}{m}\right)\left(\alpha_{t}+\beta_{t}\right)+2\eta\alpha_{t}\left(1+\frac{1}{\log m}\right)^{t}C_{2}\frac{\sqrt{n\log^{12}m}}{m}\\
 & \quad\overset{\left(\text{ii}\right)}{\leq}\left\{ 1+\eta\left(1-3\left\Vert \bm{x}^{t}\right\Vert _{2}^{2}\right)+\eta\phi_{2}\right\} \beta_{t}\left(1+\frac{1}{\log m}\right)^{t}C_{1}\frac{\sqrt{n\log^{5}m}}{m}\\
 & \quad\overset{\left(\text{iii}\right)}{\leq}\beta_{t+1}\left(1+\frac{1}{\log m}\right)^{t+1}C_{1}\frac{\sqrt{n\log^{5}m}}{m},
\end{align*}
for some $|\phi_{2}|\ll\frac{1}{\log m}$, where the inequality (i)
uses $\|\bm{x}^{t}\|_{2}\leq|x_{\|}^{t}|+\|\bm{x}_{\perp}^{t}\|_{2}=\alpha_{t}+\beta_{t}$,
the inequality (ii) holds true as long as \begin{subequations} 
\begin{align}
\frac{\sqrt{n\log^{3}m}}{m}\left(\alpha_{t}+\beta_{t}\right) & \ll\frac{1}{\log m}\beta_{t}\left(1+\frac{1}{\log m}\right)^{t}C_{1}\frac{\sqrt{n\log^{5}m}}{m},\label{eq:xt-xt-l-norm-as-long-as-1}\\
\alpha_{t}C_{2}\frac{\sqrt{n\log^{12}m}}{m} & \ll\frac{1}{\log m}\beta_{t}C_{1}\frac{\sqrt{n\log^{5}m}}{m}.\label{eq:xt-xt-l-norm-as-long-as-2}
\end{align}
\end{subequations}Here, the first condition (\ref{eq:xt-xt-l-norm-as-long-as-1})
comes from the fact that for $t<T_{0}$, 
\[
\frac{\sqrt{n\log^{3}m}}{m}\left(\alpha_{t}+\beta_{t}\right)\asymp\frac{\sqrt{n\log^{3}m}}{m}\beta_{t}\ll C_{1}\beta_{t}\frac{\sqrt{n\log^{3}m}}{m},
\]
as long as $C_{1}>0$ is sufficiently large. The other one (\ref{eq:xt-xt-l-norm-as-long-as-2})
is valid owing to the assumption of Phase I $\alpha_{t}\ll1/\log^{5}m.$
Regarding the inequality (iii) above, it is easy to check that for
some $|\phi_{3}|\ll\frac{1}{\log m}$,
\begin{align}
\left\{ 1+\eta\left(1-3\left\Vert \bm{x}^{t}\right\Vert _{2}^{2}\right)+\eta\phi_{2}\right\} \beta_{t} & =\left\{ \frac{\beta_{t+1}}{\beta_{t}}+\eta\phi_{3}\right\} \beta_{t}\nonumber \\
 & =\left\{ \frac{\beta_{t+1}}{\beta_{t}}+\eta O\left(\frac{\beta_{t+1}}{\beta_{t}}\phi_{3}\right)\right\} \beta_{t}\nonumber \\
 & \leq\beta_{t+1}\left(1+\frac{1}{\log m}\right),\label{eq:beta-t-iterative-reasoning}
\end{align}
where the second equality holds since $\frac{\beta_{t+1}}{\beta_{t}}\asymp1$
in Phase I. 
\end{itemize}
The proof is completed by applying the union bound over all $1\leq l\leq m$. 

\section{Proof of Lemma \ref{lemma:xt-xt-l-signal}\label{sec:Proof-of-Lemma-xt-xt-l-signal}}

Use (\ref{eq:xt-xt-l}) once again to deduce 
\begin{align}
 & x_{\parallel}^{t+1}-x_{\parallel}^{t+1,\left(l\right)}=\bm{e}_{1}^{\top}\big(\bm{x}^{t+1}-\bm{x}^{t+1,\left(l\right)}\big)\nonumber \\
 & \quad=\bm{e}_{1}^{\top}\left\{ \bm{I}_{n}-\eta\int_{0}^{1}\nabla^{2}f\left(\bm{x}\left(\tau\right)\right)\mathrm{d}\tau\right\} \big(\bm{x}^{t}-\bm{x}^{t,(l)}\big)-\frac{\eta}{m}\left[\big(\bm{a}_{l}^{\top}\bm{x}^{t,(l)}\big)^{2}-\big(\bm{a}_{l}^{\top}\bm{x}^{\natural}\big)^{2}\right]\bm{e}_{1}^{\top}\bm{a}_{l}\bm{a}_{l}^{\top}\bm{x}^{t,(l)}\nonumber \\
 & \quad=\left[x_{\parallel}^{t}-x_{\parallel}^{t,\left(l\right)}-\eta\int_{0}^{1}\bm{e}_{1}^{\top}\nabla^{2}f\left(\bm{x}\left(\tau\right)\right)\mathrm{d}\tau\big(\bm{x}^{t}-\bm{x}^{t,(l)}\big)\right]-\frac{\eta}{m}\left[\big(\bm{a}_{l}^{\top}\bm{x}^{t,(l)}\big)^{2}-\big(\bm{a}_{l}^{\top}\bm{x}^{\natural}\big)^{2}\right]a_{l,1}\bm{a}_{l}^{\top}\bm{x}^{t,(l)},\label{eq:xt-xt-l-signal}
\end{align}
where we recall that $\bm{x}\left(\tau\right):=\bm{x}^{t}+\tau\left(\bm{x}^{t,\left(l\right)}-\bm{x}^{t}\right)$.

We begin by controlling the second term of (\ref{eq:xt-xt-l-signal}).
Applying similar arguments as in (\ref{eq:xt-xt-l-signal-remainder})
yields
\[
\left|\frac{1}{m}\left[\big(\bm{a}_{l}^{\top}\bm{x}^{t,(l)}\big)^{2}-\left(\bm{a}_{l}^{\top}\bm{x}^{\natural}\right)^{2}\right]a_{l,1}\bm{a}_{l}^{\top}\bm{x}^{t,(l)}\right|\lesssim\frac{\log^{2}m}{m}\left\Vert \bm{x}^{t,(l)}\right\Vert _{2}
\]
with probability at least $1-O\left(m^{-10}\right)$.

Regarding the first term in (\ref{eq:xt-xt-l-signal}), one can use
the decomposition 
\[
\bm{a}_{i}^{\top}\big(\bm{x}^{t}-\bm{x}^{t,(l)}\big)=a_{i,1}\big(x_{\parallel}^{t}-x_{\parallel}^{t,\left(l\right)}\big)+\bm{a}_{i,\perp}^{\top}\big(\bm{x}_{\perp}^{t}-\bm{x}_{\perp}^{t,(l)}\big)
\]
to obtain that
\begin{align*}
\bm{e}_{1}^{\top}\nabla^{2}f\left(\bm{x}\left(\tau\right)\right)\big(\bm{x}^{t}-\bm{x}^{t,\left(l\right)}\big) & =\frac{1}{m}\sum_{i=1}^{m}\left[3\big(\bm{a}_{i}^{\top}\bm{x}\left(\tau\right)\big)^{2}-\big(\bm{a}_{i}^{\top}\bm{x}^{\natural}\big)^{2}\right]a_{i,1}\bm{a}_{i}^{\top}\big(\bm{x}^{t}-\bm{x}^{t,(l)}\big)\\
 & =\underbrace{\frac{1}{m}\sum_{i=1}^{m}\left[3\big(\bm{a}_{i}^{\top}\bm{x}\left(\tau\right)\big)^{2}-\big(\bm{a}_{i}^{\top}\bm{x}^{\natural}\big)^{2}\right]a_{i,1}^{2}\big(x_{\parallel}^{t}-x_{\parallel}^{t,\left(l\right)}\big)}_{:=\omega_{1}\left(\tau\right)}\\
 & \quad+\underbrace{\frac{1}{m}\sum_{i=1}^{m}\left[3\big(\bm{a}_{i}^{\top}\bm{x}\left(\tau\right)\big)^{2}-\big(\bm{a}_{i}^{\top}\bm{x}^{\natural}\big)^{2}\right]a_{i,1}\bm{a}_{i,\perp}^{\top}\big(\bm{x}_{\perp}^{t}-\bm{x}_{\perp}^{t,\left(l\right)}\big)}_{:=\omega_{2}\left(\tau\right)}.
\end{align*}
In the sequel, we shall bound $\omega_{1}\left(\tau\right)$ and $\omega_{2}\left(\tau\right)$
separately. 
\begin{itemize}
\item For $\omega_{1}\left(\tau\right)$, Lemma \ref{lemma:hessian-concentration}
together with the facts (\ref{eq:x-tau-implication}) tells us that
\begin{align*}
 & \left|\frac{1}{m}\sum_{i=1}^{m}\left[3\left(\bm{a}_{i}^{\top}\bm{x}\left(\tau\right)\right)^{2}-\left(\bm{a}_{i}^{\top}\bm{x}^{\natural}\right)^{2}\right]a_{i,1}^{2}-\left[3\left\Vert \bm{x}\left(\tau\right)\right\Vert _{2}^{2}+6\big|x_{\parallel}\left(\tau\right)\big|^{2}-3\right]\right|\\
 & \quad\lesssim\sqrt{\frac{n\log^{3}m}{m}}\max\left\{ \left\Vert \bm{x}\left(\tau\right)\right\Vert _{2}^{2},1\right\} \lesssim\sqrt{\frac{n\log^{3}m}{m}},
\end{align*}
which further implies that 
\[
\omega_{1}\left(\tau\right)=\left(3\left\Vert \bm{x}\left(\tau\right)\right\Vert _{2}^{2}+6\big|x_{\parallel}\left(\tau\right)\big|^{2}-3\right)\big(x_{\parallel}^{t}-x_{\parallel}^{t,\left(l\right)}\big)+r_{1}
\]
with the residual term $r_{1}$ obeying 
\begin{align*}
\left|r_{1}\right| & =O\left(\sqrt{\frac{n\log^{3}m}{m}}\big|x_{\parallel}^{t}-x_{\parallel}^{t,\left(l\right)}\big|\right).
\end{align*}
\item We proceed to bound $\omega_{2}\left(\tau\right)$. Decompose $w_{2}\left(\tau\right)$
into the following: 
\[
\omega_{2}\left(\tau\right)=\underbrace{\frac{3}{m}\sum_{i=1}^{m}\left(\bm{a}_{i}^{\top}\bm{x}\left(\tau\right)\right)^{2}a_{i,1}\bm{a}_{i,\perp}^{\top}\big(\bm{x}_{\perp}^{t}-\bm{x}_{\perp}^{t,\left(l\right)}\big)}_{:=\omega_{3}\left(\tau\right)}-\underbrace{\frac{1}{m}\sum_{i=1}^{m}\left(\bm{a}_{i}^{\top}\bm{x}^{\natural}\right)^{2}a_{i,1}\bm{a}_{i,\perp}^{\top}\big(\bm{x}_{\perp}^{t}-\bm{x}_{\perp}^{t,\left(l\right)}\big)}_{:=\omega_{4}}.
\]
\begin{itemize}
\item The term $\omega_{4}$ is relatively simple to control. Recognizing
$\left(\bm{a}_{i}^{\top}\bm{x}^{\natural}\right)^{2}=a_{i,1}^{2}$
and $a_{i,1}=\xi_{i}\left|a_{i,1}\right|$, one has 
\begin{align*}
\omega_{4} & =\frac{1}{m}\sum_{i=1}^{m}\xi_{i}\left|a_{i,1}\right|^{3}\bm{a}_{i,\perp}^{\top}\big(\bm{x}_{\perp}^{t,\text{sgn}}-\bm{x}_{\perp}^{t,\text{sgn},\left(l\right)}\big)+\frac{1}{m}\sum_{i=1}^{m}\xi_{i}\left|a_{i,1}\right|^{3}\bm{a}_{i,\perp}^{\top}\big(\bm{x}_{\perp}^{t}-\bm{x}_{\perp}^{t,\left(l\right)}-\bm{x}_{\perp}^{t,\text{sgn}}+\bm{x}_{\perp}^{t,\text{sgn},\left(l\right)}\big).
\end{align*}
In view of the independence between $\xi_{i}$ and $\left|a_{i,1}\right|^{3}\bm{a}_{i,\perp}^{\top}\big(\bm{x}_{\perp}^{t,\text{sgn}}-\bm{x}_{\perp}^{t,\text{sgn},\left(l\right)}\big)$,
one can thus invoke the Bernstein inequality (see Lemma \ref{lemma:bernstein})
to obtain 
\begin{equation}
\left|\frac{1}{m}\sum_{i=1}^{m}\xi_{i}\left|a_{i,1}\right|^{3}\bm{a}_{i,\perp}^{\top}\big(\bm{x}_{\perp}^{t,\text{sgn}}-\bm{x}_{\perp}^{t,\text{sgn},\left(l\right)}\big)\right|\lesssim\frac{1}{m}\left(\sqrt{V_{1}\log m}+B_{1}\log m\right)\label{eq:xt-l-bernstein-1}
\end{equation}
with probability at least $1-O\left(m^{-10}\right)$, where 
\begin{align*}
V_{1} & :=\sum_{i=1}^{m}\left|a_{i,1}\right|^{6}\left|\bm{a}_{i,\perp}^{\top}\big(\bm{x}_{\perp}^{t,\text{sgn}}-\bm{x}_{\perp}^{t,\text{sgn},\left(l\right)}\big)\right|^{2}\qquad\text{and}\qquad B_{1}:=\max_{1\leq i\leq m}\left|a_{i,1}\right|^{3}\left|\bm{a}_{i,\perp}^{\top}\big(\bm{x}_{\perp}^{t,\text{sgn}}-\bm{x}_{\perp}^{t,\text{sgn},\left(l\right)}\big)\right|.
\end{align*}
Regarding $V_{1}$, one can combine the fact (\ref{eq:max-a-i-1})
and Lemma \ref{lemma:hessian-concentration} to reach 
\begin{align*}
\frac{1}{m}V_{1} & \lesssim\log^{2}m\left(\bm{x}_{\perp}^{t,\text{sgn}}-\bm{x}_{\perp}^{t,\text{sgn},\left(l\right)}\right)^{\top}\left(\frac{1}{m}\sum_{i=1}^{m}\left|a_{i,1}\right|^{2}\bm{a}_{i,\perp}\bm{a}_{i,\perp}^{\top}\right)\big(\bm{x}_{\perp}^{t,\text{sgn}}-\bm{x}_{\perp}^{t,\text{sgn},\left(l\right)}\big)\\
 & \lesssim\log^{2}m\left\Vert \bm{x}_{\perp}^{t,\text{sgn}}-\bm{x}_{\perp}^{t,\text{sgn},\left(l\right)}\right\Vert _{2}^{2}.
\end{align*}
For $B_{1}$, it is easy to check from (\ref{eq:max-a-i-1}) and (\ref{eq:max-a-i-norm})
that 
\[
B_{1}\lesssim\sqrt{n\log^{3}m}\left\Vert \bm{x}_{\perp}^{t,\text{sgn}}-\bm{x}_{\perp}^{t,\text{sgn},\left(l\right)}\right\Vert _{2}.
\]
The previous two bounds taken collectively yield 
\begin{align}
\left|\frac{1}{m}\sum_{i=1}^{m}\xi_{i}\left|a_{i,1}\right|^{3}\bm{a}_{i,\perp}^{\top}\left(\bm{x}_{\perp}^{t,\text{sgn}}-\bm{x}_{\perp}^{t,\text{sgn},\left(l\right)}\right)\right| & \lesssim\sqrt{\frac{\log^{3}m}{m}\left\Vert \bm{x}_{\perp}^{t,\text{sgn}}-\bm{x}_{\perp}^{t,\text{sgn},\left(l\right)}\right\Vert _{2}^{2}}+\frac{\sqrt{n\log^{5}m}}{m}\left\Vert \bm{x}_{\perp}^{t,\text{sgn}}-\bm{x}_{\perp}^{t,\text{sgn},\left(l\right)}\right\Vert _{2}\nonumber \\
 & \lesssim\sqrt{\frac{\log^{3}m}{m}}\left\Vert \bm{x}_{\perp}^{t,\text{sgn}}-\bm{x}_{\perp}^{t,\text{sgn},\left(l\right)}\right\Vert _{2},\label{eq:xt-xt-l-signal-w4-1}
\end{align}
as long as $m\gtrsim n\log^{2}m$. The second term in $\omega_{4}$
can be simply controlled by the Cauchy-Schwarz inequality and Lemma
\ref{lemma:hessian-concentration}. Specifically, we have 
\begin{align}
 & \left|\frac{1}{m}\sum_{i=1}^{m}\xi_{i}\left|a_{i,1}\right|^{3}\bm{a}_{i,\perp}^{\top}\left(\bm{x}_{\perp}^{t}-\bm{x}_{\perp}^{t,\left(l\right)}-\bm{x}_{\perp}^{t,\text{sgn}}+\bm{x}_{\perp}^{t,\text{sgn},\left(l\right)}\right)\right|\nonumber \\
 & \leq\left\Vert \frac{1}{m}\sum_{i=1}^{m}\xi_{i}\left|a_{i,1}\right|^{3}\bm{a}_{i,\perp}^{\top}\right\Vert _{2}\left\Vert \bm{x}_{\perp}^{t}-\bm{x}_{\perp}^{t,\left(l\right)}-\bm{x}_{\perp}^{t,\text{sgn}}+\bm{x}_{\perp}^{t,\text{sgn},\left(l\right)}\right\Vert _{2}\nonumber \\
 & \lesssim\sqrt{\frac{n\log^{3}m}{m}}\left\Vert \bm{x}_{\perp}^{t}-\bm{x}_{\perp}^{t,\left(l\right)}-\bm{x}_{\perp}^{t,\text{sgn}}+\bm{x}_{\perp}^{t,\text{sgn},\left(l\right)}\right\Vert _{2},\label{eq:xt-xt-l-signal-w4-2}
\end{align}
where the second relation holds due to Lemma \ref{lemma:hessian-concentration}.
Take the preceding two bounds (\ref{eq:xt-xt-l-signal-w4-1}) and
(\ref{eq:xt-xt-l-signal-w4-2}) collectively to conclude that 
\begin{align*}
\left|\omega_{4}\right| & \lesssim\sqrt{\frac{\log^{3}m}{m}}\left\Vert \bm{x}_{\perp}^{t,\text{sgn}}-\bm{x}_{\perp}^{t,\text{sgn},\left(l\right)}\right\Vert _{2}+\sqrt{\frac{n\log^{3}m}{m}}\left\Vert \bm{x}_{\perp}^{t}-\bm{x}_{\perp}^{t,\left(l\right)}-\bm{x}_{\perp}^{t,\text{sgn}}+\bm{x}_{\perp}^{t,\text{sgn},\left(l\right)}\right\Vert _{2}\\
 & \lesssim\sqrt{\frac{\log^{3}m}{m}}\left\Vert \bm{x}_{\perp}^{t}-\bm{x}_{\perp}^{t,\left(l\right)}\right\Vert _{2}+\sqrt{\frac{n\log^{3}m}{m}}\left\Vert \bm{x}_{\perp}^{t}-\bm{x}_{\perp}^{t,\left(l\right)}-\bm{x}_{\perp}^{t,\text{sgn}}+\bm{x}_{\perp}^{t,\text{sgn},\left(l\right)}\right\Vert _{2},
\end{align*}
where the second line follows from the triangle inequality 
\[
\left\Vert \bm{x}_{\perp}^{t,\text{sgn}}-\bm{x}_{\perp}^{t,\text{sgn},\left(l\right)}\right\Vert _{2}\leq\left\Vert \bm{x}_{\perp}^{t}-\bm{x}_{\perp}^{t,\left(l\right)}\right\Vert _{2}+\left\Vert \bm{x}_{\perp}^{t}-\bm{x}_{\perp}^{t,\left(l\right)}-\bm{x}_{\perp}^{t,\text{sgn}}+\bm{x}_{\perp}^{t,\text{sgn},\left(l\right)}\right\Vert _{2}
\]
and the fact that $\sqrt{\frac{\log^{3}m}{m}}\leq\sqrt{\frac{n\log^{3}m}{m}}$. 
\item It remains to bound $\omega_{3}\left(\tau\right)$. To this end, one
can decompose 
\begin{align*}
\omega_{3}\left(\tau\right) & =\underbrace{\frac{3}{m}\sum_{i=1}^{m}\left[\left(\bm{a}_{i}^{\top}\bm{x}\left(\tau\right)\right)^{2}-\big(\bm{a}_{i}^{\text{sgn}\top}\bm{x}\left(\tau\right)\big)^{2}\right]a_{i,1}\bm{a}_{i,\perp}^{\top}\big(\bm{x}_{\perp}^{t}-\bm{x}_{\perp}^{t,\left(l\right)}\big)}_{:=\theta_{1}\left(\tau\right)}\\
 & \quad+\underbrace{\frac{3}{m}\sum_{i=1}^{m}\left[\left(\bm{a}_{i}^{\text{sgn}\top}\bm{x}\left(\tau\right)\right)^{2}-\left(\bm{a}_{i}^{\text{sgn}\top}\bm{x}^{\text{sgn}}\left(\tau\right)\right)^{2}\right]a_{i,1}\bm{a}_{i,\perp}^{\top}\big(\bm{x}_{\perp}^{t}-\bm{x}_{\perp}^{t,\left(l\right)}\big)}_{:=\theta_{2}\left(\tau\right)}\\
 & \quad+\underbrace{\frac{3}{m}\sum_{i=1}^{m}\left(\bm{a}_{i}^{\text{sgn}\top}\bm{x}^{\text{sgn}}\left(\tau\right)\right)^{2}a_{i,1}\bm{a}_{i,\perp}^{\top}\big(\bm{x}_{\perp}^{t,\text{sgn}}-\bm{x}_{\perp}^{t,\text{sgn},\left(l\right)}\big)}_{:=\theta_{3}\left(\tau\right)}\\
 & \quad+\underbrace{\frac{3}{m}\sum_{i=1}^{m}\left(\bm{a}_{i}^{\text{sgn}\top}\bm{x}^{\text{sgn}}\left(\tau\right)\right)^{2}a_{i,1}\bm{a}_{i,\perp}^{\top}\left(\bm{x}_{\perp}^{t}-\bm{x}_{\perp}^{t,\left(l\right)}-\bm{x}_{\perp}^{t,\text{sgn}}+\bm{x}_{\perp}^{t,\text{sgn},\left(l\right)}\right)}_{:=\theta_{4}\left(\tau\right)},
\end{align*}
where we denote $\bm{x}^{\text{sgn}}\left(\tau\right)=\bm{x}^{t,\text{sgn}}+\tau\left(\bm{x}^{t,\text{sgn},\left(l\right)}-\bm{x}^{t,\text{sgn}}\right)$.
A direct consequence of (\ref{subeq:consequence-norm}) and (\ref{subeq:consequence-incoherence})
is that
\begin{equation}
\left|\bm{a}_{i}^{\text{sgn}\top}\bm{x}^{\text{sgn}}\left(\tau\right)\right|\lesssim\sqrt{\log m}.\label{eq:x-tau-sgn-consequence}
\end{equation}
Recalling that $\xi_{i}=\mathrm{sgn}\left(a_{i,1}\right)$ and $\xi_{i}^{\mathrm{sgn}}=\mathrm{sgn}\left(a_{i,1}^{\mathrm{sgn}}\right)$,
one has 
\begin{align*}
\bm{a}_{i}^{\top}\bm{x}\left(\tau\right)-\bm{a}_{i}^{\mathrm{sgn}\top}\bm{x}\left(\tau\right) & =\left(\xi_{i}-\xi_{i}^{\mathrm{sgn}}\right)\left|a_{i,1}\right|x_{\parallel}\left(\tau\right),\\
\bm{a}_{i}^{\top}\bm{x}\left(\tau\right)+\bm{a}_{i}^{\mathrm{sgn}\top}\bm{x}\left(\tau\right) & =\left(\xi_{i}+\xi_{i}^{\mathrm{sgn}}\right)\left|a_{i,1}\right|x_{\parallel}\left(\tau\right)+2\bm{a}_{i,\perp}^{\top}\bm{x}_{\perp}\left(\tau\right),
\end{align*}
which implies that 
\begin{align}
\left(\bm{a}_{i}^{\top}\bm{x}\left(\tau\right)\right)^{2}-\left(\bm{a}_{i}^{\text{sgn}\top}\bm{x}\left(\tau\right)\right)^{2} & =\left(\bm{a}_{i}^{\top}\bm{x}\left(\tau\right)-\bm{a}_{i}^{\mathrm{sgn}\top}\bm{x}\left(\tau\right)\right)\cdot\left(\bm{a}_{i}^{\top}\bm{x}\left(\tau\right)+\bm{a}_{i}^{\mathrm{sgn}\top}\bm{x}\left(\tau\right)\right)\nonumber \\
 & =\left(\xi_{i}-\xi_{i}^{\mathrm{sgn}}\right)\left|a_{i,1}\right|x_{\parallel}\left(\tau\right)\left\{ \left(\xi_{i}+\xi_{i}^{\mathrm{sgn}}\right)\left|a_{i,1}\right|x_{\parallel}\left(\tau\right)+2\bm{a}_{i,\perp}^{\top}\bm{x}_{\perp}(\tau)\right\} \nonumber \\
 & =2\left(\xi_{i}-\xi_{i}^{\mathrm{sgn}}\right)\left|a_{i,1}\right|x_{\parallel}\left(\tau\right)\bm{a}_{i,\perp}^{\top}\bm{x}_{\perp}\left(\tau\right)\label{eq:identity-a-a-sgn}
\end{align}
owing to the identity $\left(\xi_{i}-\xi_{i}^{\mathrm{sgn}}\right)\left(\xi_{i}+\xi_{i}^{\mathrm{sgn}}\right)=\xi_{i}^{2}-\left(\xi_{i}^{\mathrm{sgn}}\right)^{2}=0$.
In light of (\ref{eq:identity-a-a-sgn}), we have 
\begin{align*}
\theta_{1}\left(\tau\right) & =\frac{6}{m}\sum_{i=1}^{m}\left(\xi_{i}-\xi_{i}^{\mathrm{sgn}}\right)\left|a_{i,1}\right|x_{\parallel}\left(\tau\right)\bm{a}_{i,\perp}^{\top}\bm{x}_{\perp}\left(\tau\right)a_{i,1}\bm{a}_{i,\perp}^{\top}\left(\bm{x}_{\perp}^{t}-\bm{x}_{\perp}^{t,\left(l\right)}\right)\\
 & =6x_{\parallel}\left(\tau\right)\cdot\bm{x}_{\perp}^{\top}\left(\tau\right)\left[\frac{1}{m}\sum_{i=1}^{m}\left(1-\xi_{i}\xi_{i}^{\text{sgn}}\right)\left|a_{i,1}\right|^{2}\bm{a}_{i,\perp}\bm{a}_{i,\perp}^{\top}\right]\left(\bm{x}_{\perp}^{t}-\bm{x}_{\perp}^{t,\left(l\right)}\right).
\end{align*}
First note that 
\begin{equation}
\left\Vert \frac{1}{m}\sum_{i=1}^{m}\left(1-\xi_{i}\xi_{i}^{\text{sgn}}\right)\left|a_{i,1}\right|^{2}\bm{a}_{i,\perp}\bm{a}_{i,\perp}^{\top}\right\Vert \leq2\left\Vert \frac{1}{m}\sum_{i=1}^{m}\left|a_{i,1}\right|^{2}\bm{a}_{i,\perp}\bm{a}_{i,\perp}^{\top}\right\Vert \lesssim1,\label{eq:spectral-upper-bound}
\end{equation}
where the last relation holds due to Lemma \ref{lemma:hessian-concentration}.
This results in the following upper bound on $\theta_{1}\left(\tau\right)$
\[
\left|\theta_{1}\left(\tau\right)\right|\lesssim\left|x_{\parallel}\left(\tau\right)\right|\left\Vert \bm{x}_{\perp}\left(\tau\right)\right\Vert _{2}\left\Vert \bm{x}_{\perp}^{t}-\bm{x}_{\perp}^{t,\left(l\right)}\right\Vert _{2}\lesssim\left|x_{\parallel}\left(\tau\right)\right|\left\Vert \bm{x}_{\perp}^{t}-\bm{x}_{\perp}^{t,\left(l\right)}\right\Vert _{2},
\]
where we have used the fact that $\left\Vert \bm{x}_{\perp}\left(\tau\right)\right\Vert _{2}\lesssim1$
(see (\ref{eq:x-tau-implication})). Regarding $\theta_{2}\left(\tau\right)$,
one obtains 
\[
\theta_{2}\left(\tau\right)=\frac{3}{m}\sum_{i=1}^{m}\left[\bm{a}_{i}^{\text{sgn}\top}\left(\bm{x}\left(\tau\right)-\bm{x}^{\text{sgn}}\left(\tau\right)\right)\right]\left[\bm{a}_{i}^{\text{sgn}\top}\left(\bm{x}\left(\tau\right)+\bm{x}^{\text{sgn}}\left(\tau\right)\right)\right]a_{i,1}\bm{a}_{i,\perp}^{\top}\left(\bm{x}_{\perp}^{t}-\bm{x}_{\perp}^{t,\left(l\right)}\right).
\]
Apply the Cauchy-Schwarz inequality to reach 
\begin{align*}
\left|\theta_{2}\left(\tau\right)\right| & \lesssim\sqrt{\frac{1}{m}\sum_{i=1}^{m}\left[\bm{a}_{i}^{\text{sgn}\top}\left(\bm{x}\left(\tau\right)-\bm{x}^{\text{sgn}}\left(\tau\right)\right)\right]^{2}\left[\bm{a}_{i}^{\text{sgn}\top}\left(\bm{x}\left(\tau\right)+\bm{x}^{\text{sgn}}\left(\tau\right)\right)\right]^{2}}\sqrt{\frac{1}{m}\sum_{i=1}^{m}\left|a_{i,1}\right|^{2}\left[\bm{a}_{i,\perp}^{\top}\left(\bm{x}_{\perp}^{t}-\bm{x}_{\perp}^{t,\left(l\right)}\right)\right]^{2}}\\
 & \lesssim\sqrt{\frac{1}{m}\sum_{i=1}^{m}\left[\bm{a}_{i}^{\text{sgn}\top}\left(\bm{x}\left(\tau\right)-\bm{x}^{\text{sgn}}\left(\tau\right)\right)\right]^{2}\log m}\cdot\left\Vert \bm{x}_{\perp}^{t}-\bm{x}_{\perp}^{t,\left(l\right)}\right\Vert _{2}\\
 & \lesssim\sqrt{\log m}\left\Vert \bm{x}\left(\tau\right)-\bm{x}^{\text{sgn}}\left(\tau\right)\right\Vert _{2}\left\Vert \bm{x}_{\perp}^{t}-\bm{x}_{\perp}^{t,\left(l\right)}\right\Vert _{2}.
\end{align*}
Here the second relation comes from Lemma \ref{lemma:hessian-concentration}
and the fact that 
\[
\left|\bm{a}_{i}^{\text{sgn}\top}\left(\bm{x}\left(\tau\right)+\bm{x}^{\text{sgn}}\left(\tau\right)\right)\right|\lesssim\sqrt{\log m}.
\]
When it comes to $\theta_{3}\left(\tau\right)$, we need to exploit
the independence between 
\[
\left\{ \xi_{i}\right\} \quad\text{and}\quad\big(\bm{a}_{i}^{\text{sgn}\top}\bm{x}^{\text{sgn}}\left(\tau\right)\big)^{2}\big|a_{i,1}\big|\bm{a}_{i,\perp}^{\top}\big(\bm{x}_{\perp}^{t,\text{sgn}}-\bm{x}_{\perp}^{t,\text{sgn},\left(l\right)}\big).
\]
Similar to (\ref{eq:xt-l-bernstein-1}), one can obtain 
\[
\left|\theta_{3}\left(\tau\right)\right|\lesssim\frac{1}{m}\left(\sqrt{V_{2}\log m}+B_{2}\log m\right)
\]
with probability at least $1-O\left(m^{-10}\right)$, where 
\begin{align*}
V_{2} & :=\sum_{i=1}^{m}\left(\bm{a}_{i}^{\text{sgn}\top}\bm{x}^{\text{sgn}}\left(\tau\right)\right)^{4}\left|a_{i,1}\right|^{2}\left|\bm{a}_{i,\perp}^{\top}\left(\bm{x}_{\perp}^{t,\text{sgn}}-\bm{x}_{\perp}^{t,\text{sgn},\left(l\right)}\right)\right|^{2}\\
B_{2} & :=\max_{1\leq i\leq m}\left(\bm{a}_{i}^{\text{sgn}\top}\bm{x}^{\text{sgn}}\left(\tau\right)\right)^{2}\left|a_{i,1}\right|\left|\bm{a}_{i,\perp}^{\top}\left(\bm{x}_{\perp}^{t,\text{sgn}}-\bm{x}_{\perp}^{t,\text{sgn},\left(l\right)}\right)\right|.
\end{align*}
It is easy to see from Lemma \ref{lemma:hessian-concentration}, (\ref{eq:x-tau-sgn-consequence}),
(\ref{eq:max-a-i-1}) and (\ref{eq:max-a-i-norm}) that 
\[
V_{2}\lesssim m\log^{2}m\left\Vert \bm{x}_{\perp}^{t,\text{sgn}}-\bm{x}_{\perp}^{t,\text{sgn},\left(l\right)}\right\Vert _{2}^{2}\qquad\text{and}\qquad B_{2}\lesssim\sqrt{n\log^{3}m}\left\Vert \bm{x}_{\perp}^{t,\text{sgn}}-\bm{x}_{\perp}^{t,\text{sgn},\left(l\right)}\right\Vert _{2},
\]
which implies 
\[
\left|\theta_{3}\left(\tau\right)\right|\lesssim\left(\sqrt{\frac{\log^{3}m}{m}}+\frac{\sqrt{n\log^{5}m}}{m}\right)\left\Vert \bm{x}_{\perp}^{t,\text{sgn}}-\bm{x}_{\perp}^{t,\text{sgn},\left(l\right)}\right\Vert _{2}\asymp\sqrt{\frac{\log^{3}m}{m}}\left\Vert \bm{x}_{\perp}^{t,\text{sgn}}-\bm{x}_{\perp}^{t,\text{sgn},\left(l\right)}\right\Vert _{2}
\]
with the proviso that $m\gtrsim n\log^{2}m$. We are left with $\theta_{4}\left(\tau\right)$.
Invoking Cauchy-Schwarz inequality,
\begin{align*}
\left|\theta_{4}\left(\tau\right)\right| & \lesssim\sqrt{\frac{1}{m}\sum_{i=1}^{m}\left(\bm{a}_{i}^{\text{sgn}\top}\bm{x}^{\text{sgn}}\left(\tau\right)\right)^{4}}\sqrt{\frac{1}{m}\sum_{i=1}^{m}\left|a_{i,1}\right|^{2}\left[\bm{a}_{i,\perp}^{\top}\left(\bm{x}_{\perp}^{t}-\bm{x}_{\perp}^{t,\left(l\right)}-\bm{x}_{\perp}^{t,\text{sgn}}+\bm{x}_{\perp}^{t,\text{sgn},\left(l\right)}\right)\right]^{2}}\\
 & \lesssim\sqrt{\frac{1}{m}\sum_{i=1}^{m}\left(\bm{a}_{i}^{\text{sgn}\top}\bm{x}^{\text{sgn}}\left(\tau\right)\right)^{2}\log m}\cdot\left\Vert \bm{x}_{\perp}^{t}-\bm{x}_{\perp}^{t,\left(l\right)}-\bm{x}_{\perp}^{t,\text{sgn}}+\bm{x}_{\perp}^{t,\text{sgn},\left(l\right)}\right\Vert _{2}\\
 & \lesssim\sqrt{\log m}\left\Vert \bm{x}_{\perp}^{t}-\bm{x}_{\perp}^{t,\left(l\right)}-\bm{x}_{\perp}^{t,\text{sgn}}+\bm{x}_{\perp}^{t,\text{sgn},\left(l\right)}\right\Vert _{2},
\end{align*}
where we have used the fact that $\big|\bm{a}_{i}^{\text{sgn}\top}\bm{x}^{\text{sgn}}\left(\tau\right)\big|\lesssim\sqrt{\log m}$.
In summary, we have obtained 
\begin{align*}
\left|\omega_{3}\left(\tau\right)\right| & \lesssim\left\{ \left|x_{\parallel}\left(\tau\right)\right|+\sqrt{\log m}\left\Vert \bm{x}\left(\tau\right)-\bm{x}^{\text{sgn}}\left(\tau\right)\right\Vert _{2}\right\} \left\Vert \bm{x}_{\perp}^{t}-\bm{x}_{\perp}^{t,\left(l\right)}\right\Vert _{2}\\
 & \quad+\sqrt{\frac{\log^{3}m}{m}}\left\Vert \bm{x}_{\perp}^{t,\text{sgn}}-\bm{x}_{\perp}^{t,\text{sgn},\left(l\right)}\right\Vert _{2}+\sqrt{\log m}\left\Vert \bm{x}_{\perp}^{t}-\bm{x}_{\perp}^{t,\left(l\right)}-\bm{x}_{\perp}^{t,\text{sgn}}+\bm{x}_{\perp}^{t,\text{sgn},\left(l\right)}\right\Vert _{2}\\
 & \lesssim\left\{ \left|x_{\parallel}\left(\tau\right)\right|+\sqrt{\log m}\left\Vert \bm{x}\left(\tau\right)-\bm{x}^{\text{sgn}}\left(\tau\right)\right\Vert _{2}+\sqrt{\frac{\log^{3}m}{m}}\right\} \left\Vert \bm{x}_{\perp}^{t}-\bm{x}_{\perp}^{t,\left(l\right)}\right\Vert _{2}\\
 & \quad+\sqrt{\log m}\left\Vert \bm{x}_{\perp}^{t}-\bm{x}_{\perp}^{t,\left(l\right)}-\bm{x}_{\perp}^{t,\text{sgn}}+\bm{x}_{\perp}^{t,\text{sgn},\left(l\right)}\right\Vert _{2},
\end{align*}
where the last inequality utilizes the triangle inequality 
\[
\left\Vert \bm{x}_{\perp}^{t,\text{sgn}}-\bm{x}_{\perp}^{t,\text{sgn},\left(l\right)}\right\Vert _{2}\leq\left\Vert \bm{x}_{\perp}^{t}-\bm{x}_{\perp}^{t,\left(l\right)}\right\Vert _{2}+\left\Vert \bm{x}_{\perp}^{t}-\bm{x}_{\perp}^{t,\left(l\right)}-\bm{x}_{\perp}^{t,\text{sgn}}+\bm{x}_{\perp}^{t,\text{sgn},\left(l\right)}\right\Vert _{2}
\]
and the fact that $\sqrt{\frac{\log^{3}m}{m}}\leq\sqrt{\log m}$.
This together with the bound for $\omega_{4}\left(\tau\right)$ gives
\begin{align*}
\left|\omega_{2}\left(\tau\right)\right| & \leq\left|\omega_{3}\left(\tau\right)\right|+\left|\omega_{4}\left(\tau\right)\right|\\
 & \lesssim\left\{ \left|x_{\parallel}\left(\tau\right)\right|+\sqrt{\log m}\left\Vert \bm{x}\left(\tau\right)-\bm{x}^{\text{sgn}}\left(\tau\right)\right\Vert _{2}+\sqrt{\frac{\log^{3}m}{m}}\right\} \left\Vert \bm{x}_{\perp}^{t}-\bm{x}_{\perp}^{t,\left(l\right)}\right\Vert _{2}\\
 & \quad+\sqrt{\log m}\left\Vert \bm{x}_{\perp}^{t}-\bm{x}_{\perp}^{t,\left(l\right)}-\bm{x}_{\perp}^{t,\text{sgn}}+\bm{x}_{\perp}^{t,\text{sgn},\left(l\right)}\right\Vert _{2},
\end{align*}
as long as $m\gg n\log^{2}m$. 
\end{itemize}
\item Combine the bounds to arrive at 
\begin{align*}
x_{\parallel}^{t+1}-x_{\parallel}^{t+1,\left(l\right)} & =\left\{ 1+3\eta\left(1-\int_{0}^{1}\left\Vert \bm{x}\left(\tau\right)\right\Vert _{2}^{2}\mathrm{d}\tau\right)+\eta\cdot O\left(\left|x_{\parallel}\left(\tau\right)\right|^{2}+\sqrt{\frac{n\log^{3}m}{m}}\right)\right\} \left(x_{\parallel}^{t}-x_{\parallel}^{t,\left(l\right)}\right)\\
 & \quad+O\left(\eta\frac{\log^{2}m}{m}\left\Vert \bm{x}^{t,(l)}\right\Vert _{2}\right)+O\left(\eta\sqrt{\log m}\left\Vert \bm{x}_{\perp}^{t}-\bm{x}_{\perp}^{t,\left(l\right)}-\bm{x}_{\perp}^{t,\text{sgn}}+\bm{x}_{\perp}^{t,\text{sgn},\left(l\right)}\right\Vert _{2}\right)\\
 & \quad+O\left(\eta\sup_{0\leq\tau\leq1}\left\{ \left|x_{\parallel}\left(\tau\right)\right|+\sqrt{\log m}\left\Vert \bm{x}\left(\tau\right)-\bm{x}^{\text{sgn}}\left(\tau\right)\right\Vert _{2}+\sqrt{\frac{\log^{3}m}{m}}\right\} \left\Vert \bm{x}_{\perp}^{t}-\bm{x}_{\perp}^{t,\left(l\right)}\right\Vert _{2}\right).
\end{align*}
To simplify the above bound, notice that for the last term, for any
$t<T_{0}\lesssim\log n$ and $0\leq\tau\leq1$, one has 
\begin{align*}
\left|x_{\parallel}\left(\tau\right)\right| & \leq\left|x_{\parallel}^{t}\right|+\left|x_{\parallel}^{t,\left(l\right)}-x_{\parallel}^{t}\right|\leq\alpha_{t}+\alpha_{t}\left(1+\frac{1}{\log m}\right)^{t}C_{2}\frac{\sqrt{n\log^{12}m}}{m}\lesssim\alpha_{t},
\end{align*}
as long as $m\gg\sqrt{n\log^{12}m}$. Similarly, one can show that
\begin{align*}
\sqrt{\log m}\left\Vert \bm{x}\left(\tau\right)-\bm{x}^{\text{sgn}}\left(\tau\right)\right\Vert _{2} & \leq\sqrt{\log m}\left(\left\Vert \bm{x}^{t}-\bm{x}^{t,\text{sgn}}\right\Vert _{2}+\left\Vert \bm{x}^{t}-\bm{x}^{t,\left(l\right)}-\bm{x}^{t,\text{sgn}}+\bm{x}^{t,\text{sgn},\left(l\right)}\right\Vert _{2}\right)\\
 & \lesssim\alpha_{t}\sqrt{\log m}\left(\sqrt{\frac{n\log^{5}m}{m}}+\frac{\sqrt{n\log^{9}m}}{m}\right)\lesssim\alpha_{t},
\end{align*}
with the proviso that $m\gg n\log^{6}m$. Therefore, we can further
obtain 
\begin{align*}
\left|x_{\parallel}^{t+1}-x_{\parallel}^{t+1,\left(l\right)}\right| & \leq\left\{ 1+3\eta\left(1-\left\Vert \bm{x}^{t}\right\Vert _{2}^{2}\right)+\eta\cdot O\left(\left\Vert \bm{x}^{t}-\bm{x}^{t,\left(l\right)}\right\Vert _{2}+\left|x_{\parallel}^{t}\right|^{2}+\sqrt{\frac{n\log^{3}m}{m}}\right)\right\} \left|x_{\parallel}^{t}-x_{\parallel}^{t,\left(l\right)}\right|\\
 & \quad+O\left(\eta\frac{\log^{2}m}{m}\left\Vert \bm{x}^{t}\right\Vert _{2}\right)+O\left(\eta\sqrt{\log m}\left\Vert \bm{x}_{\perp}^{t}-\bm{x}_{\perp}^{t,\left(l\right)}-\bm{x}_{\perp}^{t,\text{sgn}}+\bm{x}_{\perp}^{t,\text{sgn},\left(l\right)}\right\Vert _{2}\right)\\
 & \quad+O\left(\eta\alpha_{t}\left\Vert \bm{x}^{t}-\bm{x}^{t,\left(l\right)}\right\Vert _{2}\right)\\
 & \leq\left\{ 1+3\eta\left(1-\left\Vert \bm{x}^{t}\right\Vert _{2}^{2}\right)+\eta\phi_{1}\right\} \left|x_{\parallel}^{t}-x_{\parallel}^{t,\left(l\right)}\right|+O\left(\eta\alpha_{t}\left\Vert \bm{x}^{t}-\bm{x}^{t,\left(l\right)}\right\Vert _{2}\right)\\
 & \quad+O\left(\eta\frac{\log^{2}m}{m}\left\Vert \bm{x}^{t}\right\Vert _{2}\right)+O\left(\eta\sqrt{\log m}\left\Vert \bm{x}_{\perp}^{t}-\bm{x}_{\perp}^{t,\left(l\right)}-\bm{x}_{\perp}^{t,\text{sgn}}+\bm{x}_{\perp}^{t,\text{sgn},\left(l\right)}\right\Vert _{2}\right)
\end{align*}
for some $|\phi_{1}|\ll\frac{1}{\log m}$. Here the last inequality
comes from the sample complexity $m\gg n\log^{5}m$, the assumption
$\alpha_{t}\ll\frac{1}{\log^{5}m}$ and the fact (\ref{eq:1-logm-xt-xt-l}).
Given the inductive hypotheses (\ref{subeq:induction}), we can conclude
\begin{align*}
\left|x_{\parallel}^{t+1}-x_{\parallel}^{t+1,\left(l\right)}\right| & \leq\left\{ 1+3\eta\left(1-\left\Vert \bm{x}^{t}\right\Vert _{2}^{2}\right)+\eta\phi_{1}\right\} \alpha_{t}\left(1+\frac{1}{\log m}\right)^{t}C_{2}\frac{\sqrt{n\log^{12}m}}{m}\\
 & \quad+O\left(\frac{\eta\log^{2}m}{m}\left(\alpha_{t}+\beta_{t}\right)\right)+O\left(\eta\sqrt{\log m}\cdot\alpha_{t}\left(1+\frac{1}{\log m}\right)^{t}C_{4}\frac{\sqrt{n\log^{9}m}}{m}\right)\\
 & \quad+O\left(\eta\alpha_{t}\beta_{t}\left(1+\frac{1}{\log m}\right)^{t}C_{1}\frac{\sqrt{n\log^{5}m}}{m}\right)\\
 & \overset{\left(\text{i}\right)}{\leq}\left\{ 1+3\eta\left(1-\left\Vert \bm{x}^{t}\right\Vert _{2}^{2}\right)+\eta\phi_{2}\right\} \alpha_{t}\left(1+\frac{1}{\log m}\right)^{t}C_{2}\frac{\sqrt{n\log^{12}m}}{m}\\
 & \overset{\left(\text{ii}\right)}{\leq}\alpha_{t+1}\left(1+\frac{1}{\log m}\right)^{t+1}C_{2}\frac{\sqrt{n\log^{12}m}}{m}
\end{align*}
for some $|\phi_{2}|\ll\frac{1}{\log m}$. Here, the inequality (i)
holds true as long as \begin{subequations} 
\begin{align}
\frac{\log^{2}m}{m}\left(\alpha_{t}+\beta_{t}\right) & \ll\frac{1}{\log m}\alpha_{t}C_{2}\frac{\sqrt{n\log^{12}m}}{m}\label{eq:xt-xt-l-as-long-as-1}\\
\sqrt{\log m}C_{4}\frac{\sqrt{n\log^{9}m}}{m} & \ll\frac{1}{\log m}C_{2}\frac{\sqrt{n\log^{12}m}}{m}\label{eq:xt-xt-1-as-long-as-2}\\
\beta_{t}C_{1}\frac{\sqrt{n\log^{5}m}}{m} & \ll\frac{1}{\log m}C_{2}\frac{\sqrt{n\log^{12}m}}{m},\label{eq:xt-xt-1-as-long-as-5}
\end{align}
\end{subequations}where the first condition (\ref{eq:xt-xt-l-as-long-as-1})
is satisfied since (according to Lemma \ref{lemma:iterative})
\[
\alpha_{t}+\beta_{t}\lesssim\beta_{t}\lesssim\alpha_{t}\sqrt{n\log m}.
\]
The second condition (\ref{eq:xt-xt-1-as-long-as-2}) holds as long
as $C_{2}\gg C_{4}$. The third one (\ref{eq:xt-xt-1-as-long-as-5})
holds trivially. Moreover, the second inequality (ii) follows from
the same reasoning as in (\ref{eq:beta-t-iterative-reasoning}). Specifically,
we have for some $|\phi_{3}|\ll\frac{1}{\log m}$,
\begin{align*}
\left\{ 1+3\eta\left(1-\left\Vert \bm{x}^{t}\right\Vert _{2}^{2}\right)+\eta\phi_{2}\right\} \alpha_{t} & =\left\{ \frac{\alpha_{t+1}}{\alpha_{t}}+\eta\phi_{3}\right\} \alpha_{t}\\
 & \leq\left\{ \frac{\alpha_{t+1}}{\alpha_{t}}+\eta O\left(\frac{\alpha_{t+1}}{\alpha_{t}}\phi_{3}\right)\right\} \alpha_{t}\\
 & \leq\alpha_{t+1}\left(1+\frac{1}{\log m}\right),
\end{align*}
as long as $\frac{\alpha_{t+1}}{\alpha_{t}}\asymp1$. 
\end{itemize}
The proof is completed by applying the union bound over all $1\leq l\leq m$. 

\section{Proof of Lemma \ref{lemma:xt-xt-sgn}\label{sec:Proof-of-Lemma-xt-xt-sgn}}

By similar calculations as in (\ref{eq:xt-xt-l}), we get the identity
\begin{align}
\bm{x}^{t+1}-\bm{x}^{t+1,\mathrm{sgn}} & =\left\{ \bm{I}-\eta\int_{0}^{1}\nabla^{2}f\left(\tilde{\bm{x}}\left(\tau\right)\right)\mathrm{d}\tau\right\} \left(\bm{x}^{t}-\bm{x}^{t,\mathrm{sgn}}\right)+\eta\left(\nabla f^{\mathrm{sgn}}\left(\bm{x}^{t,\mathrm{sgn}}\right)-\nabla f\left(\bm{x}^{t,\mathrm{sgn}}\right)\right),\label{eq:difference-xt-sgn}
\end{align}
where $\tilde{\bm{x}}\left(\tau\right):=\bm{x}^{t}+\tau\left(\bm{x}^{t,\mathrm{sgn}}-\bm{x}^{t}\right)$.
The first term satisfies
\begin{align}
 & \left\Vert \left\{ \bm{I}-\eta\int_{0}^{1}\nabla^{2}f\left(\tilde{\bm{x}}\left(\tau\right)\right)\mathrm{d}\tau\right\} \left(\bm{x}^{t}-\bm{x}^{t,\mathrm{sgn}}\right)\right\Vert \nonumber \\
 & \quad\leq\left\Vert \bm{I}-\eta\int_{0}^{1}\nabla^{2}f\left(\tilde{\bm{x}}\left(\tau\right)\right)\mathrm{d}\tau\right\Vert \left\Vert \bm{x}^{t}-\bm{x}^{t,\mathrm{sgn}}\right\Vert _{2}\nonumber \\
 & \quad\leq\left\{ 1+3\eta\left(1-\int_{0}^{1}\left\Vert \tilde{\bm{x}}\left(\tau\right)\right\Vert _{2}^{2}\mathrm{d}\tau\right)+O\left(\eta\sqrt{\frac{n\log^{3}m}{m}}\right)\right\} \left\Vert \bm{x}^{t}-\bm{x}^{t,\mathrm{sgn}}\right\Vert _{2},\label{eq:xt-xt-sgn-main}
\end{align}
where we have invoked Lemma \ref{lemma:Hessian-UB-Stage1}. Furthermore,
one has for all $0\leq\tau\leq1$ 
\begin{align*}
\left\Vert \tilde{\bm{x}}\left(\tau\right)\right\Vert _{2}^{2} & \geq\left\Vert \bm{x}^{t}\right\Vert _{2}^{2}-\left|\left\Vert \tilde{\bm{x}}\left(\tau\right)\right\Vert _{2}^{2}-\left\Vert \bm{x}^{t}\right\Vert _{2}^{2}\right|\\
 & \geq\left\Vert \bm{x}^{t}\right\Vert _{2}^{2}-\left\Vert \tilde{\bm{x}}\left(\tau\right)-\bm{x}^{t}\right\Vert _{2}\left(\left\Vert \tilde{\bm{x}}\left(\tau\right)\right\Vert _{2}+\left\Vert \bm{x}^{t}\right\Vert _{2}\right)\\
 & \geq\left\Vert \bm{x}^{t}\right\Vert _{2}^{2}-\left\Vert \bm{x}^{t}-\bm{x}^{t,\text{sgn}}\right\Vert _{2}\left(\left\Vert \tilde{\bm{x}}\left(\tau\right)\right\Vert _{2}+\left\Vert \bm{x}^{t}\right\Vert _{2}\right).
\end{align*}
This combined with the norm conditions $\left\Vert \bm{x}^{t}\right\Vert _{2}\lesssim1$,
$\left\Vert \tilde{\bm{x}}\left(\tau\right)\right\Vert _{2}\lesssim1$
reveals that 
\[
\min_{0\leq\tau\leq1}\left\Vert \tilde{\bm{x}}\left(\tau\right)\right\Vert _{2}^{2}\geq\left\Vert \bm{x}^{t}\right\Vert _{2}^{2}+O\left(\left\Vert \bm{x}^{t}-\bm{x}^{t,\text{sgn}}\right\Vert _{2}\right),
\]
and hence we can further upper bound (\ref{eq:xt-xt-sgn-main}) as
\begin{align*}
 & \left\Vert \left\{ \bm{I}-\eta\int_{0}^{1}\nabla^{2}f\left(\tilde{\bm{x}}\left(\tau\right)\right)\mathrm{d}\tau\right\} \left(\bm{x}^{t}-\bm{x}^{t,\mathrm{sgn}}\right)\right\Vert \\
 & \quad\leq\left\{ 1+3\eta\left(1-\left\Vert \bm{x}^{t}\right\Vert _{2}^{2}\right)+\eta\cdot O\left(\left\Vert \bm{x}^{t}-\bm{x}^{t,\text{sgn}}\right\Vert _{2}+\sqrt{\frac{n\log^{3}m}{m}}\right)\right\} \left\Vert \bm{x}^{t}-\bm{x}^{t,\mathrm{sgn}}\right\Vert _{2}\\
 & \quad\leq\left\{ 1+3\eta\left(1-\left\Vert \bm{x}^{t}\right\Vert _{2}^{2}\right)+\eta\phi_{1}\right\} \left\Vert \bm{x}^{t}-\bm{x}^{t,\mathrm{sgn}}\right\Vert _{2},
\end{align*}
for some $|\phi_{1}|\ll\frac{1}{\log m}$, where the last line follows
from $m\gg n\log^{5}m$ and the fact (\ref{eq:1-logm-xt-xt-sgn}).

The remainder of this subsection is largely devoted to controlling
the gradient difference $\nabla f^{\mathrm{sgn}}\big(\bm{x}^{t,\mathrm{sgn}}\big)-\nabla f\big(\bm{x}^{t,\mathrm{sgn}}\big)$
in (\ref{eq:difference-xt-sgn}). By the definition of $f^{\mathrm{sgn}}\left(\cdot\right)$,
one has 
\begin{align*}
 & \nabla f^{\mathrm{sgn}}\big(\bm{x}^{t,\mathrm{sgn}}\big)-\nabla f\big(\bm{x}^{t,\mathrm{sgn}}\big)\\
 & \quad=\frac{1}{m}\sum_{i=1}^{m}\left\{ \big(\bm{a}_{i}^{\mathrm{sgn}\top}\bm{x}^{t,\mathrm{sgn}}\big)^{3}\bm{a}_{i}^{\mathrm{sgn}}-\big(\bm{a}_{i}^{\mathrm{sgn}\top}\bm{x}^{\natural}\big)^{2}\big(\bm{a}_{i}^{\mathrm{sgn}\top}\bm{x}^{t,\mathrm{sgn}}\big)\bm{a}_{i}^{\mathrm{sgn}}-\big(\bm{a}_{i}^{\top}\bm{x}^{t,\mathrm{sgn}}\big)^{3}\bm{a}_{i}+\big(\bm{a}_{i}^{\top}\bm{x}^{\natural}\big)^{2}\big(\bm{a}_{i}^{\top}\bm{x}^{t,\mathrm{sgn}}\big)\bm{a}_{i}\right\} \\
 & \quad=\underset{:=\bm{r}_{1}}{\underbrace{\frac{1}{m}\sum_{i=1}^{m}\left\{ \big(\bm{a}_{i}^{\mathrm{sgn}\top}\bm{x}^{t,\mathrm{sgn}}\big)^{3}\bm{a}_{i}^{\mathrm{sgn}}-\big(\bm{a}_{i}^{\top}\bm{x}^{t,\mathrm{sgn}}\big)^{3}\bm{a}_{i}\right\} }}-\underset{:=\bm{r}_{2}}{\underbrace{\frac{1}{m}\sum_{i=1}^{m}a_{i,1}^{2}\left(\bm{a}_{i}^{\mathrm{sgn}}\bm{a}_{i}^{\mathrm{sgn}\top}-\bm{a}_{i}\bm{a}_{i}^{\top}\right)\bm{x}^{t,\mathrm{sgn}}}}.
\end{align*}
Here, the last identity holds because of $\left(\bm{a}_{i}^{\top}\bm{x}^{\natural}\right)^{2}=\big(\bm{a}_{i}^{\mathrm{sgn}\top}\bm{x}^{\natural}\big)^{2}=a_{i,1}^{2}$
(see (\ref{eq:by-construction})). 
\begin{itemize}
\item We begin with the second term $\bm{r}_{2}$. By construction, one
has $\bm{a}_{i,\perp}^{\mathrm{sgn}}=\bm{a}_{i,\perp}$, $a_{i,1}^{\mathrm{sgn}}=\xi_{i}^{\mathrm{sgn}}\left|a_{i,1}\right|$
and $a_{i,1}=\xi_{1}\left|a_{i,1}\right|$. These taken together yield
\begin{equation}
\bm{a}_{i}^{\mathrm{sgn}}\bm{a}_{i}^{\mathrm{sgn}\top}-\bm{a}_{i}\bm{a}_{i}^{\top}=\left(\xi_{i}^{\mathrm{sgn}}-\xi_{i}\right)\left|a_{i,1}\right|\left[\begin{array}{cc}
0 & \bm{a}_{i,\perp}^{\top}\\
\bm{a}_{i,\perp} & \bm{0}
\end{array}\right],\label{eq:alal-alsgnalsgn}
\end{equation}
and hence $\bm{r}_{2}$ can be rewritten as 
\begin{align}
\bm{r}_{2} & =\left[\begin{array}{c}
\frac{1}{m}\sum_{i=1}^{m}\left(\xi_{i}^{\mathrm{sgn}}-\xi_{i}\right)\left|a_{i,1}\right|^{3}\bm{a}_{i,\perp}^{\top}\bm{x}_{\perp}^{t,\mathrm{sgn}}\\
x_{\parallel}^{t,\mathrm{sgn}}\cdot\frac{1}{m}\sum_{i=1}^{m}(\xi_{i}^{\mathrm{sgn}}-\xi_{i})\left|a_{i,1}\right|^{3}\bm{a}_{i,\perp}
\end{array}\right].\label{eq:r2-defn}
\end{align}
For the first entry of $\bm{r}_{2}$, the triangle inequality gives
\begin{align*}
\left|\frac{1}{m}\sum_{i=1}^{m}\left(\xi_{i}^{\mathrm{sgn}}-\xi_{i}\right)\left|a_{i,1}\right|^{3}\bm{a}_{i,\perp}^{\top}\bm{x}_{\perp}^{t,\mathrm{sgn}}\right| & \leq\underbrace{\left|\frac{1}{m}\sum_{i=1}^{m}\left|a_{i,1}\right|^{3}\xi_{i}\bm{a}_{i,\perp}^{\top}\bm{x}_{\perp}^{t,\mathrm{sgn}}\right|}_{:=\phi_{1}}+\underbrace{\left|\frac{1}{m}\sum_{i=1}^{m}\left|a_{i,1}\right|^{3}\xi_{i}^{\mathrm{sgn}}\bm{a}_{i,\perp}^{\top}\bm{x}_{\perp}^{t}\right|}_{:=\phi_{2}}\\
 & \quad+\underbrace{\left|\frac{1}{m}\sum_{i=1}^{m}\left|a_{i,1}\right|^{3}\xi_{i}^{\mathrm{sgn}}\bm{a}_{i,\perp}^{\top}\left(\bm{x}_{\perp}^{t,\mathrm{sgn}}-\bm{x}_{\perp}^{t}\right)\right|}_{:=\phi_{3}}.
\end{align*}
Regarding $\phi_{1}$, we make use of the independence between $\xi_{i}$
and $\left|a_{i,1}\right|^{3}\bm{a}_{i,\perp}^{\top}\bm{x}_{\perp}^{t,\mathrm{sgn}}$
and invoke the Bernstein inequality (see Lemma \ref{lemma:bernstein})
to reach that with probability at least $1-O\left(m^{-10}\right),$
\[
\phi_{1}\lesssim\frac{1}{m}\left(\sqrt{V_{1}\log m}+B_{1}\log m\right),
\]
where $V_{1}$ and $B_{1}$ are defined to be 
\[
V_{1}:=\sum_{i=1}^{m}\left|a_{i,1}\right|^{6}\left|\bm{a}_{i,\perp}^{\top}\bm{x}_{\perp}^{t,\mathrm{sgn}}\right|^{2}\qquad\text{and}\qquad B_{1}:=\max_{1\leq i\leq m}\left\{ \left|a_{i,1}\right|^{3}\left|\bm{a}_{i,\perp}^{\top}\bm{x}_{\perp}^{t,\mathrm{sgn}}\right|\right\} .
\]
It is easy to see from Lemma \ref{lemma:ai-uniform-concentration}
and the incoherence condition (\ref{eq:consequence-incoherence-ai-perp-x-sgn})
that with probability exceeding $1-O\left(m^{-10}\right)$, $V_{1}\lesssim m\left\Vert \bm{x}_{\perp}^{t,\mathrm{sgn}}\right\Vert _{2}^{2}$
and $B_{1}\lesssim\log^{2}m\left\Vert \bm{x}_{\perp}^{t,\mathrm{sgn}}\right\Vert _{2}$,
which implies 
\begin{align*}
\phi_{1} & \lesssim\left(\sqrt{\frac{\log m}{m}}+\frac{\log^{3}m}{m}\right)\left\Vert \bm{x}_{\perp}^{t,\mathrm{sgn}}\right\Vert _{2}\asymp\sqrt{\frac{\log m}{m}}\left\Vert \bm{x}_{\perp}^{t,\mathrm{sgn}}\right\Vert _{2},
\end{align*}
as long as $m\gg\log^{5}m$. Similarly, one can obtain 
\begin{align*}
\phi_{2} & \lesssim\sqrt{\frac{\log m}{m}}\left\Vert \bm{x}_{\perp}^{t}\right\Vert _{2}.
\end{align*}
The last term $\phi_{3}$ can be bounded through the Cauchy-Schwarz
inequality. Specifically, one has 
\begin{align*}
\phi_{3} & \leq\left\Vert \frac{1}{m}\sum_{i=1}^{m}\left|a_{i,1}\right|^{3}\xi_{i}^{\mathrm{sgn}}\bm{a}_{i,\perp}\right\Vert _{2}\left\Vert \bm{x}_{\perp}^{t,\mathrm{sgn}}-\bm{x}_{\perp}^{t}\right\Vert _{2}\lesssim\sqrt{\frac{n\log^{3}m}{m}}\left\Vert \bm{x}_{\perp}^{t,\mathrm{sgn}}-\bm{x}_{\perp}^{t}\right\Vert _{2},
\end{align*}
where the second relation arises from Lemma \ref{lemma:hessian-concentration}.
The previous three bounds taken collectively yield 
\begin{align}
\left|\frac{1}{m}\sum_{i=1}^{m}\left(\xi_{i}^{\mathrm{sgn}}-\xi_{i}\right)\left|a_{i,1}\right|^{3}\bm{a}_{i,\perp}^{\top}\bm{x}_{\perp}^{t,\mathrm{sgn}}\right| & \lesssim\sqrt{\frac{\log m}{m}}\left(\left\Vert \bm{x}_{\perp}^{t,\mathrm{sgn}}\right\Vert _{2}+\left\Vert \bm{x}_{\perp}^{t}\right\Vert _{2}\right)+\sqrt{\frac{n\log^{3}m}{m}}\left\Vert \bm{x}_{\perp}^{t,\mathrm{sgn}}-\bm{x}_{\perp}^{t}\right\Vert _{2}\nonumber \\
 & \lesssim\sqrt{\frac{\log m}{m}}\left\Vert \bm{x}_{\perp}^{t}\right\Vert _{2}+\sqrt{\frac{n\log^{3}m}{m}}\left\Vert \bm{x}_{\perp}^{t,\mathrm{sgn}}-\bm{x}_{\perp}^{t}\right\Vert _{2}.\label{eq:UB5}
\end{align}
Here the second inequality results from the triangle inequality $\left\Vert \bm{x}_{\perp}^{t,\mathrm{sgn}}\right\Vert _{2}\leq\left\Vert \bm{x}_{\perp}^{t}\right\Vert _{2}+\left\Vert \bm{x}_{\perp}^{t,\mathrm{sgn}}-\bm{x}_{\perp}^{t}\right\Vert _{2}$
and the fact that $\sqrt{\frac{\log m}{m}}\le\sqrt{\frac{n\log^{3}m}{m}}$.
In addition, for the second through the $n$th entries of $\bm{r}_{2}$,
one can again invoke Lemma \ref{lemma:hessian-concentration} to obtain
\begin{align}
\left\Vert \frac{1}{m}\sum_{i=1}^{m}\left|a_{i,1}\right|^{3}\left(\xi_{i}^{\mathrm{sgn}}-\xi_{i}\right)\bm{a}_{i,\perp}\right\Vert _{2} & \leq\left\Vert \frac{1}{m}\sum_{i=1}^{m}\left|a_{i,1}\right|^{3}\xi_{i}^{\mathrm{sgn}}\bm{a}_{i,\perp}\right\Vert _{2}+\left\Vert \frac{1}{m}\sum_{i=1}^{m}\left|a_{i,1}\right|^{3}\xi_{i}\bm{a}_{i,\perp}\right\Vert _{2}\nonumber \\
 & \lesssim\sqrt{\frac{n\log^{3}m}{m}}.\label{eq:xi-xi-sgn-ai-duplicate}
\end{align}
This combined with (\ref{eq:r2-defn}) and (\ref{eq:UB5}) yields
\begin{align*}
\|\bm{r}_{2}\|_{2} & \lesssim\sqrt{\frac{\log m}{m}}\left\Vert \bm{x}_{\perp}^{t}\right\Vert _{2}+\sqrt{\frac{n\log^{3}m}{m}}\left\Vert \bm{x}_{\perp}^{t,\mathrm{sgn}}-\bm{x}_{\perp}^{t}\right\Vert _{2}+\left|x_{\parallel}^{t,\mathrm{sgn}}\right|\sqrt{\frac{n\log^{3}m}{m}}.
\end{align*}
\item Moving on to the term $\bm{r}_{1}$, we can also decompose
\[
\bm{r}_{1}=\left[\begin{array}{c}
\frac{1}{m}\sum_{i=1}^{m}\left\{ \big(\bm{a}_{i}^{\mathrm{sgn}\top}\bm{x}^{t,\mathrm{sgn}}\big)^{3}a_{i,1}^{\text{sgn}}-\big(\bm{a}_{i}^{\top}\bm{x}^{t,\mathrm{sgn}}\big)^{3}a_{i,1}\right\} \\
\frac{1}{m}\sum_{i=1}^{m}\left\{ \big(\bm{a}_{i}^{\mathrm{sgn}\top}\bm{x}^{t,\mathrm{sgn}}\big)^{3}\bm{a}_{i,\perp}^{\text{sgn}}-\big(\bm{a}_{i}^{\top}\bm{x}^{t,\mathrm{sgn}}\big)^{3}\bm{a}_{i,\perp}\right\} 
\end{array}\right].
\]
For the second through the $n$th entries, we see that 
\begin{align*}
 & \frac{1}{m}\sum_{i=1}^{m}\left\{ \big(\bm{a}_{i}^{\mathrm{sgn}\top}\bm{x}^{t,\mathrm{sgn}}\big)^{3}\bm{a}_{i,\perp}^{\text{sgn}}-\big(\bm{a}_{i}^{\top}\bm{x}^{t,\mathrm{sgn}}\big)^{3}\bm{a}_{i,\perp}\right\} \overset{\left(\text{i}\right)}{=}\frac{1}{m}\sum_{i=1}^{m}\left\{ \big(\bm{a}_{i}^{\mathrm{sgn}\top}\bm{x}^{t,\mathrm{sgn}}\big)^{3}-\big(\bm{a}_{i}^{\top}\bm{x}^{t,\mathrm{sgn}}\big)^{3}\right\} \bm{a}_{i,\perp}\\
 & \quad\overset{\left(\text{ii}\right)}{=}\frac{1}{m}\sum_{i=1}^{m}\left\{ \left(\xi_{i}^{\mathrm{sgn}}-\xi_{i}\right)\left|a_{i,1}\right|x_{\parallel}^{t,\mathrm{sgn}}\left[\left(\bm{a}_{i}^{\mathrm{sgn}\top}\bm{x}^{t,\mathrm{sgn}}\right)^{2}+\left(\bm{a}_{i}^{\top}\bm{x}^{t,\mathrm{sgn}}\right)^{2}+\left(\bm{a}_{i}^{\mathrm{sgn}\top}\bm{x}^{t,\mathrm{sgn}}\right)\left(\bm{a}_{i}^{\top}\bm{x}^{t,\mathrm{sgn}}\right)\right]\right\} \bm{a}_{i,\perp}\\
 & \quad=\frac{x_{\parallel}^{t,\mathrm{sgn}}}{m}\sum_{i=1}^{m}\left\{ \left(\xi_{i}^{\mathrm{sgn}}-\xi_{i}\right)\left|a_{i,1}\right|\left[\left(\bm{a}_{i}^{\mathrm{sgn}\top}\bm{x}^{t,\mathrm{sgn}}\right)^{2}+\left(\bm{a}_{i}^{\top}\bm{x}^{t,\mathrm{sgn}}\right)^{2}+\left(\bm{a}_{i}^{\mathrm{sgn}\top}\bm{x}^{t,\mathrm{sgn}}\right)\left(\bm{a}_{i}^{\top}\bm{x}^{t,\mathrm{sgn}}\right)\right]\right\} \bm{a}_{i,\perp},
\end{align*}
where (i) follows from $\bm{a}_{i,\perp}^{\text{sgn}}=\bm{a}_{i,\perp}$
and (ii) relies on the elementary identity $a^{3}-b^{3}=\left(a-b\right)\left(a^{2}+b^{2}+ab\right)$.
Treating $\frac{1}{m}\sum_{i=1}^{m}\left(\bm{a}_{i}^{\mathrm{sgn}\top}\bm{x}^{t,\mathrm{sgn}}\right)^{2}a_{i,1}^{\text{sgn}}\bm{a}_{i,\perp}$
as the first column (except its first entry) of $\frac{1}{m}\sum_{i=1}^{m}\left(\bm{a}_{i}^{\mathrm{sgn}\top}\bm{x}^{t,\mathrm{sgn}}\right)^{2}\bm{a}_{i}^{\mathrm{sgn}}\bm{a}_{i}^{\mathrm{sgn}\top}$,
by Lemma \ref{lemma:hessian-concentration} and the incoherence condition
(\ref{eq:consequence-incoherence-ai-sgn-x-sgn}), we have 
\[
\frac{1}{m}\sum_{i=1}^{m}\xi_{i}^{\text{sgn}}\left|a_{i,1}\right|\left(\bm{a}_{i}^{\mathrm{sgn}\top}\bm{x}^{t,\mathrm{sgn}}\right)^{2}\bm{a}_{i,\perp}=\frac{1}{m}\sum_{i=1}^{m}\left(\bm{a}_{i}^{\mathrm{sgn}\top}\bm{x}^{t,\mathrm{sgn}}\right)^{2}a_{i,1}^{\text{sgn}}\bm{a}_{i,\perp}=2x_{\parallel}^{t,\text{sgn}}\bm{x}_{\perp}^{t,\text{sgn}}+\bm{v}_{1},
\]
where $\left\Vert \bm{v}_{1}\right\Vert _{2}\lesssim\sqrt{\frac{n\log^{3}m}{m}}$.
Similarly, 
\[
-\frac{1}{m}\sum_{i=1}^{m}\xi_{i}\left|a_{i,1}\right|\left(\bm{a}_{i}^{\top}\bm{x}^{t,\mathrm{sgn}}\right)^{2}\bm{a}_{i,\perp}=-2x_{\parallel}^{t,\text{sgn}}\bm{x}_{\perp}^{t,\text{sgn}}+\bm{v}_{2},
\]
where $\left\Vert \bm{v}_{2}\right\Vert _{2}\lesssim\sqrt{\frac{n\log^{3}m}{m}}$.
Moreover, we have 
\begin{align*}
 & \frac{1}{m}\sum_{i=1}^{m}\xi_{i}^{\text{sgn}}\left|a_{i,1}\right|\left(\bm{a}_{i}^{\top}\bm{x}^{t,\mathrm{sgn}}\right)^{2}\bm{a}_{i,\perp}\\
 & \quad=\frac{1}{m}\sum_{i=1}^{m}\xi_{i}^{\text{sgn}}\left|a_{i,1}\right|\left(\bm{a}_{i}^{\text{sgn}\top}\bm{x}^{t,\mathrm{sgn}}\right)^{2}\bm{a}_{i,\perp}+\frac{1}{m}\sum_{i=1}^{m}\xi_{i}^{\text{sgn}}\left|a_{i,1}\right|\left[\left(\bm{a}_{i}^{\top}\bm{x}^{t,\mathrm{sgn}}\right)^{2}-\left(\bm{a}_{i}^{\text{sgn}\top}\bm{x}^{t,\text{sgn}}\right)^{2}\right]\bm{a}_{i,\perp}\\
 & \quad=2x_{\parallel}^{t,\text{sgn}}\bm{x}_{\perp}^{t,\text{sgn}}+\bm{v}_{1}+\bm{v}_{3},
\end{align*}
where $\bm{v}_{3}$ is defined as 
\begin{align}
\bm{v}_{3} & =\frac{1}{m}\sum_{i=1}^{m}\xi_{i}^{\text{sgn}}\left|a_{i,1}\right|\left[\left(\bm{a}_{i}^{\top}\bm{x}^{t,\mathrm{sgn}}\right)^{2}-\left(\bm{a}_{i}^{\text{sgn}\top}\bm{x}^{t,\text{sgn}}\right)^{2}\right]\bm{a}_{i,\perp}\nonumber \\
 & =2x_{\parallel}^{t,\text{sgn}}\frac{1}{m}\sum_{i=1}^{m}\left(\xi_{i}-\xi_{i}^{\mathrm{sgn}}\right)\left(\bm{a}_{i,\perp}^{\top}\bm{x}_{\perp}^{t,\text{sgn}}\right)\xi_{i}^{\text{sgn}}\left|a_{i,1}\right|^{2}\bm{a}_{i,\perp}\nonumber \\
 & =2x_{\parallel}^{t,\text{sgn}}\frac{1}{m}\sum_{i=1}^{m}\left(\xi_{i}\xi_{i}^{\text{sgn}}-1\right)\left|a_{i,1}\right|^{2}\bm{a}_{i,\perp}\bm{a}_{i,\perp}^{\top}\bm{x}_{\perp}^{t,\text{sgn}}.\label{eq:def_v3}
\end{align}
Here the second equality comes from the identity (\ref{eq:identity-a-a-sgn}).
Similarly one can get 
\begin{align*}
-\frac{1}{m}\sum_{i=1}^{m}\xi_{i}\left|a_{i,1}\right|\left(\bm{a}_{i}^{\mathrm{sgn}\top}\bm{x}^{t,\mathrm{sgn}}\right)^{2}\bm{a}_{i,\perp} & =-2x_{\parallel}^{t,\text{sgn}}\bm{x}_{\perp}^{t,\text{sgn}}-\bm{v}_{2}-\bm{v}_{4},
\end{align*}
where $\bm{v}_{4}$ obeys 
\begin{align*}
\bm{v}_{4} & =\frac{1}{m}\sum_{i=1}^{m}\xi_{i}\left|a_{i,1}\right|\left[\left(\bm{a}_{i}^{\text{sgn}\top}\bm{x}^{t,\text{sgn}}\right)^{2}-\left(\bm{a}_{i}^{\top}\bm{x}^{t,\mathrm{sgn}}\right)^{2}\right]\bm{a}_{i,\perp}\\
 & =2x_{\parallel}^{t,\text{sgn}}\frac{1}{m}\sum_{i=1}^{m}\left(\xi_{i}\xi_{i}^{\text{sgn}}-1\right)\left|a_{i,1}\right|^{2}\bm{a}_{i,\perp}\bm{a}_{i,\perp}^{\top}\bm{x}_{\perp}^{t,\text{sgn}}.
\end{align*}
It remains to bound $\frac{1}{m}\sum_{i=1}^{m}\left(\xi_{i}^{\text{sgn}}-\xi_{i}\right)\left|a_{i,1}\right|\big(\bm{a}_{i}^{\mathrm{sgn}\top}\bm{x}^{t,\mathrm{sgn}}\big)\left(\bm{a}_{i}^{\top}\bm{x}^{t,\mathrm{sgn}}\right)\bm{a}_{i,\perp}$.
To this end, we have 
\begin{align*}
 & \frac{1}{m}\sum_{i=1}^{m}\xi_{i}^{\text{sgn}}\left|a_{i,1}\right|\big(\bm{a}_{i}^{\mathrm{sgn}\top}\bm{x}^{t,\mathrm{sgn}}\big)\left(\bm{a}_{i}^{\top}\bm{x}^{t,\mathrm{sgn}}\right)\bm{a}_{i,\perp}\\
 & \quad=\frac{1}{m}\sum_{i=1}^{m}\xi_{i}^{\text{sgn}}\left|a_{i,1}\right|\left(\bm{a}_{i}^{\mathrm{sgn}\top}\bm{x}^{t,\mathrm{sgn}}\right)^{2}\bm{a}_{i,\perp}+\frac{1}{m}\sum_{i=1}^{m}\xi_{i}^{\text{sgn}}\left|a_{i,1}\right|\left(\bm{a}_{i}^{\mathrm{sgn}\top}\bm{x}^{t,\mathrm{sgn}}\right)\left[\left(\bm{a}_{i}^{\top}\bm{x}^{t,\mathrm{sgn}}\right)-\left(\bm{a}_{i}^{\mathrm{sgn}\top}\bm{x}^{t,\mathrm{sgn}}\right)\right]\bm{a}_{i,\perp}\\
 & \quad=2x_{\parallel}^{t,\text{sgn}}\bm{x}_{\perp}^{t,\text{sgn}}+\bm{v}_{1}+\bm{v}_{5},
\end{align*}
where 
\[
\bm{v}_{5}=x_{\parallel}^{t,\text{sgn}}\frac{1}{m}\sum_{i=1}^{m}\left(\xi_{i}\xi_{i}^{\text{sgn}}-1\right)\left|a_{i,1}\right|^{2}\bm{a}_{i,\perp}\bm{a}_{i}^{\mathrm{sgn}\top}\bm{x}^{t,\mathrm{sgn}}.
\]
The same argument yields 
\[
-\frac{1}{m}\sum_{i=1}^{m}\xi_{i}\left|a_{i,1}\right|\big(\bm{a}_{i}^{\mathrm{sgn}\top}\bm{x}^{t,\mathrm{sgn}}\big)\left(\bm{a}_{i}^{\top}\bm{x}^{t,\mathrm{sgn}}\right)\bm{a}_{i,\perp}=-2x_{\parallel}^{t,\text{sgn}}\bm{x}_{\perp}^{t,\text{sgn}}-\bm{v}_{2}-\bm{v}_{6},
\]
where 
\[
\bm{v}_{6}=x_{\parallel}^{t,\text{sgn}}\frac{1}{m}\sum_{i=1}^{m}\left(\xi_{i}\xi_{i}^{\text{sgn}}-1\right)\left|a_{i,1}\right|^{2}\bm{a}_{i,\perp}\bm{a}_{i}^{\mathrm{sgn}\top}\bm{x}^{t,\mathrm{sgn}}.
\]
Combining all of the previous bounds and recognizing that $\bm{v}_{3}=\bm{v}_{4}$
and $\bm{v}_{5}=\bm{v}_{6}$, we arrive at 
\[
\left\Vert \frac{1}{m}\sum_{i=1}^{m}\left\{ \big(\bm{a}_{i}^{\mathrm{sgn}\top}\bm{x}^{t,\mathrm{sgn}}\big)^{3}\bm{a}_{i,\perp}^{\text{sgn}}-\big(\bm{a}_{i}^{\top}\bm{x}^{t,\mathrm{sgn}}\big)^{3}\bm{a}_{i,\perp}\right\} \right\Vert _{2}\lesssim\|\bm{v}_{1}\|_{2}+\|\bm{v}_{2}\|_{2}\lesssim\sqrt{\frac{n\log^{3}m}{m}}\left|x_{\parallel}^{t,\text{sgn}}\right|.
\]
Regarding the first entry of $\bm{r}_{1}$, one has 
\begin{align*}
 & \left|\frac{1}{m}\sum_{i=1}^{m}\left\{ \big(\bm{a}_{i}^{\mathrm{sgn}\top}\bm{x}^{t,\mathrm{sgn}}\big)^{3}a_{i,1}^{\mathrm{sgn}}-\big(\bm{a}_{i}^{\top}\bm{x}^{t,\mathrm{sgn}}\big)^{3}a_{i,1}\right\} \right|\\
 & \quad=\left|\frac{1}{m}\sum_{i=1}^{m}\left\{ \left(\xi_{i}^{\mathrm{sgn}}\left|a_{i,1}\right|x_{\|}^{t,\mathrm{sgn}}+\bm{a}_{i,\perp}^{\top}\bm{x}_{\perp}^{t,\mathrm{sgn}}\right)^{3}\xi_{i}^{\mathrm{sgn}}\left|a_{i,1}\right|-\left(\xi_{i}\left|a_{i,1}\right|x_{\|}^{t,\mathrm{sgn}}+\bm{a}_{i,\perp}^{\top}\bm{x}_{\perp}^{t,\mathrm{sgn}}\right)^{3}\xi_{i}\left|a_{i,1}\right|\right\} \right|\\
 & \quad=\left|\frac{1}{m}\sum_{i=1}^{m}\left(\xi_{i}^{\mathrm{sgn}}-\xi_{i}\right)\left|a_{i,1}\right|\left\{ 3\left|a_{i,1}\right|^{2}\left|x_{\|}^{t,\mathrm{sgn}}\right|^{2}\bm{a}_{i,\perp}^{\top}\bm{x}_{\perp}^{t,\mathrm{sgn}}+\left(\bm{a}_{i,\perp}^{\top}\bm{x}_{\perp}^{t,\mathrm{sgn}}\right)^{3}\right\} \right|.
\end{align*}
In view of the independence between $\xi_{i}$ and $\left|a_{i,1}\right|\big(\bm{a}_{i,\perp}^{\top}\bm{x}_{\perp}^{t,\mathrm{sgn}}\big)^{3}$,
from the Bernstein's inequality (see Lemma \ref{lemma:bernstein}),
we have that 
\[
\left|\frac{1}{m}\sum_{i=1}^{m}\xi_{i}\left|a_{i,1}\right|\left(\bm{a}_{i,\perp}^{\top}\bm{x}_{\perp}^{t,\mathrm{sgn}}\right)^{3}\right|\lesssim\frac{1}{m}\left(\sqrt{V_{2}\log m}+B_{2}\log m\right)
\]
holds with probability exceeding $1-O\left(m^{-10}\right)$, where
\[
V_{2}:=\sum_{i=1}^{m}\left|a_{i,1}\right|^{2}\left(\bm{a}_{i,\perp}^{\top}\bm{x}_{\perp}^{t,\mathrm{sgn}}\right)^{6}\qquad\text{and}\qquad B_{2}:=\max_{1\leq i\leq m}\left|a_{i,1}\right|\left|\bm{a}_{i,\perp}^{\top}\bm{x}_{\perp}^{t,\mathrm{sgn}}\right|^{3}.
\]
It is straightforward to check that $V_{2}\lesssim m\left\Vert \bm{x}_{\perp}^{t,\mathrm{sgn}}\right\Vert _{2}^{6}$
and $B_{2}\lesssim\log^{2}m\left\Vert \bm{x}_{\perp}^{t,\mathrm{sgn}}\right\Vert _{2}^{3}$,
which further implies 
\[
\left|\frac{1}{m}\sum_{i=1}^{m}\xi_{i}\left|a_{i,1}\right|\left(\bm{a}_{i,\perp}^{\top}\bm{x}_{\perp}^{t,\mathrm{sgn}}\right)^{3}\right|\lesssim\sqrt{\frac{\log m}{m}}\left\Vert \bm{x}_{\perp}^{t,\mathrm{sgn}}\right\Vert _{2}^{3}+\frac{\log^{3}m}{m}\left\Vert \bm{x}_{\perp}^{t,\mathrm{sgn}}\right\Vert _{2}^{3}\asymp\sqrt{\frac{\log m}{m}}\left\Vert \bm{x}_{\perp}^{t,\mathrm{sgn}}\right\Vert _{2}^{3},
\]
as long as $m\gg\log^{5}m$. For the term involving $\xi_{i}^{\text{sgn}}$,
we have 
\begin{align*}
\frac{1}{m}\sum_{i=1}^{m}\xi_{i}^{\text{sgn}}\left|a_{i,1}\right|\left(\bm{a}_{i,\perp}^{\top}\bm{x}_{\perp}^{t,\mathrm{sgn}}\right)^{3} & =\underbrace{\frac{1}{m}\sum_{i=1}^{m}\xi_{i}^{\text{sgn}}\left|a_{i,1}\right|\left(\bm{a}_{i,\perp}^{\top}\bm{x}_{\perp}^{t}\right)^{3}}_{:=\theta_{1}}+\underbrace{\frac{1}{m}\sum_{i=1}^{m}\xi_{i}^{\text{sgn}}\left|a_{i,1}\right|\left[\left(\bm{a}_{i,\perp}^{\top}\bm{x}_{\perp}^{t}\right)^{3}-\left(\bm{a}_{i,\perp}^{\top}\bm{x}_{\perp}^{t,\mathrm{sgn}}\right)^{3}\right]}_{:=\theta_{2}}.
\end{align*}
Similarly one can obtain 
\[
\left|\theta_{1}\right|\lesssim\sqrt{\frac{\log m}{m}}\left\Vert \bm{x}_{\perp}^{t}\right\Vert _{2}^{3}.
\]
Expand $\theta_{2}$ using the elementary identity $a^{3}-b^{3}=\left(a-b\right)\left(a^{2}+ab+b^{2}\right)$
to get 
\begin{align*}
\theta_{2} & =\frac{1}{m}\sum_{i=1}^{m}\xi_{i}^{\text{sgn}}\left|a_{i,1}\right|\bm{a}_{i,\perp}^{\top}\left(\bm{x}_{\perp}^{t}-\bm{x}_{\perp}^{t,\text{sgn}}\right)\left[\left(\bm{a}_{i,\perp}^{\top}\bm{x}_{\perp}^{t}\right)^{2}+\left(\bm{a}_{i,\perp}^{\top}\bm{x}_{\perp}^{t,\mathrm{sgn}}\right)^{2}+\left(\bm{a}_{i,\perp}^{\top}\bm{x}_{\perp}^{t}\right)\left(\bm{a}_{i,\perp}^{\top}\bm{x}_{\perp}^{t,\mathrm{sgn}}\right)\right]\\
 & =\frac{1}{m}\sum_{i=1}^{m}\left(\bm{a}_{i,\perp}^{\top}\bm{x}_{\perp}^{t}\right)^{2}\xi_{i}^{\text{sgn}}\left|a_{i,1}\right|\bm{a}_{i,\perp}^{\top}\left(\bm{x}_{\perp}^{t}-\bm{x}_{\perp}^{t,\text{sgn}}\right)\\
 & \quad+\frac{1}{m}\sum_{i=1}^{m}\left(\bm{a}_{i,\perp}^{\top}\bm{x}_{\perp}^{t,\mathrm{sgn}}\right)^{2}\xi_{i}^{\text{sgn}}\left|a_{i,1}\right|\bm{a}_{i,\perp}^{\top}\left(\bm{x}_{\perp}^{t}-\bm{x}_{\perp}^{t,\text{sgn}}\right)\\
 & \quad+\frac{1}{m}\sum_{i=1}^{m}\left(\bm{a}_{i,\perp}^{\top}\bm{x}_{\perp}^{t}\right)\bm{a}_{i,\perp}^{\top}\left(\bm{x}_{\perp}^{t,\mathrm{sgn}}-\bm{x}_{\perp}^{t}\right)\xi_{i}^{\text{sgn}}\left|a_{i,1}\right|\bm{a}_{i,\perp}^{\top}\left(\bm{x}_{\perp}^{t}-\bm{x}_{\perp}^{t,\text{sgn}}\right).
\end{align*}
Once more, we can apply Lemma \ref{lemma:hessian-concentration} with
the incoherence conditions (\ref{eq:consequence-incoherence-ai-perp-x-t})
and (\ref{eq:consequence-incoherence-ai-perp-x-sgn}) to obtain 
\begin{align*}
\left\Vert \frac{1}{m}\sum_{i=1}^{m}\left(\bm{a}_{i,\perp}^{\top}\bm{x}_{\perp}^{t}\right)^{2}\xi_{i}^{\text{sgn}}\left|a_{i,1}\right|\bm{a}_{i,\perp}^{\top}\right\Vert _{2} & \lesssim\sqrt{\frac{n\log^{3}m}{m}};\\
\left\Vert \frac{1}{m}\sum_{i=1}^{m}\left(\bm{a}_{i,\perp}^{\top}\bm{x}_{\perp}^{t,\mathrm{sgn}}\right)^{2}\xi_{i}^{\text{sgn}}\left|a_{i,1}\right|\bm{a}_{i,\perp}^{\top}\right\Vert _{2} & \lesssim\sqrt{\frac{n\log^{3}m}{m}}.
\end{align*}
In addition, one can use the Cauchy-Schwarz inequality to deduce that
\begin{align*}
 & \left|\frac{1}{m}\sum_{i=1}^{m}\left(\bm{a}_{i,\perp}^{\top}\bm{x}_{\perp}^{t}\right)\bm{a}_{i,\perp}^{\top}\left(\bm{x}_{\perp}^{t,\mathrm{sgn}}-\bm{x}_{\perp}^{t}\right)\xi_{i}^{\text{sgn}}\left|a_{i,1}\right|\bm{a}_{i,\perp}^{\top}\left(\bm{x}_{\perp}^{t}-\bm{x}_{\perp}^{t,\text{sgn}}\right)\right|\\
 & \quad\leq\sqrt{\frac{1}{m}\sum_{i=1}^{m}\left(\bm{a}_{i,\perp}^{\top}\bm{x}_{\perp}^{t}\right)^{2}\left[\bm{a}_{i,\perp}^{\top}\left(\bm{x}_{\perp}^{t,\mathrm{sgn}}-\bm{x}_{\perp}^{t}\right)\right]^{2}}\sqrt{\frac{1}{m}\sum_{i=1}^{m}\left|a_{i,1}\right|^{2}\left[\bm{a}_{i,\perp}^{\top}\left(\bm{x}_{\perp}^{t}-\bm{x}_{\perp}^{t,\text{sgn}}\right)\right]^{2}}\\
 & \quad\leq\sqrt{\left\Vert \frac{1}{m}\sum_{i=1}^{m}\left(\bm{a}_{i,\perp}^{\top}\bm{x}_{\perp}^{t}\right)^{2}\bm{a}_{i,\perp}\bm{a}_{i,\perp}^{\top}\right\Vert \big\|\bm{x}_{\perp}^{t,\mathrm{sgn}}-\bm{x}_{\perp}^{t}\big\|_{2}^{2}}\sqrt{\left\Vert \frac{1}{m}\sum_{i=1}^{m}\left|a_{i,1}\right|^{2}\bm{a}_{i,\perp}\bm{a}_{i,\perp}^{\top}\right\Vert \big\|\bm{x}_{\perp}^{t,\mathrm{sgn}}-\bm{x}_{\perp}^{t}\big\|_{2}^{2}}\\
 & \quad\lesssim\left\Vert \bm{x}_{\perp}^{t}-\bm{x}_{\perp}^{\text{sgn}}\right\Vert _{2}^{2},
\end{align*}
where the last inequality comes from Lemma \ref{lemma:hessian-concentration}.
Combine the preceding bounds to reach 
\[
\left|\theta_{2}\right|\lesssim\sqrt{\frac{n\log^{3}m}{m}}\left\Vert \bm{x}_{\perp}^{t}-\bm{x}_{\perp}^{\text{sgn}}\right\Vert _{2}+\left\Vert \bm{x}_{\perp}^{t}-\bm{x}_{\perp}^{\text{sgn}}\right\Vert _{2}^{2}.
\]
Applying the similar arguments as above we get
\begin{align*}
 & \left|\left|x_{\|}^{t,\mathrm{sgn}}\right|^{2}\frac{3}{m}\sum_{i=1}^{m}\big(\xi_{i}^{\mathrm{sgn}}-\xi_{i}\big)\left|a_{i,1}\right|^{3}\bm{a}_{i,\perp}^{\top}\bm{x}_{\perp}^{t,\mathrm{sgn}}\right|\\
 & \quad\lesssim\left|x_{\|}^{t,\mathrm{sgn}}\right|^{2}\left(\sqrt{\frac{\log m}{m}}\left\Vert \bm{x}_{\perp}^{t,\text{sgn}}\right\Vert _{2}+\sqrt{\frac{\log m}{m}}\left\Vert \bm{x}_{\perp}^{t}\right\Vert _{2}+\sqrt{\frac{n\log^{3}m}{m}}\left\Vert \bm{x}_{\perp}^{t}-\bm{x}_{\perp}^{t,\mathrm{sgn}}\right\Vert _{2}\right)\\
 & \quad\lesssim\left|x_{\|}^{t,\mathrm{sgn}}\right|^{2}\left(\sqrt{\frac{\log m}{m}}\left\Vert \bm{x}_{\perp}^{t}\right\Vert _{2}+\sqrt{\frac{n\log^{3}m}{m}}\left\Vert \bm{x}_{\perp}^{t}-\bm{x}_{\perp}^{t,\mathrm{sgn}}\right\Vert _{2}\right),
\end{align*}
where the last line follows from the triangle inequality $\left\Vert \bm{x}_{\perp}^{t,\text{sgn}}\right\Vert _{2}\leq\left\Vert \bm{x}_{\perp}^{t}\right\Vert _{2}+\left\Vert \bm{x}_{\perp}^{t}-\bm{x}_{\perp}^{t,\mathrm{sgn}}\right\Vert _{2}$
and the fact that $\sqrt{\frac{\log m}{m}}\leq\sqrt{\frac{n\log^{3}m}{m}}$.
Putting the above results together yields 
\begin{align*}
\|\bm{r}_{1}\|_{2} & \lesssim\sqrt{\frac{n\log^{3}m}{m}}\left|x_{\parallel}^{t,\text{sgn}}\right|+\sqrt{\frac{\log m}{m}}\left(\left\Vert \bm{x}_{\perp}^{t,\mathrm{sgn}}\right\Vert _{2}+\left\Vert \bm{x}_{\perp}^{t}\right\Vert _{2}\right)+\sqrt{\frac{n\log^{3}m}{m}}\left\Vert \bm{x}_{\perp}^{t}-\bm{x}_{\perp}^{\text{sgn}}\right\Vert _{2}\\
 & \quad+\left\Vert \bm{x}_{\perp}^{t}-\bm{x}_{\perp}^{\text{sgn}}\right\Vert _{2}^{2}+\left|x_{\|}^{t,\mathrm{sgn}}\right|^{2}\left(\sqrt{\frac{\log m}{m}}\left\Vert \bm{x}_{\perp}^{t}\right\Vert _{2}+\sqrt{\frac{n\log^{3}m}{m}}\left\Vert \bm{x}_{\perp}^{t}-\bm{x}_{\perp}^{\text{sgn}}\right\Vert _{2}\right),\ 
\end{align*}
which can be further simplified to 
\begin{align*}
\left\Vert \bm{r}_{1}\right\Vert _{2} & \lesssim\sqrt{\frac{n\log^{3}m}{m}}\left|x_{\parallel}^{t}\right|+\sqrt{\frac{\log m}{m}}\left\Vert \bm{x}_{\perp}^{t}\right\Vert _{2}+\sqrt{\frac{n\log^{3}m}{m}}\left\Vert \bm{x}^{t}-\bm{x}^{\text{sgn}}\right\Vert _{2}+\left\Vert \bm{x}^{t}-\bm{x}^{\text{sgn}}\right\Vert _{2}^{2}.
\end{align*}
\item Combine all of the above estimates to reach 
\begin{align*}
\left\Vert \bm{x}^{t+1}-\bm{x}^{t+1,\mathrm{sgn}}\right\Vert _{2} & \leq\left\Vert \left\{ \bm{I}-\eta\int_{0}^{1}\nabla^{2}f\left(\tilde{\bm{x}}(\tau)\right)\mathrm{d}\tau\right\} \big(\bm{x}^{t}-\bm{x}^{t,\mathrm{sgn}}\big)\right\Vert _{2}+\eta\left\Vert \nabla f^{\mathrm{sgn}}\big(\bm{x}^{t,\mathrm{sgn}}\big)-\nabla f\big(\bm{x}^{t,\mathrm{sgn}}\big)\right\Vert _{2}\\
 & \leq\left\{ 1+3\eta\left(1-\left\Vert \bm{x}^{t}\right\Vert _{2}^{2}\right)+\eta\phi_{2}\right\} \left\Vert \bm{x}^{t}-\bm{x}^{t,\mathrm{sgn}}\right\Vert _{2}+O\left(\eta\sqrt{\frac{\log m}{m}}\left\Vert \bm{x}_{\perp}^{t}\right\Vert _{2}\right)+\eta\sqrt{\frac{n\log^{3}m}{m}}\left|x_{\parallel}^{t}\right|
\end{align*}
for some $|\phi_{2}|\ll\frac{1}{\log m}$. Here the second inequality
follows from the fact (\ref{eq:1-logm-xt-xt-sgn}). Substitute the
induction hypotheses into this bound to reach 
\begin{align*}
\left\Vert \bm{x}^{t+1}-\bm{x}^{t+1,\mathrm{sgn}}\right\Vert _{2} & \leq\left\{ 1+3\eta\left(1-\left\Vert \bm{x}^{t}\right\Vert _{2}^{2}\right)+\eta\phi_{2}\right\} \alpha_{t}\left(1+\frac{1}{\log m}\right)^{t}C_{3}\sqrt{\frac{n\log^{5}m}{m}}\\
 & \quad+\eta\sqrt{\frac{\log m}{m}}\beta_{t}+\eta\sqrt{\frac{n\log^{3}m}{m}}\alpha_{t}\\
 & \overset{\left(\text{i}\right)}{\leq}\left\{ 1+3\eta\left(1-\left\Vert \bm{x}^{t}\right\Vert _{2}^{2}\right)+\eta\phi_{3}\right\} \alpha_{t}\left(1+\frac{1}{\log m}\right)^{t}C_{3}\sqrt{\frac{n\log^{5}m}{m}}\\
 & \overset{\left(\text{ii}\right)}{\leq}\alpha_{t+1}\left(1+\frac{1}{\log m}\right)^{t+1}C_{3}\sqrt{\frac{n\log^{5}m}{m}},
\end{align*}
for some $|\phi_{3}|\ll\frac{1}{\log m}$, where (ii) follows the
same reasoning as in (\ref{eq:beta-t-iterative-reasoning}) and (i)
holds as long as \begin{subequations} 
\begin{align}
\sqrt{\frac{\log m}{m}}\beta_{t} & \ll\frac{1}{\log m}\alpha_{t}\left(1+\frac{1}{\log m}\right)^{t}C_{3}\sqrt{\frac{n\log^{5}m}{m}},\label{eq:xt-xt-sgn-as-long-as-1}\\
\sqrt{\frac{n\log^{3}m}{m}}\alpha_{t} & \ll\frac{1}{\log m}\alpha_{t}\left(1+\frac{1}{\log m}\right)^{t}C_{3}\sqrt{\frac{n\log^{5}m}{m}}.\label{eq:xt-xt-sgn-as-long-as-2}
\end{align}
\end{subequations}Here the first condition (\ref{eq:xt-xt-sgn-as-long-as-1})
results from (see Lemma \ref{lemma:iterative})
\[
\beta_{t}\lesssim\sqrt{n\log m}\cdot\alpha_{t},
\]
and the second one is trivially true with the proviso that $C_{3}>0$
is sufficiently large.
\end{itemize}

\section{Proof of Lemma \ref{lemma:double-diff}\label{sec:Proof-of-Lemma-double-diff}}

Consider any $l\,\left(1\leq l\leq m\right)$. According to the gradient
update rules (\ref{eq:gradient_update-WF}), (\ref{eq:gradient-update-leave-WF}),
(\ref{eq:gradient-update-leave-WF-1}) and (\ref{eq:gradient-update-leave-WF-1-1}),
we have 
\begin{align*}
 & \bm{x}^{t+1}-\bm{x}^{t+1,(l)}-\bm{x}^{t+1,\mathrm{sgn}}+\bm{x}^{t+1,\mathrm{sgn},(l)}\\
 & \quad=\bm{x}^{t}-\bm{x}^{t,(l)}-\bm{x}^{t,\mathrm{sgn}}+\bm{x}^{t,\mathrm{sgn},(l)}-\eta\left[\nabla f\left(\bm{x}^{t}\right)-\nabla f^{(l)}\big(\bm{x}^{t,(l)}\big)-\nabla f^{\mathrm{sgn}}\left(\bm{x}^{t,\mathrm{sgn}}\right)+\nabla f^{\mathrm{sgn},(l)}\left(\bm{x}^{t,\mathrm{sgn},(l)}\right)\right].
\end{align*}
It then boils down to controlling the gradient difference, i.e.~$\nabla f\left(\bm{x}^{t}\right)-\nabla f^{(l)}\big(\bm{x}^{t,(l)}\big)-\nabla f^{\mathrm{sgn}}\left(\bm{x}^{t,\mathrm{sgn}}\right)+\nabla f^{\mathrm{sgn},(l)}\left(\bm{x}^{t,\mathrm{sgn},(l)}\right)$.
To this end, we first see that 
\begin{align}
\nabla f\left(\bm{x}^{t}\right)-\nabla f^{(l)}\big(\bm{x}^{t,(l)}\big) & =\nabla f\left(\bm{x}^{t}\right)-\nabla f\big(\bm{x}^{t,(l)}\big)+\nabla f\big(\bm{x}^{t,(l)}\big)-\nabla f^{(l)}\big(\bm{x}^{t,(l)}\big)\nonumber \\
 & =\left(\int_{0}^{1}\nabla^{2}f\left(\bm{x}\left(\tau\right)\right)\mathrm{d}\tau\right)\left(\bm{x}^{t}-\bm{x}^{t,(l)}\right)+\frac{1}{m}\left[\left(\bm{a}_{l}^{\top}\bm{x}^{t,(l)}\right)^{2}-\left(\bm{a}_{l}^{\top}\bm{x}^{\natural}\right)^{2}\right]\bm{a}_{l}\bm{a}_{l}^{\top}\bm{x}^{t,(l)},\label{eq:grad-diff-l}
\end{align}
where we denote $\bm{x}(\tau):=\bm{x}^{t}+\tau\left(\bm{x}^{t,(l)}-\bm{x}^{t}\right)$
and the last identity results from the fundamental theorem of calculus
\cite[Chapter XIII, Theorem 4.2]{lang1993real}. Similar calculations
yield 
\begin{align}
 & \nabla f^{\mathrm{sgn}}\left(\bm{x}^{t,\mathrm{sgn}}\right)-\nabla f^{\mathrm{sgn},(l)}\big(\bm{x}^{t,\mathrm{sgn},(l)}\big)\nonumber \\
 & =\left(\int_{0}^{1}\nabla^{2}f^{\mathrm{sgn}}\left(\tilde{\bm{x}}\left(\tau\right)\right)\mathrm{d}\tau\right)\left(\bm{x}^{t,\mathrm{sgn}}-\bm{x}^{t,\mathrm{sgn}(l)}\right)+\frac{1}{m}\left[\left(\bm{a}_{l}^{\text{sgn}\top}\bm{x}^{t,\mathrm{sgn},(l)}\right)^{2}-\left(\bm{a}_{l}^{\text{sgn}\top}\bm{x}^{\natural}\right)^{2}\right]\bm{a}_{l}^{\text{sgn}}\bm{a}_{l}^{\text{sgn}\top}\bm{x}^{t,\mathrm{sgn},(l)}\label{eq:grad-diff-sgn-l}
\end{align}
with $\tilde{\bm{x}}(\tau):=\bm{x}^{t,\mathrm{sgn}}+\tau\left(\bm{x}^{t,\mathrm{sgn},(l)}-\bm{x}^{t,\mathrm{sgn}}\right)$.
Combine (\ref{eq:grad-diff-l}) and (\ref{eq:grad-diff-sgn-l}) to
arrive at 
\begin{align}
 & \nabla f\left(\bm{x}^{t}\right)-\nabla f^{(l)}\big(\bm{x}^{t,(l)}\big)-\nabla f^{\mathrm{sgn}}\left(\bm{x}^{t,\mathrm{sgn}}\right)+\nabla f^{\mathrm{sgn},(l)}\big(\bm{x}^{t,\mathrm{sgn},(l)}\big)\nonumber \\
 & \quad=\underset{:=\bm{v}_{1}}{\underbrace{\left(\int_{0}^{1}\nabla^{2}f\big(\bm{x}(\tau)\big)\mathrm{d}\tau\right)\big(\bm{x}^{t}-\bm{x}^{t,(l)}\big)-\left(\int_{0}^{1}\nabla^{2}f^{\mathrm{sgn}}\big(\tilde{\bm{x}}(\tau)\big)\mathrm{d}\tau\right)\big(\bm{x}^{t,\mathrm{sgn}}-\bm{x}^{t,\mathrm{sgn},(l)}\big)}}\nonumber \\
 & \quad\quad+\underset{:=\bm{v}_{2}}{\underbrace{\frac{1}{m}\left[\left(\bm{a}_{l}^{\top}\bm{x}^{t,(l)}\right)^{2}-\left(\bm{a}_{l}^{\top}\bm{x}^{\natural}\right)^{2}\right]\bm{a}_{l}\bm{a}_{l}^{\top}\bm{x}^{t,(l)}-\frac{1}{m}\left[\left(\bm{a}_{l}^{\text{sgn}\top}\bm{x}^{t,\mathrm{sgn},(l)}\right)^{2}-\left(\bm{a}_{l}^{\text{sgn}\top}\bm{x}^{\natural}\right)^{2}\right]\bm{a}_{l}^{\text{sgn}}\bm{a}_{l}^{\text{sgn}\top}\bm{x}^{t,\mathrm{sgn},(l)}}}.\label{eq:dfn-v1-v2}
\end{align}
In what follows, we shall control $\bm{v}_{1}$ and $\bm{v}_{2}$
separately. 
\begin{itemize}
\item We start with the simpler term $\bm{v}_{2}$. In light of the fact
that $\big(\bm{a}_{l}^{\top}\bm{x}^{\natural}\big)^{2}=\big(\bm{a}_{l}^{\text{sgn}\top}\bm{x}^{\natural}\big)^{2}=\big|a_{l,1}\big|^{2}$
(see (\ref{eq:by-construction})), one can decompose $\bm{v}_{2}$
as 
\begin{align*}
m\bm{v}_{2} & =\underbrace{\left[\big(\bm{a}_{l}^{\top}\bm{x}^{t,(l)}\big)^{2}-\big(\bm{a}_{l}^{\text{sgn}\top}\bm{x}^{t,\mathrm{sgn},(l)}\big)^{2}\right]\bm{a}_{l}\bm{a}_{l}^{\top}\bm{x}^{t,(l)}}_{:=\bm{\theta}_{1}}\\
 & \quad+\left[\big(\bm{a}_{l}^{\text{sgn}\top}\bm{x}^{t,\mathrm{sgn},(l)}\big)^{2}-\left|a_{l,1}\right|^{2}\right]\underbrace{\left(\bm{a}_{l}\bm{a}_{l}^{\top}\bm{x}^{t,(l)}-\bm{a}_{l}^{\text{sgn}}\bm{a}_{l}^{\text{sgn}\top}\bm{x}^{t,\mathrm{sgn},(l)}\right)}_{:=\bm{\theta}_{2}}.
\end{align*}
First, it is easy to see from (\ref{eq:max-a-i-1}) and the independence
between $\bm{a}_{l}^{\text{sgn}}$ and $\bm{x}^{t,\text{sgn},\left(l\right)}$
that 
\begin{align}
\left|\big(\bm{a}_{l}^{\text{sgn}\top}\bm{x}^{t,\mathrm{sgn},(l)}\big)^{2}-\left|a_{l,1}\right|^{2}\right| & \leq\big(\bm{a}_{l}^{\text{sgn}\top}\bm{x}^{t,\mathrm{sgn},(l)}\big)^{2}+\left|a_{l,1}\right|^{2}\nonumber \\
 & \lesssim\log m\cdot\big\|\bm{x}^{t,\mathrm{sgn},(l)}\big\|_{2}^{2}+\log m\lesssim\log m\label{eq:double-diff-v-2-factor}
\end{align}
with probability at least $1-O\left(m^{-10}\right)$, where the last
inequality results from the norm condition $\left\Vert \bm{x}^{t,\mathrm{sgn},(l)}\right\Vert _{2}\lesssim1$
(see (\ref{eq:consequence-norm-xt-sgn-l})). Regarding the term $\bm{\theta}_{2}$,
one has 
\[
\bm{\theta}_{2}=\left(\bm{a}_{l}\bm{a}_{l}^{\top}-\bm{a}_{l}^{\text{sgn}}\bm{a}_{l}^{\text{sgn}\top}\right)\bm{x}^{t,\left(l\right)}+\bm{a}_{l}^{\text{sgn}}\bm{a}_{l}^{\text{sgn}\top}\big(\bm{x}^{t,\left(l\right)}-\bm{x}^{t,\text{sgn},\left(l\right)}\big),
\]
which together with the identity (\ref{eq:alal-alsgnalsgn}) gives
\[
\bm{\theta}_{2}=\left(\xi_{l}-\xi_{l}^{\text{sgn}}\right)\left|a_{l,1}\right|\left[\begin{array}{c}
\bm{a}_{l,\perp}^{\top}\bm{x}_{\perp}^{t,\left(l\right)}\\
x_{\parallel}^{t,\left(l\right)}\bm{a}_{l,\perp}
\end{array}\right]+\bm{a}_{l}^{\text{sgn}}\bm{a}_{l}^{\text{sgn}\top}\big(\bm{x}^{t,\left(l\right)}-\bm{x}^{t,\text{sgn},\left(l\right)}\big).
\]
In view of the independence between $\bm{a}_{l}$ and $\bm{x}^{t,\left(l\right)}$,
and between $\bm{a}_{l}^{\text{sgn}}$ and $\bm{x}^{t,\left(l\right)}-\bm{x}^{t,\text{sgn},\left(l\right)}$,
one can again apply standard Gaussian concentration results to obtain
that 
\[
\left|\bm{a}_{l,\perp}^{\top}\bm{x}_{\perp}^{t,\left(l\right)}\right|\lesssim\sqrt{\log m}\left\Vert \bm{x}_{\perp}^{t,\left(l\right)}\right\Vert _{2}\qquad\text{and}\qquad\left|\bm{a}_{l}^{\text{sgn}\top}\big(\bm{x}^{t,\left(l\right)}-\bm{x}^{t,\text{sgn},\left(l\right)}\big)\right|\lesssim\sqrt{\log m}\left\Vert \bm{x}^{t,\left(l\right)}-\bm{x}^{t,\text{sgn},\left(l\right)}\right\Vert _{2}
\]
with probability exceeding $1-O\left(m^{-10}\right)$. Combining these
two with the facts (\ref{eq:max-a-i-1}) and (\ref{eq:max-a-i-norm})
leads to 
\begin{align}
\left\Vert \bm{\theta}_{2}\right\Vert _{2} & \leq\left|\xi_{l}-\xi_{l}^{\text{sgn}}\right|\left|a_{l,1}\right|\left(\left|\bm{a}_{l,\perp}^{\top}\bm{x}_{\perp}^{t,\left(l\right)}\right|+\left|x_{\parallel}^{t,\left(l\right)}\right|\left\Vert \bm{a}_{l,\perp}\right\Vert _{2}\right)+\left\Vert \bm{a}_{l}^{\text{sgn}}\right\Vert _{2}\left|\bm{a}_{l}^{\text{sgn}\top}\big(\bm{x}^{t,\left(l\right)}-\bm{x}^{t,\text{sgn},\left(l\right)}\big)\right|\nonumber \\
 & \lesssim\sqrt{\log m}\left(\sqrt{\log m}\left\Vert \bm{x}_{\perp}^{t,\left(l\right)}\right\Vert _{2}+\sqrt{n}\left|x_{\parallel}^{t,\left(l\right)}\right|\right)+\sqrt{n\log m}\left\Vert \bm{x}^{t,\left(l\right)}-\bm{x}^{t,\text{sgn},\left(l\right)}\right\Vert _{2}\nonumber \\
 & \lesssim\log m\left\Vert \bm{x}_{\perp}^{t,\left(l\right)}\right\Vert _{2}+\sqrt{n\log m}\left(\left|x_{\parallel}^{t,\left(l\right)}\right|+\left\Vert \bm{x}^{t,\left(l\right)}-\bm{x}^{t,\text{sgn},\left(l\right)}\right\Vert _{2}\right).\label{eq:double-diff-theta-2-vec}
\end{align}
We now move on to controlling $\bm{\theta}_{1}$. Use the elementary
identity $a^{2}-b^{2}=\left(a-b\right)\left(a+b\right)$ to get 
\begin{equation}
\bm{\theta}_{1}=\left(\bm{a}_{l}^{\top}\bm{x}^{t,(l)}-\bm{a}_{l}^{\text{sgn}\top}\bm{x}^{t,\mathrm{sgn},(l)}\right)\left(\bm{a}_{l}^{\top}\bm{x}^{t,(l)}+\bm{a}_{l}^{\text{sgn}\top}\bm{x}^{t,\mathrm{sgn},(l)}\right)\bm{a}_{l}\bm{a}_{l}^{\top}\bm{x}^{t,(l)}.\label{eq:double-diff-theta-1}
\end{equation}
The constructions of $\bm{a}_{l}^{\text{sgn}}$ requires that 
\[
\bm{a}_{l}^{\top}\bm{x}^{t,(l)}-\bm{a}_{l}^{\text{sgn}\top}\bm{x}^{t,\mathrm{sgn},(l)}=\xi_{l}\left|a_{l,1}\right|x_{\parallel}^{t,\left(l\right)}-\xi_{l}^{\text{sgn}}\left|a_{l,1}\right|x_{\parallel}^{t,\text{sgn},\left(l\right)}+\bm{a}_{l,\perp}^{\top}\big(\bm{x}_{_{\perp}}^{t,\left(l\right)}-\bm{x}_{\perp}^{t,\text{sgn},\left(l\right)}\big).
\]
Similarly, in view of the independence between $\bm{a}_{l,\perp}$
and $\bm{x}_{\perp}^{t,\left(l\right)}-\bm{x}_{\perp}^{t,\text{sgn},\left(l\right)}$,
and the fact (\ref{eq:max-a-i-1}), one can see that with probability
at least $1-O\left(m^{-10}\right)$ 
\begin{align}
\left|\bm{a}_{l}^{\top}\bm{x}^{t,(l)}-\bm{a}_{l}^{\text{sgn}\top}\bm{x}^{t,\mathrm{sgn},(l)}\right| & \leq\left|\xi_{l}\right|\left|a_{l,1}\right|\left|x_{\parallel}^{t,\left(l\right)}\right|+\left|\xi_{l}^{\text{sgn}}\right|\left|a_{l,1}\right|\left|x_{\parallel}^{t,\text{sgn},\left(l\right)}\right|+\left|\bm{a}_{l,\perp}^{\top}\big(\bm{x}_{_{\perp}}^{t,\left(l\right)}-\bm{x}_{\perp}^{t,\text{sgn},\left(l\right)}\big)\right|\nonumber \\
 & \lesssim\sqrt{\log m}\left(\left|x_{\parallel}^{t,\left(l\right)}\right|+\left|x_{\parallel}^{t,\text{sgn},\left(l\right)}\right|+\left\Vert \bm{x}_{\perp}^{t,\left(l\right)}-\bm{x}_{\perp}^{t,\text{sgn},\left(l\right)}\right\Vert _{2}\right)\nonumber \\
 & \lesssim\sqrt{\log m}\left(\left|x_{\parallel}^{t,\left(l\right)}\right|+\left\Vert \bm{x}^{t,\left(l\right)}-\bm{x}^{t,\text{sgn},\left(l\right)}\right\Vert _{2}\right),\label{eq:double-diff-theta-1-vec-prev}
\end{align}
where the last inequality results from the triangle inequality $\big|x_{\parallel}^{t,\text{sgn},\left(l\right)}\big|\leq\big|x_{\parallel}^{t,\left(l\right)}\big|+\left\Vert \bm{x}^{t,\left(l\right)}-\bm{x}^{t,\text{sgn},\left(l\right)}\right\Vert _{2}$.
Substituting (\ref{eq:double-diff-theta-1-vec-prev}) into (\ref{eq:double-diff-theta-1})
results in 
\begin{align}
\left\Vert \bm{\theta}_{1}\right\Vert _{2} & =\left|\bm{a}_{l}^{\top}\bm{x}^{t,(l)}-\bm{a}_{l}^{\text{sgn}\top}\bm{x}^{t,\mathrm{sgn},(l)}\right|\left|\bm{a}_{l}^{\top}\bm{x}^{t,(l)}+\bm{a}_{l}^{\text{sgn}\top}\bm{x}^{t,\mathrm{sgn},(l)}\right|\left\Vert \bm{a}_{l}\right\Vert _{2}\left|\bm{a}_{l}^{\top}\bm{x}^{t,(l)}\right|\nonumber \\
 & \lesssim\sqrt{\log m}\left(\left|x_{\parallel}^{t,\left(l\right)}\right|+\big\|\bm{x}^{t,\left(l\right)}-\bm{x}^{t,\text{sgn},\left(l\right)}\big\|_{2}\right)\cdot\sqrt{\log m}\cdot\sqrt{n}\cdot\sqrt{\log m}\nonumber \\
 & \asymp\sqrt{n\log^{3}m}\left(\left|x_{\parallel}^{t,\left(l\right)}\right|+\big\|\bm{x}^{t,\left(l\right)}-\bm{x}^{t,\text{sgn},\left(l\right)}\big\|_{2}\right),\label{eq:double-diff-theta-1-vec}
\end{align}
where the second line comes from the simple facts (\ref{eq:max-a-i-norm}),
\[
\left|\bm{a}_{l}^{\top}\bm{x}^{t,(l)}+\bm{a}_{l}^{\text{sgn}\top}\bm{x}^{t,\mathrm{sgn},(l)}\right|\leq\sqrt{\log m}\qquad\text{and}\qquad\left|\bm{a}_{l}^{\top}\bm{x}^{t,\left(l\right)}\right|\lesssim\sqrt{\log m}.
\]
Taking the bounds (\ref{eq:double-diff-v-2-factor}), (\ref{eq:double-diff-theta-2-vec})
and (\ref{eq:double-diff-theta-1-vec}) collectively, we can conclude
that 
\begin{align*}
\left\Vert \bm{v}_{2}\right\Vert _{2} & \leq\frac{1}{m}\left(\left\Vert \bm{\theta}_{1}\right\Vert _{2}+\left|\left(\bm{a}_{l}^{\text{sgn}\top}\bm{x}^{t,\mathrm{sgn},(l)}\right)^{2}-\left|a_{l,1}\right|^{2}\right|\left\Vert \bm{\theta}_{2}\right\Vert _{2}\right)\\
 & \lesssim\frac{\log^{2}m}{m}\left\Vert \bm{x}_{\perp}^{t,\left(l\right)}\right\Vert _{2}+\frac{\sqrt{n\log^{3}m}}{m}\left(\left|x_{\parallel}^{t,\left(l\right)}\right|+\big\|\bm{x}^{t,\left(l\right)}-\bm{x}^{t,\text{sgn},\left(l\right)}\big\|_{2}\right).
\end{align*}
\item To bound $\bm{v}_{1}$, one first observes that 
\begin{align*}
 & \nabla^{2}f\left(\bm{x}\left(\tau\right)\right)\left(\bm{x}^{t}-\bm{x}^{t,(l)}\right)-\nabla^{2}f^{\mathrm{sgn}}\left(\tilde{\bm{x}}\left(\tau\right)\right)\big(\bm{x}^{t,\mathrm{sgn}}-\bm{x}^{t,\text{sgn},\left(l\right)}\big)\\
 & \quad=\underbrace{\nabla^{2}f\left(\bm{x}\left(\tau\right)\right)\left(\bm{x}^{t}-\bm{x}^{t,(l)}-\bm{x}^{t,\mathrm{sgn}}+\bm{x}^{t,\mathrm{sgn},(l)}\right)}_{:=\bm{w}_{1}\left(\tau\right)}+\underbrace{\left[\nabla^{2}f\left(\bm{x}\left(\tau\right)\right)-\nabla^{2}f\left(\tilde{\bm{x}}\left(\tau\right)\right)\right]\big(\bm{x}^{t,\mathrm{sgn}}-\bm{x}^{t,\text{sgn},\left(l\right)}\big)}_{:=\bm{w}_{2}\left(\tau\right)}\\
 & \quad\quad+\underbrace{\left[\nabla^{2}f\left(\tilde{\bm{x}}\left(\tau\right)\right)-\nabla^{2}f^{\text{sgn}}\left(\tilde{\bm{x}}\left(\tau\right)\right)\right]\big(\bm{x}^{t,\mathrm{sgn}}-\bm{x}^{t,\text{sgn},\left(l\right)}\big)}_{:=\bm{w}_{3}\left(\tau\right)}.
\end{align*}
\begin{itemize}
\item The first term $\bm{w}_{1}(\tau)$ satisfies 
\begin{align*}
 & \left\Vert \bm{x}^{t}-\bm{x}^{t,(l)}-\bm{x}^{t,\mathrm{sgn}}+\bm{x}^{t,\mathrm{sgn},(l)}-\eta\int_{0}^{1}\bm{w}_{1}(\tau)\mathrm{d}\tau\right\Vert _{2}\\
 & \quad=\left\Vert \left\{ \bm{I}-\eta\int_{0}^{1}\nabla^{2}f\left(\bm{x}\left(\tau\right)\right)\mathrm{d}\tau\right\} \left(\bm{x}^{t}-\bm{x}^{t,(l)}-\bm{x}^{t,\mathrm{sgn}}+\bm{x}^{t,\mathrm{sgn},(l)}\right)\right\Vert _{2}\\
 & \quad\leq\left\Vert \bm{I}-\eta\int_{0}^{1}\nabla^{2}f\left(\bm{x}\left(\tau\right)\right)\mathrm{d}\tau\right\Vert \cdot\left\Vert \bm{x}^{t}-\bm{x}^{t,(l)}-\bm{x}^{t,\mathrm{sgn}}+\bm{x}^{t,\mathrm{sgn},(l)}\right\Vert _{2}\\
 & \quad\leq\left\{ 1+3\eta\left(1-\left\Vert \bm{x}^{t}\right\Vert _{2}^{2}\right)+O\left(\eta\frac{1}{\log m}\right)+\eta\phi_{1}\right\} \left\Vert \bm{x}^{t}-\bm{x}^{t,(l)}-\bm{x}^{t,\mathrm{sgn}}+\bm{x}^{t,\mathrm{sgn},(l)}\right\Vert _{2},
\end{align*}
for some $|\phi_{1}|\ll\frac{1}{\log m}$, where the last line follows
from the same argument as in (\ref{eq:xt-xt-l-signal-main}). 
\item Regarding the second term $\bm{w}_{2}(\tau)$, it is seen that 
\begin{align}
\left\Vert \nabla^{2}f\left(\bm{x}\left(\tau\right)\right)-\nabla^{2}f\left(\tilde{\bm{x}}\left(\tau\right)\right)\right\Vert  & =\left\Vert \frac{3}{m}\sum_{i=1}^{m}\left[\big(\bm{a}_{i}^{\top}\bm{x}(\tau)\big)^{2}-\big(\bm{a}_{i}^{\top}\tilde{\bm{x}}(\tau)\big)^{2}\right]\bm{a}_{i}\bm{a}_{i}^{\top}\right\Vert \nonumber \\
 & \leq\max_{1\leq i\leq m}\left|\big(\bm{a}_{i}^{\top}\bm{x}(\tau)\big)^{2}-\big(\bm{a}_{i}^{\top}\tilde{\bm{x}}(\tau)\big)^{2}\right|\left\Vert \frac{3}{m}\sum_{i=1}^{m}\bm{a}_{i}\bm{a}_{i}^{\top}\right\Vert \nonumber \\
 & \leq\max_{1\leq i\leq m}\left|\bm{a}_{i}^{\top}\left(\bm{x}\left(\tau\right)-\tilde{\bm{x}}\left(\tau\right)\right)\right|\max_{1\leq i\leq m}\left|\bm{a}_{i}^{\top}\left(\bm{x}\left(\tau\right)+\tilde{\bm{x}}\left(\tau\right)\right)\right|\left\Vert \frac{3}{m}\sum_{i=1}^{m}\bm{a}_{i}\bm{a}_{i}^{\top}\right\Vert \nonumber \\
 & \lesssim\max_{1\leq i\leq m}\left|\bm{a}_{i}^{\top}\left(\bm{x}\left(\tau\right)-\tilde{\bm{x}}\left(\tau\right)\right)\right|\sqrt{\log m},\label{eq:double-hessian-diff}
\end{align}
where the last line makes use of Lemma \ref{lemma:ai-ai-spectral-upper-bound}
as well as the incoherence conditions 
\begin{align}
\max_{1\leq i\leq m}\left|\bm{a}_{i}^{\top}\left(\bm{x}\left(\tau\right)+\tilde{\bm{x}}\left(\tau\right)\right)\right| & \leq\max_{1\leq i\leq m}\left|\bm{a}_{i}^{\top}\bm{x}\left(\tau\right)\right|+\max_{1\leq i\leq m}\left|\bm{a}_{i}^{\top}\tilde{\bm{x}}\left(\tau\right)\right|\lesssim\sqrt{\log m}.\label{eq:double-diff-incoherence-condition}
\end{align}
Note that 
\begin{align*}
\bm{x}(\tau)-\tilde{\bm{x}}(\tau) & =\bm{x}^{t}+\tau\left(\bm{x}^{t,(l)}-\bm{x}^{t}\right)-\left[\bm{x}^{t,\mathrm{sgn}}+\tau\left(\bm{x}^{t,\mathrm{sgn},(l)}-\bm{x}^{t,\mathrm{sgn}}\right)\right]\\
 & =\left(1-\tau\right)\left(\bm{x}^{t}-\bm{x}^{t,\text{sgn}}\right)+\tau\big(\bm{x}^{t,(l)}-\bm{x}^{t,\text{sgn},\left(l\right)}\big).
\end{align*}
This implies for all $0\leq\tau\leq1$, 
\[
\left|\bm{a}_{i}^{\top}\big(\bm{x}(\tau)-\tilde{\bm{x}}(\tau)\big)\right|\leq\left|\bm{a}_{i}^{\top}\big(\bm{x}^{t}-\bm{x}^{t,\mathrm{sgn}}\big)\right|+\left|\bm{a}_{i}^{\top}\big(\bm{x}^{t,(l)}-\bm{x}^{t,\mathrm{sgn},(l)}\big)\right|.
\]
Moreover, the triangle inequality together with the Cauchy-Schwarz
inequality tells us that 
\begin{align*}
\left|\bm{a}_{i}^{\top}\big(\bm{x}^{t,(l)}-\bm{x}^{t,\text{sgn},\left(l\right)}\big)\right| & \leq\left|\bm{a}_{i}^{\top}\big(\bm{x}^{t}-\bm{x}^{t,\mathrm{sgn}}\big)\right|+\left|\bm{a}_{i}^{\top}\big(\bm{x}^{t}-\bm{x}^{t,\mathrm{sgn}}-\bm{x}^{t,(l)}+\bm{x}^{t,\mathrm{sgn},(l)}\big)\right|\\
 & \leq\left|\bm{a}_{i}^{\top}\big(\bm{x}^{t}-\bm{x}^{t,\mathrm{sgn}}\big)\right|+\left\Vert \bm{a}_{i}\right\Vert _{2}\big\|\bm{x}^{t}-\bm{x}^{t,\mathrm{sgn}}-\bm{x}^{t,(l)}+\bm{x}^{t,\mathrm{sgn},(l)}\big\|_{2}
\end{align*}
and 
\begin{align*}
\left|\bm{a}_{i}^{\top}\big(\bm{x}^{t}-\bm{x}^{t,\mathrm{sgn}}\big)\right| & \leq\left|\bm{a}_{i}^{\top}\big(\bm{x}^{t,(i)}-\bm{x}^{t,\mathrm{sgn},(i)}\big)\right|+\left|\bm{a}_{i}^{\top}\left(\bm{x}^{t}-\bm{x}^{t,\mathrm{sgn}}-\bm{x}^{t,(i)}+\bm{x}^{t,\mathrm{sgn},(i)}\right)\right|\\
 & \leq\left|\bm{a}_{i}^{\top}\big(\bm{x}^{t,(i)}-\bm{x}^{t,\mathrm{sgn},(i)}\big)\right|+\left\Vert \bm{a}_{i}\right\Vert _{2}\big\|\bm{x}^{t}-\bm{x}^{t,\mathrm{sgn}}-\bm{x}^{t,(i)}+\bm{x}^{t,\mathrm{sgn},(i)}\big\|_{2}.
\end{align*}
Combine the previous three inequalities to obtain 
\begin{align*}
 & \max_{1\leq i\leq m}\left|\bm{a}_{i}^{\top}\left(\bm{x}\left(\tau\right)-\tilde{\bm{x}}\left(\tau\right)\right)\right|\leq\max_{1\leq i\leq m}\left|\bm{a}_{i}^{\top}\left(\bm{x}^{t}-\bm{x}^{t,\mathrm{sgn}}\right)\right|+\max_{1\leq i\leq m}\left|\bm{a}_{i}^{\top}\big(\bm{x}^{t,(l)}-\bm{x}^{t,\text{sgn},\left(l\right)}\big)\right|\\
 & \quad\leq2\max_{1\leq i\leq m}\left|\bm{a}_{i}^{\top}\big(\bm{x}^{t,(i)}-\bm{x}^{t,\text{sgn},\left(i\right)}\big)\right|+3\max_{1\leq i\leq m}\left\Vert \bm{a}_{i}\right\Vert _{2}\max_{1\leq l\leq m}\big\|\bm{x}^{t}-\bm{x}^{t,\mathrm{sgn}}-\bm{x}^{t,(l)}+\bm{x}^{t,\mathrm{sgn},(l)}\big\|_{2}\\
 & \quad\lesssim\sqrt{\log m}\max_{1\leq i\leq m}\big\|\bm{x}^{t,(i)}-\bm{x}^{t,\mathrm{sgn},(i)}\big\|_{2}+\sqrt{n}\max_{1\leq l\leq m}\big\|\bm{x}^{t}-\bm{x}^{t,\mathrm{sgn}}-\bm{x}^{t,(l)}+\bm{x}^{t,\mathrm{sgn},(l)}\big\|_{2},
\end{align*}
where the last inequality follows from the independence between $\bm{a}_{i}$
and $\bm{x}^{t,\left(i\right)}-\bm{x}^{t,\text{sgn},\left(i\right)}$
and the fact (\ref{eq:max-a-i-norm}). Substituting the above bound
into (\ref{eq:double-hessian-diff}) results in 
\begin{align*}
 & \left\Vert \nabla^{2}f\left(\bm{x}\left(\tau\right)\right)-\nabla^{2}f\left(\tilde{\bm{x}}\left(\tau\right)\right)\right\Vert \\
 & \quad\lesssim\log m\max_{1\leq i\leq m}\left\Vert \bm{x}^{t,(i)}-\bm{x}^{t,\mathrm{sgn},(i)}\right\Vert _{2}+\sqrt{n\log m}\max_{1\leq l\leq m}\big\|\bm{x}^{t}-\bm{x}^{t,\mathrm{sgn}}-\bm{x}^{t,(l)}+\bm{x}^{t,\mathrm{sgn},(l)}\big\|_{2}\\
 & \quad\lesssim\log m\left\Vert \bm{x}^{t}-\bm{x}^{t,\mathrm{sgn}}\right\Vert _{2}+\sqrt{n\log m}\max_{1\leq l\leq m}\big\|\bm{x}^{t}-\bm{x}^{t,\mathrm{sgn}}-\bm{x}^{t,(l)}+\bm{x}^{t,\mathrm{sgn},(l)}\big\|_{2}.
\end{align*}
Here, we use the triangle inequality 
\[
\left\Vert \bm{x}^{t,(i)}-\bm{x}^{t,\mathrm{sgn},(i)}\right\Vert _{2}\leq\left\Vert \bm{x}^{t}-\bm{x}^{t,\mathrm{sgn}}\right\Vert _{2}+\big\|\bm{x}^{t}-\bm{x}^{t,\mathrm{sgn}}-\bm{x}^{t,(i)}+\bm{x}^{t,\mathrm{sgn},(i)}\big\|_{2}
\]
and the fact $\log m\leq\sqrt{n\log m}$. Consequently, we have the
following bound for $\bm{w}_{2}\left(\tau\right)$: 
\begin{align*}
 & \|\bm{w}_{2}(\tau)\|_{2}\leq\left\Vert \nabla^{2}f\left(\bm{x}\left(\tau\right)\right)-\nabla^{2}f\left(\tilde{\bm{x}}\left(\tau\right)\right)\right\Vert \cdot\big\|\bm{x}^{t,\mathrm{sgn}}-\bm{x}^{t,\mathrm{sgn},(l)}\big\|_{2}\\
 & \text{ }\lesssim\left\{ \log m\left\Vert \bm{x}^{t}-\bm{x}^{t,\mathrm{sgn}}\right\Vert _{2}+\sqrt{n\log m}\max_{1\leq l\leq m}\big\|\bm{x}^{t}-\bm{x}^{t,\mathrm{sgn}}-\bm{x}^{t,(l)}+\bm{x}^{t,\mathrm{sgn},(l)}\big\|_{2}\right\} \big\|\bm{x}^{t,\mathrm{sgn}}-\bm{x}^{t,\mathrm{sgn},(l)}\big\|_{2}.
\end{align*}
\item It remains to control $\bm{w}_{3}\left(\tau\right)$. To this end,
one has 
\begin{align*}
\bm{w}_{3}(\tau) & =\frac{1}{m}\sum_{i=1}^{m}\underbrace{\left[3\left(\bm{a}_{i}^{\top}\tilde{\bm{x}}\left(\tau\right)\right)^{2}-\left(\bm{a}_{i}^{\top}\bm{x}^{\natural}\right)^{2}\right]}_{:=\rho_{i}}\bm{a}_{i}\bm{a}_{i}^{\top}\big(\bm{x}^{t,\mathrm{sgn}}-\bm{x}^{t,\mathrm{sgn},(l)}\big)\\
 & \quad-\frac{1}{m}\sum_{i=1}^{m}\underbrace{\left[3\big(\bm{a}_{i}^{\mathrm{sgn}\top}\tilde{\bm{x}}\left(\tau\right)\big)^{2}-\big(\bm{a}_{i}^{\mathrm{sgn}\top}\bm{x}^{\natural}\big)^{2}\right]}_{:=\rho_{i}^{\text{sgn}}}\bm{a}_{i}^{\mathrm{sgn}}\bm{a}_{i}^{\mathrm{sgn}\top}\big(\bm{x}^{t,\mathrm{sgn}}-\bm{x}^{t,\mathrm{sgn},(l)}\big).
\end{align*}
We consider the first entry of $\bm{w}_{3}\left(\tau\right)$, i.e.~$w_{3,\parallel}\left(\tau\right)$,
and the 2nd through the $n$th entries, $\bm{w}_{3,\perp}\left(\tau\right)$,
separately. For the first entry $w_{3,\parallel}\left(\tau\right)$,
we obtain 
\begin{equation}
w_{3,\parallel}\left(\tau\right)=\frac{1}{m}\sum_{i=1}^{m}\rho_{i}\xi_{i}\left|a_{i,1}\right|\bm{a}_{i}^{\top}\big(\bm{x}^{t,\mathrm{sgn}}-\bm{x}^{t,\mathrm{sgn},(l)}\big)-\frac{1}{m}\sum_{i=1}^{m}\rho_{i}^{\text{sgn}}\xi_{i}^{\text{sgn}}\left|a_{i,1}\right|\bm{a}_{i}^{\text{sgn}\top}\big(\bm{x}^{t,\mathrm{sgn}}-\bm{x}^{t,\mathrm{sgn},(l)}\big).\label{eq:double-w-3-signal}
\end{equation}
Use the expansions 
\begin{align*}
\bm{a}_{i}^{\top}\big(\bm{x}^{t,\mathrm{sgn}}-\bm{x}^{t,\mathrm{sgn},(l)}\big) & =\xi_{i}\left|a_{i,1}\right|\left(x_{\parallel}^{t,\text{sgn}}-x_{\parallel}^{t,\text{sgn},\left(l\right)}\right)+\bm{a}_{i,\perp}^{\top}\big(\bm{x}_{\perp}^{t,\mathrm{sgn}}-\bm{x}_{\perp}^{t,\mathrm{sgn},\left(l\right)}\big)\\
\bm{a}_{i}^{\text{sgn}\top}\big(\bm{x}^{t,\mathrm{sgn}}-\bm{x}^{t,\mathrm{sgn},(l)}\big) & =\xi_{i}^{\text{sgn}}\left|a_{i,1}\right|\left(x_{\parallel}^{t,\text{sgn}}-x_{\parallel}^{t,\text{sgn},\left(l\right)}\right)+\bm{a}_{i,\perp}^{\top}\big(\bm{x}_{\perp}^{t,\mathrm{sgn}}-\bm{x}_{\perp}^{t,\mathrm{sgn},\left(l\right)}\big)
\end{align*}
to further obtain 
\begin{align*}
w_{3,\parallel}\left(\tau\right) & =\frac{1}{m}\sum_{i=1}^{m}\left(\rho_{i}-\rho_{i}^{\text{sgn}}\right)\left|a_{i,1}\right|^{2}\big(x_{\parallel}^{t,\text{sgn}}-x_{\parallel}^{t,\text{sgn},\left(l\right)}\big)+\frac{1}{m}\sum_{i=1}^{m}\left(\rho_{i}\xi_{i}-\rho_{i}^{\text{sgn}}\xi_{i}^{\text{sgn}}\right)\left|a_{i,1}\right|\bm{a}_{i,\perp}^{\top}\big(\bm{x}_{\perp}^{t,\mathrm{sgn}}-\bm{x}_{\perp}^{t,\mathrm{sgn},\left(l\right)}\big)\\
 & =\underbrace{\frac{1}{m}\sum_{i=1}^{m}\left(\rho_{i}-\rho_{i}^{\text{sgn}}\right)\left|a_{i,1}\right|^{2}\big(x_{\parallel}^{t,\text{sgn}}-x_{\parallel}^{t,\text{sgn},\left(l\right)}\big)}_{:=\theta_{1}\left(\tau\right)}\\
 & \quad+\underbrace{\frac{1}{m}\sum_{i=1}^{m}\left(\rho_{i}-\rho_{i}^{\text{sgn}}\right)\left(\xi_{i}+\xi_{i}^{\text{sgn}}\right)\left|a_{i,1}\right|\bm{a}_{i,\perp}^{\top}\big(\bm{x}_{\perp}^{t,\mathrm{sgn}}-\bm{x}_{\perp}^{t,\mathrm{sgn},\left(l\right)}\big)}_{:=\theta_{2}\left(\tau\right)}\\
 & \quad+\underbrace{\frac{1}{m}\sum_{i=1}^{m}\rho_{i}^{\text{sgn}}\xi_{i}\left|a_{i,1}\right|\bm{a}_{i,\perp}^{\top}\big(\bm{x}_{\perp}^{t,\mathrm{sgn}}-\bm{x}_{\perp}^{t,\mathrm{sgn},\left(l\right)}\big)}_{:=\theta_{3}\left(\tau\right)}-\underbrace{\frac{1}{m}\sum_{i=1}^{m}\rho_{i}\xi_{i}^{\text{sgn}}\left|a_{i,1}\right|\bm{a}_{i,\perp}^{\top}\big(\bm{x}_{\perp}^{t,\mathrm{sgn}}-\bm{x}_{\perp}^{t,\mathrm{sgn},\left(l\right)}\big)}_{:=\theta_{4}\left(\tau\right)}
\end{align*}
The identity (\ref{eq:identity-a-a-sgn}) reveals that 
\begin{equation}
\rho_{i}-\rho_{i}^{\text{sgn}}=6\left(\xi_{i}-\xi_{i}^{\mathrm{sgn}}\right)\left|a_{i,1}\right|\tilde{x}_{\parallel}\left(\tau\right)\bm{a}_{i,\perp}^{\top}\tilde{\bm{x}}_{\perp}\left(\tau\right),\label{eq:identity-square-diff-double}
\end{equation}
and hence 
\[
\theta_{1}\left(\tau\right)=\tilde{x}_{\parallel}\left(\tau\right)\cdot\frac{6}{m}\sum_{i=1}^{m}\left(\xi_{i}-\xi_{i}^{\mathrm{sgn}}\right)\left|a_{i,1}\right|^{3}\bm{a}_{i,\perp}^{\top}\tilde{\bm{x}}_{\perp}\left(\tau\right)\left(x_{\parallel}^{t,\text{sgn}}-x_{\parallel}^{t,\text{sgn},\left(l\right)}\right),
\]
which together with (\ref{eq:xi-xi-sgn-ai-duplicate}) implies
\begin{align*}
\left|\theta_{1}\left(\tau\right)\right| & \leq6\left|\tilde{x}_{\parallel}\left(\tau\right)\right|\left|x_{\parallel}^{t,\text{sgn}}-x_{\parallel}^{t,\text{sgn},\left(l\right)}\right|\left\Vert \tilde{\bm{x}}_{\perp}\left(\tau\right)\right\Vert _{2}\left\Vert \frac{1}{m}\sum_{i=1}^{m}\left(\xi_{i}-\xi_{i}^{\mathrm{sgn}}\right)\left|a_{i,1}\right|^{3}\bm{a}_{i,\perp}^{\top}\right\Vert \\
 & \lesssim\sqrt{\frac{n\log^{3}m}{m}}\left|\tilde{x}_{\parallel}\left(\tau\right)\right|\left|x_{\parallel}^{t,\text{sgn}}-x_{\parallel}^{t,\text{sgn},\left(l\right)}\right|\left\Vert \tilde{\bm{x}}_{\perp}\left(\tau\right)\right\Vert _{2}\\
 & \lesssim\sqrt{\frac{n\log^{3}m}{m}}\left|\tilde{x}_{\parallel}\left(\tau\right)\right|\left|x_{\parallel}^{t,\text{sgn}}-x_{\parallel}^{t,\text{sgn},\left(l\right)}\right|,
\end{align*}
where the penultimate inequality arises from (\ref{eq:xi-xi-sgn-ai-duplicate})
and the last inequality utilizes the fact that 
\[
\left\Vert \tilde{\bm{x}}_{\perp}\left(\tau\right)\right\Vert _{2}\leq\left\Vert \bm{x}_{\perp}^{t,\text{sgn}}\right\Vert _{2}+\left\Vert \bm{x}_{\perp}^{t,\text{sgn},\left(l\right)}\right\Vert _{2}\lesssim1.
\]
Again, we can use (\ref{eq:identity-square-diff-double}) and the
identity $\left(\xi_{i}-\xi_{i}^{\mathrm{sgn}}\right)\left(\xi_{i}+\xi_{i}^{\mathrm{sgn}}\right)=0$
to deduce that 
\[
\theta_{2}\left(\tau\right)=0.
\]
When it comes to $\theta_{3}\left(\tau\right)$, we exploit the independence
between $\xi_{i}$ and $\rho_{i}^{\text{sgn}}\left|a_{i,1}\right|\bm{a}_{i,\perp}^{\top}\left(\bm{x}_{\perp}^{t,\mathrm{sgn}}-\bm{x}_{\perp}^{t,\mathrm{sgn},\left(l\right)}\right)$
and apply the Bernstein inequality (see Lemma \ref{lemma:bernstein})
to obtain that with probability exceeding $1-O\left(m^{-10}\right)$
\[
\left|\theta_{3}\left(\tau\right)\right|\lesssim\frac{1}{m}\left(\sqrt{V_{1}\log m}+B_{1}\log m\right),
\]
where 
\[
V_{1}:=\sum_{i=1}^{m}\left(\rho_{i}^{\text{sgn}}\right)^{2}\left|a_{i,1}\right|^{2}\left|\bm{a}_{i,\perp}^{\top}\big(\bm{x}_{\perp}^{t,\mathrm{sgn}}-\bm{x}_{\perp}^{t,\mathrm{sgn},\left(l\right)}\big)\right|^{2}\quad\text{and}\quad B_{1}:=\max_{1\leq i\leq m}\left|\rho_{i}^{\text{sgn}}\right|\left|a_{i,1}\right|\left|\bm{a}_{i,\perp}^{\top}\big(\bm{x}_{\perp}^{t,\mathrm{sgn}}-\bm{x}_{\perp}^{t,\mathrm{sgn},\left(l\right)}\big)\right|.
\]
Combine the fact $\left|\rho_{i}^{\text{sgn}}\right|\lesssim\log m$
and Lemma \ref{lemma:hessian-concentration} to see that 
\[
V_{1}\lesssim\big(m\log^{2}m\big)\big\|\bm{x}_{\perp}^{t,\mathrm{sgn}}-\bm{x}_{\perp}^{t,\mathrm{sgn},\left(l\right)}\big\|_{2}^{2}.
\]
In addition, the facts $\left|\rho_{i}^{\text{sgn}}\right|\lesssim\log m$,
(\ref{eq:max-a-i-1}) and (\ref{eq:max-a-i-norm}) tell us that 
\[
B_{1}\lesssim\sqrt{n\log^{3}m}\big\|\bm{x}_{\perp}^{t,\mathrm{sgn}}-\bm{x}_{\perp}^{t,\mathrm{sgn},\left(l\right)}\big\|_{2}.
\]
Continue the derivation to reach 
\begin{equation}
\left|\theta_{3}\left(\tau\right)\right|\lesssim\left(\sqrt{\frac{\log^{3}m}{m}}+\frac{\sqrt{n\log^{5}m}}{m}\right)\big\|\bm{x}_{\perp}^{t,\mathrm{sgn}}-\bm{x}_{\perp}^{t,\mathrm{sgn},\left(l\right)}\big\|_{2}\lesssim\sqrt{\frac{\log^{3}m}{m}}\big\|\bm{x}_{\perp}^{t,\mathrm{sgn}}-\bm{x}_{\perp}^{t,\mathrm{sgn},\left(l\right)}\big\|_{2},\label{eq:double-theta-3}
\end{equation}
provided that $m\gtrsim n\log^{2}m$. This further allows us to obtain
\begin{align}
\left|\theta_{4}\left(\tau\right)\right| & =\left|\frac{1}{m}\sum_{i=1}^{m}\left[3\left(\bm{a}_{i}^{\top}\tilde{\bm{x}}\left(\tau\right)\right)^{2}-\left(\bm{a}_{i}^{\top}\bm{x}^{\natural}\right)^{2}\right]\xi_{i}^{\text{sgn}}\left|a_{i,1}\right|\bm{a}_{i,\perp}^{\top}\big(\bm{x}_{\perp}^{t,\mathrm{sgn}}-\bm{x}_{\perp}^{t,\mathrm{sgn},\left(l\right)}\big)\right|\nonumber \\
 & \leq\left|\frac{1}{m}\sum_{i=1}^{m}\left\{ 3\left(\bm{a}_{i}^{\top}\bm{x}\left(\tau\right)\right)^{2}-\left|a_{i,1}\right|^{2}\right\} \xi_{i}^{\text{sgn}}\left|a_{i,1}\right|\bm{a}_{i,\perp}^{\top}\big(\bm{x}_{\perp}^{t}-\bm{x}_{\perp}^{t,(l)}\big)\right|\nonumber \\
 & \quad+\left|\frac{1}{m}\sum_{i=1}^{m}\left\{ 3\left(\bm{a}_{i}^{\top}\tilde{\bm{x}}\left(\tau\right)\right)^{2}-3\left(\bm{a}_{i}^{\top}\bm{x}\left(\tau\right)\right)^{2}\right\} \xi_{i}^{\text{sgn}}\left|a_{i,1}\right|\bm{a}_{i,\perp}^{\top}\big(\bm{x}_{\perp}^{t,\mathrm{sgn}}-\bm{x}_{\perp}^{t,\mathrm{sgn},\left(l\right)}\big)\right|\nonumber \\
 & \quad+\left|\frac{1}{m}\sum_{i=1}^{m}\left\{ 3\left(\bm{a}_{i}^{\top}\bm{x}\left(\tau\right)\right)^{2}-\left|a_{i,1}\right|^{2}\right\} \xi_{i}^{\text{sgn}}\left|a_{i,1}\right|\bm{a}_{i,\perp}^{\top}\left(\bm{x}_{\perp}^{t}-\bm{x}_{\perp}^{t,(l)}-\bm{x}_{\perp}^{t,\mathrm{sgn}}+\bm{x}_{\perp}^{t,\mathrm{sgn},(l)}\right)\right|\nonumber \\
 & \lesssim\sqrt{\frac{\log^{3}m}{m}}\left\Vert \bm{x}_{\perp}^{t}-\bm{x}_{\perp}^{t,(l)}\right\Vert _{2}+\sqrt{\log m}\left\Vert \bm{x}_{\perp}^{t,\mathrm{sgn}}-\bm{x}_{\perp}^{t,\mathrm{sgn},(l)}\right\Vert _{2}\left\Vert \bm{x}\left(\tau\right)-\tilde{\bm{x}}\left(\tau\right)\right\Vert _{2}\nonumber \\
 & \quad+\frac{1}{\log^{3/2}m}\left\Vert \bm{x}_{\perp}^{t}-\bm{x}_{\perp}^{t,(l)}-\bm{x}_{\perp}^{t,\mathrm{sgn}}+\bm{x}_{\perp}^{t,\mathrm{sgn}(l)}\right\Vert _{2}.\label{eq:double-theta-4}
\end{align}
To justify the last inequality, we first use similar bounds as in
(\ref{eq:double-theta-3}) to show that with probability exceeding
$1-O\left(m^{-10}\right)$, 
\[
\left|\frac{1}{m}\sum_{i=1}^{m}\left\{ 3\big(\bm{a}_{i}^{\top}\bm{x}(\tau)\big)^{2}-\left|a_{i,1}\right|^{2}\right\} \xi_{i}^{\text{sgn}}\left|a_{i,1}\right|\bm{a}_{i,\perp}^{\top}\big(\bm{x}_{\perp}^{t}-\bm{x}_{\perp}^{t,(l)}\big)\right|\lesssim\sqrt{\frac{\log^{3}m}{m}}\left\Vert \bm{x}_{\perp}^{t}-\bm{x}_{\perp}^{t,(l)}\right\Vert _{2}.
\]
In addition, we can invoke the Cauchy-Schwarz inequality to get 
\begin{align*}
 & \left|\frac{1}{m}\sum_{i=1}^{m}\left\{ 3\left(\bm{a}_{i}^{\top}\tilde{\bm{x}}\left(\tau\right)\right)^{2}-3\left(\bm{a}_{i}^{\top}\bm{x}\left(\tau\right)\right)^{2}\right\} \xi_{i}^{\text{sgn}}\left|a_{i,1}\right|\bm{a}_{i,\perp}^{\top}\left(\bm{x}_{\perp}^{t,\mathrm{sgn}}-\bm{x}_{\perp}^{t,\mathrm{sgn},(l)}\right)\right|\\
 & \quad\leq\sqrt{\left(\frac{1}{m}\sum_{i=1}^{m}\left\{ 3\left(\bm{a}_{i}^{\top}\tilde{\bm{x}}\left(\tau\right)\right)^{2}-3\left(\bm{a}_{i}^{\top}\bm{x}\left(\tau\right)\right)^{2}\right\} ^{2}\left|a_{i,1}\right|^{2}\right)\left(\frac{1}{m}\sum_{i=1}^{m}\left|\bm{a}_{i,\perp}^{\top}\left(\bm{x}_{\perp}^{t,\mathrm{sgn}}-\bm{x}_{\perp}^{t,\mathrm{sgn},(l)}\right)\right|^{2}\right)}\\
 & \quad\lesssim\sqrt{\frac{1}{m}\sum_{i=1}^{m}\left\{ \left(\bm{a}_{i}^{\top}\tilde{\bm{x}}\left(\tau\right)\right)^{2}-\left(\bm{a}_{i}^{\top}\bm{x}\left(\tau\right)\right)^{2}\right\} ^{2}\left|a_{i,1}\right|^{2}}\left\Vert \bm{x}_{\perp}^{t,\mathrm{sgn}}-\bm{x}_{\perp}^{t,\mathrm{sgn},(l)}\right\Vert _{2},
\end{align*}
where the last line arises from Lemma \ref{lemma:ai-ai-spectral-upper-bound}.
For the remaining term in the expression above, we have 
\begin{align*}
\sqrt{\frac{1}{m}\sum_{i=1}^{m}\left\{ \left(\bm{a}_{i}^{\top}\tilde{\bm{x}}\left(\tau\right)\right)^{2}-\left(\bm{a}_{i}^{\top}\bm{x}\left(\tau\right)\right)^{2}\right\} ^{2}\left|a_{i,1}\right|^{2}} & =\sqrt{\frac{1}{m}\sum_{i=1}^{m}\left|a_{i,1}\right|^{2}\left[\bm{a}_{i}^{\top}\left(\bm{x}\left(\tau\right)-\tilde{\bm{x}}\left(\tau\right)\right)\right]^{2}\left[\bm{a}_{i}^{\top}\left(\bm{x}\left(\tau\right)+\tilde{\bm{x}}\left(\tau\right)\right)\right]^{2}}\\
 & \overset{\left(\text{i}\right)}{\lesssim}\sqrt{\frac{\log m}{m}\sum_{i=1}^{m}\left|a_{i,1}\right|^{2}\left[\bm{a}_{i}^{\top}\left(\bm{x}\left(\tau\right)-\tilde{\bm{x}}\left(\tau\right)\right)\right]^{2}}\\
 & \overset{\left(\text{ii}\right)}{\lesssim}\sqrt{\log m}\left\Vert \bm{x}\left(\tau\right)-\tilde{\bm{x}}\left(\tau\right)\right\Vert _{2}.
\end{align*}
Here, (i) makes use of the incoherence condition (\ref{eq:double-diff-incoherence-condition}),
whereas (ii) comes from Lemma \ref{lemma:hessian-concentration}.
Regarding the last line in (\ref{eq:double-theta-4}), we have 
\begin{align*}
 & \left|\frac{1}{m}\sum_{i=1}^{m}\left\{ 3\left(\bm{a}_{i}^{\top}\bm{x}\left(\tau\right)\right)^{2}-\left|a_{i,1}\right|^{2}\right\} \xi_{i}^{\text{sgn}}\left|a_{i,1}\right|\bm{a}_{i,\perp}^{\top}\left(\bm{x}_{\perp}^{t}-\bm{x}_{\perp}^{t,(l)}-\bm{x}_{\perp}^{t,\mathrm{sgn}}+\bm{x}_{\perp}^{t,\mathrm{sgn},(l)}\right)\right|\\
 & \quad\leq\left\Vert \frac{1}{m}\sum_{i=1}^{m}\left\{ 3\left(\bm{a}_{i}^{\top}\bm{x}\left(\tau\right)\right)^{2}-\left|a_{i,1}\right|^{2}\right\} \xi_{i}^{\text{sgn}}\left|a_{i,1}\right|\bm{a}_{i,\perp}^{\top}\right\Vert _{2}\left\Vert \bm{x}_{\perp}^{t}-\bm{x}_{\perp}^{t,(l)}-\bm{x}_{\perp}^{t,\mathrm{sgn}}+\bm{x}_{\perp}^{t,\mathrm{sgn},(l)}\right\Vert _{2}.
\end{align*}
Since $\xi_{i}^{\text{sgn}}$ is independent of $\left\{ 3\left(\bm{a}_{i}^{\top}\bm{x}\left(\tau\right)\right)^{2}-\left|a_{i,1}\right|^{2}\right\} \left|a_{i,1}\right|\bm{a}_{i,\perp}^{\top}$,
one can apply the Bernstein inequality (see Lemma \ref{lemma:bernstein})
to deduce that 
\[
\left\Vert \frac{1}{m}\sum_{i=1}^{m}\left\{ 3\left(\bm{a}_{i}^{\top}\bm{x}\left(\tau\right)\right)^{2}-\left|a_{i,1}\right|^{2}\right\} \xi_{i}^{\text{sgn}}\left|a_{i,1}\right|\bm{a}_{i,\perp}^{\top}\right\Vert _{2}\lesssim\frac{1}{m}\left(\sqrt{V_{2}\log m}+B_{2}\log m\right),
\]
where 
\begin{align*}
V_{2} & :=\sum_{i=1}^{m}\left\{ 3\left(\bm{a}_{i}^{\top}\bm{x}\left(\tau\right)\right)^{2}-\left|a_{i,1}\right|^{2}\right\} ^{2}\left|a_{i,1}\right|^{2}\bm{a}_{i,\perp}^{\top}\bm{a}_{i,\perp}\lesssim mn\log^{3}m;\\
B_{2} & :=\max_{1\leq i\leq m}\left|3\left(\bm{a}_{i}^{\top}\bm{x}\left(\tau\right)\right)^{2}-\left|a_{i,1}\right|^{2}\right|\left|a_{i,1}\right|\left\Vert \bm{a}_{i,\perp}\right\Vert _{2}\lesssim\sqrt{n}\log^{3/2}m.
\end{align*}
This further implies 
\[
\left\Vert \frac{1}{m}\sum_{i=1}^{m}\left\{ 3\left(\bm{a}_{i}^{\top}\bm{x}\left(\tau\right)\right)^{2}-\left|a_{i,1}\right|^{2}\right\} \xi_{i}^{\text{sgn}}\left|a_{i,1}\right|\bm{a}_{i,\perp}^{\top}\right\Vert _{2}\lesssim\sqrt{\frac{n\log^{4}m}{m}}+\frac{\sqrt{n}\log^{5/2}m}{m}\lesssim\frac{1}{\log^{3/2}m},
\]
as long as $m\gg n\log^{7}m$. Take the previous bounds on $\theta_{1}\left(\tau\right)$,
$\theta_{2}\left(\tau\right)$, $\theta_{3}\left(\tau\right)$ and
$\theta_{4}\left(\tau\right)$ collectively to arrive at 
\begin{align*}
\left|w_{3,\parallel}\left(\tau\right)\right| & \lesssim\sqrt{\frac{n\log^{3}m}{m}}\left|\tilde{x}_{\parallel}\left(\tau\right)\right|\left|x_{\parallel}^{t,\text{sgn}}-x_{\parallel}^{t,\text{sgn},\left(l\right)}\right|+\sqrt{\frac{\log^{3}m}{m}}\left\Vert \bm{x}_{\perp}^{t,\mathrm{sgn}}-\bm{x}_{\perp}^{t,\mathrm{sgn},\left(l\right)}\right\Vert _{2}\\
 & \quad+\sqrt{\frac{\log^{3}m}{m}}\left\Vert \bm{x}_{\perp}^{t}-\bm{x}_{\perp}^{t,(l)}\right\Vert _{2}+\sqrt{\log m}\left\Vert \bm{x}_{\perp}^{t,\mathrm{sgn}}-\bm{x}_{\perp}^{t,\mathrm{sgn},(l)}\right\Vert _{2}\left\Vert \bm{x}\left(\tau\right)-\tilde{\bm{x}}\left(\tau\right)\right\Vert _{2}\\
 & \quad+\frac{1}{\log^{3/2}m}\left\Vert \bm{x}_{\perp}^{t}-\bm{x}_{\perp}^{t,(l)}-\bm{x}_{\perp}^{t,\mathrm{sgn}}+\bm{x}_{\perp}^{t,\mathrm{sgn}(l)}\right\Vert _{2}\\
 & \lesssim\sqrt{\frac{n\log^{3}m}{m}}\left|\tilde{x}_{\parallel}\left(\tau\right)\right|\left|x_{\parallel}^{t,\text{sgn}}-x_{\parallel}^{t,\text{sgn},\left(l\right)}\right|\\
 & \quad+\sqrt{\frac{\log^{3}m}{m}}\left\Vert \bm{x}_{\perp}^{t}-\bm{x}_{\perp}^{t,(l)}\right\Vert _{2}+\sqrt{\log m}\left\Vert \bm{x}_{\perp}^{t,\mathrm{sgn}}-\bm{x}_{\perp}^{t,\mathrm{sgn},(l)}\right\Vert _{2}\left\Vert \bm{x}\left(\tau\right)-\tilde{\bm{x}}\left(\tau\right)\right\Vert _{2}\\
 & \quad+\frac{1}{\log^{3/2}m}\left\Vert \bm{x}_{\perp}^{t}-\bm{x}_{\perp}^{t,(l)}-\bm{x}_{\perp}^{t,\mathrm{sgn}}+\bm{x}_{\perp}^{t,\mathrm{sgn}(l)}\right\Vert _{2},
\end{align*}
where the last inequality follows from the triangle inequality 
\[
\left\Vert \bm{x}_{\perp}^{t,\mathrm{sgn}}-\bm{x}_{\perp}^{t,\mathrm{sgn},\left(l\right)}\right\Vert _{2}\leq\left\Vert \bm{x}_{\perp}^{t}-\bm{x}_{\perp}^{t,(l)}\right\Vert _{2}+\left\Vert \bm{x}_{\perp}^{t}-\bm{x}_{\perp}^{t,(l)}-\bm{x}_{\perp}^{t,\mathrm{sgn}}+\bm{x}_{\perp}^{t,\mathrm{sgn}(l)}\right\Vert _{2}
\]
and the fact that $\sqrt{\frac{\log^{3}m}{m}}\leq\frac{1}{\log^{3/2}m}$
for $m$ sufficiently large. Similar to (\ref{eq:double-w-3-signal}),
we have the following identity for the 2nd through the $n$th entries
of $\bm{w}_{3}\left(\tau\right)$: 
\begin{align*}
\bm{w}_{3,\perp}\left(\tau\right) & =\frac{1}{m}\sum_{i=1}^{m}\rho_{i}\bm{a}_{i,\perp}\bm{a}_{i}^{\top}\left(\bm{x}^{t,\mathrm{sgn}}-\bm{x}^{t,\mathrm{sgn},(l)}\right)-\frac{1}{m}\sum_{i=1}^{m}\rho_{i}^{\text{sgn}}\bm{a}_{i,\perp}\bm{a}_{i}^{\text{sgn}\top}\left(\bm{x}^{t,\mathrm{sgn}}-\bm{x}^{t,\mathrm{sgn},(l)}\right)\\
 & =\frac{3}{m}\sum_{i=1}^{m}\left[\left(\bm{a}_{i}^{\top}\tilde{\bm{x}}\left(\tau\right)\right)^{2}\xi_{i}-\left(\bm{a}_{i}^{\text{sgn}\top}\tilde{\bm{x}}\left(\tau\right)\right)^{2}\xi_{i}^{\mathrm{sgn}}\right]\left|a_{i,1}\right|\bm{a}_{i,\perp}\left(x_{\parallel}^{t,\text{sgn}}-x_{\parallel}^{t,\text{sgn},\left(l\right)}\right)\\
 & \quad+\frac{3}{m}\sum_{i=1}^{m}\left|a_{i,1}\right|^{2}\left(\xi_{i}-\xi_{i}^{\mathrm{sgn}}\right)\left|a_{i,1}\right|\bm{a}_{i,\perp}\left(x_{\parallel}^{t,\text{sgn}}-x_{\parallel}^{t,\text{sgn},\left(l\right)}\right)\\
 & \quad+\frac{3}{m}\sum_{i=1}^{m}\left[\left(\bm{a}_{i}^{\top}\tilde{\bm{x}}\left(\tau\right)\right)^{2}-\left(\bm{a}_{i}^{\mathrm{sgn}\top}\tilde{\bm{x}}\left(\tau\right)\right)^{2}\right]\bm{a}_{i,\perp}\bm{a}_{i,\perp}^{\top}\left(\bm{x}_{\perp}^{t,\mathrm{sgn}}-\bm{x}_{\perp}^{t,\mathrm{sgn},(l)}\right).
\end{align*}
It is easy to check by Lemma \ref{lemma:hessian-concentration} and
the incoherence conditions $\left|\bm{a}_{i}^{\top}\tilde{\bm{x}}\left(\tau\right)\right|\lesssim\sqrt{\log m}\left\Vert \tilde{\bm{x}}\left(\tau\right)\right\Vert _{2}$
and $\left|\bm{a}_{i}^{\text{sgn}\top}\tilde{\bm{x}}\left(\tau\right)\right|\lesssim\sqrt{\log m}\left\Vert \tilde{\bm{x}}\left(\tau\right)\right\Vert _{2}$
that 
\[
\frac{1}{m}\sum_{i=1}^{m}\left(\bm{a}_{i}^{\top}\tilde{\bm{x}}\left(\tau\right)\right)^{2}\xi_{i}\left|a_{i,1}\right|\bm{a}_{i,\perp}=2\tilde{x}_{1}\left(\tau\right)\tilde{\bm{x}}_{\perp}\left(\tau\right)+O\left(\sqrt{\frac{n\log^{3}m}{m}}\right),
\]
and
\begin{align*}
\frac{1}{m}\sum_{i=1}^{m}\left(\bm{a}_{i}^{\text{sgn}\top}\tilde{\bm{x}}\left(\tau\right)\right)^{2}\xi_{i}^{\text{sgn}}\left|a_{i,1}\right|\bm{a}_{i,\perp} & =2\tilde{x}_{1}\left(\tau\right)\tilde{\bm{x}}_{\perp}\left(\tau\right)+O\left(\sqrt{\frac{n\log^{3}m}{m}}\right).
\end{align*}
Besides, in view of (\ref{eq:xi-xi-sgn-ai-duplicate}), we have 
\[
\left\Vert \frac{3}{m}\sum_{i=1}^{m}\left|a_{i,1}\right|^{2}\left(\xi_{i}-\xi_{i}^{\mathrm{sgn}}\right)\left|a_{i,1}\right|\bm{a}_{i,\perp}\right\Vert _{2}\lesssim\sqrt{\frac{n\log^{3}m}{m}}.
\]
We are left with controlling $\left\Vert \frac{3}{m}\sum_{i=1}^{m}\left[\left(\bm{a}_{i}^{\top}\tilde{\bm{x}}\left(\tau\right)\right)^{2}-\left(\bm{a}_{i}^{\mathrm{sgn}\top}\tilde{\bm{x}}\left(\tau\right)\right)^{2}\right]\bm{a}_{i,\perp}\bm{a}_{i,\perp}^{\top}\left(\bm{x}_{\perp}^{t,\mathrm{sgn}}-\bm{x}_{\perp}^{t,\mathrm{sgn},(l)}\right)\right\Vert _{2}$.
To this end, one can see from (\ref{eq:identity-square-diff-double})
that
\begin{align*}
 & \left\Vert \frac{3}{m}\sum_{i=1}^{m}\left[\left(\bm{a}_{i}^{\top}\tilde{\bm{x}}\left(\tau\right)\right)^{2}-\left(\bm{a}_{i}^{\mathrm{sgn}\top}\tilde{\bm{x}}\left(\tau\right)\right)^{2}\right]\bm{a}_{i,\perp}\bm{a}_{i,\perp}^{\top}\left(\bm{x}_{\perp}^{t,\mathrm{sgn}}-\bm{x}_{\perp}^{t,\mathrm{sgn},(l)}\right)\right\Vert _{2}\\
 & \quad=\left\Vert \tilde{x}_{\parallel}\left(\tau\right)\cdot\frac{6}{m}\sum_{i=1}^{m}\left(\xi_{i}-\xi_{i}^{\mathrm{sgn}}\right)\left|a_{i,1}\right|\bm{a}_{i,\perp}\bm{a}_{i,\perp}^{\top}\tilde{\bm{x}}_{\perp}\left(\tau\right)\bm{a}_{i,\perp}^{\top}\left(\bm{x}_{\perp}^{t,\mathrm{sgn}}-\bm{x}_{\perp}^{t,\mathrm{sgn},(l)}\right)\right\Vert \\
 & \quad\leq12\max_{1\leq i\leq m}\left|a_{i,1}\right|\left|\tilde{x}_{\parallel}\left(\tau\right)\right|\max_{1\leq i\leq m}\left|\bm{a}_{i,\perp}^{\top}\tilde{\bm{x}}_{\perp}\left(\tau\right)\right|\left\Vert \frac{1}{m}\sum_{i=1}^{m}\bm{a}_{i,\perp}\bm{a}_{i,\perp}^{\top}\right\Vert \left\Vert \bm{x}_{\perp}^{t,\mathrm{sgn}}-\bm{x}_{\perp}^{t,\mathrm{sgn},(l)}\right\Vert _{2}\\
 & \quad\lesssim\log m\left|\tilde{x}_{\parallel}\left(\tau\right)\right|\left\Vert \bm{x}_{\perp}^{t,\mathrm{sgn}}-\bm{x}_{\perp}^{t,\mathrm{sgn},(l)}\right\Vert _{2},
\end{align*}
where the last relation arises from (\ref{eq:max-a-i-1}), the incoherence
condition $\max_{1\leq i\leq m}\left|\bm{a}_{i,\perp}^{\top}\tilde{\bm{x}}_{\perp}\left(\tau\right)\right|\lesssim\sqrt{\log m}$
and Lemma \ref{lemma:ai-ai-spectral-upper-bound}. Hence the 2nd through
the $n$th entries of $\bm{w}_{3}\left(\tau\right)$ obey 
\[
\left\Vert \bm{w}_{3,\perp}\left(\tau\right)\right\Vert _{2}\lesssim\sqrt{\frac{n\log^{3}m}{m}}\left|x_{\parallel}^{t,\text{sgn}}-x_{\parallel}^{t,\text{sgn},\left(l\right)}\right|+\log m\left|\tilde{x}_{\parallel}\left(\tau\right)\right|\left\Vert \bm{x}_{\perp}^{t,\mathrm{sgn}}-\bm{x}_{\perp}^{t,\mathrm{sgn},(l)}\right\Vert _{2}.
\]
Combine the above estimates to arrive at 
\begin{align*}
\left\Vert \bm{w}_{3}\left(\tau\right)\right\Vert _{2} & \leq\left|w_{3,\parallel}\left(\tau\right)\right|+\left\Vert \bm{w}_{3,\perp}\left(\tau\right)\right\Vert _{2}\\
 & \leq\log m\left|\tilde{x}_{\parallel}\left(\tau\right)\right|\left\Vert \bm{x}_{\perp}^{t,\mathrm{sgn}}-\bm{x}_{\perp}^{t,\mathrm{sgn},(l)}\right\Vert _{2}+\sqrt{\frac{n\log^{3}m}{m}}\left|x_{\parallel}^{t,\text{sgn}}-x_{\parallel}^{t,\text{sgn},\left(l\right)}\right|\\
 & \quad+\sqrt{\frac{\log^{3}m}{m}}\left\Vert \bm{x}_{\perp}^{t}-\bm{x}_{\perp}^{t,(l)}\right\Vert _{2}+\sqrt{\log m}\left\Vert \bm{x}_{\perp}^{t,\mathrm{sgn}}-\bm{x}_{\perp}^{t,\mathrm{sgn},(l)}\right\Vert _{2}\left\Vert \bm{x}\left(\tau\right)-\tilde{\bm{x}}\left(\tau\right)\right\Vert _{2}\\
 & \quad+\frac{1}{\log^{3/2}m}\left\Vert \bm{x}_{\perp}^{t}-\bm{x}_{\perp}^{t,(l)}-\bm{x}_{\perp}^{t,\mathrm{sgn}}+\bm{x}_{\perp}^{t,\mathrm{sgn}(l)}\right\Vert _{2}.
\end{align*}
\end{itemize}
\item Putting together the preceding bounds on $\bm{v}_{1}$ and $\bm{v}_{2}$
($\bm{w}_{1}\left(\tau\right)$, $\bm{w}_{2}\left(\tau\right)$ and
$\bm{w}_{3}\left(\tau\right)$), we can deduce that 
\begin{align}
 & \left\Vert \bm{x}^{t+1}-\bm{x}^{t+1,(l)}-\bm{x}^{t+1,\mathrm{sgn}}+\bm{x}^{t+1,\mathrm{sgn},(l)}\right\Vert _{2}\nonumber \\
 & \quad=\left\Vert \bm{x}^{t}-\bm{x}^{t,(l)}-\bm{x}^{t,\mathrm{sgn}}+\bm{x}^{t,\mathrm{sgn},(l)}-\eta\left(\int_{0}^{1}\bm{w}_{1}\left(\tau\right)\mathrm{d}\tau+\int_{0}^{1}\bm{w}_{2}\left(\tau\right)\mathrm{d}\tau+\int_{0}^{1}\bm{w}_{3}\left(\tau\right)\mathrm{d}\tau\right)-\eta\bm{v}_{2}\right\Vert _{2}\nonumber \\
 & \quad\leq\left\Vert \bm{x}^{t}-\bm{x}^{t,(l)}-\bm{x}^{t,\mathrm{sgn}}+\bm{x}^{t,\mathrm{sgn},(l)}-\eta\int_{0}^{1}\bm{w}_{1}\left(\tau\right)\mathrm{d}\tau\right\Vert _{2}+\eta\sup_{0\leq\tau\leq1}\left\Vert \bm{w}\left(\tau\right)\right\Vert _{2}+\eta\sup_{0\leq\tau\leq1}\left\Vert \bm{w}_{3}\left(\tau\right)\right\Vert _{2}+\eta\left\Vert \bm{v}_{2}\right\Vert _{2}\nonumber \\
 & \quad\leq\left\{ 1+3\eta\left(1-\left\Vert \bm{x}^{t}\right\Vert _{2}^{2}\right)+\eta\phi_{1}\right\} \left\Vert \bm{x}^{t}-\bm{x}^{t,(l)}-\bm{x}^{t,\mathrm{sgn}}+\bm{x}^{t,\mathrm{sgn},(l)}\right\Vert _{2}\nonumber \\
 & \quad\quad+O\left(\eta\left\{ \sqrt{n\log m}\max_{1\leq l\leq m}\big\|\bm{x}^{t}-\bm{x}^{t,\mathrm{sgn}}-\bm{x}^{t,(l)}+\bm{x}^{t,\mathrm{sgn},(l)}\big\|_{2}+\log m\left\Vert \bm{x}^{t}-\bm{x}^{t,\mathrm{sgn}}\right\Vert _{2}\right\} \left\Vert \bm{x}^{t,\text{sgn}}-\bm{x}^{t,\text{sgn},(l)}\right\Vert _{2}\right)\nonumber \\
 & \quad\quad+O\left(\eta\log m\sup_{0\leq\tau\leq1}\left|\tilde{x}_{\parallel}\left(\tau\right)\right|\left\Vert \bm{x}^{t,\mathrm{sgn}}-\bm{x}^{t,\mathrm{sgn},(l)}\right\Vert _{2}\right)+O\left(\eta\sqrt{\frac{n\log^{3}m}{m}}\left|x_{\parallel}^{t,\text{sgn}}-x_{\parallel}^{t,\text{sgn},\left(l\right)}\right|\right)\nonumber \\
 & \quad\quad+O\left(\eta\sqrt{\frac{\log^{3}m}{m}}\left\Vert \bm{x}_{\perp}^{t}-\bm{x}_{\perp}^{t,(l)}\right\Vert _{2}\right)+O\left(\eta\sqrt{\log m}\left\Vert \bm{x}_{\perp}^{t,\mathrm{sgn}}-\bm{x}_{\perp}^{t,\mathrm{sgn},(l)}\right\Vert _{2}\sup_{0\leq\tau\leq1}\left\Vert \bm{x}\left(\tau\right)-\tilde{\bm{x}}\left(\tau\right)\right\Vert _{2}\right).\nonumber \\
 & \quad\quad+O\left(\eta\frac{\log^{2}m}{m}\left\Vert \bm{x}_{\perp}^{t,\left(l\right)}\right\Vert _{2}\right)+O\left(\eta\frac{\sqrt{n\log^{3}m}}{m}\left(\left|x_{\parallel}^{t,\left(l\right)}\right|+\left\Vert \bm{x}^{t,\left(l\right)}-\bm{x}^{t,\text{sgn},\left(l\right)}\right\Vert _{2}\right)\right).\label{eq:double-diff-fist-combine}
\end{align}
To simplify the preceding bound, we first make the following claim,
whose proof is deferred to the end of this subsection. \begin{claim}\label{claim:double-diff}For
$t\leq T_{0}$, the following inequalities hold: 
\begin{align*}
\sqrt{n\log m}\left\Vert \bm{x}^{t,\text{sgn}}-\bm{x}^{t,\text{sgn},(l)}\right\Vert _{2} & \ll\frac{1}{\log m};\\
\log m\sup_{0\leq\tau\leq1}\left|\tilde{x}_{\parallel}\left(\tau\right)\right|+\log m\left\Vert \bm{x}^{t}-\bm{x}^{t,\mathrm{sgn}}\right\Vert _{2}+\sqrt{\log m}\sup_{0\leq\tau\leq1}\left\Vert \bm{x}\left(\tau\right)-\tilde{\bm{x}}\left(\tau\right)\right\Vert _{2} & +\frac{\sqrt{n\log^{3}m}}{m}\lesssim\alpha_{t}\log m;\\
\alpha_{t}\log m & \ll\frac{1}{\log m}.
\end{align*}
\end{claim}Armed with Claim \ref{claim:double-diff}, one can rearrange
terms in (\ref{eq:double-diff-fist-combine}) to obtain for some $|\phi_{2}|,|\phi_{3}|\ll\frac{1}{\log m}$
\begin{align*}
 & \left\Vert \bm{x}^{t+1}-\bm{x}^{t+1,(l)}-\bm{x}^{t+1,\mathrm{sgn}}+\bm{x}^{t+1,\mathrm{sgn},(l)}\right\Vert _{2}\\
 & \quad\leq\left\{ 1+3\eta\left(1-\left\Vert \bm{x}^{t}\right\Vert _{2}^{2}\right)+\eta\phi_{2}\right\} \max_{1\leq l\leq m}\left\Vert \bm{x}^{t}-\bm{x}^{t,(l)}-\bm{x}^{t,\mathrm{sgn}}+\bm{x}^{t,\mathrm{sgn},(l)}\right\Vert _{2}\\
 & \quad\quad+\eta O\left(\log m\cdot\alpha_{t}+\sqrt{\frac{\log^{3}m}{m}}+\frac{\log^{2}m}{m}\right)\left\Vert \bm{x}^{t}-\bm{x}^{t,(l)}\right\Vert _{2}\\
 & \quad\quad+\eta O\left(\sqrt{\frac{n\log^{3}m}{m}}+\frac{\sqrt{n\log^{3}m}}{m}\right)\left|x_{\parallel}^{t}-x_{\parallel}^{t,\left(l\right)}\right|+\eta\frac{\log^{2}m}{m}\left\Vert \bm{x}_{\perp}^{t}\right\Vert _{2}\\
 & \quad\quad+\eta O\left(\frac{\sqrt{n\log^{3}m}}{m}\right)\left(\left|x_{\parallel}^{t}\right|+\left\Vert \bm{x}^{t}-\bm{x}^{t,\text{sgn}}\right\Vert _{2}\right)\\
 & \quad\leq\left\{ 1+3\eta\left(1-\left\Vert \bm{x}^{t}\right\Vert _{2}^{2}\right)+\eta\phi_{3}\right\} \left\Vert \bm{x}^{t}-\bm{x}^{t,(l)}-\bm{x}^{t,\mathrm{sgn}}+\bm{x}^{t,\mathrm{sgn},(l)}\right\Vert _{2}\\
 & \quad\quad+O\left(\eta\log m\right)\cdot\alpha_{t}\left\Vert \bm{x}^{t}-\bm{x}^{t,(l)}\right\Vert _{2}\\
 & \quad\quad+O\left(\eta\sqrt{\frac{n\log^{3}m}{m}}\right)\left|x_{\parallel}^{t}-x_{\parallel}^{t,\left(l\right)}\right|+O\left(\eta\frac{\log^{2}m}{m}\right)\left\Vert \bm{x}_{\perp}^{t}\right\Vert _{2}\\
 & \quad\quad+O\left(\eta\frac{\sqrt{n\log^{3}m}}{m}\right)\left(\left|x_{\parallel}^{t}\right|+\left\Vert \bm{x}^{t}-\bm{x}^{t,\text{sgn}}\right\Vert _{2}\right).
\end{align*}
Substituting in the hypotheses (\ref{subeq:induction}), we can arrive
at 
\begin{align*}
 & \left\Vert \bm{x}^{t+1}-\bm{x}^{t+1,(l)}-\bm{x}^{t+1,\mathrm{sgn}}+\bm{x}^{t+1,\mathrm{sgn},(l)}\right\Vert _{2}\\
 & \quad\leq\left\{ 1+3\eta\left(1-\left\Vert \bm{x}^{t}\right\Vert _{2}^{2}\right)+\eta\phi_{3}\right\} \alpha_{t}\left(1+\frac{1}{\log m}\right)^{t}C_{4}\frac{\sqrt{n\log^{9}m}}{m}\\
 & \quad+O\left(\eta\log m\right)\alpha_{t}\beta_{t}\left(1+\frac{1}{\log m}\right)^{t}C_{1}\frac{\sqrt{n\log^{5}m}}{m}\\
 & \quad+O\left(\eta\sqrt{\frac{\log^{3}m}{m}}\right)\beta_{t}\left(1+\frac{1}{\log m}\right)^{t}C_{1}\frac{\sqrt{n\log^{5}m}}{m}\\
 & \quad+O\left(\eta\sqrt{\frac{n\log^{3}m}{m}}\right)\alpha_{t}\left(1+\frac{1}{\log m}\right)^{t}C_{2}\frac{\sqrt{n\log^{12}m}}{m}\\
 & \quad+O\left(\eta\frac{\log^{2}m}{m}\right)\beta_{t}+O\left(\eta\frac{\sqrt{n\log^{3}m}}{m}\right)\alpha_{t}\\
 & \quad+O\left(\eta\frac{\sqrt{n\log^{3}m}}{m}\right)\alpha_{t}\left(1+\frac{1}{\log m}\right)^{t}C_{3}\sqrt{\frac{n\log^{5}m}{m}}\\
 & \quad\overset{\left(\text{i}\right)}{\leq}\left\{ 1+3\eta\left(1-\left\Vert \bm{x}^{t}\right\Vert _{2}^{2}\right)+\eta\phi_{4}\right\} \alpha_{t}\left(1+\frac{1}{\log m}\right)^{t}C_{4}\frac{\sqrt{n\log^{9}m}}{m}\\
 & \quad\overset{\left(\text{ii}\right)}{\leq}\alpha_{t+1}\left(1+\frac{1}{\log m}\right)^{t+1}C_{4}\frac{\sqrt{n\log^{9}m}}{m}
\end{align*}
for some $|\phi_{4}|\ll\frac{1}{\log m}$. Here, the last relation
(ii) follows the same argument as in (\ref{eq:beta-t-iterative-reasoning})
and (i) holds true as long as \begin{subequations} 
\begin{align}
\left(\log m\right)\alpha_{t}\beta_{t}\left(1+\frac{1}{\log m}\right)^{t}C_{1}\frac{\sqrt{n\log^{5}m}}{m} & \ll\frac{1}{\log m}\alpha_{t}\left(1+\frac{1}{\log m}\right)^{t}C_{4}\frac{\sqrt{n\log^{9}m}}{m};\label{eq:double-as-long-as-1}\\
\sqrt{\frac{n\log^{3}m}{m}}\alpha_{t}\left(1+\frac{1}{\log m}\right)^{t}C_{2}\frac{\sqrt{n\log^{12}m}}{m} & \ll\frac{1}{\log m}\alpha_{t}\left(1+\frac{1}{\log m}\right)^{t}C_{4}\frac{\sqrt{n\log^{9}m}}{m};\label{eq:double-as-long-as-2}\\
\frac{\log^{2}m}{m}\beta_{t} & \ll\frac{1}{\log m}\alpha_{t}\left(1+\frac{1}{\log m}\right)^{t}C_{4}\frac{\sqrt{n\log^{9}m}}{m};\label{eq:double-as-long-as-3}\\
\frac{\sqrt{n\log^{3}m}}{m}\alpha_{t}\left(1+\frac{1}{\log m}\right)^{t}C_{3}\sqrt{\frac{n\log^{5}m}{m}} & \ll\frac{1}{\log m}\alpha_{t}\left(1+\frac{1}{\log m}\right)^{t}C_{4}\frac{\sqrt{n\log^{9}m}}{m};\label{eq:double-as-long-as-4}\\
\frac{\sqrt{n\log^{3}m}}{m}\alpha_{t} & \ll\frac{1}{\log m}\alpha_{t}\left(1+\frac{1}{\log m}\right)^{t}C_{4}\frac{\sqrt{n\log^{9}m}}{m},\label{eq:double-as-long-as-5}
\end{align}
\end{subequations}where we recall that $t\leq T_{0}\lesssim\log n$.
The first condition (\ref{eq:double-as-long-as-1}) can be checked
using $\beta_{t}\lesssim1$ and the assumption that $C_{4}>0$ is
sufficiently large. The second one is valid if $m\gg n\log^{8}m$.
In addition, the third condition follows from the relationship (see
Lemma \ref{lemma:iterative})
\[
\beta_{t}\lesssim\alpha_{t}\sqrt{n\log m}.
\]
It is also easy to see that the last two are both valid. \begin{proof}[Proof
of Claim \ref{claim:double-diff}]For the first claim, it is east
to see from the triangle inequality that 
\begin{align*}
 & \sqrt{n\log m}\left\Vert \bm{x}^{t,\text{sgn}}-\bm{x}^{t,\text{sgn},\left(l\right)}\right\Vert _{2}\\
 & \leq\sqrt{n\log m}\left(\left\Vert \bm{x}^{t}-\bm{x}^{t,\left(l\right)}\right\Vert _{2}+\left\Vert \bm{x}^{t}-\bm{x}^{t,\left(l\right)}-\bm{x}^{t,\text{sgn}}+\bm{x}^{t,\text{sgn},\left(l\right)}\right\Vert _{2}\right)\\
 & \leq\sqrt{n\log m}\beta_{t}\left(1+\frac{1}{\log m}\right)^{t}C_{1}\frac{\sqrt{n\log^{5}m}}{m}+\sqrt{n\log m}\alpha_{t}\left(1+\frac{1}{\log m}\right)^{t}C_{4}\frac{\sqrt{n\log^{9}m}}{m}\\
 & \lesssim\frac{n\log^{3}m}{m}+\frac{n\log^{5}m}{m}\ll\frac{1}{\log m},
\end{align*}
as long as $m\gg n\log^{6}m$. Here, we have invoked the upper bounds
on $\alpha_{t}$ and $\beta_{t}$ provided in Lemma \ref{lemma:iterative}.
Regarding the second claim, we have 
\begin{align*}
\left|\tilde{x}_{\parallel}\left(\tau\right)\right| & \leq\left|x_{\parallel}^{t,\text{sgn}}\right|+\left|x_{\parallel}^{t,\text{sgn},\left(l\right)}\right|\leq2\left|x_{\parallel}^{t,\text{sgn}}\right|+\left|x_{\parallel}^{t,\text{sgn}}-x_{\parallel}^{t,\text{sgn},\left(l\right)}\right|\\
 & \leq2\left|x_{\parallel}^{t}\right|+2\left\Vert \bm{x}^{t}-\bm{x}^{t,\text{sgn}}\right\Vert _{2}+\left|x_{\parallel}^{t}-x_{\parallel}^{t,\left(l\right)}\right|+\left\Vert \bm{x}^{t}-\bm{x}^{t,\left(l\right)}-\bm{x}^{t,\text{sgn}}+\bm{x}^{t,\text{sgn},\left(l\right)}\right\Vert _{2}\\
 & \lesssim\alpha_{t}\left(1+\sqrt{\frac{n\log^{5}m}{m}}+\frac{\sqrt{n\log^{12}m}}{m}+\frac{\sqrt{n\log^{9}m}}{m}\right)\lesssim\alpha_{t},
\end{align*}
as long as $m\gg n\log^{5}m$. Similar arguments can lead us to conclude
that the remaining terms on the left-hand side of the second inequality
in the claim are bounded by $O(\alpha_{t})$. The third claim is an
immediate consequence of the fact $\alpha_{t}\ll\frac{1}{\log^{5}m}$
(see Lemma \ref{lemma:iterative}). \end{proof} 
\end{itemize}

\section{Proof of Lemma \ref{lemma:xt-signal-phase-2}\label{sec:Proof-of-Lemma-xt-signal-phase-2}}

Recall from Appendix \ref{sec:Proof-of-Lemma-xt-signal} that 
\[
x_{\parallel}^{t+1}=\left\{ 1+3\eta\left(1-\left\Vert \bm{x}^{t}\right\Vert _{2}^{2}\right)+O\left(\eta\sqrt{\frac{n\log^{3}m}{m}}\right)\right\} x_{\parallel}^{t}+J_{2}-J_{4},
\]
where $J_{2}$ and $J_{4}$ are defined respectively as 
\begin{align*}
J_{2} & :=\eta\left[1-3\big(x_{\parallel}^{t}\big)^{2}\right]\cdot\frac{1}{m}\sum_{i=1}^{m}a_{i,1}^{3}\bm{a}_{i,\perp}^{\top}\bm{x}_{\perp}^{t};\\
J_{4} & :=\eta\cdot\frac{1}{m}\sum_{i=1}^{m}\left(\bm{a}_{i,\perp}^{\top}\bm{x}_{\perp}^{t}\right)^{3}a_{i,1}.
\end{align*}
Instead of resorting to the leave-one-out sequence $\left\{ \bm{x}^{t,\text{sgn}}\right\} $
as in Appendix \ref{sec:Proof-of-Lemma-xt-signal}, we can directly
apply Lemma \ref{lemma:ai-uniform-concentration} and the incoherence
condition (\ref{eq:incoherence-phase-2}) to obtain 
\begin{align*}
\left|J_{2}\right| & \leq\eta\left|1-3\big(x_{\parallel}^{t}\big)^{2}\right|\left|\frac{1}{m}\sum_{i=1}^{m}a_{i,1}^{3}\bm{a}_{i,\perp}^{\top}\bm{x}_{\perp}^{t}\right|\ll\eta\frac{1}{\log^{6}m}\left\Vert \bm{x}_{\perp}^{t}\right\Vert _{2}\ll\eta\frac{1}{\log m}\alpha_{t};\\
\left|J_{4}\right| & \leq\eta\left|\frac{1}{m}\sum_{i=1}^{m}\big(\bm{a}_{i,\perp}^{\top}\bm{x}_{\perp}^{t}\big)a_{i,1}\right|\ll\eta\frac{1}{\log^{6}m}\left\Vert \bm{x}_{\perp}^{t}\right\Vert _{2}^{3}\ll\eta\frac{1}{\log m}\alpha_{t}
\end{align*}
with probability at least $1-O\left(m^{-10}\right)$, as long as $m\gg n\log^{13}m$.
Here, the last relations come from the fact that $\alpha_{t}\geq\frac{c}{\log^{5}m}$
(see Lemma \ref{lemma:iterative}). Combining the previous estimates
gives
\[
\alpha_{t+1}=\left\{ 1+3\eta\left(1-\left\Vert \bm{x}^{t}\right\Vert _{2}^{2}\right)+\eta\zeta_{t}\right\} \alpha_{t},
\]
with $|\zeta_{t}|\ll\frac{1}{\log m}$. This finishes the proof. 

\section{Proof of Lemma \ref{lemma:xt-xt-l-phase-2}\label{sec:Proof-of-Lemma-xt-xt-l-phase-2}}

In view of Appendix \ref{sec:Proof-of-Lemma-xt-xt-l}, one has 
\begin{align*}
\left\Vert \bm{x}^{t+1}-\bm{x}^{t+1,\left(l\right)}\right\Vert _{2} & \leq\left\{ 1+3\eta\left(1-\left\Vert \bm{x}^{t}\right\Vert _{2}^{2}\right)+\eta\phi_{1}\right\} \left\Vert \bm{x}^{t}-\bm{x}^{t,\left(l\right)}\right\Vert _{2}+O\left(\eta\frac{\sqrt{n\log^{3}m}}{m}\left\Vert \bm{x}^{t}\right\Vert _{2}\right),
\end{align*}
for some $|\phi_{1}|\ll\frac{1}{\log m}$, where we use the trivial
upper bound 
\[
2\eta\left|x_{\parallel}^{t}-x_{\parallel}^{t,\left(l\right)}\right|\leq2\eta\left\Vert \bm{x}^{t}-\bm{x}^{t,\left(l\right)}\right\Vert _{2}.
\]
Under the hypotheses (\ref{eq:induction-xt-xt-l-phase-2}), we can
obtain 
\begin{align*}
\left\Vert \bm{x}^{t+1}-\bm{x}^{t+1,\left(l\right)}\right\Vert _{2} & \leq\left\{ 1+3\eta\left(1-\left\Vert \bm{x}^{t}\right\Vert _{2}^{2}\right)+\eta\phi_{1}\right\} \alpha_{t}\left(1+\frac{1}{\log m}\right)^{t}C_{6}\frac{\sqrt{n\log^{15}m}}{m}+O\left(\eta\frac{\sqrt{n\log^{3}m}}{m}\left(\alpha_{t}+\beta_{t}\right)\right)\\
 & \leq\left\{ 1+3\eta\left(1-\left\Vert \bm{x}^{t}\right\Vert _{2}^{2}\right)+\eta\phi_{2}\right\} \alpha_{t}\left(1+\frac{1}{\log m}\right)^{t}C_{6}\frac{\sqrt{n\log^{15}m}}{m}\\
 & \leq\alpha_{t+1}\left(1+\frac{1}{\log m}\right)^{t+1}C_{6}\frac{\sqrt{n\log^{15}m}}{m},
\end{align*}
for some $|\phi_{2}|\ll\frac{1}{\log m}$, as long as $\eta$ is sufficiently
small and
\[
\frac{\sqrt{n\log^{3}m}}{m}\left(\alpha_{t}+\beta_{t}\right)\ll\frac{1}{\log m}\alpha_{t}\left(1+\frac{1}{\log m}\right)^{t}C_{6}\frac{\sqrt{n\log^{15}m}}{m}.
\]
This is satisfied since, according to Lemma \ref{lemma:iterative},
\[
\frac{\sqrt{n\log^{3}m}}{m}\left(\alpha_{t}+\beta_{t}\right)\lesssim\frac{\sqrt{n\log^{3}m}}{m}\lesssim\frac{\sqrt{n\log^{13}m}}{m}\alpha_{t}\ll\frac{1}{\log m}\alpha_{t}\left(1+\frac{1}{\log m}\right)^{t}C_{6}\frac{\sqrt{n\log^{15}m}}{m},
\]
as long as $C_{6}>0$ is sufficiently large.

\section{Proof of Lemma \ref{lemma:ai-uniform-concentration}\label{sec:Proof-of-Lemma-ai-uniform-concentration}}

Without loss of generality, it suffices to consider all the\emph{
}unit vectors $\bm{z}$ obeying $\left\Vert \bm{z}\right\Vert _{2}=1$.
To begin with, for any given $\bm{z}$, we can express the quantities of interest as $\frac{1}{m}\sum_{i=1}^{m}\left(g_{i}\left(\bm{z}\right)-G\left(\bm{z}\right)\right),$
where $g_{i}(\bm{z})$ depends only on $\bm{z}$ and $\bm{a}_{i}$.
Note that 
\[
g_{i}\left(\bm{z}\right)=a_{i,1}^{\theta_{1}}\left(\bm{a}_{i,\perp}^{\top}\bm{z}\right)^{\theta_{2}}
\]
for different $\theta_{1},\theta_{2}\in\left\{ 1,2,3,4,6\right\} $
in each of the cases considered herein. It can be easily verified
from Gaussianality that in all of these cases, for any fixed \emph{unit}
vector $\bm{z}$ one has 
\begin{align}
\EE\left[g_{i}^{2}\left(\bm{z}\right)\right] & \lesssim\left(\EE\left[\left|g_{i}\left(\bm{z}\right)\right|\right]\right)^{2};\label{eq:variance-upper}\\
\EE\left[\left|g_{i}\left(\bm{z}\right)\right|\right] & \asymp1;\label{eq:uniform-mean}\\
\left|\EE\left[g_{i}\left(\bm{z}\right)\ind_{\left\{ \left|\bm{a}_{i,\perp}^{\top}\bm{z}\right|\leq\beta\left\Vert \bm{z}\right\Vert _{2},\left|a_{i,1}\right|\leq5\sqrt{\log m}\right\} }\right]-\EE\left[g_{i}\left(\bm{z}\right)\right]\right| & \leq\frac{1}{n}\EE\left[\left|g_{i}\left(\bm{z}\right)\right|\right].\label{eq:tail-bound}
\end{align}
In addition, on the event $\left\{\max_{1\leq i\leq m}\left\Vert \bm{a}_{i}\right\Vert _{2}\leq\sqrt{6n}\right\}$
which has probability at least $1-O\left(me^{-1.5n}\right)$, one
has, for any fixed unit vectors $\bm{z},\bm{z}_{0}$, that
\begin{equation}
\left|g_{i}\left(\bm{z}\right)-g_{i}\left(\bm{z}_{0}\right)\right|\leq n^{\alpha}\left\Vert \bm{z}-\bm{z}_{0}\right\Vert _{2}\label{eq:uniform-Lip}
\end{equation}
for some parameter $\alpha=O\left(1\right)$ in all cases. In light
of these properties, we will proceed by controlling $\frac{1}{m}\sum_{i=1}^{m}g_{i}\left(\bm{z}\right)-\EE\left[g_{i}\left(\bm{z}\right)\right]$
in a unified manner.

We start by looking at any fixed vector $\bm{z}$ independent of $\left\{ \bm{a}_{i}\right\} $.
Recognizing that 
\[
\frac{1}{m}\sum_{i=1}^{m}g_{i}\left(\bm{z}\right)\ind_{\left\{ \left|\bm{a}_{i,\perp}^{\top}\bm{z}\right|\leq\beta\left\Vert \bm{z}\right\Vert _{2},\left|a_{i,1}\right|\leq5\sqrt{\log m}\right\} }-\EE\left[g_{i}\left(\bm{z}\right)\ind_{\left\{ \left|\bm{a}_{i,\perp}^{\top}\bm{z}\right|\leq\beta\left\Vert \bm{z}\right\Vert _{2},\left|a_{i,1}\right|\leq5\sqrt{\log m}\right\} }\right]
\]
is a sum of $m$ i.i.d.~random variables, one can thus apply the
Bernstein inequality to obtain 
\begin{align*}
 & \mathbb{P}\left\{ \left|\frac{1}{m}\sum_{i=1}^{m}g_{i}\left(\bm{z}\right)\ind_{\left\{ \left|\bm{a}_{i,\perp}^{\top}\bm{z}\right|\leq\beta\left\Vert \bm{z}\right\Vert _{2},\left|a_{i,1}\right|\leq5\sqrt{\log m}\right\} }-\EE\left[g_{i}\left(\bm{z}\right)\ind_{\left\{ \left|\bm{a}_{i,\perp}^{\top}\bm{z}\right|\leq\beta\left\Vert \bm{z}\right\Vert _{2},\left|a_{i,1}\right|\leq5\sqrt{\log m}\right\} }\right]\right|\geq\tau\right\} \\
 & \quad\leq2\exp\left(-\frac{\tau^{2}/2}{V+\tau B/3}\right),
\end{align*}
where the two quantities $V$ and $B$ obey 
\begin{align}
V & :=\frac{1}{m^{2}}\sum_{i=1}^{m}\mathbb{E}\left[g_{i}^2 \left(\bm{z}\right)\ind_{\left\{ \left|\bm{a}_{i,\perp}^{\top}\bm{z}\right|\leq\beta\left\Vert \bm{z}\right\Vert _{2},\left|a_{i,1}\right|\leq5\sqrt{\log m}\right\} } \right]  \leq\frac{1}{m}\EE\left[g_{i}^{2}\left(\bm{z}\right)\right]\lesssim\frac{1}{m}\left(\EE\left[\left|g_{i}\left(\bm{z}\right)\right|\right]\right)^{2};\label{eq:defn-V-poly}\\
B & :=\frac{1}{m}\max_{1\leq i\leq m}\left\{ \left|g_{i}\left(\bm{z}\right)\right|\ind_{\left\{ \left|\bm{a}_{i,\perp}^{\top}\bm{z}\right|\leq\beta\left\Vert \bm{z}\right\Vert _{2},\left|a_{i,1}\right|\leq5\sqrt{\log m}\right\} }\right\} .\label{eq:defn-B-poly}
\end{align}
Here the penultimate relation of \eqref{eq:defn-V-poly} follows from (\ref{eq:variance-upper}).
Taking $\tau=\epsilon\,\EE\left[\left|g_{i}\left(\bm{z}\right)\right|\right]$,
we can deduce that 
\begin{equation}
\left|\frac{1}{m}\sum_{i=1}^{m}g_{i}\left(\bm{z}\right)\ind_{\left\{ \left|\bm{a}_{i,\perp}^{\top}\bm{z}\right|\leq\beta\left\Vert \bm{z}\right\Vert _{2},\left|a_{i,1}\right|\leq5\sqrt{\log m}\right\} }-\EE\left[g_{i}\left(\bm{z}\right)\ind_{\left\{ \left|\bm{a}_{i,\perp}^{\top}\bm{z}\right|\leq\beta\left\Vert \bm{z}\right\Vert _{2},\left|a_{i,1}\right|\leq5\sqrt{\log m}\right\} }\right]\right|\leq\epsilon\,\EE\left[\left|g_{i}\left(\bm{z}\right)\right|\right]\label{eq:even1-UB-2}
\end{equation}
with probability exceeding $1-2\min\left\{ \exp\left(-c_{1}m\epsilon^{2}\right),\exp\left(-\frac{c_{2}\epsilon\EE\left[\left|g_{i}\left(\bm{z}\right)\right|\right]}{B}\right)\right\} $
for some constants $c_{1},c_{2}>0$. In particular, when $m\epsilon^{2}/(n\log n)$
and $\epsilon\EE\left[\left|g_{i}\left(\bm{z}\right)\right|\right]/(Bn\log n)$
are both sufficiently large, the inequality (\ref{eq:even1-UB-2})
holds with probability exceeding $1-2\exp\left(-c_{3}n\log n\right)$
for some constant $c_{3}>0$ sufficiently large.

We then move on to extending this result to a uniform bound. Let $\mathcal{N}_{\theta}$
be a $\theta$-net of the unit sphere with cardinality $\left|\mathcal{N}_{\theta}\right|\leq\left(1+\frac{2}{\theta}\right)^{n}$
such that for any $\bm{z}$ on the unit sphere, one can find a point
$\bm{z}_{0}\in\mathcal{N}_{\theta}$ such that $\left\Vert \bm{z}-\bm{z}_{0}\right\Vert _{2}\leq\theta$.
Apply the triangle inequality to obtain
\begin{align*}
 & \left|\frac{1}{m}\sum_{i=1}^{m}g_{i}\left(\bm{z}\right)\ind_{\left\{ \left|\bm{a}_{i,\perp}^{\top}\bm{z}\right|\leq\beta\left\Vert \bm{z}\right\Vert _{2},\left|a_{i,1}\right|\leq5\sqrt{\log m}\right\} }-\EE\left[g_{i}\left(\bm{z}\right)\ind_{\left\{ \left|\bm{a}_{i,\perp}^{\top}\bm{z}\right|\leq\beta\left\Vert \bm{z}\right\Vert _{2},\left|a_{i,1}\right|\leq5\sqrt{\log m}\right\} }\right]\right|\\
 & \quad\leq\underbrace{\left|\frac{1}{m}\sum_{i=1}^{m}g_{i}\left(\bm{z}_{0}\right)\ind_{\left\{ \left|\bm{a}_{i,\perp}^{\top}\bm{z}_{0}\right|\leq\beta\left\Vert \bm{z}_0\right\Vert _{2},\left|a_{i,1}\right|\leq5\sqrt{\log m}\right\} }-\EE\left[g_{i}\left(\bm{z}_{0}\right)\ind_{\left\{ \left|\bm{a}_{i,\perp}^{\top}\bm{z}_{0}\right|\leq\beta\left\Vert \bm{z}_0\right\Vert _{2},\left|a_{i,1}\right|\leq5\sqrt{\log m}\right\} }\right]\right|}_{:=I_{1}}\\
 & \quad\quad+\underbrace{\left|\frac{1}{m}\sum_{i=1}^{m}\left[g_{i}\left(\bm{z}\right)\ind_{\left\{ \left|\bm{a}_{i,\perp}^{\top}\bm{z}\right|\leq\beta\left\Vert \bm{z}\right\Vert _{2},\left|a_{i,1}\right|\leq5\sqrt{\log m}\right\} }-g_{i}\left(\bm{z}_{0}\right)\ind_{\left\{ \left|\bm{a}_{i,\perp}^{\top}\bm{z}_{0}\right|\leq\beta\left\Vert \bm{z}_0\right\Vert _{2},\left|a_{i,1}\right|\leq5\sqrt{\log m}\right\} } \right]\right|}_{:=I_{2}},
\end{align*}
where the second line arises from the fact that 
\[
\EE\left[g_{i}\left(\bm{z}\right)\ind_{\left\{ \left|\bm{a}_{i,\perp}^{\top}\bm{z}\right|\leq\beta\left\Vert \bm{z}\right\Vert _{2},\left|a_{i,1}\right|\leq5\sqrt{\log m}\right\} }\right]=\EE\left[g_{i}\left(\bm{z}_{0}\right)\ind_{\left\{ \left|\bm{a}_{i,\perp}^{\top}\bm{z}_{0}\right|\leq\beta\left\Vert \bm{z}_{0}\right\Vert _{2},\left|a_{i,1}\right|\leq5\sqrt{\log m}\right\} }\right].
\]
With regard to the first term $I_{1}$, by the union bound, with probability
at least $1-2\left(1+\frac{2}{\theta}\right)^{n}\exp\left(-c_{3}n\log n\right)$,
one has 
\[
I_{1}\leq\epsilon\,\EE\left[\left|g_{i}\left(\bm{z}_{0}\right)\right|\right].
\]
It remains to bound $I_{2}$. Denoting $\mathcal{S}_{i}=\left\{ \bm{z}\mid\big|\bm{a}_{i,\perp}^{\top}\bm{z}\big|\leq\beta\left\Vert \bm{z}\right\Vert _{2},|a_{i,1}|\leq5\sqrt{\log m}\right\} $,
we have 
\begin{align}
I_{2} & =\left|\frac{1}{m}\sum_{i=1}^{m}g_{i}\left(\bm{z}\right)\ind_{\left\{ \bm{z}\in\mathcal{S}_{i}\right\} }-g_{i}\left(\bm{z}_{0}\right)\ind_{\left\{ \bm{z}_{0}\in\mathcal{S}_{i}\right\} }\right|\nonumber \\
 & \leq\left|\frac{1}{m}\sum_{i=1}^{m}\left(g_{i}\left(\bm{z}\right)-g_{i}\left(\bm{z}_{0}\right)\right)\ind_{\left\{ \bm{z}\in\mathcal{S}_{i},\bm{z}_{0}\in\cS_{i}\right\} }\right|+\left|\frac{1}{m}\sum_{i=1}^{m}g_{i}\left(\bm{z}\right)\ind_{\left\{ \bm{z}\in\mathcal{S}_{i},\bm{z}_{0}\notin\mathcal{S}_{i}\right\} }\right|+\left|\frac{1}{m}\sum_{i=1}^{m}g_{i}\left(\bm{z}_{0}\right)\ind_{\left\{ \bm{z}\notin\mathcal{S}_{i},\bm{z}_{0}\in\mathcal{S}_{i}\right\} }\right|\nonumber \\
 & \leq\frac{1}{m}\sum_{i=1}^{m}\left|g_{i}\left(\bm{z}\right)-g_{i}\left(\bm{z}_{0}\right)\right|+\frac{1}{m}\max_{1\leq i\leq m}\left|g_{i}\left(\bm{z}\right)\ind_{\left\{ \bm{z}\in\mathcal{S}_{i}\right\} }\right|\cdot\sum_{i=1}^{m}\ind_{\left\{ \bm{z}\in\mathcal{S}_{i},\bm{z}_{0}\notin\mathcal{S}_{i}\right\} }\nonumber \\
 & \quad+\frac{1}{m}\max_{1\leq i\leq m}\left|g_{i}\left(\bm{z}_{0}\right)\ind_{\left\{ \bm{z}_{0}\in\mathcal{S}_{i}\right\} }\right|\cdot\sum_{i=1}^{m}\ind_{\left\{ \bm{z}\notin\mathcal{S}_{i},\bm{z}_{0}\in\mathcal{S}_{i}\right\} }.\label{eq:uniform-I-2}
\end{align}
For the first term in (\ref{eq:uniform-I-2}), it follows from (\ref{eq:uniform-Lip})
that 
\[
\frac{1}{m}\sum_{i=1}^{m}\left|g_{i}\left(\bm{z}\right)-g_{i}\left(\bm{z}_{0}\right)\right|\leq n^{\alpha}\left\Vert \bm{z}-\bm{z}_{0}\right\Vert _{2}\leq n^{\alpha}\theta.
\]
For the second term of (\ref{eq:uniform-I-2}), we have 
\begin{align}
\ind_{\left\{ \bm{z}\in\mathcal{S}_{i},\bm{z}_{0}\notin\mathcal{S}_{i}\right\} } & \leq\ind_{\left\{ \left|\bm{a}_{i,\perp}^{\top}\bm{z}\right|\leq\beta,\left|\bm{a}_{i,\perp}^{\top}\bm{z}_{0}\right|\geq\beta\right\} }\nonumber \\
 & =\ind_{\left\{ \left|\bm{a}_{i,\perp}^{\top}\bm{z}\right|\leq\beta\right\} }\left(\ind_{\left\{ \left|\bm{a}_{i,\perp}^{\top}\bm{z}_{0}\right|\geq\beta+\sqrt{6n}\theta\right\} }+\ind_{\left\{ \beta\leq\left|\bm{a}_{i,\perp}^{\top}\bm{z}_{0}\right|<\beta+\sqrt{6n}\theta\right\} }\right)\nonumber \\
 & =\ind_{\left\{ \left|\bm{a}_{i,\perp}^{\top}\bm{z}\right|\leq\beta\right\} }\ind_{\left\{ \beta\leq\left|\bm{a}_{i,\perp}^{\top}\bm{z}_{0}\right|\leq\beta+\sqrt{6n}\theta\right\} }\label{eq:third-identity-I}\\
 & \leq\ind_{\left\{ \beta\leq\left|\bm{a}_{i,\perp}^{\top}\bm{z}_{0}\right|\leq\beta+\sqrt{6n}\theta\right\} }.\nonumber 
\end{align}
Here, the identity \eqref{eq:third-identity-I} holds due to the fact
that 
\[
\ind_{\left\{ \left|\bm{a}_{i,\perp}^{\top}\bm{z}\right|\leq\beta\right\} }\ind_{\left\{ \left|\bm{a}_{i,\perp}^{\top}\bm{z}_{0}\right|\geq\beta+\sqrt{6n}\theta\right\} }=0;
\]
in fact, under the condition $\big|\bm{a}_{i,\perp}^{\top}\bm{z}_{0}\big|\geq\beta+\sqrt{6n}\theta$
one has
\[
\left|\bm{a}_{i,\perp}^{\top}\bm{z}\right|\geq\left|\bm{a}_{i,\perp}^{\top}\bm{z}_{0}\right|-\left|\bm{a}_{i,\perp}^{\top}\left(\bm{z}-\bm{z}_{0}\right)\right|\geq\beta+\sqrt{6n}\theta-\left\Vert \bm{a}_{i,\perp}\right\Vert _{2}\left\Vert \bm{z}-\bm{z}_{0}\right\Vert _{2}>\beta+\sqrt{6n}\theta-\sqrt{6n}\theta\geq\beta,
\]
which is contradictory to $\left|\bm{a}_{i,\perp}^{\top}\bm{z}\right|\leq\beta$.
As a result, one can obtain 
\[
\sum_{i=1}^{m}\ind_{\left\{ \bm{z}\in\mathcal{S}_{i},\bm{z}_{0}\notin\mathcal{S}_{i}\right\} }\leq\sum_{i=1}^{m}\ind_{\left\{ \beta\leq\left|\bm{a}_{i,\perp}^{\top}\bm{z}_{0}\right|\leq\beta+\sqrt{6n}\theta\right\} }\leq2Cn\log n,
\]
with probability at least $1-e^{-\frac{2}{3}Cn\log n}$ for a sufficiently
large constant $C>0$, where the last inequality follows from the
Chernoff bound (see Lemma \ref{lemma:chernoff}). This together with
the union bound reveals that with probability exceeding $1-\left(1+\frac{2}{\theta}\right)^{n}e^{-\frac{2}{3}Cn\log n}$,
\[
\frac{1}{m}\max_{1\leq i\leq m}\left|g_{i}\left(\bm{z}\right)\ind_{\left\{ \bm{z}\in\mathcal{S}_{i}\right\} }\right|\cdot\sum_{i=1}^{m}\ind_{\left\{ \bm{z}\in\mathcal{S}_{i},\bm{z}_{0}\notin\mathcal{S}_{i}\right\} }\leq B\cdot2Cn\log n
\]
with $B$ defined in \eqref{eq:defn-B-poly}. Similarly, one can show
that 
\[
\frac{1}{m}\max_{1\leq i\leq m}\left|g_{i}\left(\bm{z}_{0}\right)\ind_{\left\{ \bm{z}_{0}\in\mathcal{S}_{i}\right\} }\right|\cdot\sum_{i=1}^{m}\ind_{\left\{ \bm{z}\notin\mathcal{S}_{i},\bm{z}_{0}\in\mathcal{S}_{i}\right\} }\leq B\cdot2Cn\log n.
\]
Combine the above bounds to reach that 
\[
I_{1}+I_{2}\leq\epsilon\,\EE\left[\left|g_{i}\left(\bm{z}_{0}\right)\right|\right]+n^{\alpha}\theta+4B\cdot Cn\log n\leq2\epsilon\,\EE\left[\left|g_{i}\left(\bm{z}\right)\right|\right],
\]
as long as 
\[
n^{\alpha}\theta\leq\frac{\epsilon}{2}\,\EE\left[\left|g_{i}\left(\bm{z}\right)\right|\right]\qquad\text{and}\qquad4B\cdot Cn\log n\leq\frac{\epsilon}{2}\,\EE\left[\left|g_{i}\left(\bm{z}\right)\right|\right].
\]
In view of the fact (\ref{eq:uniform-mean}), one can take $\theta\asymp\epsilon n^{-\alpha}$
to conclude that 
\begin{equation}
\left|\frac{1}{m}\sum_{i=1}^{m}g_{i}\left(\bm{z}\right)\ind_{\left\{ \left|\bm{a}_{i,\perp}^{\top}\bm{z}\right|\leq\beta\left\Vert \bm{z}\right\Vert _{2},\left|a_{i,1}\right|\leq5\sqrt{\log m}\right\} }-\EE\left[g_{i}\left(\bm{z}\right)\ind_{\left\{ \left|\bm{a}_{i,\perp}^{\top}\bm{z}\right|\leq\beta\left\Vert \bm{z}\right\Vert _{2},\left|a_{i,1}\right|\leq5\sqrt{\log m}\right\} }\right]\right|\leq2\epsilon\,\EE\left[\left|g_{i}\left(\bm{z}\right)\right|\right]\label{eq:even1-UB-1}
\end{equation}
holds for all $\bm{z}\in\RR^{n}$ with probability at least $1-2\exp\left(-c_{4}n\log n\right)$
for some constant $c_{4}>0$, with the proviso that $\epsilon\geq\frac{1}{n}$
and that $\epsilon\,\EE\left[\left|g_{i}\left(\bm{z}\right)\right|\right]/\left(Bn\log n\right)$
sufficiently large.

Further, we note that $\{\max_{i}\left|a_{i,1}\right|\leq5\sqrt{\log m}\}$
occurs with probability at least $1-O(m^{-10})$. Therefore, on an
event of probability at least $1-O(m^{-10})$, one has 
\begin{equation}
\frac{1}{m}\sum_{i=1}^{m}g_{i}\left(\bm{z}\right)=\frac{1}{m}\sum_{i=1}^{m}g_{i}\left(\bm{z}\right)\ind_{\left\{ \left|\bm{a}_{i,\perp}^{\top}\bm{z}\right|\leq\beta\left\Vert \bm{z}\right\Vert _{2},\left|a_{i,1}\right|\leq5\sqrt{\log m}\right\} }\label{eq:even1-UB-1-1}
\end{equation}
for all $\bm{z}\in\mathbb{R}^{n-1}$ obeying $\max_{i}\big|\bm{a}_{i,\perp}^{\top}\bm{z}\big|\leq\beta\left\Vert \bm{z}\right\Vert _{2}$.
On this event, one can use the triangle inequality to obtain 
\begin{align*}
\left|\frac{1}{m}\sum_{i=1}^{m}g_{i}\left(\bm{z}\right)-\EE\left[g_{i}\left(\bm{z}\right)\right]\right| & =\left|\frac{1}{m}\sum_{i=1}^{m}g_{i}\left(\bm{z}\right)\ind_{\left\{ \left|\bm{a}_{i,\perp}^{\top}\bm{z}\right|\leq\beta\left\Vert \bm{z}\right\Vert _{2},\left|a_{i,1}\right|\leq5\sqrt{\log m}\right\} }-\EE\left[g_{i}\left(\bm{z}\right)\right]\right|\\
 & \leq\left|\frac{1}{m}\sum_{i=1}^{m}g_{i}\left(\bm{z}\right)\ind_{\left\{ \left|\bm{a}_{i,\perp}^{\top}\bm{z}\right|\leq\beta\left\Vert \bm{z}\right\Vert _{2},\left|a_{i,1}\right|\leq5\sqrt{\log m}\right\} }-\EE\left[g_{i}\left(\bm{z}\right)\ind_{\left\{ \left|\bm{a}_{i,\perp}^{\top}\bm{z}\right|\leq\beta\left\Vert \bm{z}\right\Vert _{2},|a_{i,1}|\leq5\sqrt{\log m}\right\} }\right]\right|\\
 & \quad+\left|\EE\left[g_{i}\left(\bm{z}\right)\ind_{\left\{ \left|\bm{a}_{i,\perp}^{\top}\bm{z}\right|\leq\beta\left\Vert \bm{z}\right\Vert _{2},\left|a_{i,1}\right|\leq5\sqrt{\log m}\right\} }\right]-\EE\left[g_{i}\left(\bm{z}\right)\right]\right|\\
 & \leq2\epsilon\,\EE\left[\left|g_{i}\left(\bm{z}\right)\right|\right]+\frac{1}{n}\EE\left[\left|g_{i}\left(\bm{z}\right)\right|\right]\\
 & \leq3\epsilon\,\EE\left[\left|g_{i}\left(\bm{z}\right)\right|\right],
\end{align*}
as long as $\epsilon>1/n$, where the penultimate line follows from \eqref{eq:tail-bound}. This leads to the desired uniform upper
bound for $\frac{1}{m}\sum_{i=1}^{m}g_{i}\left(\bm{z}\right)-\EE\left[g_{i}\left(\bm{z}\right)\right]$,
namely, with probability at least $1-O\left(m^{-10}\right)$, 
\[
\left|\frac{1}{m}\sum_{i=1}^{m}g_{i}\left(\bm{z}\right)-\EE\left[g_{i}\left(\bm{z}\right)\right]\right|\leq3\epsilon\,\EE\left[\left|g_{i}\left(\bm{z}\right)\right|\right]
\]
holds uniformly for all $\bm{z}\in\mathbb{R}^{n-1}$ obeying $\max_{i}\big|\bm{a}_{i,\perp}^{\top}\bm{z}\big|\leq\beta\left\Vert \bm{z}\right\Vert _{2}$,
provided that
\[
m\epsilon^{2}/(n\log n)\quad\text{and}\quad\epsilon\EE\left[\left|g_{i}\left(\bm{z}\right)\right|\right]/\left(Bn\log n\right)
\]
are both sufficiently large (with $B$ defined in \eqref{eq:defn-B-poly}).

To finish up, we provide the bounds on $B$ and the resulting sample
complexity conditions for each case as follows. 
\begin{itemize}
\item For $g_{i}\left(\bm{z}\right)=a_{i,1}^{3}\bm{a}_{i,\perp}^{\top}\bm{z}$,
one has $B\lesssim\frac{1}{m}\beta\log^{\frac{3}{2}}m$, and hence
we need $m\gg\max\left\{ \frac{1}{\epsilon^{2}}n\log n,\text{ }\frac{1}{\epsilon}\beta n\log^{\frac{5}{2}}m\right\} $; 
\item For $g_{i}\left(\bm{z}\right)=a_{i,1}\left(\bm{a}_{i,\perp}^{\top}\bm{z}\right)^{3}$,
one has $B\lesssim\frac{1}{m}\beta^{3}\log^{\frac{1}{2}}m$, and hence
we need $m\gg\max\left\{ \frac{1}{\epsilon^{2}}n\log n,\text{ }\frac{1}{\epsilon}\beta^{3}n\log^{\frac{3}{2}}m\right\} $ 
\item For $g_{i}\left(\bm{z}\right)=a_{i,1}^{2}\left(\bm{a}_{i,\perp}^{\top}\bm{z}\right)^{2}$,
we have $B\lesssim\frac{1}{m}\beta^{2}\log m$, and hence $m\gg\max\left\{ \frac{1}{\epsilon^{2}}n\log n,\text{ }\frac{1}{\epsilon}\beta^{2}n\log^{2}m\right\} $; 
\item For $g_{i}\left(\bm{z}\right)=a_{i,1}^{6}\left(\bm{a}_{i,\perp}^{\top}\bm{z}\right)^{2}$,
we have $B\lesssim\frac{1}{m}\beta^{2}\log^{3}m$, and hence $m\gg\max\left\{ \frac{1}{\epsilon^{2}}n\log n,\text{ }\frac{1}{\epsilon}\beta^{2}n\log^{4}m\right\} $; 
\item For $g_{i}\left(\bm{z}\right)=a_{i,1}^{2}\left(\bm{a}_{i,\perp}^{\top}\bm{z}\right)^{6}$
, one has $B\lesssim\frac{1}{m}\beta^{6}\log m$, and hence $m\gg\max\left\{ \frac{1}{\epsilon^{2}}n\log n,\text{ }\frac{1}{\epsilon}\beta^{6}n\log^{2}m\right\} $; 
\item For $g_{i}\left(\bm{z}\right)=a_{i,1}^{2}\left(\bm{a}_{i,\perp}^{\top}\bm{z}\right)^{4}$,
one has $B\lesssim\frac{1}{m}\beta^{4}\log m$, and hence $m\gg\max\left\{ \frac{1}{\epsilon^{2}}n\log n,\text{ }\frac{1}{\epsilon}\beta^{4}n\log^{2}m\right\} $. 
\end{itemize}
Given that $\epsilon$ can be arbitrary quantity above $1/n$, we
establish the advertised results.

\section{Proof of Lemma \ref{lemma:hessian-concentration}\label{sec:Proof-of-Lemma-hessian-concentration}}

Note that if the second claim \eqref{eq:hessian-second-claim} holds, we can readily use it to justify
the first one \eqref{eq:hessian-first-claim} by observing that 
\[
\max_{1\leq i\leq m}\left|\bm{a}_{i}^{\top}\bm{x}^{\natural}\right|\leq5\sqrt{\log m}\left\Vert \bm{x}^{\natural}\right\Vert _{2}
\]
holds with probability at least $1-O\left(m^{-10}\right)$. As a consequence,
the proof is devoted to justifying the second claim in the lemma. 

First, notice that it suffices to consider all $\bm{z}$'s with unit
norm, i.e.~$\left\Vert \bm{z}\right\Vert _{2}=1$. We can then apply
the triangle inequality to obtain 
\begin{align*}
\left\Vert \frac{1}{m}\sum_{i=1}^{m}\left(\bm{a}_{i}^{\top}\bm{z}\right)^{2}\bm{a}_{i}\bm{a}_{i}^{\top}-\bm{I}_{n}-2\bm{z}\bm{z}^{\top}\right\Vert  & \leq\underbrace{\left\Vert \frac{1}{m}\sum_{i=1}^{m}\left(\bm{a}_{i}^{\top}\bm{z}\right)^{2}\bm{a}_{i}\bm{a}_{i}^{\top}\ind_{\left\{ \left|\bm{a}_{i}^{\top}\bm{z}\right|\leq c_{2}\sqrt{\log m}\right\} }-\left(\beta_{1}\bm{I}_{n}+\beta_{2}\bm{z}\bm{z}^{\top}\right)\right\Vert }_{:=\theta_{1}}\\
 & \quad+\underbrace{\left\Vert \beta_{1}\bm{I}_{n}+\beta_{2}\bm{z}\bm{z}^{\top}-\left(\bm{I}_{n}+2\bm{z}\bm{z}^{\top}\right)\right\Vert }_{:=\theta_{2}},
\end{align*}
where 
\[
\beta_{1}:=\EE\left[\xi^{2}\ind_{\left\{ \left|\xi\right|\leq c_{2}\sqrt{\log m}\right\} }\right]\qquad\text{and}\qquad\beta_{2}:=\EE\left[\xi^{4}\ind_{\left\{ \left|\xi\right|\leq c_{2}\sqrt{\log m}\right\} }\right]-\beta_{1}
\]
with $\xi\sim N\left(0,1\right)$. 
\begin{itemize}
\item For the second term $\theta_{2}$, we can further bound it as follows
\begin{align*}
\theta_{2} & \leq\left\Vert \beta_{1}\bm{I}_{n}-\bm{I}_{n}\right\Vert +\left\Vert \beta_{2}\bm{z}\bm{z}^{\top}-2\bm{z}\bm{z}^{\top}\right\Vert \\
 & \leq\left|\beta_{1}-1\right|+\left|\beta_{2}-2\right|,
\end{align*}
which motivates us to bound $\left|\beta_{1}-1\right|$ and $\left|\beta_{2}-2\right|$.
Towards this end, simple calculation yields 
\begin{align*}
1-\beta_{1} & =\sqrt{\frac{2}{\pi}}\cdot c_{2}\sqrt{\log m}e^{-\frac{c_{2}^{2}\log m}{2}}+\text{erfc}\left(\frac{c_{2}\sqrt{\log m}}{2}\right)\\
 & \overset{\left(\text{i}\right)}{\leq}\sqrt{\frac{2}{\pi}}\cdot c_{2}\sqrt{\log m}e^{-\frac{c_{2}^{2}\log m}{2}}+\frac{1}{\sqrt{\pi}}\frac{2}{c_{2}\sqrt{\log m}}e^{-\frac{c_{2}^{2}\log m}{4}}\\
 & \overset{\left(\text{ii}\right)}{\leq}\frac{1}{m},
\end{align*}
where (i) arises from the fact that for all $x>0$, $\text{erfc}\left(x\right)\leq\frac{1}{\sqrt{\pi}}\frac{1}{x}e^{-x^{2}}$and
(ii) holds as long as $c_{2}>0$ is sufficiently large. Similarly,
for the difference $\left|\beta_{2}-2\right|$, one can easily show
that
\begin{align}
\left|\beta_{2}-2\right| & \leq\left|\EE\left[\xi^{4}\ind_{\left\{ \left|\xi\right|\leq c_{2}\sqrt{\log m}\right\} }\right]-3\right|+\left|\beta_{1}-1\right|\leq\frac{2}{m}.\label{eq:uniform-beta-2}
\end{align}
Take the previous two bounds collectively to reach 
\[
\theta_{2}\leq\frac{3}{m}.
\]
\item With regards to $\theta_{1}$, we resort to the standard covering
argument. First, fix some $\bm{x},\bm{z}\in\RR^{n}$ with $\left\Vert \bm{x}\right\Vert _{2}=\left\Vert \bm{z}\right\Vert _{2}=1$
and notice that 
\[
\frac{1}{m}\sum_{i=1}^{m}\left(\bm{a}_{i}^{\top}\bm{z}\right)^{2}\left(\bm{a}_{i}^{\top}\bm{x}\right)^{2}\ind_{\left\{ \left|\bm{a}_{i}^{\top}\bm{z}\right|\leq c_{2}\sqrt{\log m}\right\} }-\beta_{1}-\beta_{2}\left(\bm{z}^{\top}\bm{x}\right)^{2}
\]
is a sum of $m$ i.i.d.~random variables with bounded sub-exponential
norms. To see this, one has 
\[
\left\Vert \left(\bm{a}_{i}^{\top}\bm{z}\right)^{2}\left(\bm{a}_{i}^{\top}\bm{x}\right)^{2}\ind_{\left\{ \left|\bm{a}_{i}^{\top}\bm{z}\right|\leq c_{2}\sqrt{\log m}\right\} }\right\Vert _{\psi_{1}}\leq c_{2}^{2}\log m\left\Vert \left(\bm{a}_{i}^{\top}\bm{x}\right)^{2}\right\Vert _{\psi_{1}}\leq c_{2}^{2}\log m,
\]
where $\|\cdot\|_{\psi_{1}}$ denotes the sub-exponential norm \cite{Vershynin2012}.
This further implies that
\[
\left\Vert \left(\bm{a}_{i}^{\top}\bm{z}\right)^{2}\left(\bm{a}_{i}^{\top}\bm{x}\right)^{2}\ind_{\left\{ \left|\bm{a}_{i}^{\top}\bm{z}\right|\leq c_{2}\sqrt{\log m}\right\} }-\beta_{1}-\beta_{2}\left(\bm{z}^{\top}\bm{x}\right)^{2}\right\Vert _{\psi_{1}}\leq2c_{2}^{2}\log m.
\]
Apply the Bernstein's inequality to show that for any $0\leq\epsilon\leq1$,
\[
\PP\left(\left|\frac{1}{m}\sum_{i=1}^{m}\left(\bm{a}_{i}^{\top}\bm{z}\right)^{2}\left(\bm{a}_{i}^{\top}\bm{x}\right)^{2}\ind_{\left\{ \left|\bm{a}_{i}^{\top}\bm{z}\right|\leq c_{2}\sqrt{\log m}\right\} }-\beta_{1}-\beta_{2}\left(\bm{z}^{\top}\bm{x}\right)^{2}\right|\geq2\epsilon c_{2}^{2}\log m\right)\leq2\exp\left(-c\epsilon^{2}m\right),
\]
where $c>0$ is some absolute constant. Taking $\epsilon\asymp\sqrt{\frac{n\log m}{m}}$
reveals that with probability exceeding $1-2\exp\left(-c_{10}n\log m\right)$
for some $c_{10}>0$, one has 
\begin{equation}
\left|\frac{1}{m}\sum_{i=1}^{m}\left(\bm{a}_{i}^{\top}\bm{z}\right)^{2}\left(\bm{a}_{i}^{\top}\bm{x}\right)^{2}\ind_{\left\{ \left|\bm{a}_{i}^{\top}\bm{z}\right|\leq c_{2}\sqrt{\log m}\right\} }-\beta_{1}-\beta_{2}\left(\bm{z}^{\top}\bm{x}\right)^{2}\right|\lesssim c_2^2\sqrt{\frac{n\log^{3}m}{m}}.\label{eq:uniform-fix-x-z}
\end{equation}
One can then apply the covering argument to extend the above result
to all unit vectors $\bm{x},\bm{z}\in\RR^{n}$. Let $\mathcal{N}_{\theta}$ be
a $\theta$-net of the unit sphere, which has cardinality at most
$\left(1+\frac{2}{\theta}\right)^{n}$. Then for every $\bm{x},\bm{z}\in\RR$
with unit norm, we can find $\bm{x}_{0},\bm{z}_{0}\in\mathcal{N}_{\theta}$
such that $\|\bm{x}-\bm{x}_{0}\|_{2}\leq\theta$ and $\|\bm{z}-\bm{z}_{0}\|_{2}\leq\theta$.
The triangle inequality reveals that 
\begin{align*}
 & \left|\frac{1}{m}\sum_{i=1}^{m}\left(\bm{a}_{i}^{\top}\bm{z}\right)^{2}\left(\bm{a}_{i}^{\top}\bm{x}\right)^{2}\ind_{\left\{ \left|\bm{a}_{i}^{\top}\bm{z}\right|\leq c_{2}\sqrt{\log m}\right\} }-\beta_{1}-\beta_{2}\left(\bm{z}^{\top}\bm{x}\right)^{2}\right|\\
 & \quad\leq\underbrace{\left|\frac{1}{m}\sum_{i=1}^{m}\left(\bm{a}_{i}^{\top}\bm{z}_{0}\right)^{2}\left(\bm{a}_{i}^{\top}\bm{x}_{0}\right)^{2}\ind_{\left\{ \left|\bm{a}_{i}^{\top}\bm{z}_{0}\right|\leq c_{2}\sqrt{\log m}\right\} }-\beta_{1}-\beta_{2}\left(\bm{z}_{0}^{\top}\bm{x}_{0}\right)^{2}\right|}_{:=I_{1}}+\underbrace{\beta_{2}\left|\left(\bm{z}^{\top}\bm{x}\right)^{2}-\left(\bm{z}_{0}^{\top}\bm{x}_{0}\right)^{2}\right|}_{:=I_{2}}\\
 & \quad\quad+\underbrace{\left|\frac{1}{m}\sum_{i=1}^{m}\left[\left(\bm{a}_{i}^{\top}\bm{z}\right)^{2}\left(\bm{a}_{i}^{\top}\bm{x}\right)^{2}\ind_{\left\{ \left|\bm{a}_{i}^{\top}\bm{z}\right|\leq c_{2}\sqrt{\log m}\right\} }-\left(\bm{a}_{i}^{\top}\bm{z}_{0}\right)^{2}\left(\bm{a}_{i}^{\top}\bm{x}_{0}\right)^{2}\ind_{\left\{ \left|\bm{a}_{i}^{\top}\bm{z}_{0}\right|\leq c_{2}\sqrt{\log m}\right\} }\right]\right|}_{:=I_{3}}.
\end{align*}
Regarding $I_{1}$, one sees from (\ref{eq:uniform-fix-x-z}) and
the union bound that with probability at least $1-2(1+\frac{2}{\theta})^{2n}\exp\left(-c_{10}n\log m\right)$,
one has 
\[
I_{1}\lesssim c_{2}^{2}\sqrt{\frac{n\log^{3}m}{m}}.
\]
 For the second term $I_{2}$, we can deduce from (\ref{eq:uniform-beta-2})
that $\beta_{2}\leq3$ and 
\begin{align*}
\left|\left(\bm{z}^{\top}\bm{x}\right)^{2}-\left(\bm{z}_{0}^{\top}\bm{x}_{0}\right)^{2}\right| & =\left|\bm{z}^{\top}\bm{x}-\bm{z}_{0}^{\top}\bm{x}_{0}\right|\left|\bm{z}^{\top}\bm{x}+\bm{z}_{0}^{\top}\bm{x}_{0}\right|\\
 & =\left|\left(\bm{z}-\bm{z}_{0}\right)^{\top}\bm{x}+\bm{z}_{0}\left(\bm{x}-\bm{x}_{0}\right)\right|\left|\bm{z}^{\top}\bm{x}+\bm{z}_{0}^{\top}\bm{x}_{0}\right|\\
 & \leq2\left(\left\Vert \bm{z}-\bm{z}_{0}\right\Vert _{2}+\left\Vert \bm{x}-\bm{x}_{0}\right\Vert _{2}\right)\leq2\theta,
\end{align*}
where the last line arises from the Cauchy-Schwarz inequality and
the fact that $\bm{x},\bm{z},\bm{x}_{0},\bm{z}_{0}$ are all unit
norm vectors. This further implies 
\[
I_{2}\leq6\theta.
\]
Now we move on to control the last term $I_{3}$. Denoting 
\[
\cS_{i}:=\left\{ \bm{u}\mid\left|\bm{a}_{i}^{\top}\bm{u}\right|\leq c_{2}\sqrt{\log m}\right\} 
\]
allows us to rewrite $I_{3}$ as 
\begin{align}
I_{3} & =\left|\frac{1}{m}\sum_{i=1}^{m}\left[\left(\bm{a}_{i}^{\top}\bm{z}\right)^{2}\left(\bm{a}_{i}^{\top}\bm{x}\right)^{2}\ind_{\left\{ \bm{z}\in\cS_{i}\right\} }-\left(\bm{a}_{i}^{\top}\bm{z}_{0}\right)^{2}\left(\bm{a}_{i}^{\top}\bm{x}_{0}\right)^{2}\ind_{\left\{ \bm{z}_{0}\in\mathcal{S}_{i}\right\} }\right]\right|\nonumber \\
 & \leq\left|\frac{1}{m}\sum_{i=1}^{m}\left[\left(\bm{a}_{i}^{\top}\bm{z}\right)^{2}\left(\bm{a}_{i}^{\top}\bm{x}\right)^{2}-\left(\bm{a}_{i}^{\top}\bm{z}_{0}\right)^{2}\left(\bm{a}_{i}^{\top}\bm{x}_{0}\right)^{2}\right]\ind_{\left\{ \bm{z}\in\mathcal{S}_{i},\bm{z}_{0}\in\mathcal{S}_{i}\right\} }\right|\nonumber \\
 & +\left|\frac{1}{m}\sum_{i=1}^{m}\left(\bm{a}_{i}^{\top}\bm{z}\right)^{2}\left(\bm{a}_{i}^{\top}\bm{x}\right)^{2}\ind_{\left\{ \bm{z}\in\cS_{i},\bm{z}_{0}\notin\mathcal{S}_{i}\right\} }\right|+\left|\frac{1}{m}\sum_{i=1}^{m}\left(\bm{a}_{i}^{\top}\bm{z}_{0}\right)^{2}\left(\bm{a}_{i}^{\top}\bm{x}_{0}\right)^{2}\ind_{\left\{ \bm{z}_{0}\in\mathcal{S}_{i},\bm{z}\notin\mathcal{S}_{i}\right\} }\right|.\label{eq:uniform-I-3}
\end{align}
Here the decomposition is similar to what we have done in (\ref{eq:uniform-I-2}).
For the first term in (\ref{eq:uniform-I-3}), one has 
\begin{align*}
\left|\frac{1}{m}\sum_{i=1}^{m}\left[\left(\bm{a}_{i}^{\top}\bm{z}\right)^{2}\left(\bm{a}_{i}^{\top}\bm{x}\right)^{2}-\left(\bm{a}_{i}^{\top}\bm{z}_{0}\right)^{2}\left(\bm{a}_{i}^{\top}\bm{x}_{0}\right)^{2}\right]\ind_{\left\{ \bm{z}\in\mathcal{S}_{i},\bm{z}_{0}\in\mathcal{S}_{i}\right\} }\right| & \leq\frac{1}{m}\sum_{i=1}^{m}\left|\left(\bm{a}_{i}^{\top}\bm{z}\right)^{2}\left(\bm{a}_{i}^{\top}\bm{x}\right)^{2}-\left(\bm{a}_{i}^{\top}\bm{z}_{0}\right)^{2}\left(\bm{a}_{i}^{\top}\bm{x}_{0}\right)^{2}\right|\\
 & \leq n^{\alpha}\theta,
\end{align*}
for some $\alpha=O(1)$. Here the last line follows from the smoothness
of the function $g\left(\bm{x},\bm{z}\right)=\left(\bm{a}_{i}^{\top}\bm{z}\right)^{2}\left(\bm{a}_{i}^{\top}\bm{x}\right)^{2}$.
Proceeding to the second term in (\ref{eq:uniform-I-3}), we see from
(\ref{eq:third-identity-I}) that 
\[
\ind_{\left\{ \bm{z}\in\cS_{i},\bm{z}_{0}\notin\mathcal{S}_{i}\right\} }\leq\ind_{\left\{ c_{2}\sqrt{\log m}\leq\left|\bm{a}_{i}^{\top}\bm{z}_{0}\right|\leq c_{2}\sqrt{\log m}+\sqrt{6n}\theta\right\} },
\]
which implies that 
\begin{align*}
\left|\frac{1}{m}\sum_{i=1}^{m}\left(\bm{a}_{i}^{\top}\bm{z}\right)^{2}\left(\bm{a}_{i}^{\top}\bm{x}\right)^{2}\ind_{\left\{ \bm{z}\in\cS_{i},\bm{z}_{0}\notin\mathcal{S}_{i}\right\} }\right| & \leq\max_{1\leq i\leq m}\left(\bm{a}_{i}^{\top}\bm{z}\right)^{2}\ind_{\left\{ \bm{z}\in\cS_{i}\right\} }\left|\frac{1}{m}\sum_{i=1}^{m}\left(\bm{a}_{i}^{\top}\bm{x}\right)^{2}\ind_{\left\{ \bm{z}\in\cS_{i},\bm{z}_{0}\notin\mathcal{S}_{i}\right\} }\right|\\
 & \leq c_{2}^{2}\log m\left|\frac{1}{m}\sum_{i=1}^{m}\left(\bm{a}_{i}^{\top}\bm{x}\right)^{2}\ind_{\left\{ c_{2}\sqrt{\log m}\leq\left|\bm{a}_{i}^{\top}\bm{z}_{0}\right|\leq c_{2}\sqrt{\log m}+\sqrt{6n}\theta\right\} }\right|.
\end{align*}
With regard to the above quantity, we have the following claim. \begin{claim}\label{claim:uniform}With
probability at least $1-c_{2}e^{-c_{3}n\log m}$ for some constants
$c_{2},c_{3}>0$, one has 
\[
\left|\frac{1}{m}\sum_{i=1}^{m}\left(\bm{a}_{i}^{\top}\bm{x}\right)^{2}\ind_{\left\{ c_{2}\sqrt{\log m}\leq\left|\bm{a}_{i}^{\top}\bm{z}_{0}\right|\leq c_{2}\sqrt{\log m}+\sqrt{6n}\theta\right\} }\right|\lesssim\sqrt{\frac{n\log m}{m}}
\]
for all $\bm{x}\in\RR^{n}$ with unit norm and for all $\bm{z}_{0}\in\mathcal{N}_{\theta}$.
\end{claim}With this claim in place, we arrive at 
\[
\left|\frac{1}{m}\sum_{i=1}^{m}\left(\bm{a}_{i}^{\top}\bm{z}\right)^{2}\left(\bm{a}_{i}^{\top}\bm{x}\right)^{2}\ind_{\left\{ \bm{z}\in\cS_{i},\bm{z}_{0}\notin\mathcal{S}_{i}\right\} }\right|\lesssim c_{2}^{2}\sqrt{\frac{n\log^{3}m}{m}}
\]
with high probability. Similar arguments lead us to conclude that
with high probability 
\[
\left|\frac{1}{m}\sum_{i=1}^{m}\left(\bm{a}_{i}^{\top}\bm{z}_{0}\right)^{2}\left(\bm{a}_{i}^{\top}\bm{x}_{0}\right)^{2}\ind_{\left\{ \bm{z}_{0}\in\mathcal{S}_{i},\bm{z}\notin\mathcal{S}_{i}\right\} }\right|\lesssim c_{2}^{2}\sqrt{\frac{n\log^{3}m}{m}}.
\]
Taking the above bounds collectively and setting $\theta\asymp m^{-\alpha-1}$
yield with high probability for all unit vectors $\bm{z}$'s and $\bm{x}$'s
\[
\left|\frac{1}{m}\sum_{i=1}^{m}\left(\bm{a}_{i}^{\top}\bm{z}\right)^{2}\left(\bm{a}_{i}^{\top}\bm{x}\right)^{2}\ind_{\left\{ \left|\bm{a}_{i}^{\top}\bm{z}\right|\leq c_{2}\sqrt{\log m}\right\} }-\beta_{1}-\beta_{2}\left(\bm{z}^{\top}\bm{x}\right)^{2}\right|\lesssim c_{2}^{2}\sqrt{\frac{n\log^{3}m}{m}},
\]
which is equivalent to saying that 
\[
\theta_{1}\lesssim c_{2}^{2}\sqrt{\frac{n\log^{3}m}{m}}.
\]
The proof is complete by combining the upper bounds on $\theta_{1}$
and $\theta_{2}$, and the fact $\frac{1}{m}=o\left(\sqrt{\frac{n\log^{3}m}{m}}\right)$. 
\end{itemize}
\begin{proof}[Proof of Claim \ref{claim:uniform}]We first apply
the triangle inequality to get 
\begin{align*}
 \left|\frac{1}{m}\sum_{i=1}^{m}\left(\bm{a}_{i}^{\top}\bm{x}\right)^{2}\ind_{\left\{ c_{2}\sqrt{\log m}\leq\left|\bm{a}_{i}^{\top}\bm{z}_{0}\right|\leq c_{2}\sqrt{\log m}+\sqrt{6n}\theta\right\} }\right|  &\leq\underbrace{\left|\frac{1}{m}\sum_{i=1}^{m}\left(\bm{a}_{i}^{\top}\bm{x}_{0}\right)^{2}\ind_{\left\{ c_{2}\sqrt{\log m}\leq\left|\bm{a}_{i}^{\top}\bm{z}_{0}\right|\leq c_{2}\sqrt{\log m}+\sqrt{6n}\theta\right\} }\right|}_{:=J_{1}}\\
 &\hspace{-0.5in} +\underbrace{\left|\frac{1}{m}\sum_{i=1}^{m}\left[\left(\bm{a}_{i}^{\top}\bm{x}\right)^{2}-\left(\bm{a}_{i}^{\top}\bm{x}_{0}\right)^{2}\right]\ind_{\left\{ c_{2}\sqrt{\log m}\leq\left|\bm{a}_{i}^{\top}\bm{z}_{0}\right|\leq c_{2}\sqrt{\log m}+\sqrt{6n}\theta\right\} }\right|}_{:=J_{2}},
\end{align*}
where $\bm{x}_{0}\in\mathcal{N}_{\theta}$ and $\|\bm{x}-\bm{x}_{0}\|_{2}\leq\theta$.
The second term can be controlled as follows 
\[
J_{2}\leq\frac{1}{m}\sum_{i=1}^{m}\left|\left(\bm{a}_{i}^{\top}\bm{x}\right)^{2}-\left(\bm{a}_{i}^{\top}\bm{x}_{0}\right)^{2}\right|\leq n^{O\left(1\right)}\theta,
\]
where we utilize the smoothness property of the function $h\left(\bm{x}\right)=\left(\bm{a}_{i}^{\top}\bm{x}\right)^{2}$.
It remains to bound $J_{1}$, for which we first fix $\bm{x}_{0}$
and $\bm{z}_{0}$. Take the Bernstein inequality to get 
\begin{align*}
 & \PP\left(\left|\frac{1}{m}\sum_{i=1}^{m}\left(\bm{a}_{i}^{\top}\bm{x}_{0}\right)^{2}\ind_{\left\{ c_{2}\sqrt{\log m}\leq\left|\bm{a}_{i}^{\top}\bm{z}_{0}\right|\leq c_{2}\sqrt{\log m}+\sqrt{6n}\theta\right\} }-\EE\left[\left(\bm{a}_{i}^{\top}\bm{x}_{0}\right)^{2}\ind_{\left\{ c_{2}\sqrt{\log m}\leq\left|\bm{a}_{i}^{\top}\bm{z}_{0}\right|\leq c_{2}\sqrt{\log m}+\sqrt{6n}\theta\right\} }\right]\right|\geq\tau\right)\\
 & \leq2e^{-cm\tau^{2}}
\end{align*}
for some constant $c>0$ and any sufficiently small $\tau>0$. Taking
$\tau\asymp\sqrt{\frac{n\log m}{m}}$ reveals that with probability
exceeding $1-2e^{-Cn\log m}$ for some large enough constant $C>0$,
\[
J_{1}\lesssim\EE\left[\left(\bm{a}_{i}^{\top}\bm{x}_{0}\right)^{2}\ind_{\left\{ c_{2}\sqrt{\log m}\leq\left|\bm{a}_{i}^{\top}\bm{z}_{0}\right|\leq c_{2}\sqrt{\log m}+\sqrt{6n}\theta\right\} }\right]+\sqrt{\frac{n\log m}{m}}.
\]
Regarding the expectation term, it follows from Cauchy-Schwarz that
\begin{align*}
\EE\left[\left(\bm{a}_{i}^{\top}\bm{x}_{0}\right)^{2}\ind_{\left\{ c_{2}\sqrt{\log m}\leq\left|\bm{a}_{i}^{\top}\bm{z}_{0}\right|\leq c_{2}\sqrt{\log m}+\sqrt{6n}\theta\right\} }\right] & \leq\sqrt{\EE\left[\left(\bm{a}_{i}^{\top}\bm{x}_{0}\right)^{4}\right]}\sqrt{\EE\left[\ind_{\left\{ c_{2}\sqrt{\log m}\leq\left|\bm{a}_{i}^{\top}\bm{z}_{0}\right|\leq c_{2}\sqrt{\log m}+\sqrt{6n}\theta\right\} }\right]}\\
 & \asymp\EE\left[\ind_{\left\{ c_{2}\sqrt{\log m}\leq\left|\bm{a}_{i}^{\top}\bm{z}_{0}\right|\leq c_{2}\sqrt{\log m}+\sqrt{6n}\theta\right\} }\right]\\
 & \leq1/m,
\end{align*}
as long as $\theta$ is sufficiently small. Combining the preceding
bounds with the union bound, we can see that with probability at least
$1-2\left(1+\frac{2}{\theta}\right)^{2n}e^{-Cn\log m}$
\[
J_{1}\lesssim\sqrt{\frac{n\log m}{m}}+\frac{1}{m}.
\]
Picking $\theta\asymp m^{-c_{1}}$ for some large enough constant
$c_{1}>0$, we arrive at with probability at least $1-c_{2}e^{-c_{3}n\log m}$
\[
\left|\frac{1}{m}\sum_{i=1}^{m}\left(\bm{a}_{i}^{\top}\bm{x}\right)^{2}\ind_{\left\{ c_{2}\sqrt{\log m}\leq\left|\bm{a}_{i}^{\top}\bm{z}_{0}\right|\leq c_{2}\sqrt{\log m}+\sqrt{6n}\theta\right\} }\right|\lesssim\sqrt{\frac{n\log m}{m}}
\]
for all unit vectors $\bm{x}$'s and for all $\bm{z}_{0}\in\mathcal{N}_{\theta}$,
where $c_{2},c_{3}>0$ are some absolute constants. \end{proof}

\section{Proof of Lemma \ref{lemma:Hessian-UB-Stage1}\label{sec:Proof-of-Lemma-Hessian-UB-Stage1}}

Recall that the Hessian matrix is given by 
\begin{align*}
\nabla^{2}f\left(\bm{z}\right) & =\frac{1}{m}\sum_{i=1}^{m}\left[3\left(\bm{a}_{i}^{\top}\bm{z}\right)^{2}-\left(\bm{a}_{i}^{\top}\bm{x}^{\natural}\right)^{2}\right]\bm{a}_{i}\bm{a}_{i}^{\top}.
\end{align*}
Lemma \ref{lemma:hessian-concentration} implies that with probability
at least $1-O\left(m^{-10}\right)$, 
\begin{equation}
\left\Vert \nabla^{2}f\left(\bm{z}\right)-6\bm{z}\bm{z}^{\top}-3\left\Vert \bm{z}\right\Vert _{2}^{2}\bm{I}_{n}+2\bm{x}^{\natural}\bm{x}^{\natural\top}+\left\Vert \bm{x}^{\natural}\right\Vert _{2}^{2}\bm{I}_{n}\right\Vert \lesssim\sqrt{\frac{n\log^{3}m}{m}}\max\left\{ \left\Vert \bm{z}\right\Vert _{2}^{2},\left\Vert \bm{x}^{\natural}\right\Vert _{2}^{2}\right\} \label{eq:nabla2-f-bound}
\end{equation}
holds simultaneously for all $\bm{z}$ obeying $\max_{1\leq i\leq m}\left|\bm{a}_{i}^{\top}\bm{z}\right|\leq c_{0}\sqrt{\log m}\left\Vert \bm{z}\right\Vert _{2}$,
with the proviso that $m\gg n\log^{3}m$. This together with the fact
$\left\Vert \bm{x}^{\natural}\right\Vert _{2}=1$ leads to 
\begin{align*}
-\nabla^{2}f\left(\bm{z}\right) & \succeq-6\bm{z}\bm{z}^{\top}-\left\{ 3\left\Vert \bm{z}\right\Vert _{2}^{2}-1+O\left(\sqrt{\frac{n\log^{3}m}{m}}\max\left\{ \left\Vert \bm{z}\right\Vert _{2}^{2},1\right\} \right)\right\} \bm{I}_{n}\\
 & \succeq-\left\{ 9\left\Vert \bm{z}\right\Vert _{2}^{2}-1+O\left(\sqrt{\frac{n\log^{3}m}{m}}\max\left\{ \left\Vert \bm{z}\right\Vert _{2}^{2},1\right\} \right)\right\} \bm{I}_{n}.
\end{align*}
As a consequence, if we pick $0<\eta<\frac{c_{2}}{\max\left\{ \left\Vert \bm{z}\right\Vert _{2}^{2},1\right\} }$
for $c_{2}>0$ sufficiently small, then $\bm{I}_{n}-\eta\nabla^{2}f\left(\bm{z}\right)\succeq\bm{0}$.
This combined with (\ref{eq:nabla2-f-bound}) gives 
\[
\left\Vert \left(\bm{I}_{n}-\eta\nabla^{2}f\left(\bm{z}\right)\right)-\left\{ \left(1-3\eta\left\Vert \bm{z}\right\Vert _{2}^{2}+\eta\right)\bm{I}_{n}+2\eta\bm{x}^{\natural}\bm{x}^{\natural\top}-6\eta\bm{z}\bm{z}^{\top}\right\} \right\Vert \lesssim\sqrt{\frac{n\log^{3}m}{m}}\max\left\{ \left\Vert \bm{z}\right\Vert _{2}^{2},1\right\} .
\]

Additionally, it follows from (\ref{eq:nabla2-f-bound}) that
\begin{align*}
\left\Vert \nabla^{2}f\left(\bm{z}\right)\right\Vert  & \leq\left\Vert 6\bm{z}\bm{z}^{\top}+3\left\Vert \bm{z}\right\Vert _{2}^{2}\bm{I}_{n}+2\bm{x}^{\natural}\bm{x}^{\natural\top}+\left\Vert \bm{x}^{\natural}\right\Vert _{2}^{2}\bm{I}_{n}\right\Vert +O\left(\sqrt{\frac{n\log^{3}m}{m}}\right)\max\left\{ \left\Vert \bm{z}\right\Vert _{2}^{2},\left\Vert \bm{x}^{\natural}\right\Vert _{2}^{2}\right\} \\
 & \leq9\|\bm{z}\|_{2}^{2}+3+O\left(\sqrt{\frac{n\log^{3}m}{m}}\right)\max\left\{ \left\Vert \bm{z}\right\Vert _{2}^{2},1\right\} \\
 & \leq10\|\bm{z}\|_{2}^{2}+4
\end{align*}
as long as $m\gg n\log^{3}m$. 

\section{Proof of Lemma \ref{lemma:consequence}\label{sec:Proof-of-Lemma-consequence}}

Note that when $t\lesssim\log n$, one naturally has 
\begin{equation}\label{eq:growth_logm}
\left(1+\frac{1}{\log m}\right)^{t}\lesssim1.
\end{equation}

Regarding the first set of consequences \eqref{subeq:consequence-norm}, one sees via the triangle
inequality that 
\begin{align*}
\max_{1\leq l\leq m}\big\Vert \bm{x}^{t,\left(l\right)}\big\Vert _{2} & \leq\left\Vert \bm{x}^{t}\right\Vert _{2}+\max_{1\leq l\leq m}\big\Vert \bm{x}^{t}-\bm{x}^{t,\left(l\right)}\big\Vert _{2}\\
 & \overset{\left(\text{i}\right)}{\leq}C_{5}+\beta_{t}\left(1+\frac{1}{\log m}\right)^{t}C_{1}\eta\frac{\sqrt{n\log^{5}m}}{m}\\
 & \overset{\left(\text{ii}\right)}{\leq}C_{5}+O\left(\frac{\sqrt{n\log^{5}m}}{m}\right)\\
 & \overset{\left(\text{iii}\right)}{\leq}2C_{5},
\end{align*}
where (i) follows from the induction hypotheses (\ref{eq:induction-xt-l})
and (\ref{eq:induction-norm-size}). The second inequality (ii) holds
true since $\beta_{t}\lesssim1$ and \eqref{eq:growth_logm}. The last one (iii) is valid as long
as $m\gg\sqrt{n\log^{5}m}$. Similarly, for the lower bound, one can
show that for each $1\leq l\leq m$, 
\begin{align*}
\big\Vert \bm{x}_{\perp}^{t,\left(l\right)} \big\Vert _{2} & \geq\left\Vert \bm{x}_{\perp}^{t}\right\Vert _{2}-\big\Vert \bm{x}_{\perp}^{t}-\bm{x}_{\perp}^{t,\left(l\right)}\big\Vert _{2}\\
 & \geq\left\Vert \bm{x}_{\perp}^{t}\right\Vert _{2}-\max_{1\leq l\leq m}\big\Vert \bm{x}^{t}-\bm{x}^{t,\left(l\right)}\big\Vert _{2}\\
 & \geq c_{5}-\beta_{t}\left(1+\frac{1}{\log m}\right)^{t}C_{1}\eta\frac{\sqrt{n\log^{3}m}}{m}\geq\frac{c_{5}}{2},
\end{align*}
as long as $m\gg\sqrt{n\log^{5}m}$. Using similar arguments ($\alpha_{t}\lesssim1$),
we can prove the lower and upper bounds for $\bm{x}^{t,\text{sgn}}$
and $\bm{x}^{t,\text{sgn},\left(l\right)}$.

For the second set of consequences \eqref{subeq:consequence-incoherence}, namely the incoherence consequences,
first notice that it is sufficient to show that the inner product
(for instance $|\bm{a}_{l}^{\top}\bm{x}^{t}|$) is upper bounded by
$C_{7}\log m$ in magnitude for some absolute constants $C_{7}>0$.
To see this, suppose for now 
\begin{equation}
\max_{1\leq l\leq m}\left|\bm{a}_{l}^{\top}\bm{x}^{t}\right|\leq C_{7}\sqrt{\log m}.\label{eq:consequnce-inner-product}
\end{equation}
One can further utilize the lower bound on $\left\Vert \bm{x}^{t}\right\Vert _{2}$
to deduce that 
\[
\max_{1\leq l\leq m}\big|\bm{a}_{l}^{\top}\bm{x}^{t}\big|\leq\frac{C_{7}}{c_{5}}\sqrt{\log m}\big\Vert \bm{x}^{t}\big\Vert _{2}.
\]
This justifies the claim that we only need to obtain bounds as in
(\ref{eq:consequnce-inner-product}). Once again we can invoke the
triangle inequality to deduce that with probability at least $1-O\left(m^{-10}\right)$,
\begin{align*}
\max_{1\leq l\leq m}\left|\bm{a}_{l}^{\top}\bm{x}^{t}\right| & \leq\max_{1\leq l\leq m}\big|\bm{a}_{l}^{\top}\big(\bm{x}^{t}-\bm{x}^{t,\left(l\right)}\big)\big|+\max_{1\leq l\leq m}\big|\bm{a}_{l}^{\top}\bm{x}^{t,\left(l\right)}\big|\\
 & \overset{\left(\text{i}\right)}{\leq}\max_{1\leq l\leq m}\left\Vert \bm{a}_{l}\right\Vert _{2}\max_{1\leq l\leq m}\big\Vert \bm{x}^{t}-\bm{x}^{t,\left(l\right)}\big\Vert _{2}+\max_{1\leq l\leq m}\left|\bm{a}_{l}^{\top}\bm{x}^{t,\left(l\right)}\right|\\
 & \overset{\left(\text{ii}\right)}{\lesssim}\sqrt{n}\beta_{t}\left(1+\frac{1}{\log m}\right)^{t}C_{1}\eta\frac{\sqrt{n\log^{5}m}}{m}+\sqrt{\log m} \max_{1\leq l\leq m}\big\Vert \bm{x}^{t,\left(l\right)}\big\Vert _{2}\\
 & \lesssim\frac{n\log^{5/2}m}{m}+C_{5}\sqrt{\log m}\lesssim C_{5}\sqrt{\log m}.
\end{align*}
Here, the first relation (i) results from the Cauchy-Schwarz inequality
and (ii) utilizes the induction hypothesis (\ref{eq:induction-xt-l}),
the fact (\ref{eq:max-a-i-norm}) and the standard Gaussian concentration,
namely, $\max_{1\leq l\leq m}\left|\bm{a}_{l}^{\top}\bm{x}^{t,\left(l\right)}\right|\lesssim\sqrt{\log m}\max_{1\leq l\leq m}\big\Vert \bm{x}^{t,\left(l\right)}\big\Vert _{2}$
with probability at least $1-O\left(m^{-10}\right)$. The last line
is a direct consequence of the fact (\ref{eq:consequence-norm-xt-l})
established above and \eqref{eq:growth_logm}. In
regard to the incoherence w.r.t.~$\bm{x}^{t,\text{sgn}}$, we resort
to the leave-one-out sequence $\bm{x}^{t,\text{sgn},\left(l\right)}$.
Specifically, we have 
\begin{align*}
\left|\bm{a}_{l}^{\top}\bm{x}^{t,\text{sgn}}\right| & \leq\left|\bm{a}_{l}^{\top}\bm{x}^{t}\right|+\left|\bm{a}_{l}^{\top}\big(\bm{x}^{t,\text{sgn}}-\bm{x}^{t}\big)\right|\\
 & \leq\left|\bm{a}_{l}^{\top}\bm{x}^{t}\right|+\left|\bm{a}_{l}^{\top}\big(\bm{x}^{t,\text{sgn}}-\bm{x}^{t}-\bm{x}^{t,\text{sgn},\left(l\right)}+\bm{x}^{t,\left(l\right)}\big)\right|+\left|\bm{a}_{l}^{\top}\big(\bm{x}^{t,\text{sgn},\left(l\right)}-\bm{x}^{t,\left(l\right)}\big)\right|\\
 & \lesssim\sqrt{\log m}+\sqrt{n}\alpha_{t}\left(1+\frac{1}{\log m}\right)^{t}C_{4}\frac{\sqrt{n\log^{9}m}}{m}+\sqrt{\log m}\\
 & \lesssim\sqrt{\log m}.
\end{align*}
The remaining incoherence conditions can be obtained through similar
arguments. For the sake of conciseness, we omit the details here. 

With regard to the third set of consequences \eqref{eq:consequence-upper-bounds}, we can directly use
the induction hypothesis and obtain 
\begin{align*}
\max_{1\leq l\leq m}\big\Vert \bm{x}^{t}-\bm{x}^{t,\left(l\right)}\big\Vert _{2} & \leq\beta_{t}\left(1+\frac{1}{\log m}\right)^{t}C_{1}\frac{\sqrt{n\log^{3}m}}{m}\\
 & \lesssim\frac{\sqrt{n\log^{3}m}}{m}\lesssim\frac{1}{\log m},
\end{align*}
as long as $m\gg\sqrt{n\log^{5}m}$. Apply similar arguments to get
the claimed bound on $\|\bm{x}^{t}-\bm{x}^{t,\text{sgn}}\|_{2}$.
For the remaining one, we have 
\begin{align*}
\max_{1\leq l\leq m}\left|x_{\parallel}^{t,\left(l\right)}\right| & \leq\max_{1\leq l\leq m}\left|x_{\parallel}^{t}\right|+\max_{1\leq l\leq m}\left|x_{\parallel}^{t,\left(l\right)}-x_{\parallel}^{t}\right|\\
 & \leq\alpha_{t}+\alpha_{t}\left(1+\frac{1}{\log m}\right)^{t}C_{2}\eta\frac{\sqrt{n\log^{12}m}}{m}\\
 & \leq2\alpha_{t},
\end{align*}
with the proviso that $m\gg\sqrt{n\log^{12}m}$.

\section{Proof of Theorem \ref{thm:dependent-init}\label{sec:Proof-of-Theorem-dependent}}

A key observation is that in the proof of Theorem \ref{thm:main}, we do not require
independence between $\bm{x}^{0}$ and the data $\{\bm{a}_{i,}y_{i}\}_{1\leq i\leq m}$.
Instead, what we really need are: 
\begin{enumerate}
\item $\bm{x}^{0,\text{sgn}}$ is independent of $\{\xi_{i}=\text{sgn}(a_{i,1})\}_{1\leq i\leq m}$; 
\item $\bm{x}^{0,(l)}$ is independent of $(\bm{a}_{l},y_{l})$ for all
$1\leq l\leq m$ and 
\item $\bm{x}^{0,\text{sgn},(l)}$ is independent of both $\{\xi_{i}\}_{1\leq i\leq m}$
and $\{\bm{a}_{l},y_{l}\}$ for all $1\leq l\leq m$.
\end{enumerate}
With this observation in mind, one can see that the claim on the convergence
holds true as long as the initialization $\bm{x}^{0}$ satisfies (\ref{eq:initilization-condition})
and we can construct $\bm{x}^{0,\text{sgn}}$, $\bm{x}^{0,(l)}$ and
$\bm{x}^{0,\text{sgn},(l)}$, which obey the required independence
mentioned above as well as the base case specified in (\ref{subeq:induction}). In the
following, we show that for \[
	\bm{x}^{0} = \sqrt{\frac{1}{m}\sum_{i=1}^{m}y_{i}} \cdot \bm{u},
\]
where $\bm{u}$ is uniformly distributed over the unit sphere, the requirements can all be satisfied.
\begin{enumerate}
\item The first restriction (\ref{eq:initilization-condition}) can be easily verified by concentration
inequalities for spherical distribution and the fact that $\frac{1}{m}\sum_{i=1}^{m}y_{i}$
sharply concentrates around $\|\bm{x}^{\natural}\|_{2}^{2}$.
\item Next, we move on to demonstrating how to construct $\bm{x}^{0,\text{sgn}}$,
$\bm{x}^{0,(l)}$ and $\bm{x}^{0,\text{sgn},(l)}$ with prescribed
independence. In view of the initialization, we have 
\[
\bm{x}^{0}=\lambda\cdot\bm{u},
\]
where $\bm{u}$ is a unit vector uniformly distributed over the unit
sphere in $\mathbb{R}^{n}$ and $\lambda=\sqrt{\sum_{i=1}^{m}y_{i}/m}$.
Moreover, one has $\lambda$ is independent of $\bm{u}$. This together with the fact
that 
\[
y_{i}=\left(\bm{a}_{i}^{\top}\bm{x}^{\natural}\right)^{2}=\left|a_{i,1}\right|^{2}
\]
reveals that $\lambda$ depends on $\{|a_{i,1}|\}_{1\leq i\leq m}$
only and $\bm{u}$ is independent of the data $\{\bm{a}_{i},y_{i}\}_{1\leq i\leq m}$.
Therefore, one can set 
\[
\bm{x}^{0,(l)}=\lambda^{(l)}\cdot\bm{u},
\]
where $\bm{u}$ is the same vector as in $\bm{x}^{0}$ and $\lambda^{(l)}=\sqrt{\sum_{i:i\neq l}^{m}y_{i}/m}$.
One can see from this construction that $\bm{x}^{0,(l)}$is independent of $\{\bm{a}_{l},y_{l}\}$.
Regarding $\bm{x}^{0,\text{sgn}}$ and $\bm{x}^{0,\text{sgn},(l)}$,
we set 
\[
\bm{x}^{0,\text{sgn}}=\bm{x}^{0},\qquad\text{and}\qquad\bm{x}^{0,\text{sgn},(l)}=\bm{x}^{0,(l)}.
\]
Since $\bm{x}^{0}$ is independent of $\{\xi_{i}=\text{sgn}(a_{i,1})\}_{1\leq i\leq m}$,
so is $\bm{x}^{0,\text{sgn}}$. The same reasoning can be applied
to show independence between $\bm{x}^{0,\text{sgn},(l)}$ and $\{\xi_{i}\}_{1\leq i\leq m}$
and $\{\bm{a}_{l},y_{l}\}$. 
\item We are left with checking the base case, i.e. (\ref{subeq:induction}): 
\begin{enumerate}
\item For the difference between $\bm{x}^{0}$ and $\bm{x}^{0,(l)}$, we
have 
\begin{align*}
\left\Vert \bm{x}^{0}-\bm{x}^{0,(l)}\right\Vert _{2} & =\left\Vert \lambda\bm{u}-\lambda^{(l)}\bm{u}\right\Vert _{2}=\left|\lambda-\lambda^{(l)}\right|\\
 & =\sqrt{\frac{1}{m}\sum_{i=1}^{m}y_{i}}-\sqrt{\frac{1}{m}\sum_{i:i\neq l}^{m}y_{i}}\\
 & =\frac{\frac{1}{m}y_{l}}{\sqrt{\frac{1}{m}\sum_{i=1}^{m}y_{i}}+\sqrt{\frac{1}{m}\sum_{i:i\neq l}^{m}y_{i}}},
\end{align*}
where the last relation holds due to the basic identity $\sqrt{a}-\sqrt{b}=(a-b)/(\sqrt{a}+\sqrt{b})$
for $a,b>0$. Noting that $\frac{1}{m}\sum_{i=1}^{m}y_{i}$ sharply
concentrates around 1 and $|y_{l}|\lesssim\log m$ with high probability,
one arrives at 
\[
\left\Vert \bm{x}^{0}-\bm{x}^{0,(l)}\right\Vert _{2}=\left|\lambda-\lambda^{(l)}\right|\lesssim\frac{\log m}{m}\leq\beta_{0}C_{1}\frac{\sqrt{n\log^{5}m}}{m}.
\]
This finishes the proof of (\ref{eq:induction-xt-l}). 
\item The base case for (\ref{eq:induction-xt-l-signal}) can be easily deduced due to 
\[
\left|x_{\parallel}^{0}-x_{\parallel}^{0,(l)}\right|\leq\left\Vert \bm{x}^{0}-\bm{x}^{0,(l)}\right\Vert _{2}\lesssim\frac{\log m}{m}\leq\alpha_{0}C_{2}\frac{\sqrt{n\log^{12}m}}{m}.
\]
\item By construction, we have $\bm{x}^{0,\text{sgn}}=\bm{x}^{0}$ and $\bm{x}^{0,\text{sgn},(l)}=\bm{x}^{0,(l)}$.
Therefore (\ref{eq:induction-xt-sgn}) and (\ref{eq:induction-double}) trivially hold.
\item The last two relations (\ref{eq:induction-norm-size}) and (\ref{eq:induction-norm-relative}) can be verified using (\ref{eq:initilization-condition}).
\end{enumerate}
\end{enumerate}
Combining all and repeating the proof of Theorem \ref{thm:main}, we finish the proof of Theorem \ref{thm:dependent-init}.

\end{document}